\pdfoutput=1
\documentclass{article}

\usepackage[main,final]{neurips_2026}

\usepackage[utf8]{inputenc}
\usepackage[T1]{fontenc}
\usepackage[hypertexnames=false]{hyperref}
\hypersetup{hidelinks,
  pdftitle={Cumulative-Goodness Free-Riding in Forward-Forward Networks: Real, Repairable, but Not Accuracy-Dominant},
  pdfauthor={Amirhossein Yousefiramandi}}
\usepackage{url}
\usepackage{booktabs}
\usepackage{amsfonts}
\usepackage{amsmath}
\usepackage{amssymb}
\usepackage{amsthm}
\usepackage{mathtools}
\usepackage{nicefrac}
\usepackage{microtype}
\usepackage{xcolor}
\usepackage{enumitem}
\usepackage{graphicx}
\usepackage{multirow}
\usepackage{algorithm}
\usepackage{algpseudocode}
\usepackage{placeins}  %
\usepackage{tikz}
\usetikzlibrary{positioning, arrows.meta, decorations.pathreplacing, calc, fit, backgrounds, shadows.blur, shapes.geometric}

\newtheorem{theorem}{Theorem}[section]
\newtheorem{proposition}[theorem]{Proposition}
\newtheorem{lemma}[theorem]{Lemma}
\newtheorem{corollary}[theorem]{Corollary}

\DeclareMathOperator{\softplus}{softplus}
\DeclareMathOperator{\sigmoid}{sigmoid}

\title{Cumulative-Goodness Free-Riding in Forward-Forward Networks:\\
Real, Repairable, but Not Accuracy-Dominant}

\author{%
  Amirhossein Yousefiramandi\thanks{Code:
  \url{https://github.com/amirhossein-yousefi/ff-free-riding}\newline
  Checkpoints: \url{https://huggingface.co/Amirhossein75/ff-free-riding}} \\
  Independent Researcher \\
  \texttt{amir.usefi75@gmail.com}
}

\begin{document}

\maketitle

\begin{abstract}
Forward-Forward (FF) training lets each layer learn from a local
goodness criterion. In cumulative-goodness variants, later
layers can inherit a task that earlier layers have partly
separated. We formalize this as \emph{layer free-riding}:
under the softplus FF criterion, the class-discrimination gradient
reaching block $d$ decays exponentially with the positive margin
accumulated by earlier blocks
(Theorem~\ref{thm:attenuation}); a squared-hinge barrier reproduces
the pathology. History-free and hardness-gated repairs raise
deeper-layer separation by up to $5{\times}$ on CIFAR-100 and
$45{\times}$ in an 8-block CIFAR-10 model, yet these and a depth-scaled auxiliary term move accuracy by under
one percentage point among non-degenerate variants; on Tiny
ImageNet, a harder cross-dataset check of the selected
configuration, layer health likewise does not track
accuracy. A theory-derived attenuation-compensated objective raises deepest-block
separation $5.9{\times}$ at a consistent cost in the FF classifier's own
accuracy that a probe trained on the frozen features recovers; in a
smaller CIFAR-10 control, most of its separation gain comes from
dropping the depth-scaled auxiliary term it replaces. Accuracy invariance is a property of
the deployed sum-based scoring rule under redistribution
(Proposition~\ref{prop:redistribution}) and, empirically, of the
structured trained readouts we test, not of what the network learns:
retrained variants disagree on 13--26\% of test predictions.
Architecture and augmentation matter far more than the
non-degenerate training rules studied.
Cumulative free-riding is therefore a real and repairable
pathology, but for the softplus and squared-hinge objectives,
architectures, and datasets we study, it is not the dominant factor
limiting accuracy.
\end{abstract}

\section{Introduction}
\label{sec:intro}

The Forward-Forward (FF) algorithm~\citep{hinton2022forward} replaces the
backward pass of backpropagation (BP) with a second forward pass on
``negative'' data, training each layer greedily to maximize a scalar
\emph{goodness} on positive examples and minimize it on negative ones.
This local, layer-wise learning rule is appealing for local-learning
and memory-efficiency reasons (no activations need be stored across
layers during training); we treat the resulting system as an
engineering testbed for local objectives, not as a biological model. Yet FF methods have consistently underperformed BP on image
classification: the original algorithm achieves roughly 54--59\% on
CIFAR-10~\citep{hinton2022forward}, and subsequent improvements have
narrowed but not closed the gap~\citep{scff2025,trifecta2024,distanceforward2024,deeperforward2025}.

\paragraph{A tested hypothesis for FF underperformance.}
We test a concrete optimization hypothesis: cumulative goodness can
make later blocks inherit an already-separated problem, causing
\textbf{layer free-riding}. We show this pathology is real and
repairable, but that repairing it does not close the accuracy gap in
our controlled setting: we bound one candidate explanation of the
FF--BP gap rather than resolve the gap. We detect it through three
signatures (quantities defined in \S\ref{sec:method}, mainly
Eq.~\ref{eq:diagnostics}):
\begin{itemize}[noitemsep, leftmargin=*]
  \item \emph{Flat separation.} Cumulative wrong-label separation
        ($\text{sep}_\text{nl}$) changes little with depth under
        cumulative ($\gamma{>}0$) training, whereas each block's own
        separation ($\text{sep}^\text{cur}_\text{nl}$) differs sharply
        between variants in later blocks (App.~Fig.~\ref{fig:freeriding},
        Tab.~\ref{tab:freeriding}).
  \item \emph{Loss collapse.} In the \textsc{baseline} recipe
        ($\gamma{=}0.7$, uniform current-block weight; \S\ref{sec:setup}),
        each block's current-block loss (not LC of
        Eq.~\ref{eq:diagnostics}) is near zero at every depth, so it
        gives little training signal, whereas the
        depth-scaled auxiliary of \textsc{cp-fair} keeps it substantial
        at blocks 1--3 (App.~Fig.~\ref{fig:loss_bpff}a).
  \item \emph{Depth saturation} (DS). In three CIFAR-10 runs ($\gamma{=}0$,
        \textsc{cp-fair}, $\gamma{=}1$ (LCFF; \S\ref{sec:setup})), two-block validation accuracy
        stays within $0.2$\,pp of four-block accuracy after epoch~60
        (DS(1) of Eq.~\ref{eq:diagnostics} as a difference;
        App.~Fig.~\ref{fig:depth_gap}); this is consistent with, but not
        specific to, free-riding.
\end{itemize}

If free-riding limited accuracy, repairing it should help. It does
not: block health improves dramatically, yet the three heuristic
repairs---history-free $\gamma{=}0$, hardness-gated $\kappa{=}0$, and
the depth-scaled auxiliary (\S\ref{sec:repairs})---move accuracy by
under 1~percentage point (pp). A theory-derived
missing-gradient-compensated (MGC) objective incurs a consistent cost
to the FF classifier's own accuracy (Stage~1) on CIFAR-100 that a
probe trained on the frozen features (Stage~2) recovers; in a smaller
CIFAR-10 control, most of MGC's separation gain comes from dropping the
depth-scaled auxiliary it replaces, not from the compensation, and in
that recipe the deep-block collapse also depends on this auxiliary
(\S\ref{sec:ablation}). Accuracy invariance is a property of the
deployed rule: FF predicts from the sum of block goodness values, so
moving evidence between blocks without changing any class's summed
score leaves predictions unchanged (Prop.~\ref{prop:redistribution});
the repairs are not exact redistributions, and their accuracy
invariance, under this rule and under the structured trained readouts
we test (the Stage-2 probe and fitted goodness readouts,
App.~Tabs.~\ref{tab:readouts} and~\ref{tab:readout_ladder}), is
empirical. The networks themselves differ: pooled-embedding,
block-3-only, and $k$-NN probes separate the variants by more than
1\,pp (App.~Tab.~\ref{tab:probes}). The claim is scoped to
cumulative-goodness FF with softplus and squared-hinge barriers and to
the architectures and datasets evaluated here, not to FF in general.
The pathology is real (Theorem~\ref{thm:attenuation}; free-riding index
$\mathcal{F}_d$, Fig.~\ref{fig:grad_attenuation}), repairable
(deepest-block separation, Tab.~\ref{tab:dissociation_merged}), and not
accuracy-dominant (the heuristic repairs' paired deltas lie within
$\pm1$\,pp, Tab.~\ref{tab:dissociation_merged}, \S\ref{sec:ablation};
architecture and augmentation move accuracy by $6$--$60$\,pp,
Tab.~\ref{tab:bp_ff_ablation}).

\paragraph{Architecture matters---but so does augmentation.}
Naive FF on a standard CNN collapses to $29.85\%$, while BP on the same
CNN reaches $90.08\%$ (3~seeds, Tab.~\ref{tab:bp_ff_ablation}), a
$60.2$\,pp gap; FF on our co-designed backbone, stripped of its
mixture-of-experts (MoE), memory, contrastive, and reconstruction
terms, reaches $89.03\%$.
Augmentation alone moves BP by $6.4$\,pp on this stripped backbone
($93.85\%$ vs.\ $87.43\%$, 3~seeds), so FF-vs-BP comparisons across
recipes are confounded. Holding the backbone fixed, history-free
training ($\gamma{=}0$) removes free-riding; the $\gamma{=}0$ model has
the highest FF validation accuracy ($91.98\%$, seed~42), but its test
accuracy stays within 1\,pp of every non-degenerate variant (all but a
greedy-freeze control; Tab.~\ref{tab:main_results}).

\paragraph{Contributions.}
\begin{enumerate}[noitemsep, leftmargin=*]
  \item \textbf{Mechanism.} We formalize cumulative-goodness
        free-riding and prove exact softplus parameter-gradient
        attenuation under positive upstream margin
        (Theorem~\ref{thm:attenuation}); the mechanism extends to any
        barrier whose gradient decays in the accumulated margin.
  \item \textbf{Repairs as diagnostic tools.} History-free
        ($\gamma{=}0$) and hardness-gated repairs restore per-layer
        discrimination (the depth-scaled auxiliary,
        Eq.~\ref{eq:depth_scaled_main}, only partially;
        App.~Tab.~\ref{tab:freeriding}), with no
        end-to-end or cross-block backpropagation; we offer them as
        mechanism-isolation and per-block debugging tools, not as
        accuracy or efficiency gains (cumulative arms exit early too;
        App.~\ref{app:early_exit}).
  \item \textbf{Negative result.} On CIFAR-10 and CIFAR-100, the
        heuristic repairs' large layer-health gains
        leave accuracy within 1\,pp among non-degenerate variants
        (Table~\ref{tab:dissociation_merged}), and MGC's CIFAR-100 Stage-1
        cost is recovered at Stage~2; architecture and
        augmentation move it by $6$--$60$\,pp, so the search for the
        FF--BP gap should turn to architecture, augmentation, and
        inference rules that exploit redistributed layer-wise
        evidence. Prior-FF comparisons
        (App.~Tab.~\ref{tab:related_work_3ds}) are calibration, not
        state-of-the-art claims.
\end{enumerate}

\section{Related Work}
\label{sec:related}

\paragraph{Forward-Forward algorithm.}
\citet{hinton2022forward} introduced FF, training each layer with a
local goodness objective on a second ``negative'' forward pass.
Trifecta~\citep{trifecta2024} adds three techniques for deeper FF
training; DeeperForward~\citep{deeperforward2025} contributes
goodness-mean normalization for deeper CNNs;
Distance-Forward~\citep{distanceforward2024} replaces binary goodness
with metric/N-pair losses; SCFF~\citep{scff2025} introduces
self-contrastive negatives; CFFM~\citep{cffm2025} trains vision
transformers with a contrastive FF objective; ASGE~\citep{asge2025} scales FF on CNNs via spatial
goodness with a local cross-entropy loss and reports the strongest CIFAR
top-1 single-crop numbers of these works (90.62\%/65.42\%); and
HFF~\citep{hff2026} introduces a hyperspherical class-prototype local
objective. These works improve accuracy mainly through architecture,
goodness design, and training recipes, and their numbers differ in
backbone size, supervision, strict locality, augmentation, and
test-time augmentation (TTA); we
therefore treat them as calibration, not as a leaderboard (full
comparison in App.~Tab.~\ref{tab:related_work}). None analyzes cumulative free-riding as
a softplus attenuation mechanism, and none separates layer-health
repair from final accuracy as its primary contribution; to our
knowledge, our work is the first to do both.

\paragraph{Layer collaboration and free-riding.}
\citet{layercollab2023} study a related but distinct failure mode:
independently trained FF layers can fail to communicate useful
information across depth. Our finding is complementary. In
cumulative-goodness FF, earlier layers communicate \emph{too much}
positive margin: later blocks see a saturated softplus objective and
receive an exponentially attenuated incentive to improve their own
margin (Theorem~\ref{thm:attenuation}). The pathology is attenuation
caused by cumulative-margin aggregation, and the repairs we test
remove, gate, or compensate that aggregation.

\paragraph{Local and greedy learning.}
Deeply Supervised Nets~\citep{lee2015dsn}, Difference Target
Propagation~\citep{lee2015difference}, predictive
coding~\citep{rao1999predictive}, PEPITA~\citep{dellaferrera2022pepita},
and local triplet-loss training~\citep{forwardlayer2024} share FF's
preference for local signals but do not use the FF goodness function,
so we exclude them from the strict-FF comparison. Channel-wise
Competitive learning (CwC;~\citealt{cwc2024}; $78.11\%$ on CIFAR-10,
$51.23\%$ on CIFAR-100) replaces binary goodness with channel-group
competition, and the Scalable Forward-Forward algorithm
(SFF;~\citealt{sff2025}) confines backprop to each block, as we do;
both are in App.~Tab.~\ref{tab:related_work}. Our diagnosis applies in principle to any local
scheme in which later modules can inherit already-solved subproblems.
Our backbone uses MoE
layers~\citep{shazeer2017outrageously} with Switch load
balancing~\citep{fedus2022switch} and ST-MoE
z-loss~\citep{zoph2022stmoe}, and a Supervised Contrastive
(SupCon) auxiliary~\citep{khosla2020supcon}; these are building blocks, not
contributions. In the tradition of analysis papers that explain why
training choices succeed or
fail~\citep{zhang2017understanding,frankle2019lottery,raghu2021vision},
we identify a failure mode, prove its mechanism, repair it, and show
that under our controls the repair does not close the accuracy gap.

\section{Method}
\label{sec:method}

We use block-local gradient training: every applied update to block
$d$ comes from block $d$'s own loss (below), with no end-to-end or
cross-block backpropagation; earlier blocks' outputs enter that loss
detached (Lemma~\ref{lem:locality}; App.~\ref{app:locality} audits this
and the one exception, a depth-order gradient discarded before use).
Training has two stages: Stage~1 (S1) trains every block on its own
loss, and the FF prediction sums block goodness; Stage~2 (S2) fits an
attentive readout head on the frozen features ($\le$20 epochs), a
probe that is not part of FF.

\subsection{Architecture}

FF learning requires architectures with different inductive biases than
those designed for backpropagation. Our model is a convolutional stem +
patch embedding feeding $L{=}4$ \emph{FF Hybrid Blocks}; each block
takes the previous block's tokens (no upstream gradient) and emits a
scalar goodness value, trained high for the correct label and low for
incorrect ones. Each block contains
(a)~multi-head self-attention with rotary position embeddings (RoPE,~\citealt{su2024roformer});
(b)~an MoE feed-forward network (FFN; 32 experts, top-4,~\citealt{shazeer2017outrageously,fedus2022switch}; with a GEGLU activation,~\citealt{shazeer2020glu});
(c)~an attention-pooling head;
(d)~a four-slot goodness head (prototype alignment, activation
energy, attention sharpness, learned; with memory disabled---our
default---the attention-sharpness slot is zero, leaving three active
aspects, see App.~\ref{app:loss_formulas});
(e)~$L_2$ token normalization. The $L_2$ normalization prevents
magnitude explosion across blocks and matches memory cross-attention
in accuracy (App.~\ref{sec:secondary_analysis}); labels are injected at
every block (unlike~\citealt{hinton2022forward}), which beats first-block-only
injection (App.~Tab.~\ref{tab:full_main_results}).

\subsection{The Free-Riding Failure Mode}

For one training example $x$, its correct label $y$, and one negative
hypothesis $y^{-}$ (either a wrong label or a wrong image), let
\begin{equation}
  m^{(d)}(x,y,y^{-})
  \;=\;
  g^{(d)}_{\text{curr},+}(x,y)
  \;-\;
  g^{(d)}_{\text{curr},-}(x,y^{-})
\end{equation}
denote the \emph{current-block margin} at depth $d$. In our implementation,
the cumulative margin used by the block-level FF loss is
\begin{equation}
  M^{(d)}(x,y,y^{-})
  \;=\;
  m^{(d)}(x,y,y^{-})
  \;+\;
  \gamma \sum_{j<d} m^{(j)}(x,y,y^{-}),
  \label{eq:cum_margin}
\end{equation}
where $\gamma \in [0,1]$ is the history weight. Equivalently, block $d$
scores each hypothesis by its \emph{cumulative goodness}
$G^{(d)}_\gamma=g^{(d)}_\text{curr}+\gamma\sum_{j<d}g^{(j)}_\text{curr}$, with earlier
blocks' goodness detached, and Eq.~\ref{eq:cum_margin} is its margin. We
call training with $\gamma{>}0$ \emph{cumulative-goodness} FF
($\gamma{=}0.7$ in our cumulative runs, $\gamma{=}1$ for the
auxiliary-free LCFF of \S\ref{sec:setup}) and
$\gamma{=}0$ \emph{history-free}. The history term enters only the
training loss: at test time every variant, including $\gamma{=}0$,
predicts from the unweighted sum of block goodness,
$\arg\max_y\sum_d g^{(d)}(x,y)$. The barrier used by both the cumulative and the
current-block objectives is
\begin{equation}
  \ell_\beta(u)
  \;=\;
  \softplus(-\beta u)
  \;=\;
  \log(1 + e^{-\beta u}),
  \qquad \beta > 0.
  \label{eq:barrier}
\end{equation}
The cumulative block loss averages this barrier over the two negative types:
wrong-label negatives (NL) and wrong-image negatives (NI). Since this average
is linear, the analysis below is stated for one generic negative stream and
extends immediately to the full loss.

This notation makes the free-riding pathology explicit. If
$\sum_{j<d} m^{(j)}$ is already large because earlier blocks have separated the
example well, then the cumulative loss seen by block $d$ becomes almost flat in
its own margin $m^{(d)}$.

\begin{theorem}[Per-example attenuation of cumulative FF discrimination]
\label{thm:attenuation}
Let $P^{(d-1)}{=}\sum_{j<d} m^{(j)}$ and
$M_\gamma^{(d)}\!=\!m^{(d)}+\gamma P^{(d-1)}$, where $P^{(d-1)}$ is
detached from the parameters $\theta_d$ of block $d$. For the
softplus barrier $\ell_\beta(u)=\softplus(-\beta u)$, $\beta>0$,
define the attenuation ratio
\begin{equation}\label{eq:attenuation_ratio}
  R_\gamma(m,P)
  \;=\;
  \frac{|\ell_\beta'(m+\gamma P)|}{|\ell_\beta'(m)|}
  \;=\;
  \frac{1+e^{\beta m}}{1+e^{\beta(m+\gamma P)}}.
\end{equation}
Then, for every finite $m=m^{(d)}(\theta_d)$, the parameter gradient
of the cumulative loss equals that of the strictly-local loss, scaled
by $R_\gamma$:
\begin{equation}\label{eq:param_attenuation}
  \nabla_{\theta_d}\ell_\beta\!\bigl(M_\gamma^{(d)}\bigr)
  \;=\;
  R_\gamma(m^{(d)},P^{(d-1)})\,
  \nabla_{\theta_d}\ell_\beta\!\bigl(m^{(d)}\bigr).
\end{equation}
If $m^{(d)}\!\ge\!0$, $P^{(d-1)}\!\ge\!0$, and $\gamma\!\ge\!0$, then
\begin{equation}
  e^{-\beta\gamma P^{(d-1)}}
  \;\le\;
  R_\gamma(m^{(d)},P^{(d-1)})
  \;\le\;
  \min\bigl\{1,\; 2\,e^{-\beta\gamma P^{(d-1)}}\bigr\}.
\end{equation}
Thus the cumulative FF discrimination gradient reaching block $d$ is
exponentially attenuated in the already-accumulated positive upstream
margin. If $P^{(d-1)}\!<\!0$ the ratio can exceed one, so the claim is
specifically about the already-separated regime.
\end{theorem}

\noindent\emph{Remark.} Equation~\eqref{eq:param_attenuation} is a
parameter-gradient statement (proof and the measurable corollary
$R_d(m,P;\gamma)$ in App.~\ref{app:proofs}) that uses only the chain
rule. Hence for any barrier whose $|\ell'|$ is non-increasing in the accumulated margin,
$R_\gamma{=}|\ell'(m{+}\gamma P)|/|\ell'(m)|\le 1$ whenever
$\gamma P\ge 0$; the barrier sets only the decay profile. For the
squared hinge $\ell(u){=}\max(0,\theta{-}u)^2$ ($\theta{=}2$, $m{<}\theta$),
$R_\gamma{=}(\theta{-}m{-}\gamma P)_+/(\theta{-}m)$ decays linearly to
a hard zero; this barrier reproduces both free-riding (deepest-block
separation $1.67$ at $\gamma{=}0.7$ vs.\ $4.44$ at $\gamma{=}0$) and the
dissociation ($87.08$ vs.\ $86.93\%$ S2~TTA; CIFAR-10 L4/D128, i.e., 4~blocks of width 128; 3~seeds).
Alternative goodness statistics remain open.

\begin{figure}[t]
  \centering
  \includegraphics[width=\textwidth]{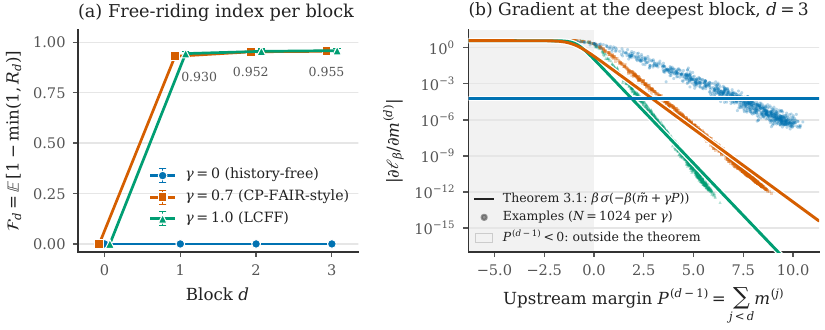}
  \caption{Empirical check of Theorem~\ref{thm:attenuation} on three
  CIFAR-10 L4/D128 Stage-1 checkpoints (seed~42, sharpness-aware
  minimization (SAM) off; first 1024 test images, one random wrong label
  each, $\beta{=}4$). The runs share architecture, optimizer and
  schedule but not every loss weight: the $\gamma{=}0.7$ run
  (\textsc{cp-fair}-style, i.e., the \emph{cumulative} setting of
  \S\ref{sec:setup}) also uses the depth-scaled auxiliary
  (Eq.~\ref{eq:depth_scaled_main}, $\lambda_0{=}0.25$) and
  $\lambda_\text{depth}{=}0.2$, both $0$ in the $\gamma{=}0$ and
  $\gamma{=}1$ (LCFF; \S\ref{sec:setup}) runs. \textbf{(a)}~Free-riding index
  $\mathcal{F}_d{=}\mathbb{E}[1{-}\min(1,R_d)]$ (App.~\ref{app:proofs})
  with 95\% confidence intervals (CIs). $\mathcal{F}_d{=}0$ under $\gamma{=}0$ and at block~0
  holds by construction ($\gamma P^{(d-1)}{=}0$ gives $R_d{\equiv}1$);
  under $\gamma{=}0.7$, $\mathcal{F}_d{=}0.930/0.952/0.955$ at blocks
  1--3. \textbf{(b)}~Per-example $|\partial\ell_\beta/\partial m^{(3)}|$
  against the upstream margin $P^{(2)}$ (all 1024 examples); curves:
  $\beta\sigma(-\beta(\tilde m+\gamma P))$ with $\tilde m$ the
  per-variant \emph{median} of $m^{(3)}$. For $\gamma{>}0$ the gradient
  decays exponentially in the upstream margin, as the theorem predicts;
  for $P{<}0$ (shaded, outside the theorem) cumulative history can
  amplify it. The snapshot uses raw weights and a non-trainer input
  normalization; recomputing on the exponential-moving-average (EMA)
  weights with the trainer's normalization changes $\mathcal{F}_d$ by at most $0.003$.}
  \label{fig:grad_attenuation}
\end{figure}

Theorem~\ref{thm:attenuation} is a precise statement of layer
free-riding: later blocks inherit a progressively easier problem and
therefore receive a progressively weaker incentive to improve their
own margin (a surrogate-loss statement, not a convergence claim). Per
example we write $R_d{:=}R_\gamma(m^{(d)},P^{(d-1)})$, the fraction of
the local-only gradient that survives cumulative aggregation, and the
\emph{free-riding index} $\mathcal{F}_d{:=}\mathbb{E}[1{-}\min(1,R_d)]\in[0,1)$,
which is $0$ without attenuation (Fig.~\ref{fig:grad_attenuation}a). The
assumptions ($m^{(d)}{\ge}0$,
$P^{(d-1)}{\ge}0$) hold on average at every block of a trained
CIFAR-10 L4/D128 cumulative model and pointwise for ${\sim}96\%$ of
test examples (App.~Tab.~\ref{tab:assumption_coverage}), and the
batch-mean $R_d$ collapses to $\le 10^{-2}$ from
block~1 onward under $\gamma{=}0.7$.

FF prediction $\hat{y}{=}\arg\max_y S_y(x)$ with $S_y{=}\sum_d g_d(\cdot,y)$
is correct iff the \emph{summed class margin}
$\Delta(x){=}S_{y^\star}{-}\max_{y\neq y^\star}\!S_y$ is positive.
\begin{proposition}[Compensating redistribution leaves FF inference unchanged]
\label{prop:redistribution}
Fix any two depths $a\ne b$ and any class-dependent function
$q(x,y)$. Define $\tilde g_a(x,y){=}g_a(x,y)+q(x,y)$,
$\tilde g_b(x,y){=}g_b(x,y)-q(x,y)$, and
$\tilde g_d{=}g_d$ for $d\notin\{a,b\}$. Then
$\sum_d \tilde g_d(x,y) = \sum_d g_d(x,y)$ for every $y$, so all
summed-goodness FF predictions are unchanged. For any $y'\ne y$, the
current-block wrong-label margin at depth $a$,
$\tilde m^{(a)}(x,y,y') = \tilde g_a(x,y) - \tilde g_a(x,y')$,
shifts by $q(x,y) - q(x,y')$ relative to the original
$m^{(a)}(x,y,y')$, and at depth $b$ by its negative. Per-block margin
diagnostics are therefore not identifiable from summed-goodness
predictions alone.
\end{proposition}
An exactly redistributing change moves \emph{which} block
contributes to $\Delta$, not $\Delta$: per-block health diagnoses
learning allocation, not accuracy. More generally:

\begin{proposition}[Prediction and accuracy stability under summed-score perturbations]
\label{prop:prediction_stability}
Let $S^A_y(x)$ and $S^B_y(x)$ be two FF classifiers' summed class
scores, with $\hat y_A(x)=\arg\max_y S^A_y(x)$,
\(
  \Delta_A(x) = S^A_{\hat y_A(x)}(x) - \max_{y\ne \hat y_A(x)} S^A_y(x)
\),
and \(E(x) = \max_y |S^A_y(x) - S^B_y(x)|\). For any $t\ge 0$,
\[
  \Pr\!\bigl[\hat y_A(x)\ne \hat y_B(x)\bigr]
  \;\le\;
  \Pr\!\bigl[\Delta_A(x)\le 2t\bigr]
  \;+\;
  \Pr\!\bigl[E(x)>t\bigr],
\]
and consequently, for true labels $Y$,
\[
  |\mathrm{Acc}(A)-\mathrm{Acc}(B)|
  \;\le\;
  \Pr\!\bigl[\hat y_A(x)\ne \hat y_B(x)\bigr]
  \;\le\;
  \Pr\!\bigl[\Delta_A(x)\le 2t\bigr]
  \;+\;
  \Pr\!\bigl[E(x)>t\bigr].
\]
\end{proposition}

\noindent The proof is in Appendix~\ref{app:proofs}. Both propositions
are arithmetic, informative for (nearly) redistributing score changes,
not for retrained networks. Re-splitting the cumulative CIFAR-100
model's real goodness sum so that the deepest block's own share moves
from $14\%$ to $98\%$ flips 0 of 10{,}000 predictions. Retrained
variants differ empirically: the three CIFAR-100 variants disagree on
${\sim}13\%$ of test predictions (S2~TTA) and MGC
(App.~\ref{app:attenuation_compensated}) vs.\ cumulative on
$22.7$--$26.2\%$, yet at Stage~2 accuracy moves by under 1\,pp
(Table~\ref{tab:dissociation_merged}).

\paragraph{Block-health diagnostics.} The empirical signatures of
free-riding we report throughout the paper are summarized by four
quantities, defined here once for reference:
\begin{equation}
\label{eq:diagnostics}
\begin{aligned}
  \text{sep}^\text{cur}_\text{nl}(d) &\;:=\; \mathbb{E}_{(x,y)}\!\left[
    g^{(d)}_{\text{curr},+}(x,y)
    \;-\; g^{(d)}_{\text{curr},-}(x,y^-)
  \right], \\[2pt]
  \text{sep}_\text{nl}(d) &\;:=\; \Pr_{(x,y)}\!\left[
    M^{(d)}_\gamma(x,y,y^-) > 0
  \right], \\[2pt]
  \text{DS}(d) &\;:=\; \frac{\mathrm{Acc}(\text{using only blocks } 0..d)}{\mathrm{Acc}(\text{full } L\text{-block model})}, \\[2pt]
  \text{LC}(d) &\;:=\; \mathbb{E}_{(x,y,y^-)}\!\left[
    \ell_\beta\!\bigl(\bar M^{(d)}(x,y,y^-)\bigr)
  \right]_{\text{end-of-Stage-1}},
\end{aligned}
\end{equation}
where $y^-$ is the hard-mined wrong label (the highest-goodness of
$k$ sampled wrong labels; App.~\ref{app:loss_formulas}),
$\bar M^{(d)}\!=\!\sum_{j\le d}m^{(j)}$ is the
summed \emph{inference} margin (of the $\gamma$-free score
$S^{(d)}_y\!=\!\sum_{j\le d}g^{(j)}$), and
$M^{(d)}_\gamma\!=\!m^{(d)}+\gamma P^{(d-1)}$ the cumulative
\emph{loss} margin (of the loss score $G^{(d)}_{\gamma}$,
App.~\ref{app:loss_formulas}). The two separations have different
units: $\text{sep}_\text{nl}$ is a fraction of examples in $[0,1]$,
whereas $\text{sep}^\text{cur}_\text{nl}$ is a goodness margin
(unbounded); both are logged on Stage-1 training batches. We do not treat a small
$\text{LC}(d)$ as healthy by itself: low $\text{LC}(d)$ together with
collapsed $\text{sep}^\text{cur}_\text{nl}(d)$ is our diagnostic of
free-riding---the summed-goodness classifier is already separating
positives from negatives, so block $d$'s own discrimination pressure
vanishes. A well-allocated solution should keep
$\text{sep}^\text{cur}_\text{nl}(d)$ non-collapsed even when the
summed score is confident. We plot the current-block loss (App.~Fig.~\ref{fig:loss_bpff}a), not LC$(d)$.

\subsection{Repairs}
\label{sec:repairs}

\paragraph{Depth-scaled auxiliary.}
To counteract free-riding, we add a loss that evaluates block $d$
\emph{in isolation}, on its current-block margin $m^{(d)}$, and scale its
weight with depth:
\begin{equation}\label{eq:depth_scaled_main}
  \mathcal{L}_\text{curr}^{(d)}
  \;=\;
  \frac{1}{B}\sum_{i=1}^B w_i^{(d)} \, \ell_\beta\!\bigl(m_i^{(d)}\bigr),
  \qquad
  \lambda_\text{curr}(d)
  \;=\;
  \lambda_0 \cdot \Bigl(1 + \rho \cdot \tfrac{d}{L-1}\Bigr).
\end{equation}
With $\lambda_0{=}0.25$ and $\rho{=}3.0$, the effective weights are
$[0.25, 0.50, 0.75, 1.00]$ for blocks $0$--$3$. The deepest block therefore
faces $4\times$ the direct discrimination pressure of the shallowest.

For $d{>}0$ examples are reweighted by \emph{residual weights}
$w_i^{(d)} \propto \sigmoid(-\beta P_i^{(d-1)})$ (batch-mean
normalized, unclipped) where $P_i^{(d-1)} = \sum_{j<d} m_i^{(j)}$, so block $d$ focuses on
examples earlier blocks have not yet separated; we call
$\lambda_\text{curr}(d)\,\mathcal{L}_\text{curr}^{(d)}$ the
\emph{depth-scaled auxiliary}. Appendix~\ref{app:proofs}
proves these weights preserve the mean batch-loss scale and states
Theorem~\ref{thm:gradient_floor} (the current-block term supplies a
per-example gradient floor of $\lambda_\text{curr}(d)\, w_i^{(d)} \beta / 2$ on
unresolved examples, which depends on the upstream margin only through
$w_i^{(d)}$).

\paragraph{A simpler alternative: $\gamma{=}0$.}
Setting $\gamma{=}0$ eliminates cumulative-loss attenuation
(Theorem~\ref{thm:attenuation}) by construction; it does
\emph{not} guarantee monotone per-block utility under all datasets
(harder data can still show peak-then-decline under $\gamma{=}0$ via
mechanisms outside our theorem; see Section~\ref{sec:tinyimagenet}).
On CIFAR-10, $\gamma{=}0$ has the highest validation accuracy and stays
within 1\,pp of every non-degenerate variant on test
(Table~\ref{tab:main_results}; full sweep:
App.~Tab.~\ref{tab:full_main_results}); the depth-scaled auxiliary is
an option when cumulative scoring ($\gamma{>}0$) is required, though in
one recipe the deep-block collapse also depends on it
(\S\ref{sec:ablation}).

\paragraph{Hardness-gated collaboration.}
This repair replaces the constant $\gamma$ by a per-example gate
$\gamma_0\,\sigma\bigl(\tau(\kappa-\textstyle\sum_{j<d}g^{(j)})\bigr)$ on
the goodness that earlier blocks gave the same input, so block $d$
inherits little history where that goodness is already high (easy
examples) and up to $\gamma_0$ where it is low (hard ones); we report
$\kappa{=}0$ (App.~\ref{app:gated_full}; App.~Tab.~\ref{tab:code_paper_map}).

\paragraph{Theory-derived compensation (MGC).}
Theorem~\ref{thm:attenuation} also suggests a direct repair: restore
the fraction of the local gradient that the cumulative loss removes.
The missing-gradient-compensated (MGC) objective adds to the cumulative
term a local term $\lambda_i^{(d)}\ell_\beta(m_i^{(d)})$ with
$\lambda_i^{(d)}{=}[c_d{-}R_i^{(d)}]_+$ (stop-gradient, $c_d{=}1$), so
that where $\gamma P_i^{(d-1)}{\ge}0$ the per-example gradient on
$m_i^{(d)}$ matches the $\gamma{=}0$ magnitude
(Prop.~\ref{prop:mgc_recovers_local}). Our runs are a faithful
implementation that approximately compensates the envelope, not exact
gradient recovery ($\epsilon{=}10^{-6}$ in the ratio, positive-part
clamp). The MGC local term \emph{replaces} the depth-scaled
auxiliary; since $R^{(0)}{\equiv}1$, the compensator is zero at
block~0, which therefore loses its auxiliary term
(App.~\ref{app:attenuation_compensated}).

\paragraph{Total loss per block.}
The implemented block loss is
\begin{equation}
\begin{aligned}
  \mathcal{L}^{(d)} = {} &
    \mathcal{L}_\text{aspects}^{(d)}
    + \lambda_\text{block} \, \mathcal{L}_\text{block}^{(d)}
    + \lambda_\text{curr}(d) \, \mathcal{L}_\text{curr}^{(d)}
    + \lambda_\text{depth} \, \mathcal{L}_\text{depth}^{(d)} \\
  & {} + \lambda_\text{con} \, \mathcal{L}_\text{SupCon}^{(d)}
    + \lambda_\text{recon} \, \mathcal{L}_\text{recon}^{(d)}
    + \lambda_\text{bal} \, \mathcal{L}_\text{bal}^{(d)}
    + \lambda_z \, \mathcal{L}_z^{(d)},
\end{aligned}
\label{eq:total_loss}
\end{equation}
where $\mathcal{L}_\text{block}^{(d)}$ is the cumulative-goodness
discrimination loss
(softplus separation between the loss score $G_{\gamma,+}^{(d)}$ and the wrong-label /
wrong-image loss scores, weighted by $\eta$;
explicit form in App.~Eq.~\ref{eq:Lblock}),
$\mathcal{L}_\text{depth}$ penalizes violations of
$G^{(d)}_{\gamma,+} - g^{(d-1)}_{+} \ge \delta_{+}$ and
$g^{(d-1)}_{-} - G^{(d)}_{\gamma,-} \ge \delta_{-}$ (block $d$'s loss
score against block $d{-}1$'s own goodness; App.~Eq.~\ref{eq:Ldepth}),
$\mathcal{L}_\text{recon}$ is the token reconstruction loss,
and $\mathcal{L}_\text{bal}, \mathcal{L}_z$ are the Switch/ST-MoE router
regularizers (full forms in App.~\ref{app:loss_formulas}). App.~\ref{app:proofs}
also shows that an idealized hard-constraint form of the depth-order
inequalities implies monotone growth of cumulative separation with
depth. Coefficients are in App.~Tab.~\ref{tab:loss_hparams}.

Diagrams of the architecture, per-block loss, and goodness head are in
App.~Figs.~\ref{fig:architecture}--\ref{fig:goodness_aspects}.

\paragraph{Training protocol.}
Stage~1 trains all blocks simultaneously, each on its own loss above
(sharpness-aware minimization (SAM;~\citealt{foret2021sam}) + AdamW,
RandAugment, hard-negative mining (HNM) with an exponential-moving-average
(EMA) teacher; 362 epochs on CIFAR-10); Stage~2 ($\le$20 epochs) trains
a separate attentive head on frozen features. The full protocol and the ablation hierarchy (which
experiments carry the causal claim) are in
App.~\ref{app:training_protocol}.

\section{Experiments}
\label{sec:experiments}

\subsection{Experimental Setup}
\label{sec:setup}

\textbf{Protocols.} \textbf{S1} = Stage-1 FF backbone with its
summed-goodness classifier (strict-local FF inference); \textbf{S2} =
Stage-2 attentive readout head trained briefly on the frozen backbone
($\le\!20$ epochs), a representation-quality probe, not part of the FF
method; we report both. The head is attentive because a single local
readout block, whether trained with the FF objective or with
cross-entropy, stays $5.3$--$6.4$\,pp below it (seed~42): the gap is
architectural (App.~\ref{app:readouts}). \textbf{no-TTA} = single-crop;
\textbf{TTA} = horizontal-flip test-time augmentation. All $\pm$ values
are sample standard deviations (SDs) over seeds.

\textbf{Variants.} \textbf{$\gamma{=}0$} is history-free FF (no
inter-block term in the loss; every variant's updates are block-local,
App.~\ref{app:locality}); \textbf{cumulative} is $\gamma{=}0.7$ with
the depth-scaled auxiliary (Eq.~\ref{eq:depth_scaled_main}); on the
32-expert CIFAR-10 model this setting is called \textbf{\textsc{cp-fair}};
\textbf{LCFF} is the plain prefix-sum baseline ($\gamma{=}1$, no
auxiliary, $\lambda_\text{curr}{=}\lambda_\text{depth}{=}0$), and any other
auxiliary-free $\gamma{>}0$ arm is named as such; hardness-gated
\textbf{$\kappa{=}0$} and \textbf{MGC} are defined in
\S\ref{sec:repairs}. Table~\ref{tab:main_results} adds \textsc{baseline}
(the default $\gamma{=}0.7$ recipe with uniform
$\lambda_\text{curr}{=}0.5$ and memory), \textsc{multi-tech}
(\textsc{cp-fair}'s loss plus 10 further techniques, 24 experts), and
\textsc{greedy-freeze} (pure greedy block freezing), the only
degenerate variant; \emph{non-degenerate} means all others
(configurations: App.~Tab.~\ref{tab:full_config}). $\gamma{=}0$
seed~42 used $\lambda_\text{curr}{=}\lambda_\text{depth}{=}0$, seeds
123/456 used $0.25$/$0.2$; with seed~42 rerun at $0.25$/$0.2$, the
3-seed S1-TTA mean is $91.30 \pm 0.17\%$
(App.~Tab.~\ref{tab:gamma0_multiseed}).

\textbf{Settings.} L$x$/D$y$ denotes $x$ blocks of width $y$. The dense
CIFAR-100 L4/D256 runs (the \emph{trio} of history rules $\gamma{=}0$,
gated, and cumulative, 3 seeds each, plus MGC; no MoE, no SAM; 362+20 epochs)
carry the central dissociation (\S\ref{sec:ablation}); CIFAR-10
L4/D128 (180+10 epochs) carries the mechanism checks
(Fig.~\ref{fig:grad_attenuation}, the second barrier, the MGC control);
CIFAR-10 L4/D256 with 32 MoE experts carries the accuracy calibration
(Tabs.~\ref{tab:main_results} and~\ref{tab:bp_ff_ablation}); L8/D128
probes depth (App.~\ref{app:gated_full}).
CIFAR-10~\citep{krizhevsky2009cifar} uses $45{,}000 / 5{,}000 /
10{,}000$ train/val/test (split seed = run seed; all variants of a
seed share splits); Tiny ImageNet (200~classes, $64{\times}64$, 100k
train / 10k val) uses the $\gamma{=}0$ configuration. The 32-expert
CIFAR-10 model (8~heads, top-4 routing per block) has $107.2$\,M
parameters in total and $18.9$\,M active per token (all $6.25$\,M
non-expert parameters plus the top-4 experts in each of the 4 blocks;
$16.2$\,M excluding training-only heads and the unused memory path);
in total this exceeds the 12- and 17-layer FF backbones
of \citet{trifecta2024} and \citet{deeperforward2025} (${\sim}8.5$ and, by our count,
${\sim}31$\,M) but not the latter's channel-tripled variant, and
\S\ref{sec:bp_ff_ablation} addresses the architecture confound.

\subsection{CIFAR-10 Accuracy and Calibration}
\label{sec:c10_accuracy}

\begin{table}[t]
  \caption{CIFAR-10 test accuracy on the co-designed FF backbone:
           five interpretable variants on Stage~1 (defined in
           \S\ref{sec:setup}). \textbf{Bold} marks the
           validation-selected configuration.
           The full 13-variant sweep and prior published FF baselines
           are in App.~\ref{sec:related_work_full}
           (Tabs.~\ref{tab:full_main_results}--\ref{tab:related_work_3ds});
           per-seed breakdowns are in App.~\ref{sec:multiseed_supplement}.
           $^\dagger$Seed~42 used a different configuration
           (\S\ref{sec:setup}; App.~Tab.~\ref{tab:gamma0_multiseed}).}
  \label{tab:main_results}
  \centering
  \small
  \setlength{\tabcolsep}{4pt}
  \begin{tabular}{lcccc}
    \toprule
    Variant & $n$ (seeds) & Best Val (s42) & Test no-TTA & Test TTA \\
    \midrule
    $\gamma{=}0$ (history-free, validation-selected) & 3 (42,123,456)$^\dagger$ & \textbf{91.98\%} & $90.50 \pm 0.05$ & $91.32 \pm 0.19$ \\
    \textsc{cp-fair}                                 & 4 (42,123,456,789)    & 91.90\%          & $90.36 \pm 0.18$          & $91.33 \pm 0.07$ \\
    \textsc{multi-tech}                              & 1 (42)                & 90.94\%          & 90.55\%                   & 91.22\% \\
    \textsc{baseline}                                & 1 (42)                & 91.74\%          & 90.69\%                   & 91.20\% \\
    \textsc{greedy-freeze} (degenerate)              & 1 (42)                & 87.66\%          & 86.23\%                   & 87.17\% \\
    \bottomrule
  \end{tabular}
\end{table}

Selected by validation accuracy, the history-free variant ($\gamma{=}0$)
reaches $90.52\%$ single-crop and $91.45\%$ TTA on seed~42, and
$91.32 \pm 0.19\%$ S1~TTA and $91.37 \pm 0.15\%$ S2~TTA over 3 seeds
(App.~Tab.~\ref{tab:gamma0_multiseed}; seed-42 note in
\S\ref{sec:setup}). Among non-degenerate variants, no-TTA accuracy
spans $<1$\,pp---small relative to the spread across backbones
(\S\ref{sec:bp_ff_ablation}); we default to single-crop in the main
text. For calibration only (not architecture-controlled; none of the
listed prior FF works report TTA), our single-crop $\gamma{=}0$
backbone is within $0.12$\,pp of the strongest single-crop FF result
listed (ASGE;~\citealt{asge2025}; linear readout) on CIFAR-10 and $+0.95$\,pp above it
on CIFAR-100, and $+12.9$\,pp above SCFF's unsupervised Tiny ImageNet
result (App.~Tab.~\ref{tab:related_work_3ds}). Per-seed tables,
including hardness-gated $\kappa{=}0$ at L4/D128 and L8/D128, are in
App.~\ref{sec:multiseed_supplement}.

\subsection{Cross-dataset Generalization to Tiny ImageNet}
\label{sec:tinyimagenet}

We replicate the $\gamma{=}0$ configuration on \textbf{Tiny
ImageNet}~\citep{le2015tinyimagenet} (200 classes, $64{\times}64$;
labeled validation split, no hidden test). The Tiny backbone
matches CIFAR-10 in depth and width (L4, D256, 362+20 epochs,
bs~256) but is dataset-tuned: dense multilayer perceptron (MLP) instead of MoE, no SAM, 8 heads, HNM
$k{=}20\!\to\!40$ (App.~\ref{sec:tiny_imagenet_metrics}); we therefore
treat Tiny as cross-dataset evidence, not an
architecture-controlled transfer. Across three seeds (42, 123, 456),
$\gamma{=}0$ reaches $48.54 \pm 0.19\%$ S1 single-crop,
$49.09 \pm 0.39\%$ S1 TTA, and $52.32 \pm 0.34\%$ S2 TTA top-1
(prior-FF context in \S\ref{sec:c10_accuracy}). Per-seed numbers and the per-block
diagnostics (App.~Fig.~\ref{fig:tinyimagenet_crossdata}) are in App.~\ref{sec:tiny_imagenet_metrics}.

On Tiny ImageNet $\text{sep}^\text{cur}_\text{nl}$ peaks at Block~1 and
declines deeper even under $\gamma{=}0$, a further depth pathology
outside Theorem~\ref{thm:attenuation} whose cause is open; per-depth
accuracy nonetheless rises with every block (seed~42;
App.~\ref{sec:tiny_imagenet_metrics}), so layer health again does not
track accuracy. A frozen-BERT sanity check on
IMDb, 20~Newsgroups, and AG~News exceeds published FF baselines by $+1.02$ to
$+2.39$\,pp (App.~\ref{app:textdomain}). \label{sec:textdomain}

\subsection{BP vs.\ FF Controlled Ablation}
\label{sec:bp_ff_ablation}

\begin{table}[h]
  \caption{Controlled ablation: training algorithm vs.\
           architecture vs.\ augmentation. FF and BP-strong both use
           strong augmentation (RandAugment + RandomResizedCrop +
           ColorJitter + GaussianBlur; neither uses Mixup/CutMix in
           Stage-1), but the recipes differ in crop range and operator
           rates (App.~\ref{sec:full_config}), so the FF-vs-BP rows
           compare pipelines, not matched augmentation. Rows 1--4:
           mean$\pm$SD, 3 seeds; others: seed~42.
           The stripped backbone has 5.54\,M
           parameters under BP and 6.74\,M under FF (plus per-block
           heads); the last two rows use the 32-expert
           model ($107.2$\,M total, $18.9$\,M active).}
  \label{tab:bp_ff_ablation}
  \centering
  \setlength{\tabcolsep}{3pt}
  \providecommand{\BPRowsNThree}{}
\renewcommand{\BPRowsNThree}{%
  Plain CNN (1.15\,M, ch.\ 64/128/256) & BP & strong & $90.91\!\pm\!0.20$ & $90.08\!\pm\!0.28$ & --- \\
  Plain CNN (1.15\,M, ch.\ 64/128/256) & FF & FF-pipeline & $31.35\!\pm\!0.94$ & $29.85\!\pm\!0.50$ & --- \\
  FF backbone (stripped, 5.54\,M) & BP & weak & $87.83\!\pm\!0.81$ & $87.43\!\pm\!0.62$ & --- \\
  FF backbone (stripped, 5.54\,M) & BP & BP-strong & $93.99\!\pm\!0.18$ & $93.85\!\pm\!0.18$ & --- \\
}
  \resizebox{\textwidth}{!}{%
  \begin{tabular}{lllccc}
    \toprule
    Architecture & Training & Aug. & Val Top-1 & Test no-TTA & Test TTA \\
    \midrule
    \BPRowsNThree
    FF backbone (stripped, 6.74\,M) & FF & FF-pipeline & 89.68\% & 89.03\% & --- \\
    \midrule
    FF backbone + $\gamma{=}0$ (validation-selected FF) & FF & FF-pipeline & \textbf{91.98\%} & $90.52\%$ & \textbf{91.45\%} \\
    FF backbone + depth-scaled (\textsc{cp-fair}) & FF & FF-pipeline & 91.90\% & 90.45\% & 91.36\% \\
    \bottomrule
  \end{tabular}%
  }
\end{table}

Three findings: (i) FF needs architecture co-design ($29.85\%$ vs.\
$89.03\%$); (ii) augmentation is a confound (BP-strong $93.85\%$ vs.\
weak-aug BP $87.43\%$ on the same backbone---the FF/weak-BP gap is
not augmentation-controlled); (iii) $\gamma{=}0$ is the cleanest
free-riding ablation ($91.45\%$ TTA). Training-cost and locality
trade-offs are in App.~\ref{app:efficiency}.

\subsection{Main Result: Separation--Accuracy Dissociation}
\label{sec:ablation}\label{sec:adaptive_gamma}\label{sec:adaptive_collab}

\paragraph{Real and repairable, but accuracy-flat.}
On CIFAR-10 (L4/D256, seed~42), deepest-block
$\text{sep}^\text{cur}_\text{nl}$(L3) is $\mathbf{3.28}$ under
$\gamma{=}0$ but collapses to $0.70$ under LCFF and $0.61$ under
\textsc{cp-fair}, while validation accuracy moves $<\!0.7$\,pp
(App.~Tab.~\ref{tab:freeriding}; per-seed values and the configuration
difference: App.~Tab.~\ref{tab:gamma0_multiseed}). At 8 blocks,
hardness gating raises block-7 separation ${\sim}45{\times}$ over LCFF
($6.94$ vs.\ $0.15$, seed~42; App.~\ref{app:gated_full}). On CIFAR-100
(Table~\ref{tab:dissociation_merged}) the three history rules
($\gamma{=}0$, gated, cumulative; all with the depth-scaled auxiliary)
move deepest-block separation ${\sim}5{\times}$ but no protocol by
1\,pp: every paired 95\% CI lies inside $\pm 1$\,pp, although
${\sim}13\%$ of test predictions change between variants (S2~TTA).

\begin{table}[!tb]
  \caption{\textbf{Central dissociation under four readout protocols (CIFAR-100, dense L4/D256, 3~seeds).}
           sep: deepest-block $\text{sep}^\text{cur}_\text{nl}$ (L3) at the last Stage-1 epoch.
           Accuracy: test top-1 (\%) of the validation-selected checkpoints (S1: EMA backbone; S2: trained head
           on the frozen EMA backbone), mean $\pm$ sample SD.
           The trio arms differ only in the history rule ($\gamma{=}0$; gated: hardness-gated $\kappa{=}0$;
           cumulative: $\gamma{=}0.7$) and all carry the depth-scaled auxiliary of Eq.~\ref{eq:depth_scaled_main};
           MGC ($c_d{=}1$, \S\ref{sec:repairs}) replaces that auxiliary by the compensator.
           $\Delta$ rows: paired difference to the cumulative arm on matched seeds ($n{=}30{,}000$ test
           predictions) with a 5000-resample bootstrap 95\% CI, conditional on the trained checkpoints;
           per-seed deltas are in App.~\ref{app:mgc_results} (MGC) and App.~\ref{app:churn} (trio).
           ``$\pm1$\,pp?'': every CI lies inside $\pm1$\,pp. Range: max/min ratio for sep and max$-$min for
           accuracy, over the unrounded variant means. $\gamma{=}0$ and gated sit slightly below cumulative
           at S1 ($\gamma{=}0$ also at S1-TTA), with CIs that exclude 0 but stay inside $\pm1$\,pp.
           MGC seed~456 and 5 of the 9 trio runs
           were resumed from checkpoints (App.~Tabs.~\ref{tab:mgc_full}
           and~\ref{tab:cifar100_dissociation_multiseed}). C10 pair: $\Delta$ [95\% CI] of \textsc{cp-fair} $-$ $\gamma{=}0$ on CIFAR-10
           (L4/D256, 3 seeds $\times$ 10k test images; per-seed exact McNemar $p{>}0.17$ under TTA;
           $\gamma{=}0$ seed-42 note in \S\ref{sec:setup}).}
  \label{tab:dissociation_merged}
  \centering
  \footnotesize
  \setlength{\tabcolsep}{1.5pt}
  \resizebox{\textwidth}{!}{%
  \begin{tabular}{lcccccl}
    \toprule
    Variant & sep (L3) & S1 & S1-TTA & S2 & S2-TTA & $\pm1$\,pp? \\
    \midrule
    $\gamma{=}0$ & $4.78\!\pm\!0.02$ & $66.37\!\pm\!0.20$ & $66.93\!\pm\!0.14$ & $68.54\!\pm\!0.25$ & $69.23\!\pm\!0.34$ &  \\
    \quad $\Delta$ vs.\ cum. &  & $-0.42$\,[$-0.76$,\,$-0.07$] & $-0.48$\,[$-0.82$,\,$-0.15$] & $+0.04$\,[$-0.26$,\,$+0.34$] & $+0.28$\,[$-0.02$,\,$+0.56$] & Yes \\
    Gated $\kappa{=}0$ & $1.71\!\pm\!0.01$ & $66.35\!\pm\!0.41$ & $67.35\!\pm\!0.66$ & $68.44\!\pm\!0.51$ & $69.07\!\pm\!0.61$ &  \\
    \quad $\Delta$ vs.\ cum. &  & $-0.43$\,[$-0.79$,\,$-0.08$] & $-0.07$\,[$-0.40$,\,$+0.27$] & $-0.07$\,[$-0.39$,\,$+0.24$] & $+0.11$\,[$-0.20$,\,$+0.42$] & Yes \\
    Cumulative & $0.96\!\pm\!0.01$ & $66.78\!\pm\!0.28$ & $67.41\!\pm\!0.35$ & $68.51\!\pm\!0.51$ & $68.96\!\pm\!0.43$ &  \\
    MGC & $5.61\!\pm\!0.04$ & $65.73\!\pm\!0.35$ & $66.58\!\pm\!0.25$ & $68.24\!\pm\!0.51$ & $69.09\!\pm\!0.27$ &  \\
    \quad $\Delta$ vs.\ cum. &  & $-1.05$\,[$-1.46$,\,$-0.62$] & $-0.83$\,[$-1.23$,\,$-0.42$] & $-0.27$\,[$-0.68$,\,$+0.14$] & $+0.13$\,[$-0.27$,\,$+0.53$] & \begin{tabular}[t]{@{}l@{}}No (S1,\\S1-TTA)\end{tabular} \\
    \midrule
    Range & $5.86{\times}$ & $1.05$\,pp & $0.83$\,pp & $0.31$\,pp & $0.28$\,pp &  \\
    \midrule
    C10 pair & --- & $-0.06$\,[$-0.35$,\,$+0.23$] & $+0.05$\,[$-0.20$,\,$+0.29$] & $+0.03$\,[$-0.26$,\,$+0.31$] & $-0.07$\,[$-0.32$,\,$+0.19$] & Yes \\
    \bottomrule
  \end{tabular}%
  }
\end{table}

\paragraph{Theory-derived repair (MGC).}
MGC raises CIFAR-100 deepest-block separation to $5.61 \pm 0.04$ vs.\
$0.96 \pm 0.01$ ($5.9{\times}$), yet its accuracy effect is a Stage-1
cost, negative in every seed ($\Delta{=}{-}1.05$\,[${-}1.46$,\,${-}0.62$]\,pp
single crop, ${-}0.83$\,[${-}1.23$,\,${-}0.42$] with TTA), that the
trained Stage-2 probe recovers (${-}0.27$\,[${-}0.68$,\,${+}0.14$] and
${+}0.13$\,[${-}0.27$,\,${+}0.53$]), while MGC and cumulative disagree
on $22.7$--$26.2\%$ of test predictions. On CIFAR-10 L4/D128, against a
SAM-matched cumulative arm (3 seeds per arm, GPU types mixed;
App.~\ref{app:f1_control}), MGC raises deepest-block separation from
$0.75 \pm 0.01$ to $5.25 \pm 0.05$; accuracy moves by
$-0.29$\,[$-0.59$,\,$+0.02$] (S1) and $+0.24$\,[$-0.02$,\,$+0.50$]\,pp
(S2~TTA). A control attributes most of this gain to dropping the
depth-scaled auxiliary, not to compensation. The auxiliary-free
cumulative arm (no compensator) reaches $4.88 \pm 0.05$ ($92\%$ of the gain; our prediction, registered before
the runs, that it would stay below $4.72$ failed). Keeping block~0's
auxiliary under MGC leaves separation unchanged ($5.28$) and adds no
accuracy. The cumulative term is attenuated with and without the
auxiliary ($\mathcal{F}_3{=}0.96$/$0.97$, seed~42, end of training), so
in this recipe the deep-block collapse is not caused by attenuation
(Theorem~\ref{thm:attenuation}) alone but also depends on the
depth-scaled auxiliary (Eq.~\ref{eq:depth_scaled_main}); how the two
interact is open, and LCFF, which has no auxiliary, still collapses
(App.~Tabs.~\ref{tab:freeriding} and~\ref{tab:dissociation_full}).
Across the four arms (cumulative, MGC, auxiliary-free cumulative,
MGC with block-0 auxiliary), separation spans $0.75$--$5.28$, yet every paired
95\% CI lies inside $\pm1$\,pp. MGC thus changes what the network
learns, but not accuracy beyond a Stage-1 cost that the probe recovers.

\paragraph{Recipe vs.\ mechanism.}
Could the co-designed recipe, rather than the mechanism, produce the
dissociation? Four controls bear on this. (i)~The central CIFAR-100
trio is dense, with no MoE and no SAM
(App.~Tab.~\ref{tab:cifar100_dissociation_multiseed}).
(ii)~In a leave-one-out ablation
of the CIFAR-10 recipe, removing SAM, SupCon, or $\gamma$ accumulation
each changes S2~TTA by $+0.25$\,pp (seed~42;
App.~Tab.~\ref{tab:component_ablation}). (iii)~FF on the stripped
backbone (no MoE, memory, contrastive, or reconstruction terms) still
reaches $89.03\%$ (Tab.~\ref{tab:bp_ff_ablation}, one seed).
(iv)~The squared-hinge barrier reproduces the dissociation on the
L4/D128 recipe (App.~Tab.~\ref{tab:hinge}); MGC, on the trio's own
dense recipe, moves no protocol toward a gain, and its only CIs outside
$\pm 1$\,pp are the Stage-1 costs that the probe recovers
(Tab.~\ref{tab:dissociation_merged}).
These controls mitigate, but do not eliminate, the concern that the
recipe drives the result.

\section{Limitations and Conclusion}
\label{sec:limitations}
\paragraph{Limitations.}
Scope: softplus and squared-hinge (CIFAR-10 L4/D128 only) barriers on
one goodness statistic, our co-designed backbones (the CIFAR-100 trio
is dense, without MoE or SAM), and CIFAR/Tiny ImageNet. ImageNet scale
(the main open boundary) and classic FF architectures are untested.
Depth is open: at L8 with 32 experts all regimes stay within 1\,pp but
the depth-scaled repair is best (one seed; App.~Tab.~\ref{tab:dissociation_full}),
so repairs may begin to matter at depth; Tiny ImageNet's Block-1 peak lies outside
Theorem~\ref{thm:attenuation}. Recipe controls mitigate, not
eliminate, a recipe confound. At L4/D128 a control resolves MGC's
block-0 confound for separation (with no masked accuracy gain at
$n{=}3$) but traces most of its separation gain to removing the depth-scaled
auxiliary it replaces (Eq.~\ref{eq:depth_scaled_main}); how accumulation
and that auxiliary interact is open, and the CIFAR-100 gain is not decomposed
(App.~\ref{app:f1_control}). Stage-2 is a backprop-trained
probe, not part of the FF method; some rows (e.g., the stripped FF row
of Tab.~\ref{tab:bp_ff_ablation}) are single-seed. The claim is
falsifiable (App.~\ref{app:long_limitations}).

\paragraph{Conclusion.}
Cumulative-goodness free-riding is real and repairable but, in the
settings we study, not the dominant accuracy bottleneck. The
attenuation is exact (Theorem~\ref{thm:attenuation}) and measurable
($\mathcal{F}_d{=}0.930$--$0.955$ at blocks 1--3 under $\gamma{=}0.7$).
Removal, gating, and the compensated objective (whose separation gain
at L4/D128 comes mostly from dropping the auxiliary it replaces) restore
block health and change 13--26\% of test predictions, yet leave accuracy within 1\,pp under the
deployed rule and the structured readouts we test, except for MGC's
Stage-1 cost, which the Stage-2 probe recovers. The repairs are
therefore diagnostic and mechanism-isolation tools, not accuracy
gains, and the search for FF's remaining gap should turn to
architecture, augmentation, and inference rules.

\begin{ack}
The author thanks the reviewers and the area chair for their careful and constructive
feedback, which led to the attenuation-compensated, second-barrier, and readout
experiments reported in this version. This work received no third-party funding. The
author declares no competing interests.
\end{ack}

\appendix
\let\origsection\section
\renewcommand{\section}{\FloatBarrier\origsection}
\let\origsubsection\subsection
\renewcommand{\subsection}{\FloatBarrier\origsubsection}
\newcommand{\appendixpart}[2]{%
  \FloatBarrier\clearpage
  \begin{center}{\Large\bfseries #1}\end{center}
  \addcontentsline{toc}{part}{#1}%
  \noindent #2\par\medskip}

\appendixpart{Appendix Part I: Evidence for the Central Claims}{Part~I
holds the evidence for the paper's central claims; the map below links
each claim to its main-text item and to its Part~I section. The
configuration, loss-formula, and protocol subsections of
App.~\ref{sec:full_config}, the early-exit analysis
(App.~\ref{app:early_exit}), and the extended limitations
(App.~\ref{app:long_limitations}) are reference material kept next to
the evidence they document. Part~II holds contextual and reference
material, some of which the main text also cites, and Part~III lists
the changes since submission and the review discussion
(App.~\ref{app:changes}).}

{\small
\noindent\begin{tabular}{@{}p{0.41\textwidth}p{0.21\textwidth}p{0.31\textwidth}@{}}
\toprule
Claim & Main text & Part~I evidence \\
\midrule
Cumulative goodness attenuates block $d$'s discrimination gradient (exact) & Thm.~\ref{thm:attenuation}, Fig.~\ref{fig:grad_attenuation} & App.~\ref{app:proofs} (Tab.~\ref{tab:assumption_coverage}); App.~\ref{app:gradshare} \\
The squared-hinge barrier reproduces it & Remark after Thm.~\ref{thm:attenuation} & App.~\ref{app:hinge} \\
Removal and gating restore deep-block separation & \S\ref{sec:ablation} & App.~\ref{app:gated_full} (Tab.~\ref{tab:gated}); App.~\ref{sec:multiseed_supplement} \\
Heuristic repairs move accuracy by under 1\,pp (paired CIs) & Tab.~\ref{tab:dissociation_merged}, \S\ref{sec:ablation} & App.~\ref{app:churn}; App.~\ref{sec:multiseed_supplement} \\
MGC: construction, Stage-1 cost, auxiliary control & \S\ref{sec:repairs}, \S\ref{sec:ablation}, Tab.~\ref{tab:dissociation_merged} & Apps.~\ref{app:attenuation_compensated}, \ref{app:mgc_results}, \ref{app:f1_control} \\
Accuracy invariance: exact under redistribution (deployed rule), empirical for the repairs and structured readouts & \S\ref{sec:intro}, Prop.~\ref{prop:redistribution} & Apps.~\ref{app:readouts}--\ref{app:churn} \\
Recipe controls & \S\ref{sec:ablation} & App.~\ref{app:components}; App.~\ref{app:hinge} \\
Applied updates are block-local & \S\ref{sec:method} & App.~\ref{app:locality} \\
\bottomrule
\end{tabular}\par}
\medskip

\section{Proofs of Theoretical Results}
\label{app:proofs}

\paragraph{Notation.}
For one example we write $m^{(d)}$ for the current-block margin and
$P^{(d-1)} = \sum_{j<d} m^{(j)}$ for the accumulated previous margin. The
cumulative block margin is $M^{(d)} = m^{(d)} + \gamma P^{(d-1)}$, with
$\gamma \in [0,1]$, and the barrier is
$\ell_\beta(u)=\log(1+e^{-\beta u})$.

\begin{proof}[Proof of Theorem~\ref{thm:attenuation}]
Let $M_\gamma^{(d)}=m^{(d)}+\gamma P^{(d-1)}$ where $P^{(d-1)}$ is
detached from $\theta_d$. By the chain rule,
\[
  \nabla_{\theta_d}\ell_\beta\!\bigl(M_\gamma^{(d)}\bigr)
  \;=\;
  \ell_\beta'\!\bigl(M_\gamma^{(d)}\bigr)\,\nabla_{\theta_d}m^{(d)},
  \qquad
  \nabla_{\theta_d}\ell_\beta\!\bigl(m^{(d)}\bigr)
  \;=\;
  \ell_\beta'\!\bigl(m^{(d)}\bigr)\,\nabla_{\theta_d}m^{(d)}.
\]
For finite $m^{(d)}$, $\ell_\beta'(m^{(d)})\ne 0$, so dividing the
two expressions gives the parameter-gradient identity
\eqref{eq:param_attenuation}. The softplus derivative
$\ell_\beta'(u)=-\beta\,\sigmoid(-\beta u)=-\beta/(1+e^{\beta u})$ then yields
\[
  R_\gamma(m,P)
  \;=\;
  \frac{|\ell_\beta'(m+\gamma P)|}{|\ell_\beta'(m)|}
  \;=\;
  \frac{1+e^{\beta m}}{1+e^{\beta(m+\gamma P)}}.
\]
Set $a=\beta m\ge 0$ and $b=\beta\gamma P\ge 0$. Writing
$(1+e^a)/(1+e^{a+b})=e^{-b}(1+e^a)/(e^{-b}+e^a)$ gives the lower
bound $e^{-b}$ since $e^{-b}\le 1$. For the upper bound,
$(1+e^a)/(1+e^{a+b})\le (1+e^a)/(e^{a+b})=e^{-b}+e^{-a-b}\le 2e^{-b}$
when $a\ge 0$ (using $e^{-a-b}\le e^{-b}$). Also the ratio is at
most $1$ for $b\ge 0$, so the upper bound is $\min\{1,2e^{-b}\}$.
Substituting $b=\beta\gamma P$ gives the stated sandwich. If $P<0$,
the denominator can be smaller than the numerator, so $R_\gamma$ can
exceed one; the attenuation claim is therefore restricted to the
already-separated regime.
\end{proof}

\begin{corollary}[Exact attenuation ratio and free-riding index]
\label{cor:attenuation_ratio}
Define the exact per-example attenuation ratio
\[
  R_d(m, P; \gamma)
  \;=\;
  \frac{|\partial_m \ell_\beta(m + \gamma P)|}{|\partial_m \ell_\beta(m)|}
  \;=\;
  \frac{1 + e^{\beta m}}{1 + e^{\beta(m + \gamma P)}}.
\]
For $m \ge 0$ and $P \ge 0$,
\[
  e^{-\beta \gamma P}
  \;\le\;
  R_d(m, P; \gamma)
  \;\le\;
  \min\bigl\{1,\; 2 e^{-\beta \gamma P}\bigr\}.
\]
\end{corollary}

\begin{proof}
The ratio and both bounds are derived in the proof of
Theorem~\ref{thm:attenuation} above, with $a=\beta m$ and $b=\beta\gamma P$.
\end{proof}

\paragraph{Free-riding index.}
Corollary~\ref{cor:attenuation_ratio} gives a measurable scalar per
example and per block. We define the \emph{free-riding index} at depth
$d$ using the clipped ratio,
\[
  \mathcal{F}_d
  \;=\;
  \mathbb{E}_{x \sim p_\text{data}}\!\bigl[\,
    1 - \min\{1,\, R_d(m^{(d)}(x), P^{(d-1)}(x); \gamma)\}
  \,\bigr]
  \;\in\; [0, 1).
\]
The clip ensures $\mathcal{F}_d \in [0,1)$ even on examples with
negative upstream margin, where cumulative scoring can amplify rather
than attenuate the current gradient ($R_d > 1$); restricting the
expectation to $\{P^{(d-1)} \ge 0\}$ reproduces the
already-separated regime where $R_d \le 1$ and the clip is inactive.
$\mathcal{F}_d \approx 0$ means block $d$'s local gradient is
essentially un-attenuated (no free-riding); $\mathcal{F}_d \to 1$
means the cumulative loss vanishes at block $d$ (severe free-riding).
This index is directly tied to the proposition: it does not require
choosing a separation threshold and can be reported alongside
$\text{sep}^\text{cur}_\text{nl}(d)$ as a complementary diagnostic.

\paragraph{Batch-mean margin diagnostics.}
Theorem~\ref{thm:attenuation} assumes $m^{(d)}\ge 0$ and
$P^{(d-1)}\ge 0$ pointwise: the already-separated regime in which
free-riding occurs. Table~\ref{tab:assumption_coverage} reports the
batch-mean current and previous margins on a CIFAR-10 L4/D128 test
batch (1024 examples) for the three $\gamma$ regimes used in
Figure~\ref{fig:grad_attenuation}, together with the batch-mean
attenuation ratio $R_d(\bar m^{(d)}, \bar P^{(d-1)}; \gamma)$
evaluated at $\beta{=}4$. Batch-mean positivity is consistent with
the already-separated regime but does not by itself establish
pointwise coverage. Pointwise, both conditions hold for
$95.6$--$97.1\%$ of the 1024 examples at every block of every
variant when recomputed on the EMA weights with the trainer's input
normalization ($96.1$--$97.0\%$ with the published raw-weight
snapshot of Figure~\ref{fig:grad_attenuation}), and the free-riding
index changes by at most $0.003$ between the two; the theorem's
regime therefore covers ${\sim}96\%$ of examples, not all of them.

\begin{table}[htbp]
  \caption{Batch-mean margin and attenuation diagnostics for
           Theorem~\ref{thm:attenuation} on a CIFAR-10 L4/D128 test
           batch ($N{=}1024$, $\beta{=}4$). The theorem assumes
           $m^{(d)}\!\ge\!0$ and $P^{(d-1)}\!\ge\!0$ pointwise; the
           rows below report \emph{batch means} (both positive across
           every variant and block) and the batch-mean attenuation
           ratio $R_d(\bar m^{(d)},\bar P^{(d-1)};\gamma)$.
           Batch-mean positivity is consistent with the
           already-separated regime but does not by itself establish
           pointwise coverage; pointwise, $m^{(d)}\!\ge\!0$ and
           $P^{(d-1)}\!\ge\!0$ hold for $95.6$--$97.1\%$ of the
           examples at every block (EMA weights; text above), and the
           per-example distribution is visualized in
           Figure~\ref{fig:grad_attenuation}(b). The batch-mean $R_d$ collapses to
           $\le 10^{-2}$ from block~1 onward under $\gamma{>}0$,
           matching Theorem~\ref{thm:attenuation}.}
  \label{tab:assumption_coverage}
  \centering
  \small
  \setlength{\tabcolsep}{4pt}
  \begin{tabular}{llrrrr}
    \toprule
    Variant & Block & $\bar m^{(d)}$ & $\bar P^{(d-1)}$ & $\gamma\bar P^{(d-1)}$ & $R_d(\bar m, \bar P; \gamma)$ \\
    \midrule
    $\gamma{=}0$ (history-free) & 0 & $+1.75$ & $+0.00$ & $0.00$ & $1.000$ \\
                                 & 1 & $+2.26$ & $+1.75$ & $0.00$ & $1.000$ \\
                                 & 2 & $+2.49$ & $+4.01$ & $0.00$ & $1.000$ \\
                                 & 3 & $+2.54$ & $+6.50$ & $0.00$ & $1.000$ \\
    \midrule
    $\gamma{=}0.7$ (CP-FAIR-style) & 0 & $+1.76$ & $+0.00$ & $0.00$ & $1.000$ \\
                                 & 1 & $+2.36$ & $+1.76$ & $1.23$ & $7.2{\times}10^{-3}$ \\
                                 & 2 & $+1.29$ & $+4.12$ & $2.89$ & $\le 10^{-4}$ \\
                                 & 3 & $+0.67$ & $+5.41$ & $3.79$ & $\le 10^{-6}$ \\
    \midrule
    $\gamma{=}1.0$ (LCFF)        & 0 & $+1.74$ & $+0.00$ & $0.00$ & $1.000$ \\
                                 & 1 & $+1.08$ & $+1.74$ & $1.74$ & $9.7{\times}10^{-4}$ \\
                                 & 2 & $+0.94$ & $+2.82$ & $2.82$ & $\le 10^{-4}$ \\
                                 & 3 & $+0.83$ & $+3.75$ & $3.75$ & $\le 10^{-6}$ \\
    \bottomrule
  \end{tabular}
\end{table}

\paragraph{Interpretation.}
Theorem~\ref{thm:attenuation} states the parameter-space attenuation
\emph{for the cumulative FF discrimination term only}: the
contribution of that term to $\nabla_{\theta_d}$ is rescaled exactly
by the scalar $R_\gamma$ as the margin-space gradient, provided
previous margins are detached. The total block update may also
contain auxiliary local losses (token reconstruction, SupCon,
depth-order, etc.) and depends on the current-block margin Jacobian
$\nabla_{\theta_d} m^{(d)}$, so the theorem does not imply that the
entire parameter-update vector is exponentially attenuated in
$P^{(d-1)}$; minibatch parameter-gradient norms can also cancel
across examples. When earlier blocks already separate an example,
the cumulative discrimination component of the per-example update at
block $d$ shrinks exponentially, which is the formal mechanism
behind the empirical observation that deep blocks of cumulative-FF
networks ``free-ride'' on their FF discrimination signal. Auxiliary
current-block losses and the depth-scaled current-block term can
deliberately restore a gradient floor
(Theorem~\ref{thm:gradient_floor}).

\begin{theorem}[The current-block term restores a depth-controlled gradient floor]
\label{thm:gradient_floor}
Consider only the FF discrimination subobjective at block $d$,
\[
  \mathcal{J}^{(d)}(m^{(d)})
  \;=\;
  \ell_\beta(M^{(d)})
  \;+\;
  \lambda_\text{curr}(d)\, w^{(d)} \, \ell_\beta(m^{(d)}),
\]
where $M^{(d)} = m^{(d)} + \gamma P^{(d-1)}$, $P^{(d-1)}$ is detached
from $\theta_d$, the residual weight $w^{(d)} > 0$ is computed from
the detached upstream margins $P^{(d-1)}$ of the batch alone
(Eq.~\ref{eq:residual_weights}; so no gradient flows
through $w^{(d)}$ during the block-$d$ backward), and
$\lambda_\text{curr}(d)$
is the depth-scaled coefficient from
equation~\eqref{eq:depth_scaled_main}. Auxiliary losses
($\mathcal{L}_\text{aspects}$, $\mathcal{L}_\text{depth}$,
$\mathcal{L}_\text{SupCon}$, $\mathcal{L}_\text{recon}$,
$\mathcal{L}_\text{bal}$, $\mathcal{L}_z$) are excluded from
$\mathcal{J}^{(d)}$ and are not covered by the bound below. Then, for
any $c \in \mathbb{R}$, if $m^{(d)} \le c$,
\[
  \bigl|\partial \mathcal{J}^{(d)} / \partial m^{(d)}\bigr|
  \;\ge\;
  \lambda_\text{curr}(d)\, w^{(d)} \beta\, \sigmoid(-\beta c).
\]
In particular, every unresolved example ($m^{(d)} \le 0$) receives
margin-gradient magnitude at least
$\lambda_\text{curr}(d)\, w^{(d)} \beta / 2$, which carries no
attenuation factor $\sigmoid(-\beta\gamma P^{(d-1)})$; its only
dependence on $P^{(d-1)}$ is through the residual weight $w^{(d)}$,
which by design is small for examples that earlier blocks already
separate.
\end{theorem}

\paragraph{Implementation note on the residual-weight lower bound.}
The bound depends linearly on the per-example residual weight
$w^{(d)}$, so a \emph{uniform} margin-gradient floor across the batch
would require a uniform lower bound on $w^{(d)}$. Our trainers do
\emph{not} clip the residual weights. Every bundled image-domain
trainer computes the mean-normalized sigmoid weight
$w_i^{(d)} = \sigmoid(-\beta P_i^{(d-1)}) \big/ \bigl(\tfrac{1}{B}\sum_k \sigmoid(-\beta P_k^{(d-1)}) + \epsilon\bigr)$
with $\epsilon = 10^{-6}$ (Eq.~\ref{eq:residual_weights}), separately
for the wrong-label and wrong-image streams. The theorem is therefore a
\emph{per-example} bound: example $i$ receives a margin-gradient floor
of $\lambda_\text{curr}(d)\, w_i^{(d)} \beta/2$, proportional to its
own residual weight, and there is no uniform batch-wide floor. Because
$w_i^{(d)}$ decreases monotonically in $P_i^{(d-1)}$
(Lemma~\ref{lem:weights}), the floor is largest on the
examples that earlier blocks have not yet separated and can approach
zero on examples they already separate by a wide margin. A uniform
floor holds only for a clipped variant of the weights, which we did
not run (remark after Eq.~\ref{eq:residual_weights}).

\paragraph{Margin-space vs.\ parameter-space gradient floor.}
Theorem~\ref{thm:gradient_floor} is a floor in
\emph{margin space}, $|\partial\mathcal{J}^{(d)}/\partial m^{(d)}|$.
The actual parameter-gradient is
$\nabla_{\theta_d}\mathcal{J}^{(d)}=
(\partial\mathcal{J}^{(d)}/\partial m^{(d)})\,\nabla_{\theta_d} m^{(d)}$,
so the floor translates to a parameter-gradient floor only under a
non-degeneracy condition on $\nabla_{\theta_d} m^{(d)}$. Concretely,
if $\|\nabla_{\theta_d} m^{(d)}\|\ge s_d>0$, then for unresolved
examples
$\|\nabla_{\theta_d}\mathcal{J}^{(d)}\|\ge\lambda_\text{curr}(d)\, w^{(d)}\,(\beta/2)\,s_d$,
ignoring cancellation from auxiliary terms.

\paragraph{Residual weights preserve batch loss scale.}
The residual weights used at $d \ge 1$ are
\begin{equation}\label{eq:residual_weights}
  w_i^{(d)} \;=\;
  \frac{\sigmoid\!\bigl(-\beta\, P_i^{(d-1)}\bigr)}
       {\tfrac{1}{B}\sum_{k=1}^{B} \sigmoid\!\bigl(-\beta\, P_k^{(d-1)}\bigr) + \epsilon},
  \qquad
  \epsilon = 10^{-6},
\end{equation}
computed without gradient from the detached $P^{(d-1)}$ and with no
clipping, so that $\tfrac{1}{B}\sum_i w_i^{(d)} = 1$ up to $\epsilon$
(Lemma~\ref{lem:weights} states the $\epsilon{=}0$ case). This keeps
the mean batch loss on the same scale as the unweighted variant while
monotonically emphasizing unresolved examples (those with small
$P_i^{(d-1)}$). After normalization, individual $w_i^{(d)}$ can be
arbitrarily small, so for the weights we use,
Theorem~\ref{thm:gradient_floor} gives only the per-example floor
$\lambda_\text{curr}(d)\, w_i^{(d)} \beta/2$.

\emph{Remark (clipped variant, not used in our runs).} Clipping the
mean-normalized weights to $[w_\text{min}, w_\text{max}]$ and then
re-normalizing bounds every final weight below by
$w_\text{min}/w_\text{max}$ (the clipped mean is at most $w_\text{max}$),
so Theorem~\ref{thm:gradient_floor} then gives the uniform floor
$\lambda_\text{curr}(d)\,(w_\text{min}/w_\text{max})\,(\beta/2)$ on every
unresolved example.

\begin{proof}[Proof of Theorem~\ref{thm:gradient_floor}]
Since $w^{(d)}$ is fixed with respect to $m^{(d)}$ inside the per-example
derivative, differentiating term by term with the derivative formula from
Theorem~\ref{thm:attenuation} gives
$\partial \mathcal{J}^{(d)}/\partial m^{(d)}
= -\beta \sigmoid(-\beta M^{(d)}) - \lambda_\text{curr}(d)\, w^{(d)} \beta \sigmoid(-\beta m^{(d)})$.
Both terms are non-positive, so
$|\partial \mathcal{J}^{(d)}/\partial m^{(d)}|
= \beta \sigmoid(-\beta M^{(d)}) + \lambda_\text{curr}(d)\, w^{(d)} \beta \sigmoid(-\beta m^{(d)})
\ge \lambda_\text{curr}(d)\, w^{(d)} \beta \sigmoid(-\beta m^{(d)})$,
dropping the first non-negative term. If $m^{(d)} \le c$, monotonicity of
$\sigmoid(-\beta u)$ in $u$ gives $\sigmoid(-\beta m^{(d)}) \ge \sigmoid(-\beta c)$,
which yields the bound. Setting $c=0$ gives
$|\partial \mathcal{J}^{(d)}/\partial m^{(d)}| \ge \lambda_\text{curr}(d)\, w^{(d)} \beta /2$
for every unresolved example with non-positive current-block margin. The bound
depends on $P^{(d-1)}$ only through $w^{(d)}$. It does not decay through the
cumulative term, but with the implemented mean-normalized weights it is small
for examples that already have a large upstream margin.
\end{proof}

\begin{lemma}[Residual weighting emphasizes unresolved examples while preserving scale]
\label{lem:weights}
For a batch of size $B$, define
\[
  w_i^{(d)}
  =
  \frac{\sigmoid(-\beta P_i^{(d-1)})}
       {\frac{1}{B}\sum_{k=1}^B \sigmoid(-\beta P_k^{(d-1)})}.
\]
Then $\frac{1}{B}\sum_{i=1}^B w_i^{(d)} = 1$. Moreover, if
$P_i^{(d-1)} < P_j^{(d-1)}$ then $w_i^{(d)} > w_j^{(d)}$.
\end{lemma}

\begin{proof}
Let $s_i = \sigmoid(-\beta P_i^{(d-1)})$. By construction,
$\frac{1}{B}\sum_i w_i^{(d)} = \bigl(\frac{1}{B}\sum_i s_i\bigr)/\bigl(\frac{1}{B}\sum_k s_k\bigr) = 1$,
so the average loss scale is preserved exactly. The map
$p \mapsto \sigmoid(-\beta p)$ is strictly decreasing for $\beta>0$, so
$P_i^{(d-1)} < P_j^{(d-1)}$ implies $s_i > s_j$, and dividing by the same
positive normalization constant preserves the ordering: examples with smaller
previous margins receive strictly larger weights.
\end{proof}

\begin{proposition}[Depth-order constraints imply monotone margin growth]
\label{prop:depthorder}
Suppose that for each depth $d \ge 1$ the cumulative positive and negative
scores satisfy
\[
  G_{+,d} - G_{+,d-1} \ge \delta_+,
  \qquad
  G_{-,d-1} - G_{-,d} \ge \delta_-,
\]
for some $\delta_+,\delta_->0$. Then
\[
  G_{+,d} \ge G_{+,0} + d\delta_+,
  \qquad
  G_{-,d} \le G_{-,0} - d\delta_-,
\]
and therefore
\[
  (G_{+,d}-G_{-,d}) - (G_{+,0}-G_{-,0})
  \ge d(\delta_+ + \delta_-).
\]
\end{proposition}

\begin{proof}
Summing the first inequality over depths $1,\dots,d$ telescopes to
$G_{+,d} - G_{+,0} = \sum_{j=1}^d (G_{+,j}-G_{+,j-1}) \ge d\delta_+$; likewise
$G_{-,0} - G_{-,d} = \sum_{j=1}^d (G_{-,j-1}-G_{-,j}) \ge d\delta_-$, i.e.\
$G_{-,d} \le G_{-,0} - d\delta_-$. Subtracting the two bounds yields
$G_{+,d}-G_{-,d} \ge G_{+,0}-G_{-,0}+d(\delta_+ + \delta_-)$, so any set of
scores satisfying the depth-order inequalities exhibits strictly increasing
separation with depth.
\end{proof}

\noindent\emph{Scope.} Proposition~\ref{prop:depthorder} is about an
\emph{idealized} form of the constraints: hard inequalities between
consecutive elements of a single cumulative sequence $G_{\pm,d}$. The
implemented depth-order loss (App.~\ref{app:loss_formulas}) differs in
two ways. It is a soft softplus penalty, so the inequalities need not
hold after training. It also compares block $d$'s loss score
$G^{(d)}_{\gamma,\pm}$ with block $d{-}1$'s \emph{current-block}
goodness $g^{(d-1)}_{\pm}$, not with $G^{(d-1)}_{\gamma,\pm}$, so the
telescoping step above does not apply to it directly. The proposition
explains the intent of the term and is not a guarantee for the
trained models.

\paragraph{Proof of Proposition~\ref{prop:prediction_stability}.}
The proposition is stated in
Section~\ref{sec:method}. We give a more explicit two-part form below
and the proof of both parts. \textbf{(a) Pointwise.} If
$E(x)\le\varepsilon$, the two classifiers can disagree on $x$ only if
$\Delta_A(x)\le 2\varepsilon$. \textbf{(b) Two-event union bound.}
For any $t\ge 0$,
\(
  \Pr[\hat y_A(x)\ne \hat y_B(x)]
  \le \Pr[\Delta_A(x)\le 2t] + \Pr[E(x)>t]
\),
which is the form quoted in the main text.

\begin{proof}
\emph{(a)} Suppose $\hat y_A(x) = y^\star$. By the bound on summed-score
perturbations, every class $y$ satisfies $|S^B_y - S^A_y| \le \varepsilon$. Then
\[
  S^B_{y^\star} \;\ge\; S^A_{y^\star} - \varepsilon
  \;\ge\; S^A_y + \Delta_A(x) - \varepsilon
  \;\ge\; S^B_y + \Delta_A(x) - 2\varepsilon
\]
for every $y \ne y^\star$. If $\Delta_A(x) > 2\varepsilon$, then
$S^B_{y^\star} > S^B_y$ for all $y \ne y^\star$, i.e.\
$\hat y_B(x) = y^\star = \hat y_A(x)$.
\emph{(b)} Partition the test set by $E(x) \le t$ and $E(x) > t$. By (a),
disagreements within the first set require $\Delta_A(x) \le 2t$;
the second set has probability mass $\Pr[E(x) > t]$, a free upper bound
on its disagreement contribution. Summing yields the union bound.
\end{proof}

\paragraph{Why the union form matters.}
Part (a) alone with $\varepsilon = \max_x E(x)$ is typically too loose to
be empirically useful. The union bound (b) lets one trade a small-$t$
threshold (tight margin band but possibly large $\Pr[E(x)>t]$) against a
larger-$t$ threshold (loose margin band but $\Pr[E(x)>t]\!\approx\!0$).
The optimal bound is read off by plotting the empirical CDF
$\Pr[\Delta_A(x) \le t]$ together with the empirical
$\Pr[E(x) > t]$. Empirically, the paired-bootstrap analysis of
Table~\ref{tab:dissociation_merged} and App.~\ref{app:churn}
provides the first-order test for the CIFAR-100 dissociation:
disagreement rate $\sim$13\%, but the disagreements are nearly net-zero
in correctness, so the part-(b) bound is tight at small $t$ where
$\Pr[\Delta_A \le 2t]$ accounts for almost all disagreement.

\paragraph{Interpretation.}
Proposition~\ref{prop:redistribution} says that any redistribution of
per-block contributions with $\sum_d r_d(x,y){=}0$ leaves the
summed-score prediction unchanged. Proposition~\ref{prop:prediction_stability}
quantifies the more general case where the redistribution is not
exactly zero-sum but is uniformly bounded by $\varepsilon$ in
summed-score space: the two classifiers can disagree only on
examples whose summed-margin distribution puts them within
$2\varepsilon$ of the decision boundary. So a perturbation that is
small in summed-score space, however large its per-block
reallocation, can only flip predictions on a thin band of marginal
examples; in particular, if $\Pr[\Delta_A \le 2\varepsilon]$ is small
(most examples are well-classified), then
$\Pr[\hat y_A \ne \hat y_B]$ is small. This arithmetic covers exactly or
nearly redistributing changes only. Retrained variants are not small
perturbations: the models of the three CIFAR-100 history rules disagree on
${\sim}13\%$ of test predictions (S2~TTA) and MGC vs.\ cumulative on
$22.7$--$26.2\%$ (App.~\ref{app:churn}), so their sub-1\,pp
Stage-2 accuracy changes
(Table~\ref{tab:dissociation_merged})
are an empirical finding, not a consequence of the propositions.

\subsection{Attenuation-compensated local loss (MGC): construction and result}
\label{app:attenuation_compensated}

\noindent\textbf{Status: run.} This objective was proposed as a
falsification target and has since been run at two scales, 3 seeds
each: CIFAR-100 dense L4/D256 (362+20 epochs, SAM off, configured
identically to the cumulative arm except for the MGC switches) and
CIFAR-10 L4/D128 (180+10 epochs, SAM on). The protocol, the
compensation telemetry, and the outcome are given at the end of this
subsection; per-seed numbers are in Table~\ref{tab:mgc_full}
(App.~\ref{app:mgc_results}), the paired four-protocol deltas in
Table~\ref{tab:dissociation_merged}, and a control that separates the
block-0 change, the auxiliary removal, and the compensation at
L4/D128 in App.~\ref{app:f1_control}.

The depth-scaled coefficient $\lambda_\text{curr}(d) = \lambda_0(1+\rho\,d/(L-1))$
in Section~\ref{sec:method} is heuristic and, as our experiments show
(Table~\ref{tab:freeriding}), provides only partial mitigation
relative to $\gamma{=}0$. We give here a \emph{theory-driven}
alternative that approximately compensates
Theorem~\ref{thm:attenuation}'s exponential envelope in the gradient.
The control of App.~\ref{app:f1_control} shows that at L4/D128 this
compensation is not the main source of the deep-block separation rise:
it adds about $8\%$ of it, and the rest comes from dropping the
auxiliary that the construction replaces.

\paragraph{Setup.}
Write $M_i^{(d)} = m_i^{(d)}+\gamma P_i^{(d-1)}$ and
$s(u)=\sigmoid(-\beta u)$, so that
$|\ell_\beta'(u)|=\beta\,s(u)$ and $s$ is strictly decreasing in $u$.
By Theorem~\ref{thm:attenuation},
\[
  \frac{|\partial_m \ell_\beta(M_i^{(d)})|}{|\partial_m \ell_\beta(m_i^{(d)})|}
  \;=\;
  \frac{s(M_i^{(d)})}{s(m_i^{(d)})}
  \;\eqcolon\;
  R_i^{(d)} \;\in\;(0,1].
\]
The bound $R_i^{(d)} \le 1$ requires only $M_i^{(d)} \ge m_i^{(d)}$,
i.e.\ $\gamma P_i^{(d-1)} \ge 0$. Since $P_i^{(d-1)}$ is a cumulative
\emph{margin} rather than a goodness magnitude, this condition can
fail at the per-example level even on positive data when an example
remains unresolved at some earlier block; in trained cumulative FF
models it holds on average and for ${\sim}96\%$ of test examples
(Table~\ref{tab:assumption_coverage}), but not for all of them. To make the
correction safe outside the assumed regime we use the clipped
attenuation-only ratio in~\eqref{eq:att_compensated} below, where the
$[\,\cdot\,]_+$ truncation prevents over-compensation when the
assumption fails. The strict positivity of $R_i^{(d)}$ follows from
$s>0$ everywhere. Under the assumed regime, no sign assumption on
$m_i^{(d)}$ itself is required---the bound holds for unresolved
current-block examples ($m_i^{(d)}<0$) as well, and the cumulative
loss delivers only an $R_i^{(d)}$ fraction of the local gradient
magnitude that a $\gamma{=}0$ trainer would see at the same example.

\paragraph{Missing-gradient compensated loss.}
The correction adds, next to the cumulative term, a local term sized
to fill the missing gradient. We present this construction
\emph{without} the residual weighting $w_i^{(d)}$ of
Eq.~\eqref{eq:residual_weights}: keeping $w_i^{(d)}$ inside the local
term mixes two distinct mechanisms (residual emphasis on unresolved
examples, and per-example envelope cancellation) and yields a
gradient magnitude that depends on $w_i^{(d)}$ and equals the
$\gamma{=}0$ target only when $w_i^{(d)}{=}1$. Define, with
stop-gradient through $R_i^{(d)}$,
\begin{equation}\label{eq:att_compensated}
  \lambda_i^{(d)}
  \;=\;
  \bigl[\,c_d - R_i^{(d)}\,\bigr]_+
  \;=\;
  \left[\,c_d - \frac{s(M_i^{(d)})}{s(m_i^{(d)})+\epsilon}\,\right]_+,
  \qquad c_d \ge 1,
\end{equation}
and the missing-gradient compensated loss
\begin{equation}
  \widetilde{\mathcal{L}}^{(d)}_\text{MGC}
  \;=\;
  \frac{1}{B}\sum_{i=1}^B
  \Bigl[\,
    \ell_\beta\!\bigl(M_i^{(d)}\bigr)
    \;+\;
    \lambda_i^{(d)}\, \ell_\beta\!\bigl(m_i^{(d)}\bigr)
  \Bigr].
\end{equation}
The depth-control coefficient is fixed: $c_d{=}1$ targets the
$\gamma{=}0$ local-gradient magnitude (exactly under the assumptions of
Proposition~\ref{prop:mgc_recovers_local}), and
$c_d{=}1+\rho_c\,d/(L{-}1) \ge 1$ (a slope $\rho_c$ separate from the
$\rho$ of $\lambda_\text{curr}$) would target a depth-boosted multiple;
only $c_d{=}1$ was run.

\paragraph{Replacement semantics in the implemented arm.}
In the runs, $\widetilde{\mathcal{L}}^{(d)}_\text{MGC}$'s local term
\emph{replaces} the baseline's depth-scaled current-block term
$\lambda_\text{curr}(d)\,\mathcal{L}^{(d)}_\text{curr}$
(Eq.~\ref{eq:depth_scaled_main}); every other term of
Eq.~\ref{eq:total_loss} is unchanged, and the local term is averaged
over the wrong-label and wrong-image streams with the same weight
$\eta$ as $\mathcal{L}_\text{block}$. At block~0 there is no upstream
margin, so $R_i^{(0)}{\equiv}1$ and the compensator is zero
($\lambda_i^{(0)}{=}0$ for $c_0{=}1$; the implementation skips block~0).
Block~0 therefore loses the $0.25$ auxiliary it has under the
baseline, while blocks 1--3 trade their depth-scaled auxiliary for the
compensator.

\begin{proposition}[Local-gradient recovery, conditional on the assumed regime]\label{prop:mgc_recovers_local}
Assume $\gamma P_i^{(d-1)} \ge 0$, $\epsilon{=}0$, and $c_d \ge 1$.
Then $R_i^{(d)} \in (0,1] \le c_d$, the truncation
$[\,c_d - R_i^{(d)}\,]_+$ is inactive (i.e., $\lambda_i^{(d)} = c_d - R_i^{(d)} \ge 0$),
and with stop-gradient through $\lambda_i^{(d)}$, the per-example
gradient on the batch-mean loss is
\[
  \left|\partial_{m_i^{(d)}} \widetilde{\mathcal{L}}^{(d)}_\text{MGC}\right|
  \;=\;
  \frac{c_d}{B}\,\beta\, s(m_i^{(d)}).
\]
The factor $1/B$ arises from the explicit $1/B$ in the definition of
$\widetilde{\mathcal{L}}^{(d)}_\text{MGC}$; the per-example
unreduced gradient is $c_d\,\beta\, s(m_i^{(d)})$.
In particular, $c_d{=}1$ exactly matches the $\gamma{=}0$
local-gradient magnitude $\beta\,s(m_i^{(d)})$, while $c_d > 1$
targets a depth-boosted multiple. \emph{Exact recovery holds only when
$\gamma P_i^{(d-1)} \ge 0$ and $\epsilon{=}0$.} Outside this regime
the construction is a no-undercompensation heuristic: the clipped
$[\,c_d - R_i^{(d)}\,]_+$ ensures
$\lambda_i^{(d)} \ge 0$ always, so the local term never subtracts
gradient, but the per-example total can exceed
$c_d\,\beta\,s(m_i^{(d)})$ when the assumption fails.
\end{proposition}

\begin{proof}
Consider the $i$-th summand
$\ell_\beta(M_i^{(d)}) + \lambda_i^{(d)}\,\ell_\beta(m_i^{(d)})$
of $\widetilde{\mathcal{L}}^{(d)}_\text{MGC}$; the derivative of the
batch mean is the derivative of this summand multiplied by $1/B$,
because $m_i^{(d)}$ enters no other summand. We compute the
unreduced per-summand derivative below and write
$\widetilde{\mathcal{L}}_i$ for the $i$-th summand.
Differentiating term by term, using $\partial M_i^{(d)}/\partial m_i^{(d)} = 1$
and $s(M_i^{(d)}) = R_i^{(d)}\,s(m_i^{(d)})$:
$\partial_m \ell_\beta(M_i^{(d)}) = -\beta\,s(M_i^{(d)}) = -\beta\,R_i^{(d)} s(m_i^{(d)})$,
and $\partial_m \ell_\beta(m_i^{(d)})=-\beta\,s(m_i^{(d)})$. Both terms
are non-positive. The hypothesis $\gamma P_i^{(d-1)} \ge 0$
implies $M_i^{(d)} \ge m_i^{(d)}$ and (since $s$ is decreasing)
$R_i^{(d)} \le 1 \le c_d$, so the $[\,\cdot\,]_+$ truncation is
inactive and $\lambda_i^{(d)} = c_d - R_i^{(d)}$ exactly. With
stop-gradient through $\lambda_i^{(d)}$, $\lambda_i^{(d)}$ contributes
no derivative; taking absolute values,
\begin{align*}
  \left|\partial_{m_i^{(d)}} \widetilde{\mathcal{L}}_i\right|
  &= \beta\,R_i^{(d)}\,s(m_i^{(d)}) + \beta\,(c_d - R_i^{(d)})\,s(m_i^{(d)})
  \;=\;
  c_d\,\beta\,s(m_i^{(d)}),
\end{align*}
and multiplying by $1/B$ gives the batch-mean statement.
\end{proof}

\paragraph{Effect of $\epsilon > 0$.}
A small $\epsilon$ in the denominator of $R_i^{(d)}$
(Eq.~\eqref{eq:att_compensated}) is included for numerical
stability when $s(m_i^{(d)})$ is near zero (very confidently
resolved examples). It perturbs the recovery:
$\widetilde R_i^{(d)} \,{=}\, s(M)/(s(m)+\epsilon) \le R_i^{(d)}$
under-estimates the true ratio, so $\lambda_i^{(d)}$ over-shoots
$c_d-R_i^{(d)}$ and the per-example gradient becomes
$c_d\,\beta\,s(m) + \beta\,(R_i^{(d)} - \widetilde R_i^{(d)})\,s(m)$,
with the slack vanishing as $\epsilon\to 0$. In practice
$\epsilon \!\le\! 10^{-6}$ keeps the slack below per-batch noise.

\paragraph{Remark (failure modes that motivated this form).}
Two superficially natural variants do \emph{not} achieve local-gradient
recovery and motivate the compensating form
\eqref{eq:att_compensated} above. (i) \emph{Multiplying}
$\ell_\beta(m)$ by the ratio $s(M)/s(m)$ with stop-gradient yields
gradient magnitude $\beta\,s(M)$ — the cumulative gradient, not the
local one. (ii) Including the residual weight $w_i^{(d)}$ inside the
local term gives $\beta\,w_i^{(d)} s(m) + \beta\,(1-w_i^{(d)}) s(M)$,
equal to $\beta\,s(m)$ only when $w_i^{(d)}{=}1$. The
compensating formulation \eqref{eq:att_compensated} avoids both
pitfalls by injecting the missing
$\beta\,(c_d\,s(m)-s(M))$ per example (exactly when $\epsilon{=}0$ and
$\gamma P\ge 0$), with $w_i^{(d)}$ omitted.

\paragraph{Protocol.}
All runs use $c_d{=}1$ at every block, $\epsilon{=}10^{-6}$,
stop-gradient through $\lambda_i^{(d)}$, the positive-part clamp of
Eq.~\eqref{eq:att_compensated}, and the cumulative term with
$\gamma{=}0.7$. On CIFAR-100 the configuration differs from the
cumulative arm of Table~\ref{tab:dissociation_merged} only in the MGC
switches (dense L4/D256, no MoE, no SAM, 362+20 epochs, seeds
42/123/456, compared on matched seeds). On CIFAR-10 L4/D128 all three
seeds were trained with SAM, whereas the seed-42 gated-ablation anchors
($\gamma{=}0.7$, $\gamma{=}0$, $\kappa{=}0$, LCFF) were trained without
SAM, so every L4/D128 MGC-versus-anchor comparison crosses the SAM
setting; the control of App.~\ref{app:f1_control} adds a SAM-on
cumulative arm (3 seeds), which serves as the L4/D128 baseline below.
The runs are therefore a faithful implementation of
Eq.~\eqref{eq:att_compensated} that approximately compensates the
envelope, not exact gradient recovery: $\epsilon{>}0$ slightly
over-compensates (above), the clamp is active wherever
$\gamma P_i^{(d-1)}{<}0$, and the compensation is not audited per
example.

\paragraph{Compensation telemetry.}
The trainer logs only the per-block batch mean and 95th percentile of
$\lambda_i^{(d)}$ (averaged over the two negative streams). On
CIFAR-100 the final-epoch mean $\lambda$ at blocks 1--3 is
$0.84$/$0.95$/$0.97$ (all three seeds within $0.01$), and the
seed-42 whole-training mean is $0.64$/$0.87$/$0.92$; at L4/D128 the
final-epoch means are $0.90$/$0.96$/$0.97$. A mean $\lambda$ close to
one means that the cumulative term delivers only a small fraction of
the local gradient and the local term restores the rest; the
final-epoch 95th percentile is $1.00$ at blocks 1--3 in every run.
CIFAR-100 seed~456 was resumed from a checkpoint: its layer and epoch
logs start at epoch 192, so its whole-training $\lambda$ mean covers
epochs 192--362 only; its final evaluation is complete.

\paragraph{Outcome.}
The MGC objective raises deepest-block
$\text{sep}^\text{cur}_\text{nl}$ to $5.61 \pm 0.04$ vs.\
$0.96 \pm 0.01$ for the cumulative arm on CIFAR-100 (and to
$5.25 \pm 0.05$ vs.\ $0.75 \pm 0.01$ for the SAM-matched cumulative
arm at L4/D128, where most of the rise comes from removing the
auxiliary and about $8\%$ from the compensation; next paragraph), with no
material accuracy gain under any protocol.
On CIFAR-100 the Stage-1 accuracy falls in every seed (MGC minus
cumulative $-1.05$\,pp, paired 95\% CI $[-1.46, -0.62]$, single crop;
$-0.83$\,pp, $[-1.23, -0.42]$, with TTA), and the trained Stage-2 probe
recovers the cost ($-0.27$\,pp, $[-0.68, +0.14]$; with TTA
$+0.13$\,pp, $[-0.27, +0.53]$; per-seed Stage-2 deltas are
sign-mixed). At L4/D128, every paired CI of MGC minus the SAM-matched
cumulative arm lies inside $\pm 1$\,pp (Table~\ref{tab:f1_control}; the
seed-42 SAM-off anchors in Table~\ref{tab:mgc_full} cross the SAM setting). MGC and cumulative disagree on
$22.7$--$26.2\%$ of CIFAR-100 test predictions
(App.~\ref{app:churn}): the objective changes what the network learns.
It is a new training objective, not a post-hoc score redistribution,
so Propositions~\ref{prop:redistribution}
and~\ref{prop:prediction_stability} do not apply to it; the
invariance of the structured readouts under MGC is an empirical
finding. Read literally, criterion~(b) of the falsification test
(App.~\ref{app:long_limitations}) is met at Stage~1, in the harmful
direction: neither Stage-1 CI lies inside $\pm 1$\,pp. The claim under
test, however, is that free-riding limits accuracy, so the refuting
outcome would be a material accuracy \emph{gain} from restored block
health; none occurs. The replacement also removes block~0's auxiliary;
at L4/D128 the control finds no compensation gain masked by that
change ($n{=}3$ per arm; next paragraph).

\paragraph{Block-0 confound and auxiliary control.}
Because the MGC arm replaces the depth-scaled term, it differs from
the cumulative arm in three ways at once: blocks 1--3 lose their
depth-scaled auxiliary and gain the compensator, and block~0 loses
its auxiliary entirely. Block-0 separation is unchanged on CIFAR-100
($0.54 \pm 0.01$ in both arms). At L4/D128 the control of
App.~\ref{app:f1_control} (Table~\ref{tab:f1_control}) resolves the
confound for separation: keeping block~0's auxiliary leaves the
deepest-block separation unchanged and adds no accuracy, and removing the
auxiliary without any compensator already gives $92\%$ of MGC's
separation gain (the compensator $8\%$; registered prediction P2 failed),
so beyond the compensator's small share the rise is not evidence that
restoring the missing gradient restores block health. On CIFAR-100
neither change was controlled, so we do not attribute the Stage-1 cost
there to the compensation alone, and its separation gain has not been
decomposed.

\section{Compensation, Second Barrier, Gradient Shares, and Readouts}
\label{app:new_evidence}

This section collects the experiments added after the review. Every table is generated from the released result files; all
$\pm$ values are sample standard deviations over seeds (in
App.~\ref{app:early_exit} and Table~\ref{tab:early_exit}, over runs), and all paired
confidence intervals are conditional on the trained checkpoints.

\subsection{Missing-gradient compensation: per-seed results}
\label{app:mgc_results}

Table~\ref{tab:mgc_full} lists every MGC run; the construction,
protocol, telemetry, and outcome are in
App.~\ref{app:attenuation_compensated}. On CIFAR-100 the Stage-1 cost
is negative in each seed under both S1 protocols ($-0.65$/$-1.19$/$-1.31$\,pp
single crop; $-0.39$/$-0.94$/$-1.16$\,pp with TTA), whereas the
Stage-2 deltas are sign-mixed ($+0.62$/$-0.66$/$-0.77$\,pp; with TTA
$+0.62$/$-0.10$/$-0.13$\,pp). Separation rises at every block from
block~1 on and is unchanged at block~0, where MGC has no compensator.
The L4/D128 runs (SAM on) are compared with a SAM-matched cumulative
arm in App.~\ref{app:f1_control}; the seed-42 SAM-off anchors in the
caption cross the SAM setting.

\begin{table}[!htb]
  \caption{\textbf{Missing-gradient compensation (MGC, $c_d{=}1$) per seed.} Test top-1 (\%) of
           the validation-selected EMA checkpoints; in parentheses, the paired difference to the
           cumulative ($\gamma{=}0.7$) arm on the same seed. sep: $\text{sep}^\text{cur}_\text{nl}$
           per block at the last Stage-1 epoch; $\bar\lambda$: final-epoch mean compensation weight
           (block~0 has no compensator, $\bar\lambda{\equiv}0$). The MGC arm replaces the baseline's
           depth-scaled current-block term, so block~0 also loses its auxiliary. Seed-42 whole-training
           mean $\bar\lambda$ (CIFAR-100): 0.64/0.87/0.92.
           Faithful implementation ($\epsilon{=}10^{-6}$, positive-part clamp; only the mean and p95
           of $\lambda$ are logged). $^\dagger$Resumed run: layer log starts at epoch
           192 (final evaluation complete).
           $^\ddagger$Local run; seeds 123/456 ran on Colab. The seed-42 L4/D128 SAM-off anchors
           ($\gamma{=}0.7$/$\gamma{=}0$/$\kappa{=}0$) reach S1 86.18/85.75/85.86, S1-TTA 87.25/86.87/86.81, and
           S2-TTA 87.42/87.21/87.27; these single-seed comparisons cross the SAM setting (SAM-matched
           cumulative arm: Table~\ref{tab:f1_control}).}
  \label{tab:mgc_full}
  \centering
  \scriptsize
  \setlength{\tabcolsep}{2pt}
  \resizebox{\textwidth}{!}{%
  \begin{tabular}{llccccccccccc}
    \toprule
    & & \multicolumn{4}{c}{Test top-1 (\%)} & \multicolumn{4}{c}{$\text{sep}^\text{cur}_\text{nl}$ per block} & \multicolumn{3}{c}{$\bar\lambda$ per block} \\
    \cmidrule(lr){3-6} \cmidrule(lr){7-10} \cmidrule(lr){11-13}
    Setting & Seed & S1 & S1-TTA & S2 & S2-TTA & L0 & L1 & L2 & L3 & L1 & L2 & L3 \\
    \midrule
    \multirow{5}{*}{\shortstack[l]{CIFAR-100\\L4/D256\\(SAM off)}} & 42 & 66.05 ($-0.65$) & 66.79 ($-0.39$) & 68.71 ($+0.62$) & 69.19 ($+0.62$) & 0.55 & 3.45 & 5.67 & 5.65 & 0.84 & 0.95 & 0.97 \\
     & 123 & 65.36 ($-1.19$) & 66.30 ($-0.94$) & 67.70 ($-0.66$) & 68.78 ($-0.10$) & 0.53 & 3.33 & 5.57 & 5.57 & 0.83 & 0.95 & 0.97 \\
     & 456$^\dagger$ & 65.79 ($-1.31$) & 66.66 ($-1.16$) & 68.30 ($-0.77$) & 69.29 ($-0.13$) & 0.53 & 3.38 & 5.59 & 5.61 & 0.84 & 0.95 & 0.97 \\
     & mean & 65.73$\pm$0.35 ($-1.05$) & 66.58$\pm$0.25 ($-0.83$) & 68.24$\pm$0.51 ($-0.27$) & 69.09$\pm$0.27 ($+0.13$) & 0.54$\pm$0.01 & 3.39$\pm$0.06 & 5.61$\pm$0.05 & 5.61$\pm$0.04 & 0.84 & 0.95 & 0.97 \\
     & cum. & 66.78$\pm$0.28 & 67.41$\pm$0.35 & 68.51$\pm$0.51 & 68.96$\pm$0.43 & 0.54$\pm$0.01 & 2.13$\pm$0.03 & 1.94$\pm$0.02 & 0.96$\pm$0.01 & --- & --- & --- \\
    \midrule
    \multirow{4}{*}{\shortstack[l]{CIFAR-10\\L4/D128\\(SAM on)}} & 42$^\ddagger$ & 85.23 & 86.64 & 86.00 & 87.05 & 1.06 & 4.31 & 5.95 & 5.24 & 0.90 & 0.96 & 0.97 \\
     & 123 & 85.88 & 87.09 & 86.48 & 87.60 & 1.11 & 4.54 & 6.17 & 5.31 & 0.91 & 0.97 & 0.97 \\
     & 456 & 85.25 & 86.89 & 85.73 & 87.17 & 1.07 & 4.38 & 5.98 & 5.21 & 0.90 & 0.96 & 0.97 \\
     & mean & 85.45$\pm$0.37 & 86.87$\pm$0.23 & 86.07$\pm$0.38 & 87.27$\pm$0.29 & 1.08$\pm$0.02 & 4.41$\pm$0.12 & 6.03$\pm$0.12 & 5.25$\pm$0.05 & 0.90 & 0.96 & 0.97 \\
    \bottomrule
  \end{tabular}%
  }
\end{table}

\FloatBarrier
\subsubsection{Control: block-0 auxiliary and auxiliary removal}
\label{app:f1_control}

MGC replaces the depth-scaled auxiliary of the
cumulative objective (\S\ref{sec:repairs}; residual-weighted at blocks
1--3), so it
differs from cumulative training in three ways at once: blocks 1--3
lose the auxiliary and gain the compensator, and block~0 loses its
auxiliary with nothing in its place
(App.~\ref{app:attenuation_compensated}). To separate these changes we
trained three further arms in the CIFAR-10 L4/D128 MGC recipe (SAM on,
seeds 42/123/456; Table~\ref{tab:f1_control}): (a)~the cumulative
objective with its auxiliary, a SAM-matched cumulative baseline;
(b)~the cumulative objective with the auxiliary removed at every block
and no compensator; and (c)~MGC with block~0's auxiliary kept. The arms
use a fork of the MGC trainer with two added switches: with both off, a
training step is bitwise identical to the original trainer's, and CPU
tests confirm that each switch changes only the intended terms. Four
predictions were written down before any run.

\begin{table}[!htb]
  \caption{\textbf{MGC control at CIFAR-10 L4/D128} (constant $\gamma{=}0.7$, SAM on, 180+10 epochs, seeds
           42/123/456 per arm; mean $\pm$ sample SD, $n{=}3$). Aux.: the depth-scaled current-block auxiliary
           ($0.25$ at block~0, $0.50$/$0.75$/$1.00$ at blocks 1--3; residual-weighted at blocks 1--3 only). (a) keeps it at
           every block; (b) removes it at every block, with no compensator; MGC (Tab.~\ref{tab:mgc_full})
           replaces it by the compensator at blocks 1--3 and has none at block~0; (c) is MGC with block~0's
           auxiliary kept. sep: $\text{sep}^\text{cur}_\text{nl}$ at block~3, last Stage-1 epoch. $\Delta$
           rows: for sep, mean $\pm$ SD of the seed-paired differences; for accuracy, the paired difference
           pooled over the three seeds ($n{=}30{,}000$ test predictions) with a 5000-resample bootstrap 95\%
           CI that resamples test images only, so it omits seed-to-seed variation (SD of the per-seed S2-TTA
           differences of (c) and (b) vs.\ MGC: 0.26 and 0.26\,pp). Every CI lies inside
           $\pm1$\,pp. MGC's deepest-block gain over (a), $4.50$, splits into $4.13$ from removing the auxiliary ((b)$-$(a), 92\%) and $0.37$ from the compensator (MGC$-$(b), 8\%). GPU per seed (G, NVIDIA GB10 on a DGX Spark; A, Colab A100): (a) A/A/A; (b) G/G/G; (c) G/G/A; MGC G/A/A; the
           arms thus mix hardware and software stacks, and (b)$-$(a) is fully confounded with them. No run
           was resumed. Registered before the runs: P1, (c) in $[4.72, 5.78]$: passes; P2, (b) below
           $4.72$ with every (b) seed below the lowest MGC seed ($5.21$): \textbf{fails} ((b) is
           $4.88$; the second clause holds); P3, every arm's S2-TTA within 1\,pp of MGC's $87.27$:
           passes; P4, (a) below $4.72$: passes.}
  \label{tab:f1_control}
  \centering
  \footnotesize
  \setlength{\tabcolsep}{3pt}
  \resizebox{\textwidth}{!}{%
  \begin{tabular}{lccccc}
    \toprule
    Arm & sep (L3) & S1 & S1-TTA & S2 & S2-TTA \\
    \midrule
    (a) Cumulative + aux. & $0.75\!\pm\!0.01$ & $85.74\!\pm\!0.37$ & $86.75\!\pm\!0.34$ & $85.98\!\pm\!0.52$ & $87.04\!\pm\!0.51$ \\
    (b) Cumulative, no aux. & $4.88\!\pm\!0.05$ & $85.31\!\pm\!0.10$ & $86.47\!\pm\!0.04$ & $85.88\!\pm\!0.34$ & $86.90\!\pm\!0.22$ \\
    \quad $\Delta$ vs.\ (a) & $+4.13\!\pm\!0.05$ & $-0.43$\,[$-0.74$,\,$-0.12$] & $-0.28$\,[$-0.54$,\,$-0.01$] & $-0.10$\,[$-0.40$,\,$+0.21$] & $-0.14$\,[$-0.40$,\,$+0.11$] \\
    \quad $\Delta$ vs.\ MGC & $-0.37\!\pm\!0.01$ & $-0.14$\,[$-0.46$,\,$+0.16$] & $-0.40$\,[$-0.68$,\,$-0.13$] & $-0.19$\,[$-0.50$,\,$+0.11$] & $-0.38$\,[$-0.64$,\,$-0.11$] \\
    (c) MGC + block-0 aux. & $5.28\!\pm\!0.03$ & $85.32\!\pm\!0.19$ & $86.53\!\pm\!0.19$ & $85.66\!\pm\!0.25$ & $86.93\!\pm\!0.09$ \\
    \quad $\Delta$ vs.\ (a) & $+4.54\!\pm\!0.03$ & $-0.42$\,[$-0.74$,\,$-0.12$] & $-0.22$\,[$-0.49$,\,$+0.06$] & $-0.32$\,[$-0.62$,\,$-0.02$] & $-0.11$\,[$-0.37$,\,$+0.16$] \\
    \quad $\Delta$ vs.\ MGC & $+0.03\!\pm\!0.02$ & $-0.14$\,[$-0.44$,\,$+0.17$] & $-0.35$\,[$-0.62$,\,$-0.07$] & $-0.41$\,[$-0.72$,\,$-0.11$] & $-0.35$\,[$-0.62$,\,$-0.08$] \\
    MGC & $5.25\!\pm\!0.05$ & $85.45\!\pm\!0.37$ & $86.87\!\pm\!0.23$ & $86.07\!\pm\!0.38$ & $87.27\!\pm\!0.29$ \\
    \quad $\Delta$ vs.\ (a) & $+4.50\!\pm\!0.04$ & $-0.29$\,[$-0.59$,\,$+0.02$] & $+0.13$\,[$-0.14$,\,$+0.40$] & $+0.09$\,[$-0.21$,\,$+0.40$] & $+0.24$\,[$-0.02$,\,$+0.50$] \\
    \bottomrule
  \end{tabular}%
  }
\end{table}

\emph{Block~0 is not the confound.} Keeping block~0's auxiliary leaves
deepest-block separation unchanged ((c) $5.28 \pm 0.03$ vs.\ MGC
$5.25 \pm 0.05$; P1 passes) and adds no accuracy: (c) is $0.14$--$0.41$\,pp
below MGC under every protocol, with CIs that exclude 0 at S1-TTA, S2
and S2-TTA. We find no compensation gain masked by the block-0 change
(three seeds per arm).

\emph{The separation gain comes mostly from removing the auxiliary.}
Without the auxiliary and without any compensator, (b) reaches
$4.88 \pm 0.05$ (seeds $4.88$/$4.93$/$4.82$), far above the
SAM-matched cumulative baseline (a) ($0.75 \pm 0.01$; P4 passes, so SAM
does not drive the contrast). P2 predicted that (b) would stay below
$4.72$, near the cumulative regime; \textbf{it fails}. Of MGC's gain
over (a) at block~3 ($4.50 \pm 0.04$), removing the auxiliary accounts
for $4.13 \pm 0.05$ ($92\%$) and the compensator for $0.37 \pm 0.01$
($8\%$); this increment is small but positive in every seed (every (b)
seed lies below the lowest MGC seed, P2's second clause). Blocks 1 and
2 show the same pattern ((a)/(b)/MGC: $1.90$/$4.17$/$4.41$ and
$1.19$/$5.75$/$6.03$), and block-0 separation is $1.08$ in every arm.
As registered, this outcome changes a claim: at L4/D128 the separation
rise under MGC cannot be attributed mainly to missing-gradient
compensation; the compensation of the gradient envelope
(Prop.~\ref{prop:mgc_recovers_local}) explains only about $8\%$ of it. CIFAR-100
has no auxiliary-removal arm, so its MGC separation gain has not been
decomposed.

\emph{The dissociation holds across the four arms.} Deepest-block
separation ranges from $0.75$ to $5.28$, yet every paired 95\% CI in
Table~\ref{tab:f1_control} lies inside $\pm 1$\,pp and every arm's
S2-TTA mean is within 1\,pp of MGC's (P3 passes). No arm gains on (a)
with a CI that excludes 0; the largest difference is MGC's
$+0.24$\,[$-0.02$,\,$+0.50$]\,pp at S2-TTA. At S1,
the three arms without the auxiliary at blocks 1--3 sit $0.29$--$0.43$\,pp
below (a) (MGC $-0.29$\,[$-0.59$,\,$+0.02$], negative in every seed;
(b) and (c) with CIs that exclude 0); no direction was registered for
S1, so we report this descriptively. Two arms trained with the same
seed disagree on $6.7$--$7.0\%$ of test predictions (S2-TTA), fewer
than two seeds of the same arm ($8.6$--$8.8\%$).

\emph{Gradient attenuation.} On the seed-42 EMA checkpoints (first
1024 test images and the wrong labels of
Fig.~\ref{fig:grad_attenuation}), $\mathcal{F}_d$ at blocks 1--3 is
$0.929$/$0.966$/$0.971$ in (b) and $0.927$/$0.954$/$0.958$ in (a)
($0.929$/$0.950$/$0.952$ for the figure's SAM-off $\gamma{=}0.7$ run on
its EMA weights). Re-evaluated on the CPU, each checkpoint's
single-crop Stage-1 accuracy is within one test image of its logged
value, and 6 of the 10{,}000 predictions in (b) and 3 in (a) differ
from the logged ones: this passes the check registered for the earlier EMA
re-measurement of $\mathcal{F}_d$ (${\ge}99.9\%$ agreement) but not
the snapshot script's stricter, unregistered limit of two.
The cumulative term is attenuated in both arms, so at this
end-of-training snapshot $\mathcal{F}_d$ does not distinguish the
healthy arm from the collapsed one. In this recipe, attenuation alone
does not produce the collapse; it also depends on the depth-scaled
auxiliary, and how the two interact is open. The dependence need not
hold elsewhere: the LCFF arm ($\gamma{=}1$, no auxiliary and no
depth-order term, SAM off, seed~42) does collapse at L4/D128 ($0.94$,
Table~\ref{tab:dissociation_full}). The snapshot covers one seed and
one checkpoint per arm, not the training trajectory.

\emph{Hardware and resumes.} Arm (a) ran on Colab A100 GPUs; (b),
(c) for seeds 42/123, and MGC for seed~42 ran on DGX Spark (GB10)
machines, and (c) for seed~456 and MGC for seeds 123/456 on A100s,
with different PyTorch and CUDA versions across runs. GPU
kernels are not bitwise reproducible across these stacks, so every
contrast includes hardware and software variation, and (b) minus (a)
is confounded with it. The separation effects are large against this
variation: (c) minus MGC is $+0.01$ to $+0.05$ across same-GPU and
cross-GPU seed pairs, several times smaller than the compensator's
increment ($7$--$34{\times}$). The accuracy differences are not: restricted to seeds whose
two runs share a GPU model, the S2-TTA differences of (c) and (b) from
MGC shrink to $-0.22$ (two seeds) and $-0.08$\,pp (one seed). No run
was resumed.

\subsection{Squared-hinge replication}
\label{app:hinge}

To test whether the pathology depends on the softplus barrier, we
replace it by the squared hinge $\max(0,\theta-u)^2$ ($\theta{=}2$),
whose attenuation ratio decays linearly to a hard zero (Remark after
Theorem~\ref{thm:attenuation}). The cumulative ($\gamma{=}0.7$) and
history-free ($\gamma{=}0$) arms otherwise share the L4/D128 recipe with
SAM. The cumulative objective again starves the deepest block
(separation $1.67$ vs.\ $4.44$, a $2.66{\times}$ ratio) while block~0
is unchanged ($2.55$ vs.\ $2.54$), and accuracy moves by at most
$0.16$\,pp under any protocol (Table~\ref{tab:hinge}).

\begin{table}[!htb]
  \caption{\textbf{Squared-hinge replication} (barrier $\max(0,\theta-u)^2$, $\theta{=}2$; CIFAR-10
           L4/D128, 180+10 epochs, SAM on, 3 seeds per arm, mean $\pm$ sample SD). The deepest block
           loses separation under the cumulative objective while block~0 is unchanged, and test
           accuracy differs between the arms by at most 0.16\,pp. All seeds come from one run lineage; the
           $\gamma{=}0.7$ seed-456 layer log is a resumed suffix starting at epoch
           171.}
  \label{tab:hinge}
  \centering
  \small
  \setlength{\tabcolsep}{3pt}
  \resizebox{\textwidth}{!}{%
  \begin{tabular}{lcccccc}
    \toprule
    Arm & S1 & S1-TTA & S2 & S2-TTA & sep (L0) & sep (L3) \\
    \midrule
    Cumulative $\gamma{=}0.7$ & $85.54\!\pm\!0.29$ & $86.85\!\pm\!0.26$ & $85.94\!\pm\!0.27$ & $87.08\!\pm\!0.10$ & $2.55\!\pm\!0.02$ & $1.67\!\pm\!0.02$ \\
    History-free $\gamma{=}0$ & $85.57\!\pm\!0.29$ & $86.69\!\pm\!0.34$ & $85.83\!\pm\!0.43$ & $86.93\!\pm\!0.33$ & $2.54\!\pm\!0.06$ & $4.44\!\pm\!0.17$ \\
    \midrule
    $\Delta$ ($\gamma{=}0.7$ $-$ $\gamma{=}0$); sep ratio $\gamma{=}0/\gamma{=}0.7$ & $-0.03$ & $+0.16$ & $+0.11$ & $+0.15$ & --- & $2.66{\times}$ \\
    \bottomrule
  \end{tabular}%
  }
\end{table}

\subsection{Per-term gradient shares}
\label{app:gradshare}

Theorem~\ref{thm:attenuation} predicts that under $\gamma{>}0$ the
discrimination term on cumulative goodness contributes a vanishing
share of each deep block's gradient. Table~\ref{tab:gradshare}
measures this on the seed-42 L4/D128 checkpoints: under $\gamma{=}0.7$
the term's share falls from $0.247$ at block~0 to $0.030$ at block~3,
and the current-block auxiliary dominates at blocks 1--3 (share
$0.52$--$0.72$); under $\gamma{=}0$ the share stays between $0.18$ and
$0.26$, and under MGC it is $0.26$--$0.43$. Together with the
free-riding index $\mathcal{F}_d$ (Figure~\ref{fig:grad_attenuation})
and the $\lambda$ logs (App.~\ref{app:attenuation_compensated}), this
table monitors the mechanism at the level of gradients.

\begin{table}[!htb]
  \caption{\textbf{Per-term gradient share at convergence} (seed~42, CIFAR-10 L4/D128 gated-ablation
           arms and MGC; raw training weights, not EMA). Share $=\|\nabla_t\|/\|\nabla_\text{total}\|$
           over each block's own parameters, mean over 20 training batches; in parentheses
           $\cos(\nabla_t,\nabla_\text{total})$. ``Discrimination (cum.)'' is the loss on cumulative
           goodness (the term of Thm.~3.1). Term gradients sum to the total gradient up to a relative
           error below $3{\times}10^{-4}$ at block~0 and ${\le}10^{-6}$ at blocks 1--3. MGC was trained
           with SAM, the two gated arms without.}
  \label{tab:gradshare}
  \centering
  \small
  \setlength{\tabcolsep}{3pt}
  \begin{tabular}{llcccc}
    \toprule
    Arm & Term & Block 0 & Block 1 & Block 2 & Block 3 \\
    \midrule
    $\gamma{=}0.7$ (constant) & discrimination (cum.) & 0.247 (0.87) & 0.112 (0.82) & 0.050 (0.83) & 0.030 (0.78) \\
     & current-block aux. & 0.062 (0.87) & 0.516 (0.81) & 0.724 (0.79) & 0.671 (0.75) \\
    $\gamma{=}0$ & discrimination (cum.) & 0.256 (0.86) & 0.230 (0.75) & 0.202 (0.66) & 0.182 (0.61) \\
    MGC ($c_d{=}1$) & discrimination (cum.) & 0.262 (0.87) & 0.367 (0.64) & 0.430 (0.55) & 0.369 (0.50) \\
     & MGC local term & --- & 0.054 (0.40) & 0.076 (0.34) & 0.094 (0.30) \\
    \bottomrule
  \end{tabular}
\end{table}

\subsection{Frozen-backbone readouts}
\label{app:readouts}

The Stage-2 head is a representation probe, not part of the FF method.
Table~\ref{tab:readouts} shows why it is an attentive head: on the
seed-42 trio, a single local readout block trained on the frozen
backbone, whether with the FF goodness objective (strict FF) or with
cross-entropy, stays $5.3$--$6.4$\,pp below the attentive probe, and
switching its objective moves it by at most $0.27$\,pp, so the gap is
architectural rather than a matter of the training signal.
\emph{Bounded exception on MGC backbones.} The strict-FF block alone
reaches $63.95 \pm 0.19\%$ on the three MGC backbones. Comparing like
with like, the matched seed-42 MGC value ($64.07\%$) is
$+1.31$/$+1.47$/$+2.21$\,pp above the seed-42
cumulative/$\kappa{=}0$/$\gamma{=}0$ values, which we read as more
standalone class information left in the repaired backbone's output;
the 3-seed MGC mean is not compared with the seed-42 trio values. The
advantage disappears once the frozen blocks' goodness is added
(seed-42 span $0.55$\,pp across the four backbones).

\begin{table}[!htb]
  \caption{\textbf{Frozen-backbone readouts} (CIFAR-100 L4/D256, single crop, full 10k test, top-1 \%).
           All readout blocks see the frozen Stage-1 backbone only; the Stage-2 attentive head is a
           representation probe, not part of the FF method. Seed-42 columns are matched; the 3-seed
           MGC column is compared only with itself. On matched seed~42 the strict-FF block on the MGC
           backbone is $+1.31$/$+1.47$/$+2.21$\,pp above cumulative/$\kappa{=}0$/$\gamma{=}0$,
           and the gap closes once frozen-block goodness is added (span 0.55\,pp).
           Cross-entropy training moves the block by at most 0.27\,pp on the three original
           seed-42 backbones, by 0.40\,pp on the MGC seed-42 backbone, and by 0.27\,pp
           for the MGC 3-seed means.
           Goodness-sum row: logged S1 (the readout runs reproduce it within 0.01\,pp).}
  \label{tab:readouts}
  \centering
  \small
  \setlength{\tabcolsep}{3pt}
  \resizebox{\textwidth}{!}{%
  \begin{tabular}{lccccc}
    \toprule
    Readout on frozen backbone & $\gamma{=}0$ (s42) & $\kappa{=}0$ (s42) & Cum.\ (s42) & MGC (s42) & MGC (3 seeds) \\
    \midrule
    Goodness sum (deployed S1 rule) & 66.17 & 66.07 & 66.70 & 66.05 & $65.73\!\pm\!0.35$ \\
    Strict-FF readout block alone & 61.86 & 62.60 & 62.76 & 64.07 & $63.95\!\pm\!0.19$ \\
    Strict-FF block $+$ frozen goodness & 66.01 & 65.78 & 66.33 & 65.95 & $65.66\!\pm\!0.40$ \\
    CE-trained block alone & 62.13 & 62.87 & 62.74 & 64.47 & $64.22\!\pm\!0.31$ \\
    CE-trained block $+$ frozen goodness & 64.99 & 65.00 & 64.96 & 65.79 & $65.55\!\pm\!0.30$ \\
    \midrule
    Stage-2 attentive head (probe, not FF) & 68.27 & 68.29 & 68.09 & 68.71 & $68.24\!\pm\!0.51$ \\
    \bottomrule
  \end{tabular}%
  }
\end{table}

\subsection{Label-free probes}
\label{app:probes}

Table~\ref{tab:probes} fits simple probes on each run's own validation
split and evaluates them once on the full test set. Structured
readouts of the goodness scores keep the trio within 1\,pp (on the
400-D block-softmax goodness: ridge span $0.05$\,pp, logistic
$0.58$\,pp). The stability of $k$-NN depends on per-block
normalization: cosine 10-NN spans $1.47$\,pp on block-softmax-normalized
goodness but $8.58$\,pp on raw goodness. Block-3-only goodness probes
(logistic $8.40$\,pp, 10-NN $6.70$\,pp) and pooled-embedding probes
(all blocks $1.30$/$4.70$\,pp; block~3 $4.0$--$5.1$\,pp) are bounded
exceptions where the variants separate by more than 1\,pp; the
cumulative variant is highest on block-3 goodness and lowest on every
pooled-embedding panel. The invariance claim is therefore scoped to
the deployed rule and the structured readouts, not to every probe.

\begin{table}[!htb]
  \caption{\textbf{Label-free probes on the frozen CIFAR-100 trio} (fit on each run's own 5k validation
           split, evaluated once on the full 10k test; top-1 \%, mean $\pm$ sample SD over 3 seeds; span =
           max $-$ min of the unrounded variant means). Structured readouts of the goodness scores keep the
           trio within 1\,pp; k-NN stability depends on per-block normalization, and block-3-only and
           pooled-embedding probes are bounded exceptions where the span exceeds 1\,pp. With a zero label
           token the pooled spans are 1.32/5.30\,pp (all blocks,
           logistic/10-NN) and 5.09/4.81\,pp (block~3).}
  \label{tab:probes}
  \centering
  \small
  \setlength{\tabcolsep}{3pt}
  \begin{tabular}{lcccc}
    \toprule
    Readout / probe & $\gamma{=}0$ & $\kappa{=}0$ & Cumulative & Span (pp) \\
    \midrule
    Goodness sum (S1, deployed) & $66.37\!\pm\!0.20$ & $66.35\!\pm\!0.41$ & $66.78\!\pm\!0.28$ & 0.43 \\
    Stage-2 attentive probe (S2) & $68.54\!\pm\!0.25$ & $68.44\!\pm\!0.51$ & $68.51\!\pm\!0.51$ & 0.11 \\
    \midrule
    \multicolumn{5}{l}{\emph{400-D goodness, block-softmax normalized}} \\
    \quad ridge & $67.05\!\pm\!0.37$ & $67.10\!\pm\!0.70$ & $67.11\!\pm\!0.31$ & 0.05 \\
    \quad logistic & $66.48\!\pm\!0.43$ & $66.28\!\pm\!0.57$ & $65.91\!\pm\!0.59$ & 0.58 \\
    \quad 10-NN (cosine) & $59.34\!\pm\!0.73$ & $60.81\!\pm\!1.33$ & $60.18\!\pm\!0.50$ & 1.47 \\
    \multicolumn{5}{l}{\emph{400-D goodness, raw}} \\
    \quad logistic & $61.83\!\pm\!0.17$ & $62.31\!\pm\!0.70$ & $62.17\!\pm\!0.56$ & 0.49 \\
    \quad 10-NN (cosine) & $39.68\!\pm\!0.86$ & $46.80\!\pm\!0.76$ & $48.25\!\pm\!0.73$ & 8.58 \\
    \multicolumn{5}{l}{\emph{Block-3 goodness only (100-D)}} \\
    \quad logistic & $47.40\!\pm\!0.45$ & $54.74\!\pm\!0.72$ & $55.80\!\pm\!0.30$ & 8.40 \\
    \quad 10-NN (cosine) & $34.21\!\pm\!0.79$ & $38.30\!\pm\!0.72$ & $40.91\!\pm\!0.74$ & 6.70 \\
    \multicolumn{5}{l}{\emph{Pooled embedding, all blocks (mean label token)}} \\
    \quad logistic & $55.37\!\pm\!1.03$ & $55.45\!\pm\!1.09$ & $54.15\!\pm\!0.70$ & 1.30 \\
    \quad 10-NN (cosine) & $37.52\!\pm\!1.73$ & $36.49\!\pm\!1.87$ & $32.82\!\pm\!1.47$ & 4.70 \\
    \multicolumn{5}{l}{\emph{Pooled embedding, block 3 (mean label token)}} \\
    \quad logistic & $38.04\!\pm\!1.08$ & $36.58\!\pm\!3.61$ & $33.67\!\pm\!1.03$ & 4.37 \\
    \quad 10-NN (cosine) & $23.74\!\pm\!0.93$ & $23.78\!\pm\!2.14$ & $19.77\!\pm\!2.35$ & 4.02 \\
    \bottomrule
  \end{tabular}
\end{table}

\subsection{Validation-fitted readout ladder}
\label{app:ladder}

Table~\ref{tab:readout_ladder} asks whether a readout fitted to the
per-block goodness scores, rather than the plain sum, separates the
trio. It does not: the trio span is at most $0.43$\,pp at every rung.
R2 (per-depth temperatures) is R1 restricted to positive weights, so
the two coincide whenever the fitted R1 weights are positive. The
learned weights do help the $\gamma{=}0$ and $\kappa{=}0$ variants more
than the cumulative one (R1 gains $+1.39 \pm 0.06$ and $+1.38 \pm 0.18$ vs.\
$+0.60 \pm 0.30$\,pp; paired differences
$+0.79$\,[$+0.48$,\,$+1.11$] and $+0.78$\,[$+0.46$,\,$+1.10$]\,pp), so
the differential gain is real even though the spans stay small. The
analysis is exploratory and conditional on the trained checkpoints.

\begin{table}[!htb]
  \caption{\textbf{Validation-fitted readout ladder on the frozen CIFAR-100 trio} (exploratory and
           checkpoint-conditional). Each readout is fit on the run's own 5k validation split and evaluated
           once on the full 10k test (S1 single crop, EMA weights); top-1 \%, mean $\pm$ sample SD over
           3 seeds, in parentheses the gain over R0. R2 equals R1 restricted to positive weights, so the
           two coincide whenever the fitted R1 weights are positive. The trio span stays at most
           0.43\,pp at every rung, while the learned weights help the
           $\gamma{=}0$ and $\kappa{=}0$ variants more than the cumulative one (last row: paired 5000-resample
           bootstrap over the test set).
           R0 is recomputed from the dumped scores and matches the logged S1 within 0.01\,pp.}
  \label{tab:readout_ladder}
  \centering
  \small
  \setlength{\tabcolsep}{3pt}
  \resizebox{\textwidth}{!}{%
  \begin{tabular}{lcccc}
    \toprule
    Readout & $\gamma{=}0$ & $\kappa{=}0$ & Cumulative & Span (pp) \\
    \midrule
    R0: plain goodness sum (deployed) & $66.37\!\pm\!0.20$ & $66.36\!\pm\!0.41$ & $66.79\!\pm\!0.28$ & 0.43 \\
    R1: per-depth weights $w_d$ & $67.76\!\pm\!0.22$ ($+1.39\!\pm\!0.06$) & $67.74\!\pm\!0.46$ ($+1.38\!\pm\!0.18$) & $67.39\!\pm\!0.29$ ($+0.60\!\pm\!0.30$) & 0.38 \\
    R2: per-depth temperatures ($w_d{>}0$) & $67.78\!\pm\!0.25$ ($+1.41\!\pm\!0.06$) & $67.76\!\pm\!0.49$ ($+1.41\!\pm\!0.19$) & $67.45\!\pm\!0.29$ ($+0.66\!\pm\!0.33$) & 0.34 \\
    R3: margin-gated $w_d(x)$ & $67.79\!\pm\!0.29$ ($+1.41\!\pm\!0.11$) & $67.76\!\pm\!0.51$ ($+1.40\!\pm\!0.19$) & $67.46\!\pm\!0.23$ ($+0.67\!\pm\!0.27$) & 0.33 \\
    \midrule
    R1 gain minus cumulative's R1 gain [95\% CI] & $+0.79$\,[$+0.48$,\,$+1.11$] & $+0.78$\,[$+0.46$,\,$+1.10$] & --- &  \\
    \bottomrule
  \end{tabular}%
  }
\end{table}

\subsection{Prediction churn and arithmetic redistribution}
\label{app:churn}

\emph{Arithmetic.} Re-splitting the cumulative CIFAR-100 seed-42
model's goodness sum across blocks so that the deepest block's own
share moves from $14\%$ to $98\%$ leaves every class score unchanged
and flips 0 of 10{,}000 predictions
(Proposition~\ref{prop:redistribution}). This covers exactly
redistributing changes only, not objective changes such as MGC.
\emph{Empirical.} Retrained variants change many predictions. The
fraction of the 10k CIFAR-100 test images whose predicted label
differs between two variants of the same seed is, averaged over the
three seeds, $13.0$--$13.7\%$ for the trio pairs under S2~TTA
($17.3$--$18.0\%$ under S1); for MGC vs.\ cumulative it is
$25.3$/$25.9$/$26.2\%$ (S1) and $22.74$/$22.76$/$22.74\%$ (S2~TTA) on
seeds 42/123/456. For the trio pairs under S2~TTA, the disagreements
are nearly balanced in correctness. For MGC vs.\ cumulative at S1 they
are not:
cumulative-only-correct predictions outnumber MGC-only-correct ones in
every seed ($704/639$, $748/629$, $763/632$; exact McNemar
$p{=}0.08$, $0.0015$, $0.0005$), which is the Stage-1 cost.
Pooled over the three seeds ($n{=}30{,}000$ per pair;
5000-resample paired bootstrap), the S2~TTA differences are
$+0.17$\,[$-0.13$,\,$+0.46$]\,pp ($\gamma{=}0$ minus $\kappa{=}0$),
$+0.28$\,[$-0.02$,\,$+0.56$]\,pp ($\gamma{=}0$ minus cumulative), and
$+0.11$\,[$-0.20$,\,$+0.42$]\,pp ($\kappa{=}0$ minus cumulative); for
$\gamma{=}0$ vs.\ cumulative the $\gamma{=}0$-only-correct and
cumulative-only-correct counts are $294/257$, $468/440$, and $299/281$
on seeds 42/123/456. A large per-block-health intervention thus moves
a non-trivial slice of predictions that sits near the cumulative
decision boundary and contributes no detectable net accuracy change at
three seeds (the other protocols are in
Table~\ref{tab:dissociation_merged}; the paired statistics reproduce
from \texttt{analysis/paired\_dissociation\_stats.py} in the
code repository). Table~\ref{tab:trio_perseed} gives the per-seed paired
differences of the trio under all four protocols.

\begin{table}[!htb]
  \caption{\textbf{Per-seed paired differences of the CIFAR-100 trio} (pp, seeds 42/123/456; each on the
           seed's full 10k test set). The pooled differences and their 5000-resample bootstrap CIs are in
           Table~\ref{tab:dissociation_merged} (vs.\ cumulative) and App.~\ref{app:churn}; every pooled CI of
           the three pairs lies inside $\pm1$\,pp under all four protocols, while one per-seed value exceeds
           1\,pp in magnitude ($\gamma{=}0$ $-$ $\kappa{=}0$, S1-TTA, seed 456: $-1.03$). MGC's per-seed differences are in
           Table~\ref{tab:mgc_full}.}
  \label{tab:trio_perseed}
  \centering
  \small
  \setlength{\tabcolsep}{3pt}
  \resizebox{\textwidth}{!}{%
  \begin{tabular}{lcccc}
    \toprule
    Pair & S1 & S1-TTA & S2 & S2-TTA \\
    \midrule
    $\gamma{=}0$ $-$ cumulative & $-0.53$/$-0.19$/$-0.53$ & $-0.26$/$-0.44$/$-0.75$ & $+0.18$/$+0.24$/$-0.31$ & $+0.37$/$+0.28$/$+0.18$ \\
    $\kappa{=}0$ $-$ cumulative & $-0.63$/$-0.38$/$-0.28$ & $-0.33$/$-0.15$/$+0.28$ & $+0.20$/$-0.34$/$-0.07$ & $+0.07$/$-0.09$/$+0.35$ \\
    $\gamma{=}0$ $-$ $\kappa{=}0$ & $+0.10$/$+0.19$/$-0.25$ & $+0.07$/$-0.29$/$-1.03$ & $-0.02$/$+0.58$/$-0.24$ & $+0.30$/$+0.37$/$-0.17$ \\
    \bottomrule
  \end{tabular}%
  }
\end{table}

\subsection{Validation-selected early exit}
\label{app:early_exit}

Because every block adds its goodness to a running class score, these
networks can stop after any block. Table~\ref{tab:early_exit} and
Figure~\ref{fig:early_exit_pareto} measure what stopping early saves
when the operating point is chosen without test data. On each run's own
5k validation split we select a truncation depth (the primary rule T,
or T$'$) or a dynamic-exit margin threshold (rule D) and evaluate each
choice once on the 10k test split (S1 single crop, EMA weights); no
operating point is chosen on test. The rules were fixed in writing
before the validation dumps were made (except one smoke test), but
after the rebuttal's test-selected analysis of the same checkpoints,
so they are not blind to those checkpoints' test curves; they are the
simplest iso-accuracy, argmax, and fewest-blocks rules. Early exit
costs little test accuracy. On the nine CIFAR-10 L4/D128 arms, T saves
$48.1 \pm 29.5\%$ of the FLOPs at $+0.09 \pm 0.19$\,pp and D saves
$70.8 \pm 2.4\%$ at $+0.11 \pm 0.22$\,pp. On the nine CIFAR-100 trio
runs, T saves $63.5 \pm 13.1\%$ at $+0.25 \pm 0.45$\,pp (6/9 runs with
$\Delta{\geq}0$), less than we expected when fixing the rules, because
T stops after block~0 in five of the nine runs; the validation-best
depth T$'$ selects depth~2 in every run and raises test accuracy in
each ($+0.79 \pm 0.37$\,pp at $49.7\%$ saved), and D saves
$73.4 \pm 1.6\%$ at $+0.01 \pm 0.36$\,pp. \emph{The saving is not
specific to the repairs}: the cumulative runs exit early too. On
CIFAR-10 the LCFF arm stops after block~0 under T, and the
constant-$\gamma{=}0.7$ arm keeps full depth under T and T$'$ but
saves $64.5\%$ under D. Under T the two cumulative arms save $0.0$ and
$72.2\%$ at $+0.00$ and $-0.22$\,pp, and the seven repaired arms
$51.6 \pm 25.7\%$ at $+0.15 \pm 0.16$\,pp; this difference rests on
one seed and two cumulative arms and is not evidence of a repair
advantage. On CIFAR-100 the cumulative runs have the same mean saving
as $\kappa{=}0$ under T ($58.0\%$) and nearly as much under D ($72.6$
vs.\ $73.2\%$). Only $\gamma{=}0$ exits earlier than cumulative, under
T (block~0 in all three seeds vs.\ one of three; paired,
$+16.6 \pm 14.3$ percentage points more FLOPs saved). Its test
accuracy at the exit is then $0.52 \pm 0.30$\,pp below cumulative's,
but $0.40 \pm 0.19$\,pp of that is its lower full-depth accuracy;
measured against each run's own full depth, the paired $\Delta$ is
$-0.12 \pm 0.17$\,pp. Under D the paired differences in saving are
small ($+2.0 \pm 2.3$ for $\gamma{=}0$, $+0.7 \pm 1.3$ for
$\kappa{=}0$). The repaired runs' larger gains under T$'$
($+0.61 \pm 0.34$ and $+0.57 \pm 0.53$\,pp over cumulative, paired)
come from their lower full-depth S1 accuracy
(Table~\ref{tab:dissociation_merged}), not from a better exit: at
depth~2 their test accuracy differs from cumulative by only
$+0.21 \pm 0.20$ ($\gamma{=}0$) and $+0.15 \pm 0.52$\,pp
($\kappa{=}0$). Anytime inference therefore comes from the running
class score, which every FF variant we trained has, not from the
repairs. \emph{Caveats.} The CIFAR-10 arms are nine
configurations at one seed each. Validation sampling noise on 5k
images is as large as the depth effects (one image is $0.02$\,pp), and
a validation-selected point can lose test accuracy (down to
$-0.56$\,pp, D on one CIFAR-100 run). FLOPs are counted rather than timed, and dynamic exit saves
wall-clock time only if examples leave the batch as they exit. For the
$\tau{=}3,\kappa{=}0$ arm (whose block-0 figure is corrected in
App.~\ref{app:changes}, item B7), T and T$'$ (chosen on validation)
also stop the arm after block~0, where its test accuracy is
$0.39$\,pp above full depth (S1 single crop; $+0.17$\,pp with flip
TTA, $87.15\%$ vs.\ $86.98\%$).

\begin{table}[!htb]
  \caption{\textbf{Validation-selected early exit} (S1 single crop, EMA weights). Every operating point is chosen on the run's own 5k validation split and
           evaluated once on the 10k test split. \emph{T}: the shallowest depth whose validation accuracy is at least
           the full-depth value; \emph{T$'$}: the validation-best depth (ties to the shallower); \emph{D}: dynamic exit
           after block $d{<}L$ once the top-1 minus top-2 margin of the running class score exceeds a threshold, set to the
           value with the fewest mean validation blocks at no loss of validation accuracy. $\Delta$: test accuracy at the
           operating point minus full-depth test accuracy of the same checkpoint (pp); ${\geq}0$: runs with $\Delta{\geq}0$;
           Saved: FLOPs saved vs.\ full depth (\%; stem plus every label hypothesis's blocks; stopping after block~0 saves
           72.2\% on CIFAR-10 and 74.5\% on CIFAR-100). Mean $\pm$ sample SD over runs (on CIFAR-10, over nine
           configurations at one seed). CIFAR-10 cumulative arms: constant $\gamma{=}0.7$ / LCFF $\gamma{=}1.0$, listed
           individually; repaired arms: $\gamma{=}0$ and the six hardness-gated arms. \emph{runs at $d{=}1/2/3/4$}: how
           many runs select each depth (T, T$'$); D: range of mean test blocks. Paired rows: repaired minus cumulative of the same seed, mean
           $\pm$ SD over the three seeds, of $\Delta$ (pp) and of Saved (percentage points). Each re-scored checkpoint reproduces its logged full-depth S1
           test accuracy within one test image (14 runs) or three images (4 runs; TF32 training-time vs.\ FP32
           re-scoring numerics).}
  \label{tab:early_exit}
  \centering
  \small
  \setlength{\tabcolsep}{3pt}
  \resizebox{\textwidth}{!}{%
  \begin{tabular}{lccccccccc}
    \toprule
    & \multicolumn{3}{c}{T} & \multicolumn{3}{c}{T$'$} & \multicolumn{3}{c}{D (dynamic exit)} \\
    \cmidrule(lr){2-4} \cmidrule(lr){5-7} \cmidrule(lr){8-10}
    Runs & $\Delta$ (pp) & ${\geq}0$ & Saved (\%) & $\Delta$ (pp) & ${\geq}0$ & Saved (\%) & $\Delta$ (pp) & ${\geq}0$ & Saved (\%) \\
    \midrule
    \multicolumn{10}{@{}l}{\emph{CIFAR-10 L4/D128, seed 42, no SAM}} \\
    All 9 arms & $+0.09\!\pm\!0.19$ & 7/9 & $48.1\!\pm\!29.5$ & $+0.07\!\pm\!0.15$ & 6/9 & $40.1\!\pm\!26.9$ & $+0.11\!\pm\!0.22$ & 7/9 & $70.8\!\pm\!2.4$ \\
    \quad cumulative (2 arms) & $+0.00$ / $-0.22$ & 1/2 & $0.0$ / $72.2$ & $+0.00$ / $-0.04$ & 1/2 & $0.0$ / $48.1$ & $+0.01$ / $-0.22$ & 1/2 & $64.5$ / $72.2$ \\
    \quad repaired (7 arms) & $+0.15\!\pm\!0.16$ & 6/7 & $51.6\!\pm\!25.7$ & $+0.10\!\pm\!0.16$ & 5/7 & $44.7\!\pm\!25.7$ & $+0.18\!\pm\!0.19$ & 6/7 & $71.5\!\pm\!0.7$ \\
    \quad runs at $d{=}1/2/3/4$ & \multicolumn{3}{c}{4/3/0/2} & \multicolumn{3}{c}{2/4/1/2} & \multicolumn{3}{c}{mean blocks 1.00--1.32} \\
    \midrule
    \multicolumn{10}{@{}l}{\emph{CIFAR-100 L4/D256 trio, seeds 42/123/456}} \\
    All 9 runs & $+0.25\!\pm\!0.45$ & 6/9 & $63.5\!\pm\!13.1$ & $+0.79\!\pm\!0.37$ & 9/9 & $49.7\!\pm\!0.0$ & $+0.01\!\pm\!0.36$ & 4/9 & $73.4\!\pm\!1.6$ \\
    \quad $\gamma{=}0$ (history-free) & $+0.02\!\pm\!0.22$ & 1/3 & $74.5\!\pm\!0.0$ & $+1.01\!\pm\!0.12$ & 3/3 & $49.7\!\pm\!0.0$ & $+0.02\!\pm\!0.22$ & 1/3 & $74.5\!\pm\!0.0$ \\
    \quad $\kappa{=}0$ (gated) & $+0.60\!\pm\!0.67$ & 2/3 & $58.0\!\pm\!14.3$ & $+0.96\!\pm\!0.22$ & 3/3 & $49.7\!\pm\!0.0$ & $+0.15\!\pm\!0.51$ & 1/3 & $73.2\!\pm\!1.6$ \\
    \quad cumulative & $+0.14\!\pm\!0.12$ & 3/3 & $58.0\!\pm\!14.3$ & $+0.39\!\pm\!0.35$ & 3/3 & $49.7\!\pm\!0.0$ & $-0.13\!\pm\!0.38$ & 2/3 & $72.6\!\pm\!2.3$ \\
    \quad $\gamma{=}0$ $-$ cum., paired & $-0.12\!\pm\!0.17$ & --- & $+16.6\!\pm\!14.3$ & $+0.61\!\pm\!0.34$ & --- & $+0.0\!\pm\!0.0$ & $+0.14\!\pm\!0.22$ & --- & $+2.0\!\pm\!2.3$ \\
    \quad $\kappa{=}0$ $-$ cum., paired & $+0.46\!\pm\!0.64$ & --- & $+0.0\!\pm\!24.8$ & $+0.57\!\pm\!0.53$ & --- & $+0.0\!\pm\!0.0$ & $+0.28\!\pm\!0.39$ & --- & $+0.7\!\pm\!1.3$ \\
    \quad runs at $d{=}1/2/3/4$ & \multicolumn{3}{c}{5/4/0/0} & \multicolumn{3}{c}{0/9/0/0} & \multicolumn{3}{c}{mean blocks 1.00--1.18} \\
    \bottomrule
  \end{tabular}%
  }
\end{table}

\begin{figure}[!htbp]
  \centering
  \includegraphics[width=0.82\textwidth]{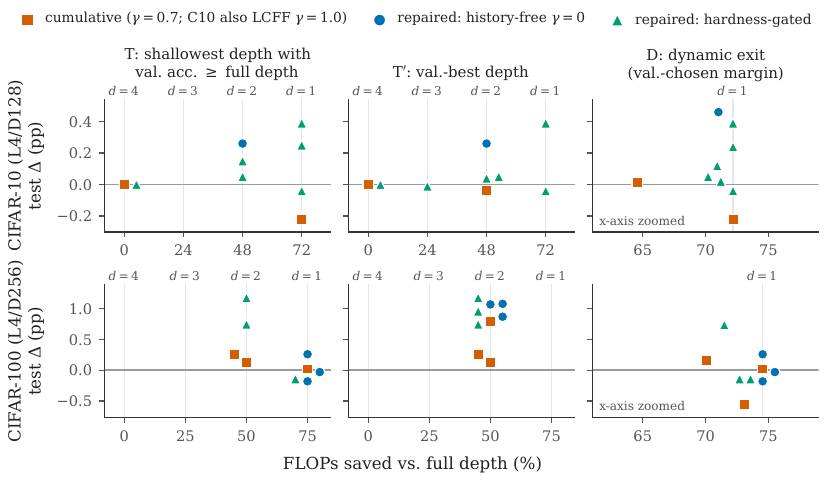}
  \caption{\textbf{Validation-selected early exit, per run.} Test
  accuracy change vs.\ FLOPs saved, both relative to full depth, at the
  operating points of Table~\ref{tab:early_exit}. Columns: selection
  rule; rows: CIFAR-10 L4/D128 (nine arms, seed~42) and the CIFAR-100
  L4/D256 trio (three seeds per variant). Squares are cumulative runs,
  circles and triangles repaired runs: at least one cumulative run exits
  early under every rule (the CIFAR-10 constant-$\gamma$ arm keeps full
  depth under T and T$'$), so the saving is not repair-specific (paired
  comparison in Table~\ref{tab:early_exit}). Points at the same depth are spread
  horizontally by family for visibility, and the dynamic-exit column's
  $x$-axis is zoomed.}
  \label{fig:early_exit_pareto}
\end{figure}

\section{Adaptive-Collaboration Sweep: Full Results}
\label{app:gated_full}

This appendix expands Section~\ref{sec:adaptive_gamma}: the full L4 /
D128 / 180-epoch sweep (Table~\ref{tab:gated}), the at-$L{=}8$ cross-depth
validation, and the CIFAR-100 baseline replication of the free-riding
signature; the goodness-decomposition and depth-wise / epoch-wise dynamics
figures are in App.~\ref{app:per_layer_metrics}. Naming in this appendix
and App.~\ref{app:per_layer_metrics}: adaptive collaboration
(``Adpt'') is the hardness-gated repair of \S\ref{sec:repairs} with
temperature $\tau$ and threshold $\kappa$; ``(cumul)'' gates on
$\sum_{j<d} g^{(j)}$ and ``(prev)'' (prev-only) on $g^{(d-1)}$ alone,
while the gated history term is $\sum_{j<d} g^{(j)}$ in both;
``no collab'' is $\gamma{=}0$; ``constant $\gamma{=}0.7$'' is the
ungated cumulative rule (``baseline'' in the legends of the
gated-sweep figures, not the \textsc{baseline} variant of
Table~\ref{tab:main_results}); and ``LCFF-prefix'' is LCFF.

\begin{table}[htbp]
  \caption{L4/D256 $\gamma$ sweep on CIFAR-10. Free-riding signatures
           across the three regimes: $\text{sep}_\text{nl}$ is
           loss-score wrong-label separation (primary accuracy
           predictor; cumulative only for $\gamma{>}0$,
           Eq.~\ref{eq:diagnostics}); $\text{sep}^\text{cur}_\text{nl}$ is current-block
           separation (per-block independence); aux.: the depth-scaled
           auxiliary (Eq.~\ref{eq:depth_scaled_main}). Validation accuracy
           varies by $<0.7$\,pp despite the ${\sim}4.7{\times}$ (LCFF) to
           ${\sim}5.4{\times}$ (\textsc{cp-fair}) deepest-block
           collapse, anticipating the dissociation finding. All rows are
           seed~42, read at the end of Stage~1 (epoch 362). Trend labels
           follow the depth profile of $\text{sep}^\text{cur}_\text{nl}$
           (Table~\ref{tab:gamma_sweep_perlayer}), not the L3 value alone:
           \emph{healthy}, every deeper block stays above L0;
           \emph{partial}, separation first rises above L0 and then
           collapses; \emph{severe}, it falls below L0 from L1 on. By the
           L3/L0 $g^+_\text{cur}$ ratio (e.g., $0.232$ in
           App.~\ref{sec:cifar100_metrics}), both collapsed variants are
           severe ($0.13$ and $0.16$).}
  \label{tab:freeriding}
  \centering
  \small
  \setlength{\tabcolsep}{3pt}
  \begin{tabular}{lcccccc}
    \toprule
    Variant & $\gamma$ & $\text{sep}_\text{nl}$ (L3) & $\text{sep}^\text{cur}_\text{nl}$ (L0) & $\text{sep}^\text{cur}_\text{nl}$ (L3) & Trend & Val \\
    \midrule
    $\gamma{=}0$ (history-free, \textbf{validation-selected}) & 0.0 & \textbf{0.995} & 2.30 & \textbf{3.28} & $\uparrow$ \textbf{healthy} & \textbf{91.98\%} \\
    \midrule
    \textsc{cp-fair} ($\gamma{=}0.7$ + aux.) & 0.7 & \textbf{0.995} & 2.30 & 0.61 & $\downarrow$ partial & 91.90\% \\
    LCFF-prefix ($\gamma{=}1.0$, no aux.) & 1.0 & \textbf{0.995} & 2.32 & 0.70 & $\downarrow$ \textbf{severe} & 91.60\% \\
    \bottomrule
  \end{tabular}
\end{table}

\begin{table}[htbp]
  \caption{Full cross-depth dissociation (companion of
           Table~\ref{tab:dissociation_merged} in the main paper).
           $\text{sep}^\text{cur}_\text{nl}$ at the deepest layer measures
           per-block discrimination; higher values indicate less
           free-riding. The table carries no qualitative free-riding labels:
           the depth-profile rule of
           Table~\ref{tab:freeriding} reproduces those of the 32-expert rows
           ($\gamma{=}0$ healthy; \textsc{cp-fair} partial, $2.25$ at L1 vs.\
           $1.25$ at L0, then down to $0.42$; LCFF severe, below L0 from L1 on,
           $0.86$ at L1) but not all of the others.
           All accuracies are S2~TTA top-1 on the held-out
           evaluation split (official test for CIFAR; labeled validation
           for Tiny~ImageNet).
           CIFAR-100 rows report mean over 3 seeds (42, 123, 456).
           CIFAR-10 rows without a seed count are single-seed (42) runs
           trained without SAM.
           $^\sharp$Mean $\pm$ sample std; per-seed values and SAM
           settings are in Table~\ref{tab:gated_k0_multiseed}. At
           L4/D128, seed~42 was trained without SAM and seeds 123/456
           with SAM, so the aggregate mixes two optimizer settings and the
           like-for-like comparison with the single-seed rows is the
           seed-42 $\kappa{=}0$ run ($87.27\%$). At L8/D128 all three
           seeds were trained without SAM; the seed-42 test accuracy was
           computed post hoc from its selected Stage-1 and Stage-2
           checkpoints.
           L4/D128 also shows the dissociation: $\kappa{=}0$ has the
           largest deepest-block separation but not the highest accuracy
           (constant $\gamma{=}0.7$, with a collapsed deepest block, is
           $0.15$\,pp above the seed-42 $\kappa{=}0$ run and $0.27$\,pp
           above the 3-seed mean), and the three rows span $0.58$\,pp.
           The dissociation also holds at L8/D128 and on CIFAR-100. At
           L8/D128 with 32 experts (seed~42) the three regimes stay
           inside the $<\!1$\,pp band, but the depth-scaled repair has
           the best single-seed accuracy there, so whether repairs begin
           to matter at greater depth is open
           (App.~\ref{app:long_limitations}).}
  \label{tab:dissociation_full}
  \centering
  \small
  \setlength{\tabcolsep}{3pt}
  \resizebox{\textwidth}{!}{%
  \begin{tabular}{llcc}
    \toprule
    Scale & Variant & $\text{sep}^\text{cur}_\text{nl}$ (deep) & Test (S2 TTA) \\
    \midrule
    \multirow{3}{*}{L4, D128}
    & Gated $\kappa{=}0$ (3 seeds)$^\sharp$ & $\mathbf{2.29 \pm 0.71}$ (L3) & $87.15 \pm 0.17\%$ \\
    & Constant $\gamma{=}0.7$ & 0.74 (L3) & $\mathbf{87.42\%}$ \\
    & LCFF $\gamma{=}1.0$ & 0.94 (L3) & $86.84\%$ \\
    \midrule
    \multirow{4}{*}{L8, D128}
    & Gated $\kappa{=}0$ (3 seeds)$^\sharp$ & $\mathbf{6.65 \pm 0.33}$ (L7) & $87.12 \pm 0.35\%$ \\
    & $\gamma{=}0$ (history-free) & 3.02 (L7) & $\mathbf{87.21\%}$ \\
    & Constant $\gamma{=}0.7$ & 0.82 (L7) & 86.97\% \\
    & LCFF $\gamma{=}1.0$ & 0.15 (L7) & 86.79\% \\
    \midrule
    \multirow{3}{*}{L8, D128, 32e}
    & $\gamma{=}0$ (history-free) & \textbf{2.74} (L7) & 89.22\% \\
    & CP-FAIR ($\gamma{=}0.7$) & 0.42 (L7) & \textbf{89.61\%} \\
    & LCFF $\gamma{=}1.0$ & 0.53 (L7) & 89.20\% \\
    \midrule
    \multirow{3}{*}{CIFAR-100, L4, D256 (3 seeds)}
    & $\gamma{=}0$ (history-free) & $\mathbf{4.78 \pm 0.02}$ (L3) & $\mathbf{69.23 \pm 0.34\%}$ \\
    & Gated $\kappa{=}0$ & $1.71 \pm 0.01$ (L3) & $69.07 \pm 0.61\%$ \\
    & Cumulative ($\gamma{=}0.7$) & $0.96 \pm 0.01$ (L3) & $68.96 \pm 0.43\%$ \\
    \midrule
    Tiny ImageNet, L4, D256 (3 seeds)
    & $\gamma{=}0$ (history-free) & 1.52 (L3) & $52.32 \pm 0.34\%$ \\
    \bottomrule
  \end{tabular}%
  }
\end{table}

\begin{table}[htbp]
  \caption{Full hardness-gated results across nine conditions. $\text{sep}^\text{cur}_\text{nl}$
           is per-block discrimination; $g^+_\text{cur}$ is per-block positive goodness;
           $g^-_\text{cur}$ is per-block wrong-label negative goodness (more negative = better).
           Acc (S1) is the Stage-1 single-crop (no TTA) official-test top-1
           from each run's final evaluation (L4, D128, seed~42).}
  \label{tab:gated}
  \centering
  \scriptsize
  \setlength{\tabcolsep}{2pt}
  \resizebox{\textwidth}{!}{%
  \begin{tabular}{lccccccccccccl}
    \toprule
    & \multicolumn{4}{c}{$\text{sep}^\text{cur}_\text{nl}$} & \multicolumn{4}{c}{$g^+_\text{cur}$} & \multicolumn{4}{c}{$g^-_\text{cur}$} & \\
    \cmidrule(lr){2-5} \cmidrule(lr){6-9} \cmidrule(lr){10-13}
    Variant & L0 & L1 & L2 & L3 & L0 & L1 & L2 & L3 & L0 & L1 & L2 & L3 & Acc (S1) \\
    \midrule
    $\gamma{=}0$ (no collab)  & 1.16 & 2.06 & 2.38 & 2.61 & 1.70 & 1.62 & 1.74 & 1.82 & 0.54 & $-$0.44 & $-$0.64 & $-$0.79 & 85.75\% \\
    Constant $\gamma{=}0.7$   & 1.16 & 2.16 & 1.32 & 0.74 & 1.71 & 1.10 & 0.55 & 0.38 & 0.56 & $-$1.06 & $-$0.77 & $-$0.36 & \textbf{86.18\%} \\
    \midrule
    Adpt $\tau{=}1,\kappa{=}0$ (cumul)   & 1.14 & 2.96 & \textbf{3.39} & \textbf{3.11} & 1.71 & 1.76 & 2.10 & 2.36 & 0.56 & $-$1.19 & $-$1.29 & $-$0.74 & 85.86\% \\
    Adpt $\tau{=}3,\kappa{=}0$ (cumul)   & 1.17 & 3.11 & 3.32 & 3.05 & 1.72 & 1.88 & 2.07 & 2.36 & 0.55 & $-$1.23 & $-$1.25 & $-$0.69 & 85.98\% \\
    Adpt $\tau{=}1,\kappa{=}2$ (cumul)   & 1.16 & 2.58 & 2.54 & 2.19 & 1.72 & 1.40 & 1.38 & 1.54 & 0.56 & $-$1.18 & $-$1.17 & $-$0.65 & 85.95\% \\
    Adpt $\tau{=}1,\kappa{=}4$ (cumul)   & 1.16 & 2.23 & 1.52 & 0.96 & 1.72 & 1.16 & 0.65 & 0.42 & 0.56 & $-$1.08 & $-$0.87 & $-$0.54 & 85.61\% \\
    Adpt $\tau{=}1,\kappa{=}0$ (prev)    & 1.15 & 2.97 & 3.22 & 3.01 & 1.71 & 1.77 & 2.02 & \textbf{2.43} & 0.56 & $-$1.20 & $-$1.20 & $-$0.58 & 85.92\% \\
    Adpt $\tau{=}1,\kappa{=}2$ (prev)    & 1.17 & 2.58 & 1.62 & 0.76 & 1.72 & 1.42 & 0.70 & 0.35 & 0.55 & $-$1.16 & $-$0.92 & $-$0.41 & 86.14\% \\
    \midrule
    LCFF ($\gamma{=}1.0$)     & 1.14 & 1.11 & 1.01 & 0.94 & 1.68 & 0.73 & 0.61 & 0.54 & 0.54 & $-$0.38 & $-$0.39 & $-$0.40 & 85.80\% \\
    \bottomrule
  \end{tabular}%
  }
\end{table}

\paragraph{Own vs.\ inherited goodness.}
Under constant $\gamma{=}0.7$, Block~3 produces only \textbf{14\%} of
its total goodness from its own computation
($g^+_\text{cur}{=}0.38$ of $g^+_\text{total}{=}2.73$); the remaining
86\% is inherited from earlier blocks. Under adaptive $\kappa{=}0$,
Block~3 produces \textbf{98\%} of its total goodness independently
($g^+_\text{cur}{=}2.36$ of $2.42$).
Figure~\ref{fig:goodness_decomposition} (App.~\ref{sec:gated_full}) visualizes this per layer; yet
this improved block health does not translate to higher accuracy
(S1 single-crop test $85.86\%$ vs.\ $86.18\%$; S2~TTA $87.27\%$ vs.\
$87.42\%$ for the constant baseline). High-threshold
gating ($\kappa{=}4$) recreates the constant-$\gamma$ collapse,
confirming that $\kappa$ controls a sharp phase transition
(Table~\ref{tab:gated}).

\paragraph{Cross-depth at $L{=}8$.} At $L{=}8$, constant $\gamma{=}0.7$
produces severe free-riding ($g^+_\text{cur}$ ratio L7/L0${=}0.32$):
per-block separation rises from $1.12$ at Block~0 to a peak of $2.53$
at Block~1 and then falls to $0.82$ at Block~7. Under $\gamma{=}0$, separation increases
through the middle blocks and then plateaus ($1.13\to3.04$ at Block~6,
$3.02$ at Block~7); adaptive $\kappa{=}0$ achieves
$\text{sep}^\text{cur}_\text{nl}{=}6.94$ at Block~7 (seed~42). Yet
accuracy barely moves: $\gamma{=}0$ reaches $87.21\%$ S2~TTA and
constant $\gamma{=}0.7$ $86.97\%$, while adaptive $\kappa{=}0$---with
$8{\times}$ healthier deepest-block diagnostics than constant
$\gamma{=}0.7$---reaches $87.12 \pm 0.35\%$ over three seeds ($86.84\%$
on seed~42; Table~\ref{tab:gated_k0_multiseed}). LCFF ($\gamma{=}1.0$)
exhibits the most extreme free-riding at $L{=}8$
($\text{sep}^\text{cur}_\text{nl}{=}0.15$ at Block~7), yet still achieves
$86.79\%$ TTA (seed~42)---within $0.42$\,pp of the best 16-expert
L8/D128 row ($\gamma{=}0$, $87.21\%$, also seed~42) and $0.33$\,pp below
the three-seed $\kappa{=}0$ mean. In the healthy adaptive runs at
$L{=}8$ (Table~\ref{tab:L8_gpos}), per-block $g^+_\text{cur}$ grows with
depth from Block~1 or~2 onward, reaching $5.38$--$6.34$ at Block~7
under the cumulative gate; it is strictly monotonic from Block~0 only
for the two cumulative $\kappa{=}0$ runs, and the prev-only
$\kappa{=}0$ run plateaus after Block~6 ($3.92\to3.91$). Under
$\gamma{=}0$ it dips at Block~1 ($1.67\to1.57$) and then rises only to
$1.98$ at Block~7. The exploratory $L{=}8$ depth-truncation
curves are nearly flat: for every adaptive run, the first block alone is
within $0.3$\,pp of the full-depth validation accuracy
(Table~\ref{tab:depth_truncation_L8}).  Across all 15 L4+L8 conditions
the free-riding ratio and Stage-2 validation accuracy have a descriptive
Pearson $r = -0.49$ that is carried by the depth contrast and is near zero
within each depth, and across seeds of a configuration deepest-block
separation at the validation-selected checkpoints shows no detectable
association with test accuracy (both exploratory;
Figure~\ref{fig:accuracy_vs_ratio} and Table~\ref{tab:e1_aligned}). The dissociation claim rests on the controlled
paired comparisons in Table~\ref{tab:dissociation_merged}, where the
three history rules vary deepest-block separation
${\sim}5{\times}$ but move accuracy by $<\!1$\,pp under every protocol.

\paragraph{Are individual blocks already strong enough?}
Three observations rule out the ``blocks are already saturated'' reading
of Table~\ref{tab:dissociation_merged}. First, Block~0 achieves only
$\text{sep}_\text{nl}{=}0.918$, missing ${\sim}8\%$ of examples; final
FF accuracy ($87\%$) is well below BP's $93.9\%$ on the same backbone,
so substantial room exists. Second, under adaptive $\kappa{=}0$, the
deepest block (Block~7) produces $45{\times}$ more per-block
discrimination than under LCFF ($\text{sep}^\text{cur}_\text{nl}$
$6.94$ vs.\ $0.15$) and incurs ${\sim}3.2{\times}$ its current-block
loss ($2.12$ vs.\ $0.67$), yet this extra
discrimination yields only $+0.05$\,pp over LCFF in the matched
seed-42 comparison ($86.84\%$ vs.\ $86.79\%$ S2~TTA at L8/D128;
$+0.33$\,pp for the three-seed $\kappa{=}0$ mean). The bottleneck is
therefore not per-block quality but the \emph{aggregation mechanism}:
scalar goodness summation cannot effectively leverage complementary
block-level features. Third, MoE expert utilization is uniform across
free-riding severity (effective experts ${\sim}8.7$--$10.0$ per layer),
so ``dead experts'' cannot explain the dissociation.

\paragraph{Does free-riding generalize to CIFAR-100?}
Yes. The $\gamma{=}0.7$ baseline (L4, D256, 362+20 epochs; 3 seeds
in Table~\ref{tab:cifar100_multiseed}) shows the signature (seed~42:
L3/L0 $g^+_\text{cur}$ ratio $0.232$, close to the CIFAR-10 D128 baseline's
$0.22$; App.~\ref{sec:cifar100_metrics}), and the dense $\gamma{=}0$ /
$\kappa{=}0$ / cumulative trio reproduces the dissociation
(Tables~\ref{tab:dissociation_merged} and~\ref{tab:dissociation_full}):
the variant with the healthiest per-block diagnostics ($\gamma{=}0$) is
not the one with the highest Stage-1 accuracy (cumulative). Free-riding
is therefore a real optimization pathology on CIFAR-100 as well, but, as
on CIFAR-10, not by itself the dominant accuracy bottleneck.

\paragraph{Paired-bootstrap evidence for the CIFAR-100 dissociation.}
The variants share the test order, so we pair per-example correctness
(matched seeds, $n{=}30{,}000$ predictions per pair, 5000-resample
bootstrap): four-protocol deltas in Table~\ref{tab:dissociation_merged},
disagreement rates and flip counts in App.~\ref{app:churn}. When a CI
is not fully inside $\pm 1$\,pp we do not claim formal equivalence.

\paragraph{Tiny ImageNet configuration.}
The Tiny ImageNet multi-seed results (Section~\ref{sec:tinyimagenet})
use the $\gamma{=}0$ configuration informed by the component ablation
(Appendix~\ref{app:components}): we drop MoE, depth-order, and SAM
(identified as redundant or harmful on CIFAR-10), and retain HNM and
multi-aspect goodness (load-bearing). Architecture is otherwise
identical to the CIFAR-10 configuration (D256, 8 heads, 4 blocks,
stem channels 384), except patch size~$4$ instead of~$2$ for the
$64{\times}64$ images (Table~\ref{tab:tiny_imagenet_config}). Notably, the Stage-2 attentive
head adds more on Tiny ImageNet than on CIFAR-100 and CIFAR-10, and
per-depth accuracy under $\gamma{=}0$ rises monotonically (values in
App.~\ref{sec:tiny_imagenet_metrics}).

\renewcommand{\subsection}{\origsubsection}%
\section{Configuration, Training Protocol, and Multi-Seed Tables}
\label{sec:full_config}

\subsection{Full Variant Configuration Table}

Table~\ref{tab:full_config} summarizes the key configuration differences among
all 13 variants relative to the baseline (\textsc{baseline}). All variants
share: $d{=}256$, $L{=}4$ blocks, 8 attention heads, $\text{lr}{=}10^{-3}$ (SAM+AdamW),
EMA decay 0.999, and identical data splits for a given seed (the
split is seeded with the run seed).

\begin{table}[htbp]
  \caption{Key configuration differences among all 13 variants. Cells show only
           values that differ from the default (32/4 experts, $\text{bs}{=}512$,
           362+20 epochs, $\gamma{=}0.7$, $\lambda_\text{curr}{=}0.5$,
           $\lambda_\text{con}{=}0.5$, memory enabled, no $L_2$ norm).}
  \label{tab:full_config}
  \centering
  \scriptsize
  \setlength{\tabcolsep}{3pt}
  \begin{tabular}{lcccccccc}
    \toprule
    Variant & Experts & bs & Epochs & $\gamma$ & $\lambda_\text{curr}$ (slope) & $\lambda_\text{con}$ & Mem & Special \\
    \midrule
    \textsc{baseline} & --- & --- & 287+20 & --- & --- & --- & \checkmark & Baseline \\
    \textsc{nomem-l2} & --- & --- & --- & --- & --- & --- & \texttimes & $L_2$ norm \\
    \textsc{topk-depth} & --- & --- & 287+20 & --- & sched (1.0) & --- & \checkmark & Top-$k$ NL \\
    \textsc{greedy-freeze} & --- & --- & 287+20 & --- & --- & --- & \checkmark & Greedy freeze \\
    \textsc{prefix-freeze} & --- & --- & 287+20 & --- & --- & --- & \checkmark & Prefix + triangle \\
    \textsc{fixed-nl} & --- & \textbf{128} & \textbf{460}+300 & \textbf{1.0} & 0.25 & --- & \checkmark & mine\_once \\
    \textsc{block0-label} & --- & \textbf{128} & \textbf{460}+300 & \textbf{1.0} & --- & --- & \checkmark & First-block labels \\
    \textsc{gamma-ramp} & --- & --- & 200+0$^\dagger$ & --- & depth-scaled & last-only & \checkmark & $\gamma$-ramp, multi-neg \\
    \textsc{cp-fair} & --- & --- & --- & --- & 0.25 (\textbf{3.0}) & --- & \texttimes & Prog.\ $\lambda_\text{curr}$ \\
    \textsc{cp-fair-lc} & --- & --- & --- & --- & 0.25 (\textbf{3.0}) & \textbf{0.15} & \texttimes & Low contrastive \\
    \textsc{dp-fair} & \textbf{24/3} & --- & 362+20$^\S$ & \textbf{prog} & prog (0.25) & dual & \texttimes & $L_2$ norm, 6 mech. \\
    \textsc{multi-tech} & \textbf{24/3} & --- & --- & \textbf{ramp} & 0.25 (\textbf{3.0}) & dual & \texttimes & $L_2$, triangle, 10 tech. \\
    $\gamma{=}0$$^\ddagger$ & --- & --- & --- & \textbf{0.0} & \textbf{0} / 0.25 (\textbf{3.0}) & --- & \texttimes & No accumulation \\
    \bottomrule
  \end{tabular}
  \vspace{-0.3em}
  {\footnotesize $^\dagger$Incomplete run (Stage~1 only). $^\S$Stage~2 ran as a
  separate 20-epoch continuation from the saved Stage-1 checkpoint.
  ``---'' = same as default.
  $^\ddagger$Per-seed configuration: seed~42 used
  $\lambda_\text{curr}{=}\lambda_\text{depth}{=}0$; seeds 123/456 used
  $\lambda_\text{curr}{=}0.25$ (slope 3.0) and $\lambda_\text{depth}{=}0.2$
  (Table~\ref{tab:gamma0_multiseed}).}
\end{table}

\subsection{Loss formulas, HNM, and augmentation}
\label{app:loss_formulas}

This subsection makes every loss term in Eq.~\ref{eq:total_loss}
explicit and pins the implementation conventions to mathematical
notation. Code-name $\leftrightarrow$ paper-name mapping is in
Tab.~\ref{tab:code_paper_map}.

\paragraph{Building blocks.}
Let $g_+^{(d)}$, $g_{nl}^{(d)}$, $g_{ni}^{(d)}$ denote the
current-block positive, wrong-label, and wrong-image goodness at
block $d$ (\S\ref{sec:method}). Each is the block's aggregate of its
aspect goodnesses, $g_\bullet^{(d)} = \sum_a \pi_a^{(d)}\, g_{\bullet,a}^{(d)}$,
with mixing weights $\pi^{(d)} = \mathrm{softmax}(\texttt{alpha\_raw}^{(d)})$
learned by block $d$. We use two different cumulative scores and
keep separate symbols for them.
The \emph{loss score} that enters the cumulative discrimination loss is
\begin{equation}
G_{\gamma,\bullet}^{(d)} \;=\; g_\bullet^{(d)} + \gamma \sum_{j<d} \mathrm{sg}\bigl(g_\bullet^{(j)}\bigr),
\qquad \bullet\in\{+,nl,ni\},
\label{eq:loss_score}
\end{equation}
where $\mathrm{sg}$ is stop-gradient: the earlier-block terms are
computed under \texttt{torch.no\_grad} and detached
(\texttt{\_sum\_prev\_goodness}), so only $g_\bullet^{(d)}$ depends on
block $d$'s parameters. This is a plain $\gamma$-weighted sum of the
previous blocks, not a recursion with geometric weights $\gamma^{d-j}$,
and matches $M_\gamma^{(d)} = m^{(d)} + \gamma P^{(d-1)}$ in
\S\ref{sec:method} (with $P^{(d-1)} = \sum_{j<d} m^{(j)}$). The
\emph{inference score} used for FF prediction is
$S_y^{(d)}(x) = \sum_{j\le d} g^{(j)}(x,y)$; it does not depend on
$\gamma$ and coincides with the loss score only at $\gamma{=}1$.
The scalar losses used below are
\begin{align}
\ell_{\mathrm{margin}}(g_+,g_-;\theta)
&= \mathbb{E}\,\softplus(\theta-g_+) + \mathbb{E}\,\softplus(g_- - \theta), \\
\ell_{\mathrm{sep}}(a,b;\alpha)
&= \mathbb{E}\,\softplus(-\alpha(a-b)), \\
\ell_{\mathrm{rank}}(a,b)
&= \ell_{\mathrm{sep}}(a,b;1),
\end{align}
where $\softplus(u)=\log(1+e^u)$ and $\mathbb{E}$ is the batch mean.
The code name \texttt{symba\_loss} corresponds to $\ell_{\mathrm{sep}}$
and is used as the cumulative discrimination loss; we write it
$\ell_{\mathrm{sep}}$ below to avoid the unexplained acronym. Its
sharpness $\alpha$ (\texttt{symba\_alpha}) is the barrier constant
written $\beta$ in \S\ref{sec:method} and App.~\ref{app:proofs}, so
$\alpha=\beta=4$; the same constant sets the temperature of the
residual weights below. All coefficients are listed in
Tab.~\ref{tab:loss_hparams}. Apart from the aspect weights
$\lambda_a$ inside $\mathcal{L}_{\mathrm{aspects}}$, the definitions
in this subsection contain no loss weights; the weights multiply the
terms only once, in Eq.~\ref{eq:total_loss}.

\paragraph{Block-cumulative discrimination loss
$\mathcal{L}_{\mathrm{block}}^{(d)}$.}
With wrong-label / wrong-image blend weight $\eta = \texttt{sff\_ni\_weight}$,
\begin{equation}
\boxed{
\mathcal{L}_{\mathrm{block}}^{(d)}
= (1-\eta)\,\ell_{\mathrm{sep}}(G_{\gamma,+}^{(d)},\,G_{\gamma,nl}^{(d)};\,\alpha)
+ \eta\,\ell_{\mathrm{sep}}(G_{\gamma,+}^{(d)},\,G_{\gamma,ni}^{(d)};\,\alpha).
}
\label{eq:Lblock}
\end{equation}
Eq.~\ref{eq:total_loss} multiplies this term by $\lambda_{\mathrm{block}}$.

\paragraph{Current-block residual loss
$\mathcal{L}_{\mathrm{curr}}^{(d)}$.}
For $d\ge 1$, with residual weights computed without gradient from
the detached accumulated margins
$P_{\bullet,i}^{(d-1)} = \sum_{j<d}\bigl(g_{+,i}^{(j)} - g_{\bullet,i}^{(j)}\bigr)$,
\begin{equation}
w_{\bullet,i}^{(d)}
= \frac{\sigma\!\bigl(-\alpha P_{\bullet,i}^{(d-1)}\bigr)}
       {\tfrac{1}{B}\sum_{k=1}^{B}\sigma\!\bigl(-\alpha P_{\bullet,k}^{(d-1)}\bigr)+\epsilon},
\qquad \bullet\in\{nl,ni\},\ \epsilon=10^{-6}
\end{equation}
(mean-normalized per batch and per negative type, not clipped;
Eq.~\ref{eq:residual_weights} and Theorem~\ref{thm:gradient_floor} in
App.~\ref{app:proofs}),
\begin{equation}
\mathcal{L}_{\mathrm{curr}}^{(d)}
= (1-\eta)\,\mathbb{E}\!\left[w_{nl}^{(d)}\,
   \softplus(-\alpha(g_+^{(d)}-g_{nl}^{(d)}))\right]
+ \eta\,\mathbb{E}\!\left[w_{ni}^{(d)}\,
   \softplus(-\alpha(g_+^{(d)}-g_{ni}^{(d)}))\right].
\end{equation}
At $d{=}0$ the weights are omitted ($w\equiv 1$).

\paragraph{Multi-aspect head.}
The default head has four aspect slots
$a\!\in\!\{\text{prototype},\text{energy},\text{attn-sharp},\text{learned}\}$
with per-aspect goodness $g_{\bullet,a}^{(d)}$ and weights
$\lambda_a$. With \texttt{use\_mem}=False (the default in this
paper's runs), the attn-sharpness slot evaluates identically to zero
because $\_$cross\_attend\_mem returns zero attention sharpness when
memory is disabled; the slot is retained for architectural
consistency but contributes no nonzero gradient, so three of the four
slots are active. The aspect head loss is
\begin{equation}
\mathcal{L}_{\mathrm{aspects}}^{(d)}
= \sum_a \lambda_a\!\left[
   (1-\eta)\,\ell_a(g_{+,a}^{(d)},g_{nl,a}^{(d)})
 + \eta\,\ell_a(g_{+,a}^{(d)},g_{ni,a}^{(d)})\right],
\end{equation}
with $\ell_a$ chosen per aspect (\texttt{aspect\_loss\_types}):
$\ell_{\mathrm{margin}}(\cdot,\cdot;\theta_a)$ for the prototype and
learned slots and $\ell_{\mathrm{rank}}$ for the energy and
attn-sharp slots. The thresholds
$\theta_a=\softplus(\theta_a^{\mathrm{raw}})$ are learned per block.
This is the only term whose weights $\lambda_a$ sit inside its
definition; Eq.~\ref{eq:total_loss} adds it without a further
coefficient.

\paragraph{Depth-order loss (as implemented).}
For $d\ge 1$, with tolerances $\delta_+,\delta_-$,
\begin{equation}
\mathcal{L}_{\mathrm{depth}}^{(d)}
= \mathbb{E}\,\softplus\bigl(\delta_+ - (G_{\gamma,+}^{(d)}-g_+^{(d-1)})\bigr)
 + \!\!\sum_{- \in \{nl, ni\}}\!\!
   \mathbb{E}\,\softplus\bigl(\delta_- - (g_-^{(d-1)}-G_{\gamma,-}^{(d)})\bigr),
\label{eq:Ldepth}
\end{equation}
and $\mathcal{L}_{\mathrm{depth}}^{(0)}=0$. The comparison is between
block $d$'s loss score $G_{\gamma}^{(d)}$ and block $d{-}1$'s
\emph{current-block} goodness $g^{(d-1)}$, not $G_{\gamma}^{(d-1)}$,
and the three terms are summed without the $\eta$ blend. The
idealized telescoping form of Proposition~\ref{prop:depthorder}
(App.~\ref{app:proofs}) is therefore a motivation for this term
rather than a statement about it. The goodness $g^{(d-1)}$ is recomputed from
block $d{-}1$'s detached aspect vector with block $d{-}1$'s mixing
weights $\pi^{(d-1)}$. In the trainers of the reported runs this ran
outside \texttt{torch.no\_grad}, so block $d$'s backward wrote a
gradient into block $d{-}1$'s four mixing logits, which was discarded
before it was ever applied; the released trainers compute it under
\texttt{torch.no\_grad}, with bit-identical parameters
(App.~\ref{app:locality}).

\paragraph{Reconstruction loss.}
\begin{equation}
\mathcal{L}_{\mathrm{recon}}^{(d)}
= \mathrm{MSE}\bigl(r_1, \mathrm{sg}(t_1)\bigr) + \mathrm{MSE}\bigl(r_2, \mathrm{sg}(t_2)\bigr),
\end{equation}
where $t_v$ ($v\in\{1,2\}$) is block $d$'s input token sequence for
the $v$-th augmented positive view (the patch-embedding output at
$d{=}0$; the $L_2$-normalized, detached output tokens of block $d{-}1$
for $d\ge 1$), $r_v$ is block $d$'s reconstruction decoder applied
to its post-MoE hidden tokens for that view, $\mathrm{sg}$ is
stop-gradient, and $\mathrm{MSE}$ averages over tokens and channels.

\paragraph{MoE auxiliary losses.}
For top-$k$ routing over $E$ experts with token-fraction $f_i$ and
mean routing probability $P_i$ (\texttt{moe\_balance\_type}=switch,
\texttt{moe\_zloss\_type}=stmoe in the bundled CIFAR-10 trainer):
\begin{equation}
\mathcal{L}_{\mathrm{bal}} = E\sum_{i=1}^E f_i\,P_i,\qquad
\mathcal{L}_{z} = \mathbb{E}_t\!\left[\big(\log\textstyle\sum_i e^{z_{t,i}}\big)^2\right].
\end{equation}

\begin{table}[htbp]
  \caption{Loss coefficients of the published CIFAR-10 \textsc{cp-fair}
           run (seed~42; \texttt{config.json} of the run, identical to
           the trainer defaults). Other runs change some of these
           coefficients where their tables say so (e.g.,
           Tab.~\ref{tab:full_config}; the $\gamma{=}0$ seed-42 run uses
           $\lambda_0{=}\lambda_{\mathrm{depth}}{=}0$).
           Each coefficient multiplies its term once, in
           Eq.~\ref{eq:total_loss}; the aspect weights $\lambda_a$ sit
           inside $\mathcal{L}_{\mathrm{aspects}}$.}
  \label{tab:loss_hparams}
  \centering
  \footnotesize
  \setlength{\tabcolsep}{4pt}
  \begin{tabular}{llll}
    \toprule
    Symbol & Code name & Value & Role \\
    \midrule
    $\alpha{=}\beta$ & \texttt{symba\_alpha} & $4$ & barrier sharpness; residual-weight temperature \\
    $\gamma$ & \texttt{gamma\_scale} & $0.7$ & history weight in $G_{\gamma}^{(d)}$ \\
    $\eta$ & \texttt{sff\_ni\_weight} & $0.5$ & wrong-label / wrong-image blend \\
    $\lambda_{\mathrm{block}}$ & \texttt{block\_lambda} & $1.0$ & weight of $\mathcal{L}_{\mathrm{block}}$ \\
    $\lambda_0$, $\rho$ & \texttt{block\_curr\_lambda}, \texttt{\_depth\_slope} & $0.25$, $3.0$ & $\lambda_{\mathrm{curr}}(d)=0.25,0.50,0.75,1.00$ \\
    $\lambda_{\mathrm{depth}}$ & \texttt{depth\_order\_lambda} & $0.2$ & weight of $\mathcal{L}_{\mathrm{depth}}$ \\
    $\delta_+$, $\delta_-$ & \texttt{depth\_margin\_pos/neg} & $0.05$, $0.1$ & depth-order tolerances \\
    $\lambda_a$ & \texttt{aspect\_lambdas} & $1.5, 0.2, 0.3, 0.3$ & prototype, energy, attn-sharp, learned \\
    $\theta_a$ (init) & \texttt{aspect\_theta\_inits} & $5.0, 2.5, 0.5, 1.0$ & used by the two margin slots (5.0, 1.0) \\
    $\lambda_{\mathrm{con}}$, $\tau$ & \texttt{contrastive\_lambda/\_temp} & $0.5$, $0.2$ & SupCon weight, temperature \\
    $\lambda_{\mathrm{recon}}$ & \texttt{recon\_lambda} & $0.05$ & weight of $\mathcal{L}_{\mathrm{recon}}$ \\
    $\lambda_{\mathrm{bal}}$, $\lambda_z$ & \texttt{moe\_balance/zloss\_coef} & $0.01$, $0.001$ & MoE router regularizers (no ramp) \\
    $\epsilon$ & (hard-coded) & $10^{-6}$ & residual-weight normalizer \\
    \bottomrule
  \end{tabular}
\end{table}

\paragraph{Hard-negative mining (HNM).}
$k$ denotes the number of \emph{candidate wrong-label draws}
sampled \emph{with replacement} from $\{0,\dots,C{-}1\}\setminus\{y\}$,
not the number of unique wrong labels (Algorithm~\ref{alg:hnm}). On
CIFAR-10 ($C{=}10$) the ramp $k{=}8\!\to\!16$ may include duplicates
(this is by design); on Tiny ImageNet ($C{=}200$) the ramp is
$k{=}20\!\to\!40$. Expected unique-label coverage is
$1-(1-1/(C-1))^k$: $0.61$--$0.85$ on CIFAR-10 but only
$0.096$--$0.18$ on Tiny ImageNet, so the Tiny ramp is not
coverage-matched to CIFAR-10.

\begin{algorithm}[h]
\caption{Hard-negative mining for one example
(\texttt{\_choose\_hard\_negatives} in the bundled CIFAR-10
trainer).}
\label{alg:hnm}
\begin{algorithmic}[1]
\Require image $x$, true label $y$, class count $C$, candidate
count $k$, EMA teacher $T$
\State $S \gets \emptyset$
\For{$j = 1,\dots,k$}
  \State sample $y_j \sim \mathrm{Uniform}\bigl(\{0,\dots,C{-}1\}\setminus\{y\}\bigr)$
         \Comment{with replacement; duplicates allowed}
  \State $s_j \gets T.\textsc{block-goodness}(x, y_j)$
  \State $S \gets S \cup \{(y_j, s_j)\}$
\EndFor
\State \Return $\arg\max_{(y_j, s_j)\in S}\;s_j$
   \Comment{candidate with highest wrong-label goodness}
\end{algorithmic}
\end{algorithm}

\paragraph{Augmentation recipes (CIFAR-10).}
Both ``FF-pipeline'' and ``BP-strong'' rows of
Tab.~\ref{tab:bp_ff_ablation} use \emph{strong} augmentation, but the
recipes are not identical (Tab.~\ref{tab:aug_recipes}); neither uses
Mixup/CutMix in Stage-1, since the FF objective requires discrete
positive and negative labels (\texttt{use\_mixup\_cutmix=False} in the
bundled CIFAR-10 trainer; the BP-strong control likewise omits it).
Stage-2 (frozen-feature attentive readout) optionally enables
Mixup/CutMix when \texttt{use\_mixup\_cutmix=True}.

\begin{table}[htbp]
  \caption{Augmentation recipes used for the FF and BP-strong rows of
           Tab.~\ref{tab:bp_ff_ablation} (CIFAR-10 Stage-1).
           Mixup/CutMix are disabled in both. The recipes differ
           primarily in crop aggressiveness and operator
           application rate, so the FF-vs-BP rows are an
           FF-pipeline vs.\ BP-pipeline comparison rather than an
           identical-augmentation control.}
  \label{tab:aug_recipes}
  \centering
  \footnotesize
  \setlength{\tabcolsep}{4pt}
  \begin{tabular}{lll}
    \toprule
    Component & FF-pipeline (Stage-1) & BP-strong control \\
    \midrule
    RandomResizedCrop scale     & $(0.3, 1.0)$ & $(0.8, 1.0)$ \\
    RandomHorizontalFlip        & yes & yes \\
    RandAugment $(N, M)$        & $(2, 8\!-\!10)$ & $(2, 9)$ \\
    ColorJitter                 & $p{=}0.8$, jitter $0.4$ & always, jitter $0.4$ \\
    GaussianBlur $\sigma\!\in\![0.1,2.0]$ & $p{=}0.5$ & always \\
    Mixup / CutMix              & disabled & disabled \\
    Normalization               & CIFAR-10 mean/std & CIFAR-10 mean/std \\
    \bottomrule
  \end{tabular}
\end{table}

\begin{table}[htbp]
  \caption{Code-name to paper-notation map for the released image
           trainers. The code names are configuration fields
           (\texttt{FFConfig}) of \texttt{trainers/cp\_fair\_*.py}; the
           \texttt{FF\_*} environment variables documented in the code
           README override a subset ($\gamma$, $\lambda_0$,
           $\lambda_{\mathrm{depth}}$, $k_0/k_L$, the aspect count, the
           gate, MGC, barrier and block-0 switches). The upper
           block is shared by all trainers; the
           hardness gate and MGC exist in the CIFAR-100 and hardness-gated
           trainers, the barrier and block-0 switches only in the latter, and
           every camera-ready switch is off by default.}
  \label{tab:code_paper_map}
  \centering
  \footnotesize
  \begin{tabular}{lll}
    \toprule
    Code symbol & Paper notation & Description \\
    \midrule
    \texttt{symba\_loss}, \texttt{symba\_alpha}
       & $\ell_{\mathrm{sep}}$, $\alpha\,(=\!\beta)$ & Cumulative discrimination loss \\
    \texttt{block\_lambda}            & $\lambda_{\mathrm{block}}$ & Eq.~\ref{eq:Lblock} weight \\
    \texttt{block\_curr\_lambda}      & $\lambda_0$  & Base weight of $\mathcal{L}_{\mathrm{curr}}$ \\
    \texttt{block\_curr\_depth\_slope}& $\rho$ & Depth scaling of $\lambda_{\mathrm{curr}}$ \\
    \texttt{block\_theta\_init}       & --- & Logged block threshold; unused by any loss \\
    \texttt{sff\_ni\_weight}          & $\eta$ & Wrong-label / wrong-image blend \\
    \texttt{depth\_order\_lambda}     & $\lambda_{\mathrm{depth}}$ & Eq.~\ref{eq:Ldepth} weight \\
    \texttt{depth\_margin\_pos/neg}   & $\delta_+,\delta_-$ & Depth-order tolerances \\
    \texttt{recon\_lambda}            & $\lambda_{\mathrm{recon}}$ & Reconstruction weight \\
    \texttt{moe\_balance\_coef}       & $\lambda_{\mathrm{bal}}$ & Switch-MoE balance weight \\
    \texttt{moe\_zloss\_coef}         & $\lambda_z$ & ST-MoE z-loss weight \\
    \texttt{hard\_negative\_k\_first/last} & $k_0,k_L$ & HNM candidate-draws ramp \\
    \texttt{n\_aspects}, \texttt{aspect\_lambdas} & $\{a\},\{\lambda_a\}$ & Multi-aspect head \\
    \texttt{gamma\_scale}             & $\gamma$ & Constant cumulative weight \\
    \midrule
    \texttt{gamma\_gating\_mode}      & --- & Gate switch (\texttt{constant} / \texttt{adaptive}) \\
    \texttt{gamma\_0}, \texttt{gamma\_tau}, \texttt{gamma\_kappa}
       & $\gamma_0,\tau,\kappa$ & Gate $\gamma_0\,\sigma(\tau(\kappa{-}\sum_{j<d}g^{(j)}))$ \\
    \texttt{use\_mgc}, \texttt{mgc\_c0}, \texttt{mgc\_rho}
       & $c_d{=}c_0(1{+}\rho_c\tfrac{d}{L-1})$ & MGC target (run: $c_d{=}1$) \\
    \texttt{barrier\_type}, \texttt{barrier\_margin}
       & $\ell$, $\theta$ & Softplus or squared hinge ($\theta{=}2$) \\
    \texttt{mgc\_keep\_l0\_aux}, \texttt{curr\_aux\_off}
       & --- & Control arms (c) and (b) of App.~\ref{app:f1_control} \\
    \bottomrule
  \end{tabular}
\end{table}

\subsection{Training protocol and ablation hierarchy}
\label{app:training_protocol}

\paragraph{Training protocol.}
\textbf{Stage~1} (362 epochs on CIFAR-10): all four blocks are trained
jointly with the loss of Eq.~\ref{eq:total_loss} using
SAM~\citep{foret2021sam} + AdamW~\citep{loshchilov2019adamw}
($\eta_\text{max}{=}10^{-3}$, cosine schedule, warmup 15 epochs, batch
size 512; RandAugment~\citep{cubuk2020randaugment}, no
Mixup~\citep{zhang2018mixup}/CutMix~\citep{yun2019cutmix} in Stage-1
since FF needs discrete labels). Each image is paired with a
hard-mined incorrect label: we draw $k$ candidate wrong labels
\emph{with replacement} from $\{0,\dots,C{-}1\}\setminus\{y\}$, score
them with an EMA teacher, and keep the hardest ($k$ counts
\emph{candidate draws}, not unique labels; CIFAR-10 ramps
$8\!\to\!16$, Tiny $20\!\to\!40$; Algorithm~\ref{alg:hnm}).
\textbf{Stage~2} (20 epochs) trains a separate attentive head on
frozen backbone features; across 12 seed-42 CIFAR-10 variants (three
with a 300-epoch Stage~2) it changes TTA test accuracy by only
$+0.03$\,pp on average and can be omitted without significant loss,
though it adds $+2.3$\,pp on CIFAR-100 and $+3.2$\,pp on Tiny ImageNet
($\gamma{=}0$, 3-seed means). Settings that deviate from this recipe
(the dense, SAM-free CIFAR-100 trio and Tiny ImageNet runs; the
180-epoch L4/D128 and L8/D128 sweeps) are stated with their tables.

\paragraph{Ablation hierarchy.}
We separate three experimental questions. (i)~The \emph{causal
free-riding test} varies only how much cumulative history enters the
discrimination objective (history-free $\gamma{=}0$, adaptive/gated
history, or cumulative $\gamma{>}0$); architecture, hard-negative
mining, augmentation, and the Stage-1 schedule are otherwise held
fixed within each dataset-specific sweep. (ii)~\emph{Architecture and
augmentation calibrations} (Tab.~\ref{tab:bp_ff_ablation}) compare
co-designed vs.\ plain backbones and weak vs.\ BP-strong
augmentation; these are calibrations for effect-size, not the basis
of the free-riding causal claim. (iii)~\emph{Component and contextual
ablations} (App.~\ref{sec:hnm_deep_dive},~\ref{app:components}) report
leave-one-out and aspect-isolation sweeps separately and are not used
to support the free-riding causal claim.

\subsection{Multi-seed stability tables}
\label{sec:multiseed_supplement}

Tables~\ref{tab:multiseed},~\ref{tab:gamma0_multiseed},
and~\ref{tab:gated_k0_multiseed} are the multi-seed extensions of the
single-seed CIFAR-10 entries in Section~\ref{sec:experiments};
Table~\ref{tab:cifar100_dissociation_multiseed} (CIFAR-100 trio) and
Table~\ref{tab:tiny_imagenet_multiseed} (Tiny ImageNet) complete the
multi-seed evidence. The main paper carries only the headline
mean$\pm$SD for each variant; per-seed numbers and
standard-deviation breakdowns are here (per-run SAM settings of the
hardness-gated runs: Table~\ref{tab:gated_k0_multiseed}).

\begin{table}[htbp]
  \caption{Multi-seed stability of \textsc{cp-fair}. All runs use identical
           architecture and hyperparameters (D256, 362+20 epochs).}
  \label{tab:multiseed}
  \centering
  \small
  \begin{tabular}{lcccc}
    \toprule
    Seed & S1 (no-TTA) & S1 (TTA) & S2 (no-TTA) & S2 (TTA) \\
    \midrule
    42  & 90.45\% & 91.36\% & 90.42\% & 91.27\% \\
    123 & 90.45\% & 91.34\% & 90.39\% & 91.21\% \\
    456 & 90.43\% & 91.39\% & 90.59\% & 91.44\% \\
    789 & 90.09\% & 91.22\% & 90.19\% & 91.19\% \\
    \midrule
    Mean $\pm$ Std & $90.36 \pm 0.18$ & $91.33 \pm 0.07$ & $90.40 \pm 0.16$ & $91.28 \pm 0.11$ \\
    \bottomrule
  \end{tabular}
\end{table}

\begin{table}[h]
  \caption{Multi-seed stability of $\gamma{=}0$ (history-free). All runs share the architecture, optimizer
           (SAM + AdamW) and schedule (L4, D256, 32 experts top-4, bs512, 362+20 epochs), but the published
           seed~42 used $\lambda_\text{curr}{=}\lambda_\text{depth}{=}0$, whereas seeds 123 and 456 used
           $\lambda_\text{curr}{=}0.25$ and $\lambda_\text{depth}{=}0.2$ (all other settings are identical),
           so the published $n{=}3$ pools two configurations; it remains the headline.
           $^\ast$Seed~42 rerun in the seed-123/456 configuration; the matched $n{=}3$ is reported next to
           the published one. All accuracies come from each run's final official-test evaluation. sep:
           deepest-block $\text{sep}^\text{cur}_\text{nl}$ at the last Stage-1 epoch and at the
           validation-selected epoch (in parentheses; as in App.~Tab.~\ref{tab:e1_aligned}). GPU: Colab A100 or DGX Spark GB10, so both $n{=}3$
           SDs mix seed and hardware variation.}
  \label{tab:gamma0_multiseed}
  \centering
  \small
  \setlength{\tabcolsep}{3pt}
  \resizebox{\textwidth}{!}{%
  \begin{tabular}{lcccccccc}
    \toprule
    Seed & $\lambda_\text{curr}$ / $\lambda_\text{depth}$ & GPU & S1 (no-TTA) & S1 (TTA) & S2 (no-TTA) & S2 (TTA) & sep (last) & sep (sel.) \\
    \midrule
    42 (published) & 0 / 0 & A100 & 90.52\% & 91.45\% & 90.37\% & 91.39\% & 3.28 & 3.22 (320) \\
    42 (rerun)$^\ast$ & 0.25 / 0.2 & A100 & 90.58\% & 91.39\% & 90.62\% & 91.31\% & 7.72 & 3.69 (200) \\
    123 & 0.25 / 0.2 & GB10 & 90.54\% & 91.40\% & 90.51\% & 91.51\% & 7.85 & 5.65 (260) \\
    456 & 0.25 / 0.2 & GB10 & 90.44\% & 91.10\% & 90.44\% & 91.22\% & 7.97 & 6.41 (260) \\
    \midrule
    Published $n{=}3$ (42, 123, 456) & mixed & mixed & $90.50 \pm 0.05$ & $91.32 \pm 0.19$ & $90.44 \pm 0.07$ & $91.37 \pm 0.15$ & --- & --- \\
    Matched $n{=}3$ (42 rerun, 123, 456) & 0.25 / 0.2 & mixed & $90.52 \pm 0.07$ & $91.30 \pm 0.17$ & $90.52 \pm 0.09$ & $91.35 \pm 0.15$ & $7.85 \pm 0.13$ & $5.25 \pm 1.41$ \\
    \bottomrule
  \end{tabular}%
  }

  \vspace{1ex}
  \begin{flushleft}\footnotesize
  Because the configurations differ, the separation gap between the published seed~42 and seeds 123/456
  should not be read as seed-to-seed variance. The rerun (one seed) supports this: in the seed-123/456
  configuration, seed~42 reaches $7.72$ at the last epoch (published:
  $3.28$), while its accuracy moves by $-0.08$ to $+0.25$\,pp
  (paired exact McNemar $p{\ge}0.35$). It does not separate $\lambda_\text{curr}$ from
  $\lambda_\text{depth}$, and the published run used an older software stack (PyTorch 2.10 vs.\
  2.11) and trainer, checked only through its configuration file. Registered before the rerun,
  S1~TTA within the seed-123/456 range $[91.10, 91.40]$ (by one test image), S1~TTA within $\pm0.30$\,pp
  of their mean, and a last-epoch separation in $[6.9, 8.9]$ all pass. At its validation-selected epoch
  (200, vs.\ 260 for seeds 123/456; separation was still rising) the rerun's separation is
  $3.69$, close to the published run's $3.22$; at the common
  evaluation epochs from 200 on it lies $-0.80$ to $+0.10$ from the seed-123/456 mean.
  The rerun trained on a Colab A100 in three segments (Stage-1 epochs 1--3, 4--274,
  and 275--362 plus Stage~2; an earlier attempt ended before its first epoch) with two resumes
  that restored the random-number state but re-seeded the data-loader workers, so it is not bitwise
  equal to an uninterrupted run; epoch~4 was trained twice (the second segment resumed from the epoch-3 checkpoint).
  Seed~42's Stage-2 values come from its run log (correction: App.~\ref{app:changes}, item B4).
  \end{flushleft}
\end{table}

\begin{table}[htbp]
  \caption{Multi-seed stability of hardness-gated $\kappa{=}0$ at the
           L4/D128 and L8/D128 reduced scales ($\gamma_0{=}0.7$,
           $\tau{=}1.0$, $M$-mode cumulative; 180+10 epochs, 16 experts
           top-2). The SAM column gives each run's optimizer setting:
           at L4/D128 seed~42 was trained \emph{without} SAM and seeds
           123/456 \emph{with} SAM, so the L4/D128 aggregate mixes two
           optimizer settings; all three L8/D128 runs are without SAM.
           Seed-42 rows are the single-seed runs of the hardness-gated
           sweep (App.~\ref{app:gated_full}); seeds 123 and 456 are new
           runs after fixing an environment-variable seed-override bug
           in the gated trainer. Accuracies are
           official-test values from each run's final evaluation, except
           the L8/D128 seed-42 row ($^\ast$).}
  \label{tab:gated_k0_multiseed}
  \centering
  \small
  \setlength{\tabcolsep}{3pt}
  \resizebox{\textwidth}{!}{%
  \begin{tabular}{llccccc}
    \toprule
    Scale & Seed & SAM & S1 (no-TTA) & S1 (TTA) & S2 (no-TTA) & S2 (TTA) \\
    \midrule
    \multirow{4}{*}{L4, D128}
        & 42  & off & 85.86\% & 86.81\% & 86.35\% & 87.27\% \\
        & 123 & on  & 85.87\% & 86.91\% & 86.03\% & 87.23\% \\
        & 456 & on  & 85.55\% & 86.62\% & 85.82\% & 86.96\% \\
        & Mean $\pm$ Std (n=3) & mixed & $85.76 \pm 0.18$ & $86.78 \pm 0.15$ & $86.07 \pm 0.27$ & $87.15 \pm 0.17$ \\
    \midrule
    \multirow{4}{*}{L8, D128}
        & 42$^\ast$ & off & 85.66\% & 86.65\% & 86.33\% & 86.84\% \\
        & 123 & off & 85.94\% & 87.12\% & 86.47\% & 87.51\% \\
        & 456 & off & 85.68\% & 86.84\% & 85.92\% & 87.01\% \\
        & Mean $\pm$ Std (n=3) & off & $85.76 \pm 0.16$ & $86.87 \pm 0.24$ & $86.24 \pm 0.29$ & $87.12 \pm 0.35$ \\
    \bottomrule
  \end{tabular}%
  }

  \vspace{1ex}
  \begin{flushleft}\footnotesize
  $^\ast$Evaluated post hoc from its selected Stage-1 and Stage-2
  checkpoints on CPU (FP32; the training-time evaluation used TF32),
  because this run logged no final test evaluation. Recomputed
  validation accuracy reproduces the logged values only to within 2 of
  5000 examples (S1~TTA $87.00\%$ vs.\ logged $86.96\%$, one example
  beyond the 1-example tolerance we set in advance; S2~TTA $86.90\%$
  vs.\ $86.88\%$), so the seed-42 test values carry about
  $\pm 0.03$\,pp of numerical uncertainty. All $\pm$ values here are
  sample standard deviations (corrections: App.~\ref{app:changes},
  items A6, B6, and B8).
  \par\smallskip
  \textbf{L4/D128} per-block diagnostics at end of Stage~1 (deepest layer L3):
  $\text{sep}^\text{cur}_\text{nl}$ = $3.11$ (seed~42, no SAM), $1.94$ (seed~123),
  $1.83$ (seed~456); mean $\pm$ std $= 2.29 \pm 0.71$.
  \textbf{L8/D128} (deepest layer L7): $\text{sep}^\text{cur}_\text{nl}$ =
  $6.94$ (seed~42), $6.70$ (seed~123), $6.30$ (seed~456); mean
  $\pm$ std $= 6.65 \pm 0.33$. The larger L4/D128 spread comes from
  seed~42, the only L4/D128 run without SAM, so it cannot be attributed
  to seed variance alone; at L8/D128, where all three seeds share the
  optimizer setting, the spread is $0.33$. All three L8 runs land far
  above the constant-$\gamma{=}0.7$ baseline (sep $0.82$) and the LCFF
  baseline (sep $0.15$)---preserving the dissociation ranking in
  Table~\ref{tab:dissociation_full}.
  \end{flushleft}
\end{table}

\paragraph{CIFAR-100 dissociation contrast.}
We additionally run the full $\gamma{=}0$ vs.\ gated-$\kappa{=}0$ vs.\
cumulative ($\gamma{>}0$) contrast on CIFAR-100 to test the dissociation
finding (Section~\ref{sec:adaptive_gamma}) on harder data; the
four-protocol paired comparison, including MGC, is the main-text
Table~\ref{tab:dissociation_merged}. All three
variants share the L4, D256, bs256, 362+20-epoch protocol; only the
collaboration rule differs.\footnote{These runs use a refined trainer
that yields slightly higher accuracy than the older $\gamma{=}0.7$
baseline in App.~Table~\ref{tab:cifar100_multiseed}; we therefore report them
as a separate evidence stream rather than replacing the earlier numbers.
Free-riding diagnostics (L3/L0 goodness ratio) are consistent across
both trainers ($0.207$ here vs.\ $0.232$ in Table~\ref{tab:cifar100_metrics}).}
Table~\ref{tab:cifar100_dissociation_multiseed} is its companion: it
adds the positive-goodness ratio, top-5 accuracy, and the resumed-run
record for the same nine runs.

\begin{table}[htbp]
  \caption{Companion to Table~\ref{tab:dissociation_merged} (CIFAR-100,
           L4, D256, bs256, 362+20 epochs, seeds 42/123/456), which gives
           the deepest-block separation, all four top-1 protocols, and the
           paired CIs; this table adds only what Table~\ref{tab:dissociation_merged}
           lacks.
           $g^+_\text{cur}$ ratio: end-of-Stage-1 deepest-to-first-block
           positive goodness, mean of the per-seed ratios. Top-5: Stage-2
           with TTA, mean $\pm$ sample SD.
           \emph{Resumed runs.} The last column gives the first Stage-1
           epoch of each seed's layer log (42/123/456); runs whose log
           starts after epoch~1 ($\gamma{=}0$ seeds 123/456, $\kappa{=}0$
           seed 456, cumulative seeds 123/456) were resumed from a
           checkpoint after an interrupted session, and their final
           evaluations are complete. The trainer did not checkpoint RNG
           state, so a resumed run is not a bitwise continuation of an
           uninterrupted one.
           \emph{Trainer protocol.} The three rows use the refined
           dense (no-MoE) L4/D256 trainer recommended by the CIFAR-10
           component ablation (\texttt{FF\_N\_EXPERTS=1}
           \texttt{FF\_MOE\_TOP\_K=1} \texttt{FF\_USE\_SAM=0});
           the older 32-expert $\gamma{=}0.7$ baseline is reported
           separately in
           Table~\ref{tab:cifar100_multiseed}. Reproduction commands
           are in \texttt{repro\_manifest.csv} in the code repository.}
  \label{tab:cifar100_dissociation_multiseed}
  \centering
  \small
  \setlength{\tabcolsep}{4pt}
  \begin{tabular}{lccc}
    \toprule
    Variant
      & $g^+_\text{cur}$ ratio L3/L0
      & S2 TTA top-5
      & First logged epoch \\
    \midrule
    $\gamma{=}0$ (history-free)
      & $0.988$ & $90.78 \pm 0.21$ & 1/85/285 \\
    Gated $\kappa{=}0$
      & $0.463$ & $\mathbf{90.89 \pm 0.31}$ & 1/1/208 \\
    Cumulative ($\gamma{=}0.7$)
      & $0.207$ & $90.68 \pm 0.33$ & 1/70/274 \\
    \midrule
    Range across regimes
      & $4.77{\times}$ & $0.20$\,pp & --- \\
    \bottomrule
  \end{tabular}
\end{table}

\paragraph{Tiny ImageNet.} Per-seed validation accuracies of the
$\gamma{=}0$ configuration (App.~\ref{sec:tiny_imagenet_metrics};
labeled validation split, since the test labels are not public).

\begin{table}[htbp]
  \caption{Multi-seed Tiny ImageNet validation accuracy with $\gamma{=}0$
           (L4, D256). Mean and standard deviation are taken across three
           seeds (42, 123, 456). Single-crop Stage-1 top-1 (no TTA) is
           $48.75$ / $48.45$ / $48.41\%$ ($48.54 \pm 0.19$). The SCFF
           line is a published number from an unsupervised protocol
           without TTA; it is listed for context and is not comparable
           to any column of this table.}
  \label{tab:tiny_imagenet_multiseed}
  \centering
  \small
  \begin{tabular}{lcccc}
    \toprule
    Seed & S1 Top-1 (TTA) & S1 Top-5 (TTA) & S2 Top-1 (TTA) & S2 Top-5 (TTA) \\
    \midrule
    42  & 49.44\% & 76.39\% & 52.60\% & 77.10\% \\
    123 & 49.16\% & 75.72\% & 52.42\% & 76.92\% \\
    456 & 48.67\% & 75.71\% & 51.95\% & 77.15\% \\
    \midrule
    Mean $\pm$ Std & $49.09 \pm 0.39$ & $75.94 \pm 0.39$ & $\mathbf{52.32 \pm 0.34}$ & $77.06 \pm 0.12$ \\
    \midrule
    \multicolumn{5}{l}{\emph{Context only (different protocol, not in the columns above):}} \\
    \multicolumn{5}{l}{SCFF~\citep{scff2025}, unsupervised training, as reported: top-1 $35.67 \pm 0.42\%$, top-5 $59.75 \pm 0.18\%$} \\
    \bottomrule
  \end{tabular}
\end{table}

\renewcommand{\subsection}{\FloatBarrier\origsubsection}
\section{Architectural Components: Full Leave-One-Out Table}
\label{app:components}

\begin{table}[htbp]
  \caption{Component ablation on the full CP-FAIR baseline
           ($\gamma{=}0.7$, D256, 362 epochs). Each row removes one
           component; $\Delta$ is the change in Stage-2 TTA accuracy
           relative to the full model ($91.24\%$). Positive $\Delta$
           means the component \emph{hurts} performance. The
           ``Multi-aspect $\to$ single aspect'' row is a
           leave-one-out (LOO) intervention that drops the head from
           four slots (three active) to one, the prototype-alignment
           slot (\texttt{n\_aspects}${=}1$ keeps the first slot); this is \emph{not}
           comparable to the per-aspect isolation rows in
           Tab.~\ref{tab:aspect_ema_ablation}, which retain each
           aspect at its original $\lambda_a$ while zeroing the
           other $\lambda$'s. The two ablation families therefore
           probe different things and their numbers cannot be summed.}
  \label{tab:component_ablation}
  \centering
  \small
  \begin{tabular}{lcccc}
    \toprule
    Ablated Component & S1 TTA & S2 (no-TTA) & S2 (TTA) & $\Delta$ (pp) \\
    \midrule
    None (full CP-FAIR baseline) & 91.31\% & 90.42\% & 91.24\% & --- \\
    \midrule
    SAM optimizer $\to$ plain AdamW         & 91.29\% & 90.77\% & 91.49\% & $+0.25$ \\
    SupCon loss ($\lambda_\text{con}{=}0$)  & 91.43\% & 90.78\% & 91.49\% & $+0.25$ \\
    $\gamma$ accumulation ($\gamma{=}0$)    & 91.39\% & 90.71\% & 91.49\% & $+0.25$ \\
    MoE $\to$ dense MLP                     & 91.58\% & 90.74\% & 91.41\% & $+0.17$ \\
    Block-curr loss                         & 91.35\% & 90.63\% & 91.40\% & $+0.16$ \\
    Depth-order loss                        & 91.43\% & 90.58\% & 91.24\% & $\phantom{+}0.00$ \\
    Multi-aspect $\to$ single aspect        & 90.87\% & 90.06\% & 90.83\% & $-0.41$ \\
    Hard negative mining $\to$ random neg.  & 90.76\% & 89.93\% & 90.79\% & $-0.45$ \\
    \bottomrule
  \end{tabular}
\end{table}

Appendix~\ref{sec:hnm_deep_dive} decomposes HNM further (mild hardness
$k{=}4$ outperforms the adaptive schedule by $+0.19$\,pp; wrong-label
negatives carry 100\% of the signal). Appendix~\ref{sec:aspect_ema_deep_dive}
decomposes multi-aspect goodness (the energy aspect alone surpasses the
4-slot, 3-active model by $+0.27$\,pp; EMA smoothing contributes a genuine
$-0.34$\,pp).

\section{Locality Audit of the Stage-1 Training Loop}
\label{app:locality}

\paragraph{Canonical block-local implementation.} The released code
uses the graph-severed implementation as the canonical FF training
loop. Concretely, the function that aggregates previous-block
goodness (\texttt{\_sum\_prev\_goodness}) is wrapped in
\texttt{torch.no\_grad()} with a final \texttt{.detach()} on its
output before it enters any block-$d$ loss. Therefore, the
cumulative-goodness path of the block-$d$ loss contains parameters of
block $d$ only. A small verification harness
(\texttt{analysis/verify\_locality.py} in the code repository)
asserts that this path produces no non-zero gradient on
\texttt{blocks[$0$\ldots$\ell{-}1$].}\allowbreak\texttt{alpha\_raw}; we
run this harness against all four image-domain trainers (CIFAR-10,
CIFAR-100, Tiny~ImageNet, hardness-gated).

\paragraph{Exception: the depth-order term (reported runs).} When
$\lambda_\text{depth}{>}0$, the depth-order loss
(Eq.~\ref{eq:Ldepth}) recomputes block $d{-}1$'s current-block
goodness with \texttt{net.blocks[d-1].aggregate\_goodness(...)}. In
the trainers that produced the reported runs (v1), these three calls
ran \emph{outside} \texttt{torch.no\_grad}. Their input aspect vector
was detached, but each call re-entered block $d{-}1$'s four mixing
logits (\texttt{alpha\_raw}), so the block-$d$ backward wrote a
non-zero gradient into \texttt{blocks[d-1].alpha\_raw.grad}. For this
term, the claim that the block-$d$ graph contains only block-$d$
parameters was therefore false in v1. The \emph{applied} updates were
still block-local: that gradient was never used, because block
$d{-}1$'s optimizer contains only block $d{-}1$'s parameters and calls
\texttt{zero\_grad(set\_to\_none=True)} before its next backward, which
discards it (the orphaned-gradient regime described below). The
released trainers wrap the three calls in \texttt{torch.no\_grad}
(in \texttt{cp\_fair\_cifar10.py} and the CIFAR-100, Tiny~ImageNet and
hardness-gated trainers), and \texttt{verify\_locality.py
-{}-depth-order} checks this path by running the real
\texttt{train\_epoch\_stage1} of each of the four trainers on a tiny
configuration ($\lambda_\text{depth}{=}0.2$, with and without SAM). In
v1 the only cross-block gradient is on
\texttt{blocks[d-1].alpha\_raw}, and it is still present when
\texttt{opt$_{d-1}$.zero\_grad} wipes it; the released version has no
cross-block gradient and leaves every online and EMA parameter
bit-identical to v1 after three batches. The runs with
$\lambda_\text{depth}{=}0$ (e.g., the $\gamma{=}0$ seed-42 run and
the Tiny~ImageNet runs) do not execute this code path.

\begin{lemma}[Block-locality under stopped previous-goodness aggregation]
\label{lem:locality}
Let the block-$d$ objective be
\[
  \mathcal{L}_d(\theta_d)
  \;=\;
  \Psi_d\!\bigl(h_d(\theta_d;\,\operatorname{sg}(h_{d-1})),\;
                \operatorname{sg}(G_{<d})\bigr),
\]
where $\operatorname{sg}(\cdot)$ denotes stop-gradient
(\texttt{torch.no\_grad}/\texttt{.detach}), $h_{d-1}$ is the
detached input token sequence to block $d$, and $G_{<d}$ is any
previous-block goodness summary computed under stop-gradient. Then
for all $j<d$,
\[
  \nabla_{\theta_j}\mathcal{L}_d \;=\; 0.
\]
Consequently, the update generated by $\mathcal{L}_d$ is block-local:
only $\theta_d$ can receive a non-zero gradient from the block-$d$
loss, and no end-to-end or cross-block backpropagation occurs.
\end{lemma}

\begin{proof}
Because $h_{d-1}$ and $G_{<d}$ enter $\mathcal{L}_d$ only through
$\operatorname{sg}(\cdot)$, their Jacobians with respect to earlier
parameters $\theta_j$ ($j<d$) are zero by construction. The only
non-stopped computational path from $\theta_j$ ($j<d$) to
$\mathcal{L}_d$ would have to pass through one of these two arguments,
so $\nabla_{\theta_j}\mathcal{L}_d = 0$. The remaining path through
$h_d(\theta_d; \cdot)$ contributes only to $\nabla_{\theta_d}\mathcal{L}_d$.
\end{proof}

\paragraph{Implementation note.} An earlier internal version of the
training loop did not wrap \texttt{\_sum\_prev\_goodness} in
\texttt{torch.no\_grad()}, so the additive path
$g^\ell_{\text{blk}} = g^\ell_{\text{curr}} + \gamma \sum_{j<\ell} \text{agg}_j(g^{(j)})$
put non-zero gradients on \texttt{blocks[j].alpha\_raw} ($j<\ell$). Applied
updates were still block-local: each block has its own optimizer over its own
parameters, and \texttt{opt$_j$.zero\_grad(set\_to\_none=True)} discards any
leaked gradient before \texttt{opt$_j$.step()} runs (orphaned-gradient regime).
The canonical implementation removes this reliance on optimizer-level isolation
for the cumulative-goodness path, and the reported numbers do not depend on which
implementation was used;
in the reported runs the depth-order term above also relied on it.

\subsection{Additional limitations}
\label{app:long_limitations}

This appendix expands the limitations summarized in
Section~\ref{sec:limitations}.

\paragraph{Scope of the negative result (extended).}
Our conclusions are limited to the FF objectives (softplus and
squared-hinge barriers on one goodness statistic), the co-designed
attention backbones (MoE blocks on CIFAR-10; the central
CIFAR-100 trio and Tiny ImageNet use dense blocks without MoE or SAM),
and the image benchmarks evaluated here. \emph{ImageNet scale is
untested and is the main open boundary of the result}: the largest
settings we study are CIFAR-100 and Tiny ImageNet at D256. The
conclusions do not rule out other local objectives---contrastive-only,
energy-based, biologically motivated training rules, or alternative
goodness statistics (the squared-hinge replication changes only the
barrier and covers CIFAR-10 L4/D128 only, 3 seeds per arm)---under which repairing free-riding could also close the
accuracy gap. Nor do they extrapolate to FF on substantially different
modalities, to classic FF architectures (a replication there is future
work), or to backbones an order of magnitude larger than the ones
tested. Within our experimental budget, architecture and augmentation
account for much larger accuracy changes than per-block free-riding
repair, but the dominance claim is \emph{about the intervention space
we explored}, not a universal claim about FF. Our recipe controls (the
dense, no-MoE, no-SAM CIFAR-100 trio; the leave-one-out ablations of
App.~\ref{app:components}; the stripped backbone of
Tab.~\ref{tab:bp_ff_ablation}; the squared-hinge objective, which
reproduces the dissociation on the L4/D128 recipe, and the
attenuation-compensated objective, which on the trio's own dense recipe
moves no protocol toward a gain and whose only CIs outside $\pm 1$\,pp
are the Stage-1 costs that the probe recovers) mitigate, but do not
eliminate, the concern that the co-designed recipe drives the result
(main-text paragraph ``Recipe vs.\ mechanism'', \S\ref{sec:ablation}).

\paragraph{The falsification test, run.} A falsifiable
counter-experiment is a local-objective intervention that (a) cancels
the cumulative attenuation envelope of Theorem~\ref{thm:attenuation}
and (b) moves accuracy materially; it would refute the dissociation.
We ran the attenuation-compensated loss of
Appendix~\ref{app:attenuation_compensated} (CIFAR-100, dense L4/D256,
3 seeds; protocol, telemetry and per-seed values there and in
App.~\ref{app:mgc_results}). It approximately compensates the envelope
and raises deepest-block separation ${\sim}5.9{\times}$ (at L4/D128 a
control traces most of such a rise to the removal of the auxiliary;
App.~\ref{app:f1_control}), but its accuracy effect is a Stage-1 cost
that is negative in every seed and that the trained Stage-2 probe
recovers (Table~\ref{tab:dissociation_merged}). Read literally,
criterion (b) is met at Stage~1, in the harmful direction: neither
Stage-1 CI lies inside the $\pm 1$\,pp band. The claim under test,
however, is that free-riding limits accuracy, so the refuting outcome
is a material accuracy \emph{gain} from restored block health; no
protocol shows one, and no arm of the L4/D128 control gains on the
SAM-matched cumulative arm with a CI that excludes 0.
In this sense the dissociation survived the test, in a sharper form:
accuracy invariance is a property of the deployed sum rule under
redistribution (Proposition~\ref{prop:redistribution}) and,
empirically, of the structured readouts we test (including the
Stage-2 probe that recovers MGC's Stage-1 cost), not of what the
network learns (MGC and cumulative disagree on $22.7$--$26.2\%$ of test
predictions).

\paragraph{Readout scope.} The claim is scoped to the deployed S1/S2
rules and structured goodness readouts, not to every conceivable
readout. On the frozen CIFAR-100 trio, fitted readouts of the per-block
goodness scores keep the three variants within $0.38$\,pp ($0.43$\,pp
for the plain goodness sum; Table~\ref{tab:readout_ladder}); outside that family (pooled-embedding
probes; $k$-NN on goodness, $1.47$\,pp block-softmax-normalized and
$8.58$\,pp raw; block-3-only probes) the variants separate by more
than 1\,pp (Table~\ref{tab:probes}), and a strict-FF
readout block on the MGC backbone is a further, bounded exception
(Table~\ref{tab:readouts}, App.~\ref{app:readouts}).

\paragraph{Depth: open boundaries.} At $L{=}8$ with 32 experts
(seed~42 only; Table~\ref{tab:dissociation_full}) all three regimes
stay inside the $<\!1$\,pp band, but the depth-scaled repair has the
best single-seed accuracy: S2~TTA $89.61\%$ vs.\ $89.22\%$
($\gamma{=}0$) and $89.20\%$ (LCFF); S1~TTA $89.31\%$ vs.\ $88.73\%$
and $89.02\%$. With one seed and a $0.39$--$0.41$\,pp S2~TTA margin,
whether repairs begin to matter at greater depth is open. Separately,
on Tiny ImageNet $\text{sep}^\text{cur}_\text{nl}$
peaks at Block~1 and declines deeper even under $\gamma{=}0$
(App.~\ref{sec:tiny_imagenet_metrics}). Theorem~\ref{thm:attenuation}
explains the cumulative-goodness pathology specifically and does not
claim to be the only depth pathology; this one is theorem-independent
and its cause is open (task difficulty or capacity allocation is a
hypothesis, not an identified mechanism). Per-depth accuracy under
$\gamma{=}0$ still rises monotonically on Tiny ImageNet (seed~42;
App.~\ref{sec:tiny_imagenet_metrics}), so this pathology, like free-riding,
leaves the conclusion that deep-block health is not the accuracy
bottleneck intact.

\paragraph{Theory is local-surrogate rather than end-to-end.} Our
analysis (Theorem~\ref{thm:attenuation}, Appendix~\ref{app:proofs})
characterizes the signal allocation induced by the FF surrogate and
proves a gradient floor for unresolved examples, but it does \emph{not}
establish global convergence of the overall optimizer, sample-complexity
guarantees, or generalization bounds.

\paragraph{MoE expert efficiency.} Our models exhibit 19--31\% dead
expert waste with 32 experts; reducing to 24 did not fully resolve this.
MoE routing efficiency under FF training remains an open problem in
this setting.

\paragraph{Factorial component interactions.} The $\gamma$-sweep
($0, 0.7, 1.0$) cleanly identifies cumulative accumulation as the root
cause of free-riding. Appendix~\ref{sec:hnm_deep_dive}
and~\ref{sec:aspect_ema_deep_dive} provide fine-grained ablations of
the two load-bearing components (hard negative mining and multi-aspect
goodness), but full factorial interactions between components are not
covered.

\paragraph{Hardness-gated experiments at reduced scale.} The adaptive
$\gamma$ experiments (Section~\ref{sec:adaptive_gamma}) were conducted
at reduced scale (D128, 180 epochs vs.\ D256, 362 epochs for the main
results) to enable a broad sweep. The $L{=}8$ experiments confirm that
the free-riding elimination and accuracy dissociation hold at double the
depth (8 blocks, 16 experts), strengthening the generality of the
finding; absolute accuracies should be interpreted within this
reduced-scale context.

\paragraph{Graph-level locality.} Several reported runs, including
the seed-42 \textsc{cp-fair}, $\gamma{=}0$ and hardness-gated runs, were
trained with the earlier implementation described in
App.~\ref{app:locality}; because every leaked gradient is zeroed before
any optimizer step, the applied updates, and hence the reported numbers,
are the same under both.

\appendixpart{Appendix Part II: Context and Reference Material}{Part~II
contains contextual and reference material, some of which the main text
also cites: the detailed comparison with prior FF methods
(App.~\ref{sec:related_work_full}), per-layer metrics, goodness
decomposition, training dynamics, and the CIFAR-100 and Tiny ImageNet
details (App.~\ref{app:per_layer_metrics}), per-class accuracy
(App.~\ref{app:per_class}), deep-dive ablations of the load-bearing
components (App.~\ref{sec:component_deep_dives}), the text-domain
protocol (App.~\ref{app:textdomain}), training cost
(App.~\ref{app:efficiency}), architecture diagrams
(App.~\ref{sec:arch_diagrams}), and secondary analyses
(App.~\ref{sec:secondary_analysis}).}

\renewcommand{\subsection}{\origsubsection}
\section{Detailed FF Method Comparison (CIFAR-10)}
\label{sec:related_work_full}

\paragraph{Full 13-variant CIFAR-10 sweep.}
Table~\ref{tab:full_main_results} is the unabridged version of the main
paper's Table~\ref{tab:main_results}. The main paper reports five
interpretable rows; the remaining eight variants below are exploratory
configurations either redundant with the headline rows or single-seed
incomplete; they are reported here for full transparency.

\begin{table}[htbp]
  \caption{CIFAR-10 test accuracy, all 13 variants (Stage~1 backbone,
           single seed 42 unless otherwise noted; multi-seed mean$\pm$std
           where $n>1$). \textbf{Bold} marks the
           validation-selected variant.
           $^\dagger$Incomplete run (Stage 1 only).
           $^\S$Stage~2 ran as a separate continuation from the saved
           Stage-1 checkpoint (Stage-2 TTA $91.21\%$).
           $^\ddagger$Seed~42 used $\lambda_\text{curr}{=}\lambda_\text{depth}{=}0$;
           seeds 123/456 used $\lambda_\text{curr}{=}0.25$, $\lambda_\text{depth}{=}0.2$
           (Table~\ref{tab:gamma0_multiseed}).}
  \label{tab:full_main_results}
  \centering
  \small
  \setlength{\tabcolsep}{4pt}
  \resizebox{\textwidth}{!}{%
  \begin{tabular}{lccccc}
    \toprule
    Variant & $n$ (seeds) & Best Val (s42) & Test no-TTA & Test TTA & Key design choice \\
    \midrule
    $\gamma{=}0$ (history-free, \textbf{validation-selected})  & 3 (42,123,456)$^\ddagger$ & \textbf{91.98\%} & $90.50 \pm 0.05$ & $\mathbf{91.32 \pm 0.19}$ & No accumulation ($\gamma{=}0$) \\
    \textsc{cp-fair}                            & 4 (42,123,456,789) & 91.90\%          & $90.36 \pm 0.18$ & $91.33 \pm 0.07$          & Depth-scaled $\lambda_\text{curr}$ \\
    \textsc{nomem-l2}                           & 1 (42)             & 91.84\%          & 90.56\%          & 91.31\%                    & No memory, $L_2$ norm \\
    \textsc{cp-fair-lc}                         & 1 (42)             & 91.40\%          & 90.46\%          & 91.32\%                    & Low contrastive ($\lambda{=}0.15$) \\
    \textsc{multi-tech}                         & 1 (42)             & 90.94\%          & 90.55\%          & 91.22\%                    & 10-technique synthesis \\
    \textsc{baseline}                           & 1 (42)             & 91.74\%          & 90.69\%          & 91.20\%                    & Baseline with val fix \\
    \textsc{topk-depth}                         & 1 (42)             & 91.22\%          & 89.98\%          & 90.69\%                    & Top-$k$ HNM + depth sched. \\
    \textsc{prefix-freeze}                      & 1 (42)             & 91.14\%          & 90.13\%          & 90.95\%                    & Greedy-prefix + triangle act. \\
    \textsc{fixed-nl}                           & 1 (42)             & 91.20\%          & 90.53\%          & 91.15\%                    & Fixed NL mining \\
    \textsc{block0-label}                       & 1 (42)             & 91.16\%          & 90.11\%          & 90.82\%                    & Label in first block only \\
    \textsc{gamma-ramp}$^\dagger$               & 1 (42)             & 90.74\%          & 90.82\%          & 90.82\%                    & $\gamma$-ramp + multi-neg \\
    \textsc{dp-fair}$^\S$                      & 1 (42)             & 91.18\%          & 90.33\%          & 91.12\%                    & 6 progressive mechanisms \\
    \textsc{greedy-freeze} (degenerate)         & 1 (42)             & 87.66\%          & 86.23\%          & 87.17\%                    & Pure greedy freeze \\
    \bottomrule
  \end{tabular}%
  }
\end{table}

\paragraph{Detailed FF method comparison.}
Table~\ref{tab:related_work} reproduces the full comparison of supervised FF
methods on CIFAR-10 along the axes that affect a fair head-to-head comparison.
This table appears here in the appendix; the main paper's
\S\ref{sec:related} carries a one-sentence pointer to it.

\begin{table}[htbp]
  \caption{Comparison of supervised FF methods on CIFAR-10 along the axes
           that affect the fairness of a head-to-head comparison. The set
           is restricted to algorithms in the Forward-Forward lineage that
           train each layer or block with a local goodness-based objective
           (explicit or implicit negatives); broader local-learning
           methods with non-FF objectives (deep supervision with auxiliary
           per-layer losses,~\citealp{lee2015dsn}; forward-only
           error-driven input modulation,~\citealp{dellaferrera2022pepita};
           layer-wise local triplet-loss training,~\citealp{forwardlayer2024})
           are out of scope here and discussed in
           \S\ref{sec:related} ``Local and greedy learning''.
           ``Strict FF'' = no feedback gradients and no overlapping
           gradient signals across layers or blocks (gradients may flow
           within a layer or block, as in HFF's intra-layer
           backprop and in our blocks); ``OLU'' denotes overlapping
           local updates that cross two adjacent blocks. ``TTA''
           indicates test-time augmentation; ``n/r'' = not reported.
           Our row reports the single-crop / TTA pair.}
  \label{tab:related_work}
  \centering
  \footnotesize
  \setlength{\tabcolsep}{3pt}
  \resizebox{\textwidth}{!}{%
  \begin{tabular}{lccccccc}
    \toprule
    Method & Arch.\ (approx.) & Params & Strict FF & Aug. & TTA & Seeds & C10 acc. \\
    \midrule
    \citet{hinton2022forward}    & locally connected, 2--3 hidden layers & n/r & yes    & none           & no & n/r & $59\%$\,$^\dagger$ \\
    \citet{scff2025}             & 3-layer CNN              & n/r            & yes        & Flip           & no & 3   & $80.8\%$\,$^\ddag$ \\
    \citet{layercollab2023}      & MLP                      & n/r            & yes        & none           & no & 5   & $48.4\%$\,$^\sharp$ \\
    \citet{cwc2024}              & CNN, $L{=}4$ (2 blocks)  & ${\sim}0.59$\,M & yes       & n/r            & no & n/r & $78.11 \pm 0.44\%$\,$^\|$ \\
    \citet{trifecta2024}         & VGG-like CNN, $L{=}12$   & ${\sim}8.5$\,M & no (OLU)   & RandCrop+Rot   & no & 3 & $83.51 \pm 0.78\%$\,$^\S$ \\
    \citet{cffm2025}             & ViT, $D{=}240$, $L{=}7$  & n/r            & yes        & RC+RandAug+Flip & no & n/r & $85.37\%$ \\
    \citet{distanceforward2024}  & CNN, $L{=}10$            & n/r            & no (OG/DF) & n/r            & no & 3   & $88.15 \pm 0.28\%$ \\
    \citet{deeperforward2025}    & ResNet-like CNN, $L{=}17$ & ${\sim}31$\,M$^\P$ & yes & none          & no & 5   & $86.22 \pm 0.17\%$\,$^\P$ \\
    \citet{sff2025}              & CNN, 3 blocks            & ${\sim}0.15$\,M & yes       & n/r            & no & 3   & $81.38 \pm 0.35\%$\,$^\diamond$ \\
    \citet{hff2026}              & 3-layer CNN              & n/r            & yes        & RandCrop+Flip  & no & 3 & $83.08 \pm 0.45\%$\,$^\flat$ \\
    \citet{asge2025}             & VGG11 CNN                & n/r            & yes        & RandCrop+Flip+CJ+Gray & no & 3 & $\mathbf{90.62 \pm 0.17\%}$\,$^\natural$ \\
    \midrule
    Ours ($\gamma{=}0$) & MoE-Transformer, $L{=}4$, $D{=}256$ & $107.2$\,M ($18.9$\,M active) & yes & RandAug+RRC+CJ+Blur & yes/no & 3$^\ast$ & $90.50 \pm 0.05 / 91.32 \pm 0.19$ \\
    \bottomrule
  \end{tabular}%
  }

  \vspace{1ex}
  \begin{flushleft}\footnotesize
  $^\ast$Seed~42 used $\lambda_\text{curr}{=}\lambda_\text{depth}{=}0$;
  seeds 123/456 used $0.25$/$0.2$ (Table~\ref{tab:gamma0_multiseed}).
  $^\dagger$ \citet{hinton2022forward} Table~1: the best FF test error
  on CIFAR-10 is $41\%$ (FF $41$--$46\%$, BP $37$--$39\%$) for
  non-convolutional, locally connected nets ($11{\times}11$ receptive
  fields) with 2 or 3 hidden layers of 3072 ReLUs; no parameter count is
  given.
  $^\ddag$ \citet{scff2025} report $80.8\%$ on CIFAR-10 (Table~1, \emph{Nature
  Communications} 16:5978, 2025; $80.75 \pm 0.12\%$, 3 seeds) for a
  3-layer CNN trained layer by layer without feedback connections,
  read out by a linear classifier on the frozen features; their primary
  setting is unsupervised local learning. They additionally report
  $77.3\%$ on STL-10~\citep{coates2011stl} and $35.7\%/59.8\%$ top-1/top-5 on Tiny ImageNet.
  $^\S$ \citet{trifecta2024} report ``around $84\%$'' in their text;
  $83.51 \pm 0.78\%$ is their 12-layer model after 500 epochs with a
  last-layer readout (their Table~2), also listed in
  \citet{deeperforward2025} Table~1.
  $^\P$ \citet{deeperforward2025} 17-layer ResNet without data
  augmentation (their Table~1, 150 epochs, 5 runs). They report no
  parameter count; ${\sim}31$\,M (and ${\sim}275$\,M for ResNet-CHx3)
  is counted from the layer definitions in their official code. Under a
  separate 1000-epoch AdamW protocol (their Appendix~C, Table~6) the
  same architecture reaches $87.16\%$ without augmentation and
  $88.72\%$ with enhanced augmentation (flips, $10^\circ$ rotation,
  $0.8$--$1.2$ scaling, color jitter). They additionally report
  $53.09 \pm 0.79\%$ on CIFAR-100 (ResNet, Table~2; their best CIFAR-100
  result is $60.28\%$ with the ResNet-CHx3 channel-tripled variant).
  $^\|$ \citet{cwc2024} CFSE CwC+Sf model: 4 convolutional layers in 2
  blocks, each layer trained with its own channel-wise competitive loss,
  read out by a softmax layer trained with cross-entropy on the last
  block's outputs; $588{,}133$ parameters (their Table~2) and a
  $21.89 \pm 0.44\%$ test error (their Table~1; the number of runs is not
  stated). They additionally report a $48.77\%$ CIFAR-100 error (their
  Table~3) with 3 blocks, the first two trained on the 20 superclasses.
  $^\sharp$ \citet{layercollab2023} report a $51.6\%$ CIFAR-10
  error for collaborative FF ($\gamma_{<t}$, network-level training;
  their Table~1, mean of five runs, all SDs below $0.2$; identical in the
  arXiv and AAAI versions); we list the corresponding accuracy,
  $48.4\%$, which \citet{hff2026} Table~1 also lists.
  $^\flat$ \citet{hff2026} HFF-CNN, a 3-layer CNN on CIFAR (their
  Tables~5 and~11), mean $\pm$ std over 3 trials (Table~2 of
  arXiv:2605.00082). HFF additionally reports
  $54.34 \pm 0.50\%$ on CIFAR-100 (CNN) and $25.70\%$ on
  ImageNet-1K from scratch (VGG16) / $65.96\%$ with transfer learning.
  $^\natural$ \citet{asge2025} VGG11, mean $\pm$ std over 3 runs
  (Table~4 of arXiv:2509.12394), read out by a linear classifier on the
  concatenated detached FF features of all convolutional layers but the
  first (``Fusion Pred''), trained alongside the FF layers without
  backpropagating into them; augmentation as in their technical
  appendix.
  $^\diamond$ \citet{sff2025} CNNB (their Tables~1 and~2): 3 blocks,
  backprop within each block only, $150{,}174$ parameters, predicting
  from the class goodness vectors averaged over the blocks, mean $\pm$ std
  over 3 seeds, trained on $80\%$ of the training images (the rest for
  validation) for at most 100 epochs. They additionally report
  $55.34 \pm 0.07\%$ on CIFAR-100 (CNNB, $296{,}044$ parameters).
  Our $\gamma{=}0$ model ($107.2$\,M parameters in total, of which
  $18.9$\,M are active per token: 32 MoE experts per block with top-4
  routing) is larger in total than the listed FF backbones of known size
  (SFF ${\sim}0.15$\,M, CwC ${\sim}0.59$\,M, Trifecta ${\sim}8.5$\,M,
  DeeperForward ${\sim}31$\,M), though not
  than DeeperForward's ResNet-CHx3 (${\sim}275$\,M), and its active
  count is below DeeperForward's ResNet; it is
  competitive with ASGE at single-crop (3-seed mean $90.50\%$ vs.\ $90.62\%$); the TTA
  number ($91.32\%$) is reported alongside the no-TTA number, and we note that
  none of the prior FF methods listed report TTA, so the TTA-augmented
  comparison is not like-for-like. The comparison is not
  augmentation-controlled (see Section~\ref{sec:bp_ff_ablation}).
  \end{flushleft}
\end{table}

\paragraph{Three-dataset summary.} Table~\ref{tab:related_work_3ds}
collects the headline single-crop test accuracies that each FF method
of Table~\ref{tab:related_work} reports on CIFAR-10, CIFAR-100, and Tiny ImageNet, alongside our
$\gamma{=}0$ Stage-1 backbone numbers. Of the three datasets, Hinton's
FF and Layer Collaboration report only CIFAR-10; ASGE, CwC, DeeperForward,
Distance-Forward, HFF, SFF and Trifecta also report CIFAR-100 top-1 (CFFM
reports CIFAR-100 top-5 and Tiny ImageNet top-10 only,
Table~\ref{tab:cifar100_comparison}); only SCFF and our work report
Tiny ImageNet top-1 at this scale. SCFF's Tiny ImageNet
number is unsupervised (per their primary setting); HFF's
``ImageNet'' number is full ImageNet-1K (not Tiny ImageNet) and is
included for context only. Cells marked ``n/r'' = not reported by
that work; ``n/a'' = applicable but no comparable number available.

\begin{table}[htbp]
  \caption{Cross-dataset comparison of FF methods (supervised except
           SCFF, whose primary setting is unsupervised). Numbers
           are the headline accuracies (\%) each work reports on
           CIFAR-10, CIFAR-100, and Tiny ImageNet (single-crop / no
           test-time augmentation unless noted). Our row gives
           Stage-1 strict-local FF inference; with the Stage-2
           frozen-feature readout the CIFAR-100 figure becomes
           $69.23\pm 0.34\%$ and Tiny ImageNet becomes
           $52.32\pm 0.34\%$ (representation probe, not strict FF).
           See Table~\ref{tab:related_work} above for axes (architecture, parameters,
           strict-FF status, augmentation, seeds). The CIFAR-10 column uses
           the 32-expert MoE backbone ($107.2$\,M total, $18.9$\,M
           active); the CIFAR-100 and Tiny ImageNet columns use a dense-MLP
           backbone ($9.49$\,M on CIFAR-100).}
  \label{tab:related_work_3ds}
  \centering
  \small
  \setlength{\tabcolsep}{4pt}
  \begin{tabular}{lccc}
    \toprule
    Method & CIFAR-10 & CIFAR-100 & Tiny ImageNet \\
    \midrule
    \citet{hinton2022forward}    & $59$\,$^l$       & n/r              & n/r \\
    \citet{layercollab2023}      & $48.4$\,$^a$     & n/r              & n/r \\
    \citet{cwc2024}              & $78.11$\,$^n$    & $51.23$\,$^n$    & n/r \\
    \citet{trifecta2024}         & ${\sim}84$\,$^b$ & $35.79$\,$^k$    & n/r \\
    \citet{cffm2025}             & $85.37$          & n/r\,$^j$        & n/r\,$^j$ \\
    \citet{distanceforward2024}  & $88.15$          & $59.01$          & n/r \\
    \citet{deeperforward2025}    & $86.22$\,$^c$    & $53.09$\,$^c$    & n/r \\
    \citet{sff2025}              & $81.38$\,$^o$    & $55.34$\,$^o$    & n/r \\
    \citet{scff2025}             & $80.8$\,$^d$        & n/r                 & \underline{$35.7$}\,$^d$ \\
    \citet{hff2026}              & $83.08$             & $54.34$             & n/a\,$^e$ \\
    \citet{asge2025}             & \underline{$90.62$}\,$^m$ & \underline{$65.42$}\,$^m$ & n/r \\
    \midrule
    Ours ($\gamma{=}0$, S1, single-crop)              & $90.50 \pm 0.05$\,$^i$ & $\mathbf{66.37 \pm 0.20}$ & $48.54 \pm 0.19$\,$^f$ \\
    Ours ($\gamma{=}0$, S1, TTA)\,$^g$                & $\mathbf{91.32 \pm 0.19}$\,$^i$ & $\mathbf{66.93 \pm 0.14}$ & $\mathbf{49.09 \pm 0.39}$ \\
    Ours ($\gamma{=}0$, S2, TTA)\,$^h$                & $\mathit{91.37 \pm 0.15}$\,$^i$ & $\mathit{69.23 \pm 0.34}$ & $\mathit{52.32 \pm 0.34}$ \\
    \bottomrule
  \end{tabular}

  \vspace{0.6ex}
  \begin{flushleft}\footnotesize
  \emph{Notes.} \underline{Underline} = highest number among the
  prior FF methods listed in that column (all single-crop; the only
  prior Tiny ImageNet number, SCFF's, comes from unsupervised
  training, so that column is not a like-for-like comparison).
  \textbf{Bold} = highest number across all
  rows in that column. \emph{Italic} = our Stage-2 frozen-feature
  readout, included for context as a representation probe (not
  strict FF inference). \emph{Context.} With horizontal-flip
  TTA (`Ours, S1, TTA' row), our $\gamma{=}0$ Stage-1 backbone is
  competitive with the strongest single-crop FF results listed on
  every dataset; without TTA (`Ours, S1, single-crop' row), we are
  within $0.12$\,pp of ASGE on CIFAR-10 and $+0.95$\,pp ahead
  of ASGE on CIFAR-100. We treat these comparisons as
  \emph{contextual} cross-paper references, not as a controlled
  SOTA claim: prior FF works use different backbones, training
  compute, readouts (ASGE, like CFFM, CwC and SCFF, trains a classifier
  on its features; our Stage-1 rows use the goodness readout), and
  (where reported) augmentation pipelines.
  $^{a,b,c,d,l,m,n,o}$ As in the notes to Table~\ref{tab:related_work}
  ($^\sharp$, $^\S$, $^\P$, $^\ddag$, $^\dagger$, $^\natural$, $^\|$,
  $^\diamond$); in particular, ASGE's numbers ($^m$) come from a linear
  classifier trained on its FF features, not from a goodness readout.
  $^e$ \citet{hff2026} reports full ImageNet-1K, not Tiny ImageNet
  (values in the notes to Table~\ref{tab:related_work}).
  $^f$ Single-crop (no-TTA) Stage-1 top-1 on the labeled validation
  split, from each run's final evaluation: $48.75$ / $48.45$ /
  $48.41\%$ for seeds 42 / 123 / 456 (TTA per-seed numbers in
  Table~\ref{tab:tiny_imagenet_multiseed}). This is $+12.9$\,pp above
  SCFF's $35.67\%$; because SCFF is trained without labels, the gap is
  contextual and is not a like-for-like comparison.
  $^g$ With horizontal-flip test-time augmentation, our Stage-1
  strict-local FF backbone is above every prior FF number listed
  (CIFAR-10: $+0.70$\,pp over ASGE; CIFAR-100: $+1.51$\,pp over ASGE;
  Tiny ImageNet: $+13.4$\,pp over SCFF). These gaps compare our TTA
  numbers with single-crop numbers, because the prior FF works in this
  table do not report TTA; they are context only. The protocol-closer
  single-crop gaps are in the row above: $-0.12$\,pp vs.\ ASGE on
  CIFAR-10, $+0.95$\,pp vs.\ ASGE on CIFAR-100, and $+12.9$\,pp vs.\
  SCFF on Tiny ImageNet, where supervised training is also compared
  with SCFF's unsupervised protocol.
  $^h$ Stage-2 frozen-feature BP-trained attentive readout
  (representation probe), included for completeness; not pure FF
  inference.
  $^i$ The CIFAR-10 entries pool seed~42
  ($\lambda_\text{curr}{=}\lambda_\text{depth}{=}0$) with seeds 123/456
  ($0.25$/$0.2$); see Table~\ref{tab:gamma0_multiseed}.
  $^j$ \citet{cffm2025} report CIFAR-100 top-5 ($84.29\%$) and Tiny
  ImageNet top-10 ($76.06\%$) only, with no top-1 number (Table~3 of
  arXiv:2502.00571).
  $^k$ Shallow 6-layer CNN after 200 epochs, last-layer readout
  (Table~2 of \citet{trifecta2024}); their 12-layer model reaches
  $35.26\%$ after 500 epochs.
  \end{flushleft}
\end{table}

\renewcommand{\subsection}{\FloatBarrier\origsubsection}%
\section{Per-Layer Metrics, Training Dynamics, and Dataset Details}
\label{app:per_layer_metrics}

\subsection{\texorpdfstring{$\gamma$-Sweep (D256, 362 Epochs)}{gamma-Sweep (D256, 362 Epochs)}}

\begin{table}[htbp]
  \caption{Per-layer $\text{sep}^\text{cur}_\text{nl}$ and $g^+_\text{cur}$ at convergence (D256 scale;
           seed~42, last Stage-1 epoch 362; the runs of Table~\ref{tab:freeriding}).}
  \label{tab:gamma_sweep_perlayer}
  \centering
  \small
  \begin{tabular}{lcccc|cccc}
    \toprule
    & \multicolumn{4}{c|}{$\text{sep}^\text{cur}_\text{nl}$} & \multicolumn{4}{c}{$g^+_\text{cur}$} \\
    Variant & L0 & L1 & L2 & L3 & L0 & L1 & L2 & L3 \\
    \midrule
    $\gamma{=}0$ (history-free) & 2.30 & 2.89 & 3.17 & \textbf{3.28} & 2.58 & 2.23 & 2.29 & 2.33 \\
    CP-FAIR ($\gamma{=}0.7$)    & 2.30 & 3.35 & 1.57 & 0.61 & 2.59 & 2.19 & 0.74 & 0.34 \\
    LCFF ($\gamma{=}1.0$)       & 2.32 & 1.02 & 0.94 & 0.70 & 2.60 & 0.66 & 0.55 & 0.41 \\
    \bottomrule
  \end{tabular}
\end{table}

Figure~\ref{fig:loss_bpff}a shows the current-block loss behind the
loss-collapse signature of \S\ref{sec:intro} on this backbone (a loss
on each block's own margin $m^{(d)}$, not LC of
Eq.~\ref{eq:diagnostics}, which uses the summed margin); panel~(b)
plots the BP calibration of Table~\ref{tab:bp_ff_ablation}.
Figure~\ref{fig:freeriding} compares cumulative and current-block
separation across depth for four variants, and
Figure~\ref{fig:freeriding_onset} shows when the current-block
separation of the three runs of Table~\ref{tab:freeriding} diverges
during training.

\begin{figure}[htbp]
  \centering
  \includegraphics[width=\textwidth]{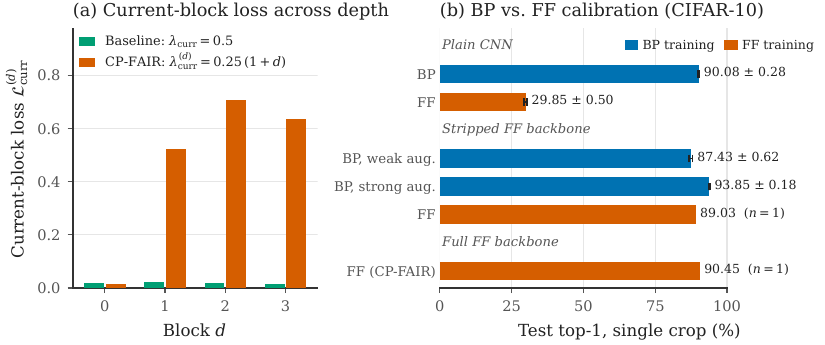}
  \caption{\textbf{(a)}~Current-block discrimination loss per block at the last logged
  Stage-1 epoch (CIFAR-10 L4/D256, seed~42; baseline epoch 286, \textsc{cp-fair} epoch 362).
  The baseline ($\lambda_\text{curr}{=}0.5$ at every block) has near-zero loss at all depths
  (its current-block term is already satisfied and supplies little gradient); with the depth-scaled $\lambda_\text{curr}{=}0.25(1{+}d)$ of
  \textsc{cp-fair}, blocks 1--3 keep substantial loss. \textbf{(b)}~Controlled BP vs.\ FF
  calibration, single-crop test top-1 for the rows of Tab.~\ref{tab:bp_ff_ablation}: the BP
  rows and the plain-CNN FF row are mean $\pm$ sample SD over 3 seeds, the two FF-backbone
  rows are seed~42. FF fails on the plain CNN ($29.85\%$ vs.\ $90.08\%$ for BP); on the
  stripped backbone, BP moves from $87.43\%$ (weak augmentation) to $93.85\%$ (BP-strong),
  and FF reaches $89.03\%$; the full FF backbone with \textsc{cp-fair} reaches $90.45\%$. The
  FF and BP recipes are not augmentation-controlled, so panel~(b) calibrates effect sizes and
  is not an FF-vs-BP ranking.}
  \label{fig:loss_bpff}
\end{figure}

\begin{figure}[htbp]
  \centering
  \includegraphics[width=\textwidth]{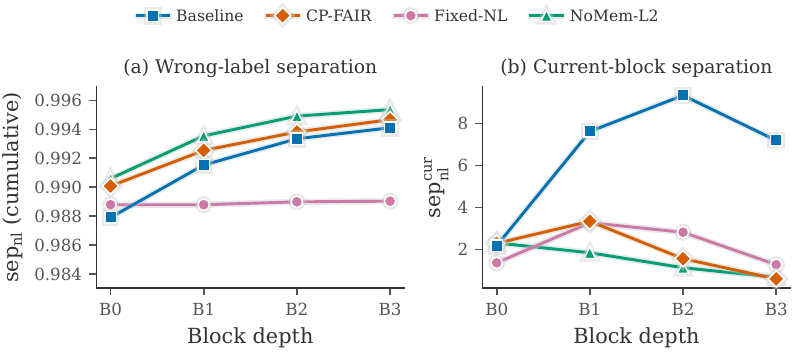}
  \caption{Separation across depth for four CIFAR-10 L4/D256 variants
           (seed~42; configurations in Table~\ref{tab:full_config}), at
           the last Stage-1 epoch of each plotted log: 362 for CP-FAIR and
           NoMem-L2, 460 for Fixed-NL, and 286 for the baseline, the end
           of its second training segment (its final Stage-1 epoch, 287,
           is logged in a third). \textbf{(a)}~Cumulative
           wrong-label separation $\text{sep}_\text{nl}$ (zoomed axis)
           stays between $0.987$ and $0.996$ at every block: flat for
           Fixed-NL ($\gamma{=}1$), rising slightly with depth for the
           others. \textbf{(b)}~Current-block separation
           $\text{sep}^\text{cur}_\text{nl}$ reaches its lowest value at
           Block~3 for CP-FAIR ($0.61$), NoMem-L2 ($0.68$) and Fixed-NL
           ($1.29$). The baseline stays far higher ($7.19$ at Block~3);
           its goodness is on a larger scale ($g_+{\approx}11.5$
           vs.\ ${\approx}3.4$ for NoMem-L2 at Block~3;
           App.~\ref{sec:secondary_analysis}), so its level is not
           comparable with the other three.}
  \label{fig:freeriding}
\end{figure}

\begin{figure}[htbp]
  \centering
  \includegraphics[width=\textwidth]{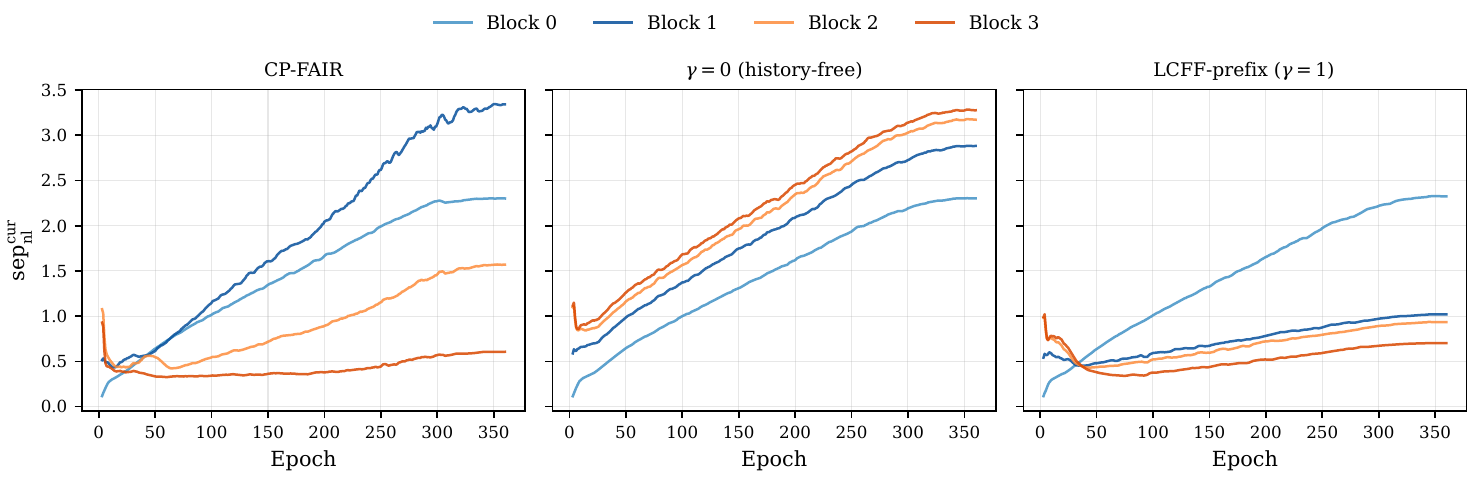}
  \caption{Free-riding onset dynamics: per-layer $\text{sep}^\text{cur}_\text{nl}$
           over training epochs (CIFAR-10 L4/D256, seed~42; five-epoch
           moving average).
           \textbf{Left}: CP-FAIR ($\gamma{=}0.7$ + depth-scaled loss)---Block~3
           falls early and stays the lowest block.
           \textbf{Center}: $\gamma{=}0$ (history-free, best)---Blocks~2--3
           dip briefly at the start; then all four blocks grow steadily to
           the end of training.
           \textbf{Right}: LCFF-prefix ($\gamma{=}1.0$, no auxiliary)---Block~0
           grows as under $\gamma{=}0$, but Blocks~1--3 stay at or below
           about~$1$ throughout, showing free-riding emerges early and
           persists.}
  \label{fig:freeriding_onset}
\end{figure}

\subsection{\texorpdfstring{Hardness-Gated $\gamma$ (D128, 180 Epochs)}{Hardness-Gated gamma (D128, 180 Epochs)}}
\label{sec:gated_full}

Table~\ref{tab:gated} (App.~\ref{app:gated_full}) lists all nine L4/D128 conditions ($\text{sep}^\text{cur}_\text{nl}$, $g^+_\text{cur}$ and $g^-_\text{cur}$ per block, S1 accuracy); Fig.~\ref{fig:goodness_decomposition} splits the positive goodness of four of them into own and inherited parts.

\begin{figure}[htbp]
  \centering
  \includegraphics[width=\textwidth]{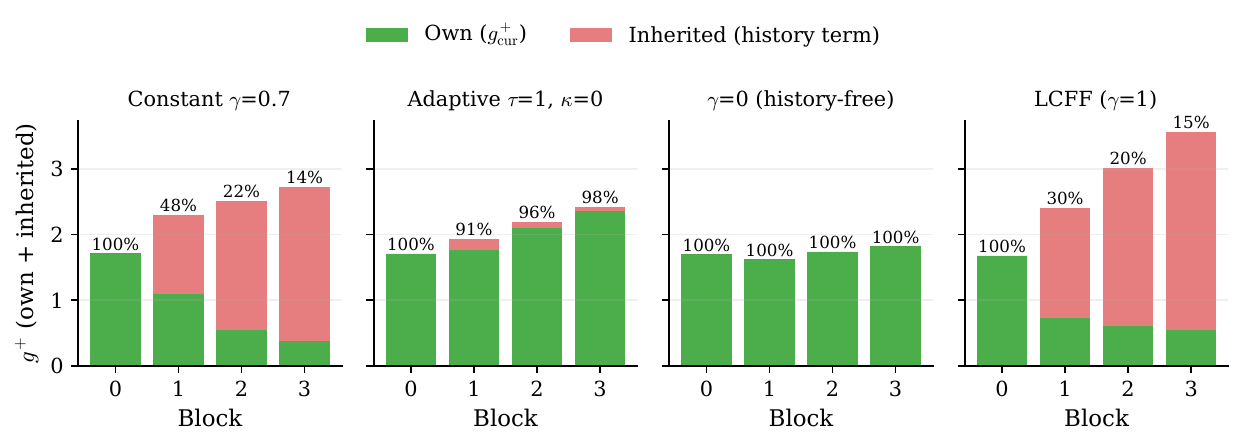}
  \caption{Goodness decomposition (CIFAR-10 L4/D128, seed~42, epoch 180;
           runs of Table~\ref{tab:gated}): each block's positive goodness
           split into its own part $g^+_\text{cur}$ (green) and the
           inherited history term (red), i.e., the effective $\gamma$
           times the detached goodness of the earlier blocks; the
           percentages give the own share. The own share falls with depth
           under constant $\gamma{=}0.7$ and LCFF ($\gamma{=}1$), stays
           dominant under hardness gating ($\tau{=}1,\kappa{=}0$), and is
           all of the goodness under $\gamma{=}0$
           (App.~\ref{app:gated_full}).}
  \label{fig:goodness_decomposition}
\end{figure}

\subsection{\texorpdfstring{Hardness-Gated $\gamma$ at 8-Block Depth (L8, D128, 180 Epochs)}{Hardness-Gated gamma at 8-Block Depth (L8, D128, 180 Epochs)}}
\label{sec:gated_L8}

\begin{table}[htbp]
  \caption{L8 per-layer $\text{sep}^\text{cur}_\text{nl}$ and accuracy at epoch 180.
           The eight per-layer columns report $\text{sep}^\text{cur}_\text{nl}$;
           the \textbf{L7/L0} column reports the ratio of current-block
           \emph{positive} goodness $g^{+}_{\text{cur}}$(L7)$/g^{+}_{\text{cur}}$(L0)
           (the free-riding indicator used in Fig.~\ref{fig:accuracy_vs_ratio}),
           not the ratio of $\text{sep}^\text{cur}_\text{nl}$ values. Values
           ${<}1$ indicate free-riding (deepest block produces less
           current-block positive goodness $g^{+}_{\text{cur}}$ than the first); values ${>}1$
           indicate the deepest block is still contributing. All
           runs: 8 blocks, D128, 16 experts, seed 42, no SAM.
           Unmarked accuracies are official-test Stage-2 TTA (the
           $\tau{=}1,\kappa{=}0$ run was evaluated post hoc from its
           selected Stage-1 and Stage-2 checkpoints;
           Table~\ref{tab:gated_k0_multiseed});
           $^\dagger$marks the best Stage-2 TTA accuracy on the 5k
           validation split for the three adaptive runs that have no
           test evaluation. Bold accuracy = highest test value.}
  \label{tab:L8_sep}
  \centering
  \scriptsize
  \begin{tabular}{lcccccccccc}
    \toprule
    & \multicolumn{8}{c}{$\text{sep}^\text{cur}_\text{nl}$} & & \\
    \cmidrule(lr){2-9}
    Variant & L0 & L1 & L2 & L3 & L4 & L5 & L6 & L7 & L7/L0 & Acc (S2 TTA) \\
    \midrule
    $\gamma{=}0$ (no collab)                    & 1.13 & 2.08 & 2.54 & 2.80 & 2.92 & 3.01 & 3.04 & 3.02 & 1.18 & \textbf{87.21\%} \\
    Constant $\gamma{=}0.7$                     & 1.12 & 2.53 & 2.24 & 1.09 & 0.86 & 0.86 & 0.87 & 0.82 & 0.32 & 86.97\% \\
    LCFF ($\gamma{=}1.0$)                        & 1.13 & 1.15 & 1.28 & 1.13 & 0.90 & 0.40 & 0.27 & 0.15 & 0.06 & 86.79\% \\
    \midrule
    Adpt $\tau{=}1,\kappa{=}0$ (cumul)          & 1.15 & 3.46 & 4.71 & 4.43 & 4.73 & 5.54 & 6.23 & \textbf{6.94} & \textbf{3.73} & 86.84\% \\
    Adpt $\tau{=}3,\kappa{=}0$ (cumul)          & 1.15 & 3.64 & 4.80 & 4.45 & 4.33 & 4.88 & 5.54 & 5.94 & 3.14 & 87.26\%$^\dagger$ \\
    Adpt $\tau{=}1,\kappa{=}2$ (cumul)          & 1.13 & 3.04 & 4.22 & 4.34 & 4.75 & 5.24 & 5.86 & 6.34 & 3.33 & 86.62\%$^\dagger$ \\
    Adpt $\tau{=}1,\kappa{=}4$ (cumul)          & 1.13 & 2.63 & 2.63 & 2.66 & 3.56 & 4.47 & 5.35 & 6.49 & 3.51 & 87.24\%$^\dagger$ \\
    Adpt $\tau{=}1,\kappa{=}0$ (prev)           & 1.15 & 3.51 & 4.56 & 3.97 & 3.86 & 4.19 & 4.49 & 4.46 & 2.28 & 87.20\% \\
    Adpt $\tau{=}1,\kappa{=}2$ (prev)           & 1.15 & 3.05 & 3.26 & 3.05 & 4.56 & 0.94 & 0.85 & 0.88 & 0.45 & 86.78\% \\
    \bottomrule
  \end{tabular}
\end{table}

The L8 results (Table~\ref{tab:L8_sep}) confirm the qualitative
findings from L4 (Table~\ref{tab:gated}): constant $\gamma{=}0.7$ and LCFF free-ride,
while the adaptive $\kappa{=}0$ variants keep far healthier blocks whose
separation is not monotone in depth (it dips at Blocks~3--4 before
rising to its maximum at Block~7 in the two cumulative runs). None of
this translates to accuracy gains: $\gamma{=}0$ reaches 87.21\% and the
adaptive $\kappa{=}0$ three-seed test mean is $87.12 \pm 0.35\%$
(Table~\ref{tab:gated_k0_multiseed}). Figure~\ref{fig:scaling_L8}
shows the per-block diagnostics of a separate 32-expert L8 sweep;
Figs.~\ref{fig:goodness_decomposition_L8}
and~\ref{fig:training_dynamics_L8} split the per-block goodness of the
adaptive runs into own and inherited parts and trace it over training.

\begin{figure}[htbp]
  \centering
  \includegraphics[width=\textwidth]{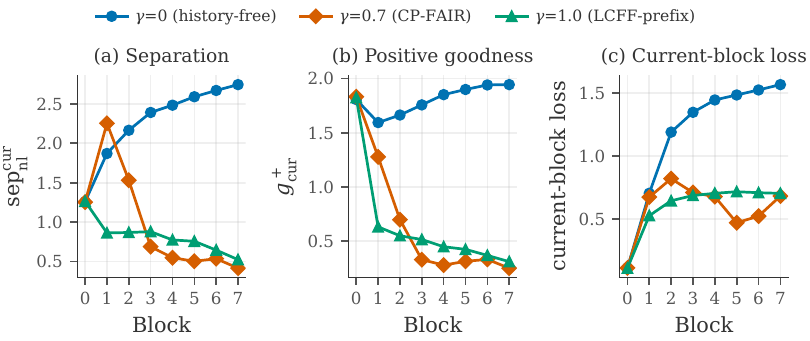}
  \caption{Per-block diagnostics at $L{=}8$ for a \emph{separate}
           32-expert (top-4) L8/D128 sweep: seed~42, no SAM, 180+20
           epochs; these are the ``L8, D128, 32e'' rows of
           Table~\ref{tab:dissociation_full}, not the 16-expert runs of
           Table~\ref{tab:L8_sep}, and the sweep has no adaptive-$\kappa$
           run. Panels: (a)~$\text{sep}^\text{cur}_\text{nl}$,
           (b)~$g^+_\text{cur}$, (c)~current-block loss (the loss on
           block $d$'s own margin $m^{(d)}$), all at the last
           Stage-1 epoch (180). Under $\gamma{=}0$, separation increases
           monotonically with depth ($1.26$ at Block~0 to $2.74$ at
           Block~7). Constant $\gamma{=}0.7$ (CP-FAIR) peaks at Block~1
           ($2.25$) and falls to $0.42$ at Block~7; LCFF-prefix
           ($\gamma{=}1.0$) drops to $0.86$ at Block~1 and ends at $0.53$.
           Test S2~TTA accuracy is $89.22\%$ ($\gamma{=}0$), $89.61\%$
           ($\gamma{=}0.7$) and $89.20\%$ ($\gamma{=}1.0$), so the run
           with the lowest deepest-block separation is the most
           accurate.}
  \label{fig:scaling_L8}
\end{figure}

\begin{figure}[htbp]
  \centering
  \includegraphics[width=\textwidth]{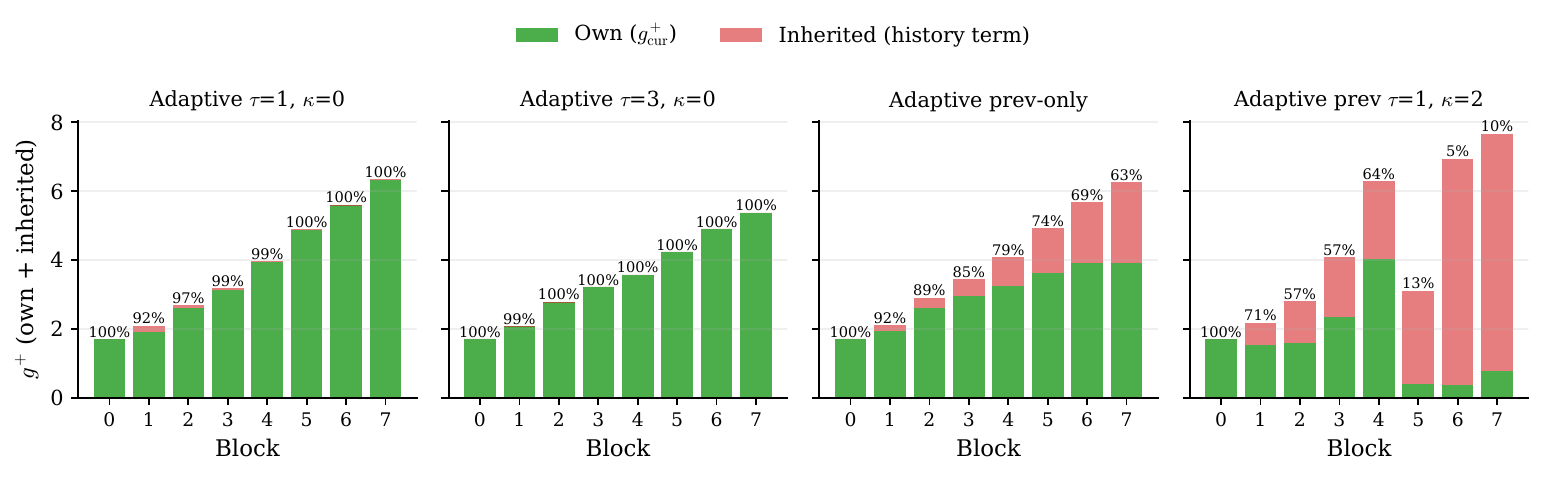}
  \caption{Goodness decomposition at $L{=}8$ (D128, 16 experts, seed~42,
           epoch 180): each block's positive goodness split into its own
           part $g^+_\text{cur}$ and the inherited history term, as in
           Fig.~\ref{fig:goodness_decomposition}, for the four adaptive
           runs shown (percentages: own share). Under adaptive $\kappa{=}0$,
           per-block goodness increases monotonically from $1.70$ (Block~0)
           to $6.34$ (Block~7), confirming progressive specialization. In
           the prev-only $\tau{=}1,\kappa{=}2$ run, own goodness collapses
           at Blocks~5--7 and the inherited term dominates
           (Table~\ref{tab:L8_gpos}). No constant-$\gamma$ run is shown;
           its $L{=}8$ free-riding is in Table~\ref{tab:L8_sep}.}
  \label{fig:goodness_decomposition_L8}
\end{figure}

\begin{figure}[htbp]
  \centering
  \includegraphics[width=\textwidth]{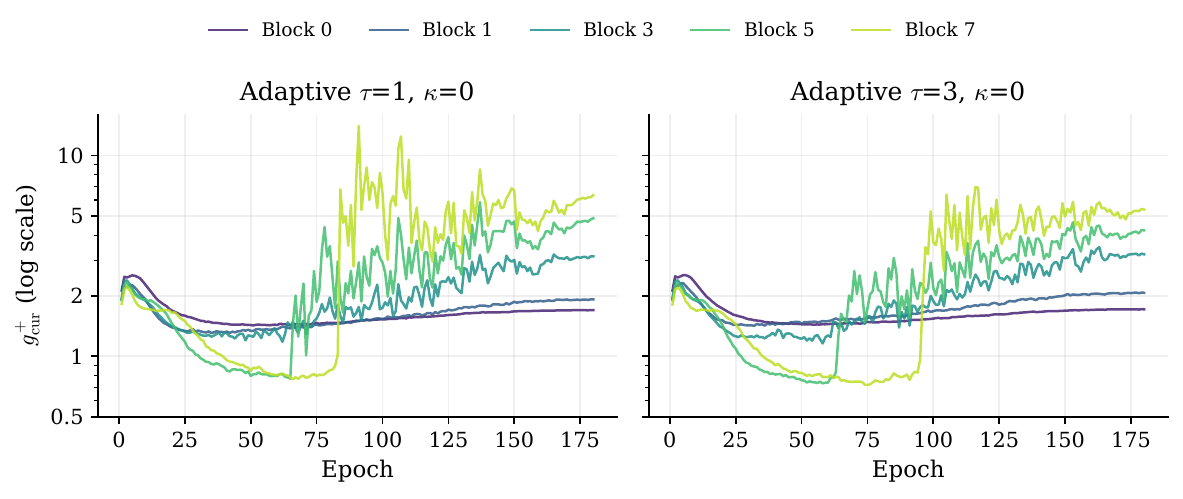}
  \caption{Training dynamics at $L{=}8$ (D128, 16 experts, seed~42, 180
           epochs): per-block $g^+_\text{cur}$ (log scale) of Blocks 0, 1,
           3, 5 and 7 over training for the two cumulative adaptive
           $\kappa{=}0$ runs ($\tau{=}1$ and $\tau{=}3$). Blocks~0 and~1
           fall from their early peak (${\approx}2.5$ for Block~0) to
           $1.3$--$1.5$ by epoch~40 and then rise only slowly (to
           $1.7$--$2.1$ at epoch~180). Blocks~5 and~7 first fall below $1$ and rise
           above Block~0 only after epoch~60 (Block~7 after epoch~80),
           with large epoch-to-epoch swings, before reaching the
           end-of-training values of Table~\ref{tab:L8_gpos} (Block~7:
           $6.34$ for $\tau{=}1$, $5.38$ for $\tau{=}3$). No constant-$\gamma$
           run is shown; its $L{=}8$ free-riding is in
           Table~\ref{tab:L8_sep}.}
  \label{fig:training_dynamics_L8}
\end{figure}

\subsubsection{L8 Per-Layer Positive Goodness}
\label{sec:gated_L8_gpos}

\begin{table}[htbp]
  \caption{L8 per-layer $g^+_\text{cur}$ (current-block positive goodness) at epoch 180.
           In the healthy conditions, $g^+_\text{cur}$ grows with
           depth---the opposite of free-riding---but only the two
           cumulative $\kappa{=}0$ runs are strictly monotonic:
           $\tau{=}1,\kappa{=}2$ dips at L1 ($1.69\to1.57$) and prev-only
           $\kappa{=}0$ plateaus at L6--L7 ($3.92\to3.91$).
           The $\kappa{=}4$ U-shape (dip at L1--L2, recovery at L3+) confirms that
           deeper networks ``outgrow'' a high threshold.
           All runs: L8, D128, 16 experts, seed 42.}
  \label{tab:L8_gpos}
  \centering
  \scriptsize
  \begin{tabular}{lcccccccccc}
    \toprule
    & \multicolumn{8}{c}{$g^+_\text{cur}$} & & \\
    \cmidrule(lr){2-9}
    Variant & L0 & L1 & L2 & L3 & L4 & L5 & L6 & L7 & L7/L0 & Pattern \\
    \midrule
    Adpt $\tau{=}1,\kappa{=}0$ (cumul)  & 1.70 & 1.92 & 2.61 & 3.14 & 3.93 & 4.87 & 5.59 & \textbf{6.34} & \textbf{3.73} & Monotonic $\uparrow$ \\
    Adpt $\tau{=}3,\kappa{=}0$ (cumul)  & 1.71 & 2.07 & 2.79 & 3.22 & 3.58 & 4.23 & 4.90 & 5.38 & 3.14 & Monotonic $\uparrow$ \\
    Adpt $\tau{=}1,\kappa{=}2$ (cumul)  & 1.69 & 1.57 & 2.37 & 3.17 & 3.93 & 4.54 & 5.16 & 5.64 & 3.33 & Dip at L1, then $\uparrow$ \\
    Adpt $\tau{=}1,\kappa{=}4$ (cumul)  & 1.69 & 1.28 & 0.99 & 1.82 & 3.01 & 3.95 & 4.82 & 5.94 & 3.51 & U-shape \\
    Adpt $\tau{=}1,\kappa{=}0$ (prev)   & 1.71 & 1.94 & 2.60 & 2.94 & 3.23 & 3.62 & 3.92 & 3.91 & 2.28 & $\uparrow$, plateau L6--L7 \\
    Adpt $\tau{=}1,\kappa{=}2$ (prev)   & 1.70 & 1.55 & 1.59 & 2.34 & 4.04 & 0.39 & 0.37 & 0.77 & 0.45 & Collapse L5--L7 \\
    \bottomrule
  \end{tabular}
\end{table}

Table~\ref{tab:L8_gpos} is the companion of Table~\ref{tab:L8_sep}:
in the healthy conditions per-block goodness grows with depth from
Block~1 or~2 onward, the opposite of free-riding, and the $\kappa{=}4$
variant recovers steeply once accumulated goodness exceeds its
threshold ($0.99 \to 5.94$ from L2 to L7).

\paragraph{The adaptive\_prev $\kappa{=}2$ anomaly.}
The prev-only $\kappa{=}2$ configuration is the only condition exhibiting
free-riding at $L{=}8$ (L7/L0$=$0.45). A pathological collapse occurs at
L5--L7: $g^+_\text{cur}$ drops from $4.04$ (L4) to $0.39$ (L5). The
mechanism is clear: in prev-only mode, the gate depends solely on the
immediately previous layer's goodness. When L4 produces high goodness
($4.04$), the sigmoid $\sigma(1 \cdot (2 - 4.04)) = \sigma(-2.04) \approx
0.12$ nearly shuts the gate, starving L5 of inherited signal. Unable to
compensate in isolation at this late training stage, L5--L7 effectively
stop learning. Its full-depth Stage-1 validation accuracy is nevertheless
the highest in Table~\ref{tab:depth_truncation_L8} (87.30\%), although its
test S2~TTA accuracy (86.78\%) is the lowest of the test-evaluated runs in
Table~\ref{tab:L8_sep}.
This reinforces the dissociation: free-riding does not predict accuracy.

\subsubsection{L4 vs.\ L8 Free-Riding Comparison}
\label{sec:depth_comparison}

\begin{table}[htbp]
  \caption{Depth amplifies anti-free-riding: L3/L0 (at $L{=}4$) vs.\
           L7/L0 (at $L{=}8$) for matched configurations. Five of six
           conditions show dramatically increased anti-free-riding at depth;
           the $\kappa{=}4$ reversal ($0.24 \to 3.51$, $14.6{\times}$)
           is the most striking.}
  \label{tab:depth_comparison}
  \centering
  \small
  \begin{tabular}{lcccc}
    \toprule
    Configuration & L4: L3/L0 & L8: L7/L0 & Change & Reversal \\
    \midrule
    Adpt $\tau{=}1,\kappa{=}0$ (cumul) & 1.39 & \textbf{3.73} & $+2.7{\times}$ & --- \\
    Adpt $\tau{=}3,\kappa{=}0$ (cumul) & 1.37 & 3.14 & $+2.3{\times}$ & --- \\
    Adpt $\tau{=}1,\kappa{=}2$ (cumul) & 0.90 & 3.33 & $+3.7{\times}$ & FR $\to$ healthy \\
    Adpt $\tau{=}1,\kappa{=}4$ (cumul) & \textbf{0.24} & 3.51 & $+14.6{\times}$ & FR $\to$ strong anti-FR \\
    Adpt $\tau{=}1,\kappa{=}0$ (prev)  & 1.42 & 2.28 & $+1.6{\times}$ & --- \\
    Adpt $\tau{=}1,\kappa{=}2$ (prev)  & 0.20 & 0.45 & $+2.3{\times}$ & Still FR \\
    \bottomrule
  \end{tabular}
\end{table}

\begin{table}[htbp]
  \caption{Accuracy degradation from $L{=}4$ to $L{=}8$ (best Stage-2
           TTA accuracy on the 5k validation split, not test, because
           three of the six L8 runs have no test evaluation; D128, 180
           epochs, seed 42, no SAM). All configurations show a slight accuracy drop at
           $L{=}8$, indicating that 180 epochs is insufficient for deeper
           networks. The $\kappa{=}4$ and prev-only $\kappa{=}0$ variants degrade least
           ($-0.18$\,pp), consistent with their strong recovery at depth.}
  \label{tab:accuracy_degradation}
  \centering
  \small
  \begin{tabular}{lccc}
    \toprule
    Configuration & L4 Acc (\%) & L8 Acc (\%) & $\Delta$ (pp) \\
    \midrule
    Adpt $\tau{=}1,\kappa{=}0$ & 87.20 & 86.88 & $-0.32$ \\
    Adpt $\tau{=}3,\kappa{=}0$ & 87.66 & 87.26 & $-0.40$ \\
    Adpt $\tau{=}1,\kappa{=}2$ & 87.40 & 86.62 & $-0.78$ \\
    Adpt $\tau{=}1,\kappa{=}4$ & 87.42 & 87.24 & $-0.18$ \\
    Adpt prev-only $\kappa{=}0$ & 87.30 & 87.12 & $-0.18$ \\
    Adpt prev-only $\kappa{=}2$ & 87.64 & 87.02 & $-0.62$ \\
    \bottomrule
  \end{tabular}
\end{table}

Table~\ref{tab:depth_comparison} shows that depth \emph{amplifies}
anti-free-riding under hardness gating: more layers accumulate more
goodness, which exceeds the $\kappa$ threshold and closes the gate,
forcing deeper layers to learn independently. This suggests that
free-riding under hardness-gated collaboration is primarily a
\emph{shallow-network problem} that self-corrects with sufficient depth.

\subsubsection{L8 Depth-Truncation Accuracy}
\label{sec:depth_truncation_L8}

Table~\ref{tab:depth_truncation_L8} lists the per-depth Stage-1
validation accuracies of the six adaptive $L{=}8$ runs.

\begin{table}[htbp]
  \caption{Depth-truncation accuracy at $L{=}8$ (D128, 180 epochs, seed~42):
           Stage-1 flip-TTA accuracy on the 5k validation split using only
           the first $d$ blocks, logged at each run's best-S1 epoch
           (epoch 180 for all six runs). The curves are nearly flat:
           within each run, accuracy varies by at most $0.26$\,pp across
           $d{=}1\ldots8$. This is an exploratory analysis: these are
           validation accuracies with no held-out test evaluation of a
           selected depth, so we draw no quantitative early-exit
           conclusion from them.}
  \label{tab:depth_truncation_L8}
  \centering
  \small
  \setlength{\tabcolsep}{4pt}
  \begin{tabular}{lcccccccc}
    \toprule
    Configuration & $d{=}1$ & $d{=}2$ & $d{=}3$ & $d{=}4$ & $d{=}5$ & $d{=}6$ & $d{=}7$ & $d{=}8$ \\
    \midrule
    Adpt $\tau{=}1,\kappa{=}0$  & 86.96 & 87.02 & 86.84 & 86.94 & 86.98 & 86.98 & 86.96 & 86.96 \\
    Adpt $\tau{=}3,\kappa{=}0$  & 87.08 & 87.06 & 86.88 & 86.86 & 86.92 & 86.92 & 86.90 & 86.82 \\
    Adpt $\tau{=}1,\kappa{=}2$  & 86.50 & 86.36 & 86.46 & 86.46 & 86.50 & 86.52 & 86.44 & 86.42 \\
    Adpt $\tau{=}1,\kappa{=}4$  & 86.78 & 87.02 & 86.90 & 86.96 & 86.94 & 86.90 & 86.96 & 87.00 \\
    Adpt prev-only $\kappa{=}0$ & 86.94 & 87.02 & 87.02 & 87.04 & 87.10 & 87.04 & 87.02 & 87.04 \\
    Adpt prev-only $\kappa{=}2$ & 87.06 & 87.22 & 87.16 & 87.14 & 87.18 & 87.26 & 87.28 & 87.30 \\
    \bottomrule
  \end{tabular}
\end{table}

\subsection{Separation vs.\ Accuracy Across Variants and Seeds}
\label{sec:sep_vs_acc}

Figure~\ref{fig:accuracy_vs_ratio} relates free-riding severity to accuracy across the 15 L4 and
L8 hardness-gated conditions, and Table~\ref{tab:e1_aligned} relates deepest-block separation to
test accuracy across seeds within each configuration.

\begin{figure}[htbp]
  \centering
  \includegraphics[width=0.9\textwidth]{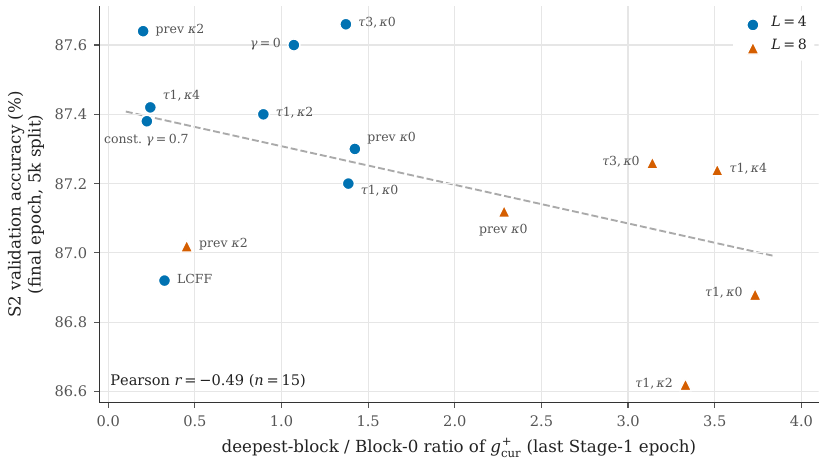}
  \caption{Accuracy vs.\ free-riding ratio across all 15 L4+L8 conditions
           (D128, 180 epochs, seed~42; circles $L{=}4$, triangles
           $L{=}8$; dashed line: least-squares fit).
           The Pearson sample correlation between the deepest-block /
           Block-0 ratio of current-block positive goodness
           $g^{+}_{\text{cur}}$ and final-epoch Stage-2 validation
           accuracy (5k split) is $r = -0.49$ ($n{=}15$), a descriptive
           association carried by the depth contrast: on average, the
           $L{=}8$ runs have higher ratios and lower 180-epoch accuracy,
           and within each depth the association is near zero and not
           significant (Pearson $r{=}{+}0.10$ at $L{=}4$, $n{=}9$, and
           $r{=}{-}0.10$ at $L{=}8$, $n{=}6$; Spearman ${-}0.23$ and
           ${-}0.20$). No single run carries it (leave-one-out $r$ from
           ${-}0.62$ to ${-}0.36$), and no significant relation
           appears within a depth or across seeds (below). Several
           of the 15 interventions also change collaboration mode,
           depth, and HNM gating simultaneously.
           \emph{Per-seed analysis (exploratory; values in
           Tab.~\ref{tab:e1_aligned}).} Across 19 seed-points
           in 6 configurations (CIFAR-10 $\gamma{=}0$, CIFAR-10
           \textsc{cp-fair}, L4/D128 $\kappa{=}0$, and the three
           CIFAR-100 trio variants), with deepest-block $\text{sep}^\text{cur}_\text{nl}$
           as logged in the epoch of each run's selected Stage-1 checkpoint, the
           within-configuration (demeaned) Spearman correlation with
           S2~TTA test accuracy is $\rho{=}{+}0.08$ (permutation
           $p{=}0.73$); no protocol reaches $p{<}0.05$ (largest
           $\rho{=}{+}0.38$, $p{=}0.06$, S2 without TTA), and splitting
           the two mixed-configuration groups (CIFAR-10 $\gamma{=}0$
           seed~42, different auxiliary weights; L4/D128 $\kappa{=}0$
           seed~42, no SAM) gives $\rho{=}{-}0.03$ ($p{=}0.91$,
           $n{=}17$). Recomputing $\text{sep}^\text{cur}_\text{nl}$ on the
           selected EMA checkpoints themselves (18 points, because the
           published CIFAR-10 $\gamma{=}0$ seed-42 checkpoint is no longer
           available; 7 configurations, with the L4/D128 $\kappa{=}0$ group
           split by SAM) leads to the same conclusion: the Pearson correlation of
           configuration-centered separation and S2~TTA accuracy is
           $r{=}{-}0.003$ (stratified permutation $p{=}0.99$; S1 single
           crop: $r{=}{+}0.16$, $p{=}0.62$), and Spearman $\rho$ is
           ${-}0.06$ ($p{=}0.82$) across the 18 points and ${+}0.36$
           (exact $p{=}0.44$) across the 7 configuration means. Neither
           measurement shows an association across seeds within
           configurations (Spearman for the logged values, Pearson for the
           recomputed ones), although with at most four seeds per
           configuration this analysis cannot rule one out. An earlier analysis paired last-epoch separation
           with selected-checkpoint accuracy ($\rho{=}0.28$, $p{=}0.20$
           for S2~TTA; $\rho{=}{+}0.53$, $p{=}0.007$ for S2 without
           TTA); the aligned values above supersede every such
           misaligned pairing. The dissociation argument therefore rests on
           controlled paired comparisons (Tab.~\ref{tab:dissociation_merged},
           with paired-bootstrap CIs), where the three history rules move accuracy by
           $<\!1$\,pp despite large changes in layer-health
           diagnostics.}
  \label{fig:accuracy_vs_ratio}
\end{figure}

\begin{table}[!htb]
  \caption{\textbf{Deepest-block separation at the validation-selected checkpoints} (exploratory; per-seed
           analysis of Fig.~\ref{fig:accuracy_vs_ratio}).
           \emph{Logged}: $\text{sep}^\text{cur}_\text{nl}$ of the deepest block as logged during training in the
           epoch of the selected Stage-1 checkpoint (training mode, raw weights, epoch mean).
           \emph{Recomputed}: the same quantity recomputed on that checkpoint's EMA weights in eval mode through the
           trainer's own diagnostic code (hard negatives mined by the EMA teacher; 20 consecutive training batches
           drawn with a fixed seed, i.e.\ 10,240 images on CIFAR-10 and 5,120 on CIFAR-100; standard
           error over batches ${\leq}0.06$). The two measurements differ in weights, mode, and negatives (the recomputed value
           is higher at 15 of 18 points), so compare values within a column only. Accuracies are test top-1 of the
           same checkpoints (S2~TTA: trained head on the frozen EMA backbone, flip TTA; S1: EMA backbone, single crop).
           $^\mathrm{a}$Trained with $\lambda_\text{curr}{=}\lambda_\text{depth}{=}0$ (seeds 123/456: $0.25$/$0.2$); its
           published checkpoint is no longer available, so the recomputation covers 18 of the 19 points.
           $^\mathrm{b}$Trained without SAM (seeds 123/456 with SAM); a separate configuration in the recomputed analysis.
           Each recomputed checkpoint reproduces its logged S1 test accuracy within one test image (16 points)
           or two images (2 points; TF32 training-time vs.\ FP32 recomputation numerics). As a code check, on the
           CIFAR-10 L4/D128 MGC seed-42 checkpoint in training mode the code gives 5.23, against
           5.23 from the independent gradient-share measurement
           (App.~\ref{app:gradshare}).}
  \label{tab:e1_aligned}
  \centering
  \footnotesize
  \begin{tabular}{llcccc}
    \toprule
    & & \multicolumn{2}{c}{Deepest-block $\text{sep}^\text{cur}_\text{nl}$} & \multicolumn{2}{c}{Test top-1 (\%)} \\
    \cmidrule(lr){3-4} \cmidrule(lr){5-6}
    Configuration & Seed & Logged & Recomputed & S2~TTA & S1 \\
    \midrule
    CIFAR-10 L4/D256 $\gamma{=}0$ & 42$^{\mathrm{a}}$ & 3.22 & --- & 91.39 & 90.52 \\
     & 123 & 5.65 & 6.65 & 91.51 & 90.54 \\
     & 456 & 6.41 & 6.75 & 91.22 & 90.44 \\
    \addlinespace
    CIFAR-10 L4/D256 \textsc{cp-fair} & 42 & 0.61 & 0.60 & 91.27 & 90.45 \\
     & 123 & 0.44 & 0.46 & 91.21 & 90.45 \\
     & 456 & 0.39 & 0.43 & 91.44 & 90.43 \\
     & 789 & 0.46 & 0.49 & 91.19 & 90.09 \\
    \addlinespace
    CIFAR-10 L4/D128 gated $\kappa{=}0$ & 42$^{\mathrm{b}}$ & 3.11 & 2.49 & 87.27 & 85.86 \\
     & 123 & 1.94 & 2.05 & 87.23 & 85.87 \\
     & 456 & 1.83 & 1.76 & 86.96 & 85.55 \\
    \midrule
    CIFAR-100 L4/D256 $\gamma{=}0$ & 42 & 4.77 & 5.11 & 68.94 & 66.17 \\
     & 123 & 4.79 & 5.14 & 69.16 & 66.36 \\
     & 456 & 4.74 & 5.02 & 69.60 & 66.57 \\
    \addlinespace
    CIFAR-100 L4/D256 gated $\kappa{=}0$ & 42 & 1.69 & 1.86 & 68.64 & 66.07 \\
     & 123 & 1.69 & 1.82 & 68.79 & 66.17 \\
     & 456 & 1.72 & 1.87 & 69.77 & 66.82 \\
    \addlinespace
    CIFAR-100 L4/D256 cumulative & 42 & 0.95 & 1.04 & 68.57 & 66.70 \\
     & 123 & 0.95 & 1.05 & 68.88 & 66.55 \\
     & 456 & 0.96 & 1.04 & 69.42 & 67.10 \\
    \bottomrule
  \end{tabular}
\end{table}

\subsection{L4 Depth Truncation and Training Dynamics}
\label{sec:l4_dynamics}

\begin{figure}[htbp]
  \centering
  \includegraphics[width=\textwidth]{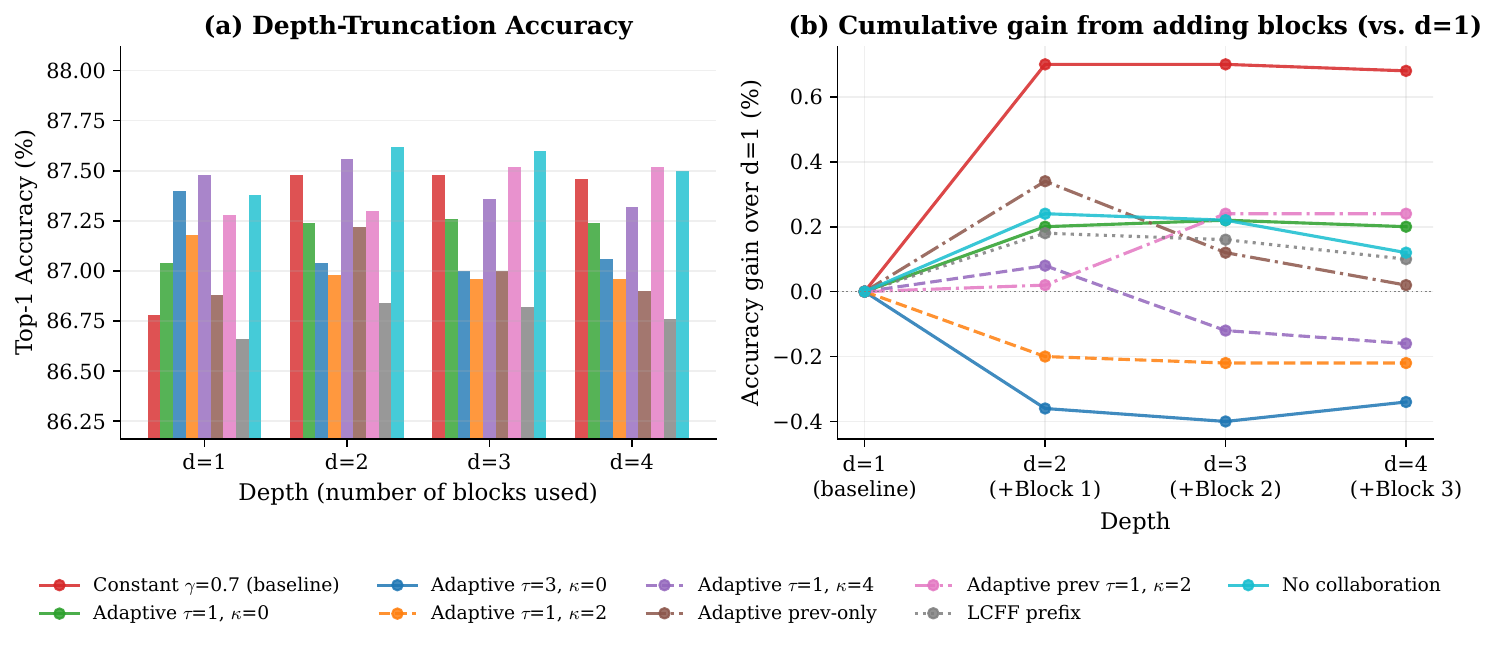}
  \caption{Depth-truncation accuracy (L4, D128, 180 epochs, seed~42):
           Stage-1 flip-TTA accuracy on the 5k validation split using only
           the first $d$ blocks, logged at each run's best-S1 epoch. These are
           validation curves, not test results. For the same nine
           checkpoints, the validation-selected operating points evaluated
           once on test are in App.~\ref{app:early_exit}
           (Table~\ref{tab:early_exit}). The $87.40\%$ for the
           $\tau{=}3,\kappa{=}0$ arm at $d{=}1$ is a validation value
           (App.~\ref{app:changes}, item B7).}
  \label{fig:depth_truncation}
\end{figure}

\paragraph{Convergence speed.}
In the three runs of Table~\ref{tab:freeriding} (CIFAR-10 L4/D256,
seed~42), Stage-1 validation accuracy (5k split, evaluated every 20
epochs) first reaches 95\% of its final value at epoch~80 and 99\% at
epoch~140 for \textsc{cp-fair}, $\gamma{=}0$ (history-free) and
LCFF-prefix alike (Fig.~\ref{fig:convergence_speed}), so the
depth-scaled loss does not slow convergence despite adding extra loss
terms. Over training, Fig.~\ref{fig:training_dynamics} traces cumulative
separation for four variants, Fig.~\ref{fig:moe_routing} the MoE
routing, and Fig.~\ref{fig:training_stability} the stability of
validation accuracy.

\begin{figure}[htbp]
  \centering
  \includegraphics[width=0.6\textwidth]{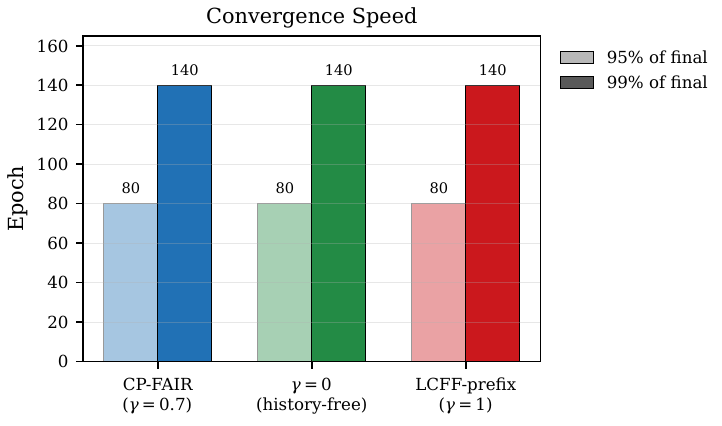}
  \caption{Convergence speed (CIFAR-10 L4/D256, seed~42): the first
           20-epoch evaluation at which Stage-1 validation accuracy
           reaches 95\% and 99\% of its final value. All three variants
           reach both thresholds at the same evaluations (epochs 80 and
           140).}
  \label{fig:convergence_speed}
\end{figure}

\begin{figure}[htbp]
  \centering
  \includegraphics[width=\textwidth]{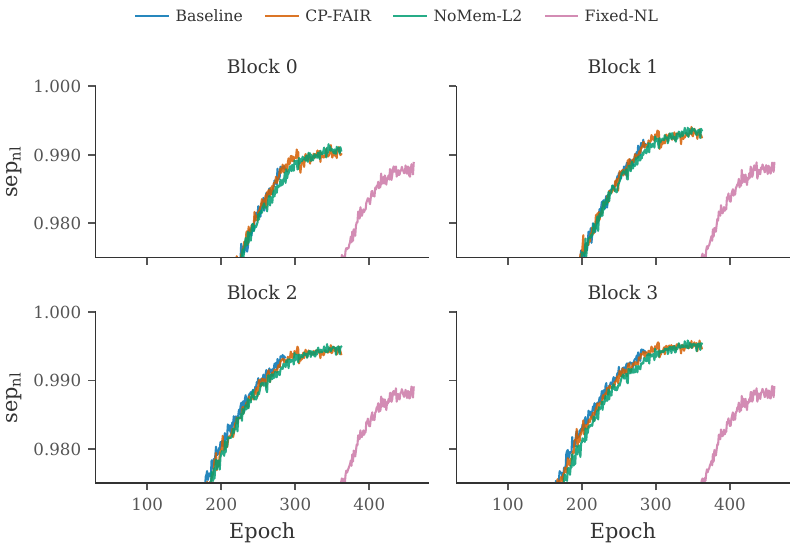}
  \caption{Training dynamics: $\text{sep}_\text{nl}$ over epochs per block depth
           (CIFAR-10 L4/D256, seed~42; epochs after~50 of each plotted
           log, which starts at epoch~145 for the baseline and~153 for
           CP-FAIR, their last logged segments).
           At deeper blocks (Block~2, Block~3), CP-FAIR and NoMem-L2 converge to
           higher separation than Fixed-NL, which rises later and plateaus
           lower on its 460-epoch, batch-size-128 schedule
           (Table~\ref{tab:full_config}). The zoomed
           $y$-axis ($[0.975, 1.0]$) reveals that differences emerge primarily
           at later training stages and deeper blocks.}
  \label{fig:training_dynamics}
\end{figure}

\begin{figure}[htbp]
  \centering
  \includegraphics[width=\textwidth]{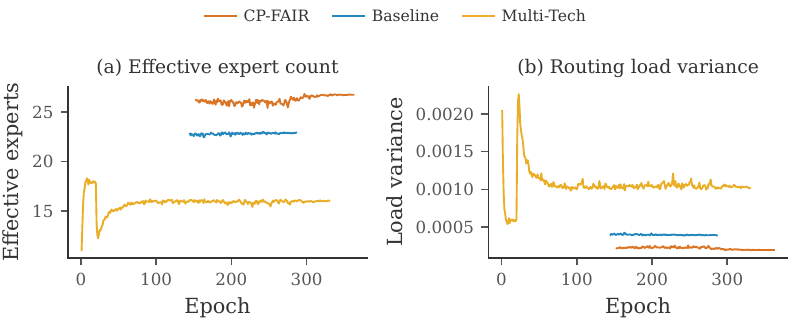}
  \caption{MoE routing quality over training (CIFAR-10 L4/D256, seed~42). \textbf{(a)}~Effective expert count
           (higher is better): CP-FAIR uses 25--27 of 32 experts, while Multi-Tech
           (24 total experts) uses 15--16. \textbf{(b)}~Routing load variance:
           CP-FAIR maintains low variance, indicating balanced expert utilization.
           Each curve spans its plotted log: epochs 153--362 for CP-FAIR
           and 145--286 for the baseline (the last logged segment of each
           run), and 1--330 for Multi-Tech (its first run; epochs
           331--362 are in the log of the resumed run).}
  \label{fig:moe_routing}
\end{figure}

\begin{figure}[htbp]
  \centering
  \includegraphics[width=\textwidth]{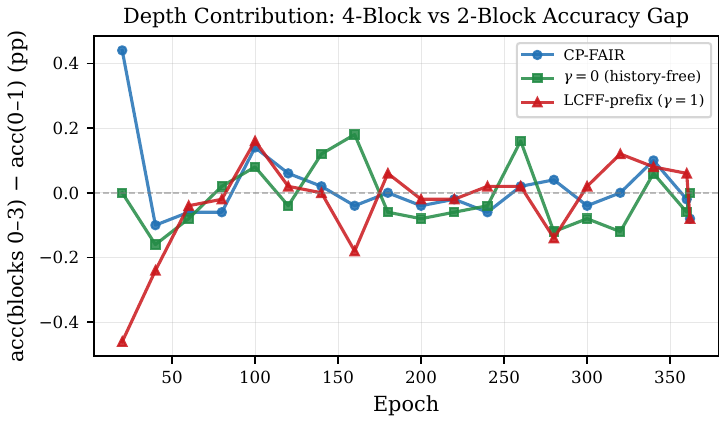}
  \caption{Depth saturation (CIFAR-10, L4/D256, seed~42): Stage-1
           flip-TTA accuracy on the 5k validation split using the first four
           blocks minus that using the first two, at each 20-epoch
           evaluation and the last epoch (362). After epoch~60 the gap stays within
           ${\pm}0.2$\,pp for all three variants, so blocks 2--3 add little
           to truncated accuracy in any regime. This is an exploratory
           analysis: no truncation depth of these L4/D256 runs is evaluated
           on the test set, so we draw no quantitative early-exit conclusion
           from it. Validation-selected truncation and dynamic exit,
           evaluated once on test, are reported for the CIFAR-10 L4/D128
           arms and the CIFAR-100 trio in App.~\ref{app:early_exit}.}
  \label{fig:depth_gap}
\end{figure}

\begin{figure}[htbp]
  \centering
  \includegraphics[width=\textwidth]{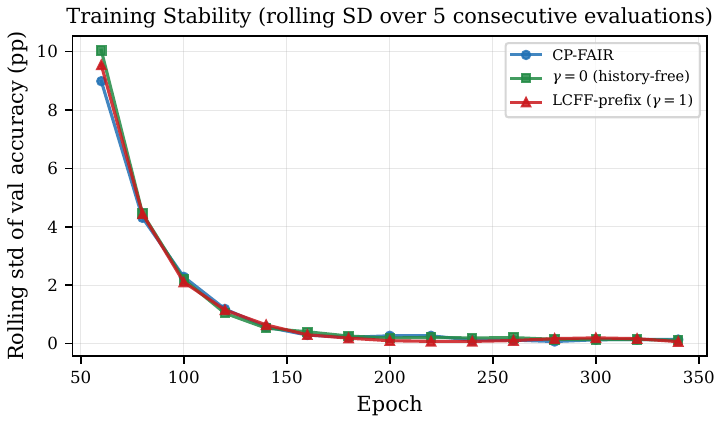}
  \caption{Training stability (CIFAR-10 L4/D256, seed~42): rolling
           population standard deviation (divisor $n{=}5$) of Stage-1
           validation accuracy over five consecutive evaluations (every
           20 epochs, plus the final epoch~362), plotted at the middle evaluation, for
           CP-FAIR, $\gamma{=}0$ (history-free) and LCFF-prefix
           ($\gamma{=}1$). The three curves nearly coincide and stay below
           $0.3$\,pp from epoch~180 on, so neither the depth-scaled loss nor
           the history rule makes validation accuracy less stable.}
  \label{fig:training_stability}
\end{figure}

\subsection{CIFAR-100 Baseline (L4, D256, 362 Epochs)}
\label{sec:cifar100_metrics}

\begin{table}[htbp]
  \caption{Per-layer metrics for CIFAR-100 baseline ($\gamma{=}0.7$, L4, D256,
           bs256, seed~42, 362+20 epochs).
           Free-riding is severe: $g^+_\text{cur}$ decays $4.3{\times}$ from L0 to L3
           (L3/L0$=$0.232), comparable to the CIFAR-10 D128 baseline (0.22).
           $\text{sep}^\text{cur}_\text{nl}$ peaks at L1 before declining, unlike the
           monotonic collapse observed on CIFAR-10 under constant~$\gamma$.}
  \label{tab:cifar100_metrics}
  \centering
  \small
  \begin{tabular}{lcccc}
    \toprule
    Metric & L0 & L1 & L2 & L3 \\
    \midrule
    $g^+_\text{cur}$            & 2.275 & 2.046 & 1.153 & 0.528 \\
    $\text{sep}^\text{cur}_\text{nl}$ & 1.787 & 4.158 & 2.999 & 1.284 \\
    $\text{sep}_\text{nl}$      & 0.971 & 0.987 & 0.994 & 0.996 \\
    \midrule
    L3/L0 ratio ($g^+_\text{cur}$) & \multicolumn{4}{c}{0.232 (severe free-riding)} \\
    L3/L0 ratio ($\text{sep}^\text{cur}_\text{nl}$) & \multicolumn{4}{c}{0.718} \\
    \midrule
    S2 Top-1 / Top-5 (TTA) & \multicolumn{4}{c}{66.52\% / 86.82\%} \\
    S1 Top-1 (TTA)         & \multicolumn{4}{c}{66.12\%} \\
    \bottomrule
  \end{tabular}
\end{table}

The CIFAR-100 baseline confirms that free-riding generalizes beyond
CIFAR-10 (Table~\ref{tab:cifar100_metrics}); cumulative separation
$\text{sep}_\text{nl}$ still reaches $0.996$ at Block~3 (vs.\ $0.982$ on
CIFAR-10 D128).

\paragraph{Multi-seed stability on CIFAR-100.}
To verify that the CIFAR-100 result is not seed-dependent, we train two
additional seeds (123, 456) with identical configuration ($\gamma{=}0.7$,
L4, D256, bs256, 362+20 epochs). Table~\ref{tab:cifar100_multiseed} shows
that accuracy is stable across seeds, with S2 TTA variance of $\pm 0.15$\,pp.

\begin{table}[htbp]
  \caption{CIFAR-100 multi-seed stability ($\gamma{=}0.7$, L4, D256, bs256,
           362+20 epochs). All runs use identical architecture and
           hyperparameters.}
  \label{tab:cifar100_multiseed}
  \centering
  \small
  \setlength{\tabcolsep}{4pt}
  \resizebox{\textwidth}{!}{%
  \begin{tabular}{lcccc}
    \toprule
    Seed & S1 (no-TTA) & S1 (TTA) & S2 (no-TTA) & S2 (TTA) \\
    \midrule
    42           & 64.62\% & 66.12\% & 65.30\% & 66.52\% \\
    123          & 65.59\% & 66.84\% & 65.65\% & 67.05\% \\
    456          & 64.78\% & 65.71\% & 65.07\% & 66.75\% \\
    \midrule
    Mean $\pm$ Std & $65.00 \pm 0.52$ & $66.22 \pm 0.57$ & $65.34 \pm 0.29$ & $66.77 \pm 0.27$ \\
    \bottomrule
  \end{tabular}%
  }
  \vspace{1ex}
  \begin{flushleft}\footnotesize
  \textbf{Data source.} Seeds 123 and 456 use the run-level test-set
  predictions (\texttt{preds\_stage\{1,2\}\_test\_\{no\_,\}tta.npz}) saved
  in the run directories of the CIFAR-100 baseline runs. Seed 42 uses the older
  trainer's run-log \texttt{eval\_top1} (TTA-enabled at eval time); its
  per-stage prediction files were not retained, but the eval-time top-1
  on the held-out test set was logged each epoch and we report the
  last-epoch value. All four columns reproduce from these sources to
  the printed precision.
  \end{flushleft}
\end{table}

\paragraph{CIFAR-100 dissociation contrast.}
The multi-seed $\gamma{=}0$ vs.\ gated-$\kappa{=}0$ vs.\ cumulative
contrast on CIFAR-100 (dense L4/D256, 3 seeds) is in the main-text
Table~\ref{tab:dissociation_merged}, with the positive-goodness ratio,
top-5 accuracy, and resumed-run record in
App.~Table~\ref{tab:cifar100_dissociation_multiseed}.

\paragraph{Comparison to prior FF work on CIFAR-100.}
Eight of the Forward-Forward methods in Table~\ref{tab:related_work}
report CIFAR-100 results: \citet{asge2025}, \citet{cwc2024},
\citet{deeperforward2025}, \citet{distanceforward2024}, \citet{hff2026},
\citet{sff2025} and \citet{trifecta2024}
report Top-1 accuracy; \citet{cffm2025} reports Top-5 only. The other
three (\citealt{hinton2022forward}, \citealt{layercollab2023},
\citealt{scff2025}) do not report CIFAR-100 results.\footnote{\citet{forwardlayer2024} report $72.49\%$
on CIFAR-100 with VGG19, but their method
trains each layer with a local triplet loss using local backprop and
does not use the FF goodness function; we therefore exclude it from
the FF comparison set (see~\S\ref{sec:related} ``Local and greedy
learning'').} Table~\ref{tab:cifar100_comparison} reports the
comparison.

\begin{table}[htbp]
  \caption{Prior published Forward-Forward results on CIFAR-100,
           alongside our $\gamma{=}0$ numbers from
           Tables~\ref{tab:dissociation_merged}
           and~\ref{tab:cifar100_dissociation_multiseed} (mean over
           3 seeds). We split our row into Stage-1 (strict-local FF
           backbone) and Stage-2 (frozen-feature attentive readout)
           because Stage-2 trains a separate head with backprop and is
           not part of the strict-local FF backbone; the prior rows use
           each method's reported readout (for ASGE, a linear classifier
           trained alongside the FF layers on their detached features;
           for CFFM, a linear classifier trained on the frozen encoder;
           for CwC, a softmax layer trained on the last block's
           outputs).}
  \label{tab:cifar100_comparison}
  \centering
  \small
  \setlength{\tabcolsep}{3pt}
  \resizebox{\textwidth}{!}{%
  \begin{tabular}{llcccc}
    \toprule
    Method & Backbone & Readout & Params & Top-1 & Top-5 \\
    \midrule
    \citet{trifecta2024}          & 6-layer CNN                   & FF                                                   & ${\sim}2.8$\,M      & $35.79 \pm 0.17\%$ & --- \\
    \citet{cwc2024}               & CNN, 3 blocks                 & softmax layer (last block)                           & n/r                  & $51.23\%$          & --- \\
    \citet{deeperforward2025}     & 17-layer ResNet               & FF                                                   & ${\sim}31$\,M      & $53.09 \pm 0.79\%$ & --- \\
    \citet{deeperforward2025}     & 17-layer ResNet-CHx3 (3$\times$ ch.) & FF                                            & ${\sim}275$\,M     & $60.28 \pm 1.02\%$ & --- \\
    \citet{sff2025}               & CNN, 3 blocks (CNNB)          & FF (goodness averaged over blocks)                   & ${\sim}0.30$\,M     & $55.34 \pm 0.07\%$ & --- \\
    \citet{distanceforward2024}  & CNN                           & FF                                                   & n/r                  & $59.01 \pm 0.35\%$ & --- \\
    \citet{hff2026}               & 3-layer CNN                   & FF                                                   & n/r                  & $54.34 \pm 0.50\%$ & --- \\
    \citet{asge2025}              & VGG11 CNN                     & linear (detached FF features)                        & n/r                  & $65.42\%$          & --- \\
    \citet{cffm2025}              & ViT, $D{=}240$, $L{=}7$       & linear (frozen encoder)                              & n/r                  & not reported       & $84.29\%$ \\
    \midrule
    Ours, S1 (strict-local FF backbone) & Transformer (dense MLP), $L{=}4$, $D{=}256$ & none (FF goodness)               & $9.49$\,M & $\mathbf{66.93 \pm 0.14\%}$ (TTA) & $90.29 \pm 0.24\%$ \\
    Ours, S1 + S2 (frozen-feature attentive head) & same                       & frozen-feature attentive head (BP-trained) & ${+}25.2$\,M (head) & $\mathbf{69.23 \pm 0.34\%}$ & $\mathbf{90.78 \pm 0.21\%}$ \\
    \bottomrule
  \end{tabular}%
  }

  \vspace{1ex}
  \begin{flushleft}\footnotesize
  Our \emph{strict-local FF backbone} (Stage-1) exceeds ASGE
  ($65.42\%$;~\citealt{asge2025}, VGG11) by $+0.95$\,pp single-crop
  ($66.37\%$) and by $+1.51$\,pp with flip-TTA ($66.93\%$, compared
  with their single-crop number; context only). Both of our rows use
  flip-TTA for Top-1 and Top-5; the prior FF rows are single-crop. The
  Stage-2 frozen-feature attentive readout is a separate post-hoc
  head trained with backprop on the frozen FF features---it should
  not be interpreted as the pure strict-local FF backbone number, and
  we list it as a separate row so the strict-FF row can be compared
  with prior FF readouts. ASGE's $65.42\%$ is itself read out by a
  linear classifier trained alongside the FF layers on their detached
  features (their ``Fusion Pred''), so our Stage-1 margins over ASGE
  compare a goodness readout with a trained linear readout. CwC trains
  the first two of its three blocks on the 20 CIFAR-100 superclasses
  (their Table~3 reports a $48.77\%$ error). SFF's entry is its CNNB
  model (their Table~2; 3 seeds; $296{,}044$ parameters).
  \citet{cffm2025} report Top-5 only on CIFAR-100; the Top-5
  comparison is therefore restricted to that metric. The
  \citet{trifecta2024} entry is their shallow model after 200 epochs
  with a last-layer readout (their Table~2). DeeperForward reports no
  parameter counts; ours are counted from its official code (notes to
  Table~\ref{tab:related_work}).
  \end{flushleft}
\end{table}

\subsection{\texorpdfstring{Tiny ImageNet ($\gamma{=}0$, L4, D256, 362 Epochs)}{Tiny ImageNet (gamma=0, L4, D256, 362 Epochs)}}
\label{sec:tiny_imagenet_metrics}

Tiny ImageNet~\citep{le2015tinyimagenet} is a 200-class subset of
ILSVRC-2012~\citep{russakovsky2015imagenet} at $64{\times}64$ pixels.
Its split convention: 100k labeled training images, 10k
labeled validation images, 10k unlabeled test images. Following the
established convention (the public test labels are not released), we
report accuracy on the 10k labeled validation split throughout this
appendix and the main paper's Section~\ref{sec:tinyimagenet}; only the
convergence curves of Fig.~\ref{fig:tinyimagenet_dynamics} use the
trainer's 10k held-out part of the training images, on which the
Stage-1 checkpoint is selected.

The multi-seed Tiny ImageNet accuracies are in
App.~Table~\ref{tab:tiny_imagenet_multiseed}; Figs.~\ref{fig:tinyimagenet_perseed}
and~\ref{fig:tinyimagenet_dynamics} show them per seed and the Stage-1
convergence of the three seeds. Fig.~\ref{fig:tinyimagenet_crossdata}
compares the per-block diagnostics with CIFAR-10: under $\gamma{=}0$,
current-block separation on Tiny ImageNet peaks at Block~1 and declines
deeper.

\begin{figure}[htbp]
  \centering
  \includegraphics[width=\textwidth]{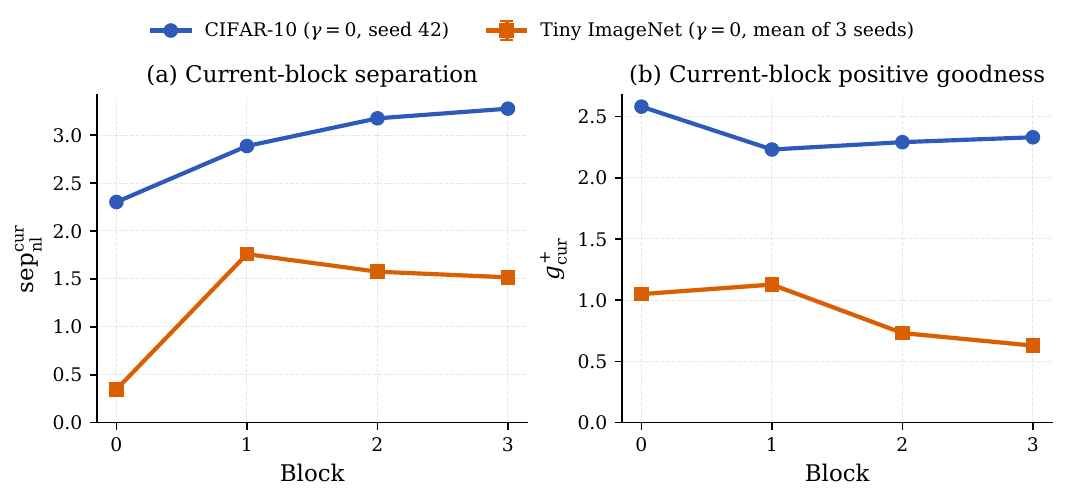}
  \caption{Per-block diagnostics under $\gamma{=}0$ at the last Stage-1
           epoch (362): CIFAR-10 (blue; seed~42, the $\gamma{=}0$ row of
           Table~\ref{tab:gamma_sweep_perlayer}) vs.\ Tiny ImageNet
           (orange; mean of 3 seeds, whose sample-std bands are narrower
           than the markers).
           \textbf{(a)}~Current-block wrong-label separation:
           CIFAR-10 grows monotonically with depth, whereas Tiny ImageNet
           peaks at Block~1 and declines at deeper blocks.
           \textbf{(b)}~Per-block positive goodness shows an analogous
           peak-then-decline on Tiny ImageNet. The same training rule
           therefore produces different per-layer health on harder data.}
  \label{fig:tinyimagenet_crossdata}
\end{figure}

\begin{table}[htbp]
  \caption{Tiny ImageNet configuration and seed-42 results
           ($\gamma{=}0$, L4, D256, bs256, 362+20 epochs). The configuration
           is derived from the CIFAR-10 component ablation: components
           identified as redundant (MoE, depth ordering, SAM) are removed;
           load-bearing components (HNM, multi-aspect) are retained.
           HNM $k$ denotes the number of \emph{candidate wrong-label
           draws} (with replacement; see App.~\ref{app:loss_formulas},
           Algorithm~\ref{alg:hnm}); $k$ is raised from $8{\to}16$ on
           CIFAR-10 (10 classes) to $20{\to}40$ on Tiny ImageNet (200
           classes), i.e., $2.5{\times}$ more draws for $20{\times}$ more
           classes. The ramp is \emph{not} coverage-matched: the expected
           unique-label coverage $1-(1-1/(C-1))^k$ is $0.61$--$0.85$ on
           CIFAR-10 ($k{=}8$--$16$, $C{=}10$) but only $0.096$--$0.18$ on
           Tiny ImageNet ($k{=}20$--$40$, $C{=}200$), so each Tiny
           negative is mined from a far smaller share of the wrong labels.
           Accuracies are top-1 / top-5 on the official CIFAR-10 test
           set and on the Tiny ImageNet labeled validation split.}
  \label{tab:tiny_imagenet_config}
  \centering
  \small
  \begin{tabular}{lcc}
    \toprule
    Parameter & CIFAR-10 ($\gamma{=}0$) & Tiny ImageNet \\
    \midrule
    Classes         & 10  & 200 \\
    Image size      & $32{\times}32$ & $64{\times}64$ \\
    d\_model        & 256 & 256 \\
    Blocks          & 4   & 4 \\
    Heads           & 8   & 8 \\
    Patch size      & 2   & 4 \\
    Stem channels   & 384 & 384 \\
    MoE experts     & 32 (top-4) & 1 (dense MLP) \\
    Batch size      & 512 & 256 \\
    HNM $k$         & $8 \to 16$ & $20 \to 40$ \\
    $\gamma$        & 0.0 & 0.0 \\
    Depth order $\lambda$ & 0.0 (s42) / 0.2 (s123, s456) & 0.0 \\
    SAM             & Yes & No \\
    SupCon $\lambda$ & 0.5 & 0.5 \\
    \midrule
    S1 Top-1 (TTA)  & 91.45\% & 49.44\% \\
    S2 Top-1 (TTA)  & 91.39\% & \textbf{52.60\%} \\
    S2 Top-5 (TTA)  & --- & 77.10\% \\
    \bottomrule
  \end{tabular}
\end{table}

The Tiny ImageNet experiment validates that the ablation-derived configuration
transfers to a substantially harder task. Three observations are noteworthy:

\begin{enumerate}[noitemsep, leftmargin=*]
  \item \textbf{Stage-2 head gains increase with task difficulty.} Measured
        as S2~TTA minus S1~TTA, the attentive head contributes
        $+3.23 \pm 0.06$\,pp on Tiny ImageNet (3 seeds; $+3.16$\,pp on
        seed~42, Table~\ref{tab:tiny_imagenet_config}) vs.\ $+2.30 \pm 0.26$\,pp on CIFAR-100
        ($\gamma{=}0$, 3 seeds) and $+0.03$\,pp on CIFAR-10 (mean over 12
        seed-42 variants, App.~\ref{sec:secondary_analysis}; $-0.06$\,pp
        for the matched $\gamma{=}0$ seed-42 run in Table~\ref{tab:tiny_imagenet_config}),
        suggesting that cross-class attention becomes more valuable as the
        number of confusable classes grows.
  \item \textbf{Per-depth accuracy is monotonically increasing.} Under
        $\gamma{=}0$, each block contributes useful discrimination:
        $d_1{=}35.71\%$, $d_2{=}46.10\%$, $d_3{=}48.41\%$,
        $d_4{=}48.75\%$ (S1, no-TTA). All four blocks are load-bearing.
  \item \textbf{Removing MoE was essential for scaling.} With 32 experts at
        $d{=}384$, the original configuration OOMed at batch sizes above 64
        on A100-40GB. The ablation-validated dense MLP at $d{=}256$ enabled
        batch size 256, which is critical for contrastive learning quality
        (${\sim}1.28$ same-class samples per batch vs.\ ${\sim}0.32$ at
        batch size 64).
\end{enumerate}

Training cost: ${\sim}29$ A100-SXM4-40GB GPU-hours per seed ($28.6$--$28.8$\,h
for 362 Stage-1 + 20 Stage-2 epochs). All three seeds were resumed across
sessions, so the cost is estimated from the median per-epoch time in the run
logs (correction: App.~\ref{app:changes}, item B11).

\begin{figure}[htbp]
  \centering
  \begin{minipage}[t]{0.49\textwidth}
  \centering
  \includegraphics[width=\linewidth]{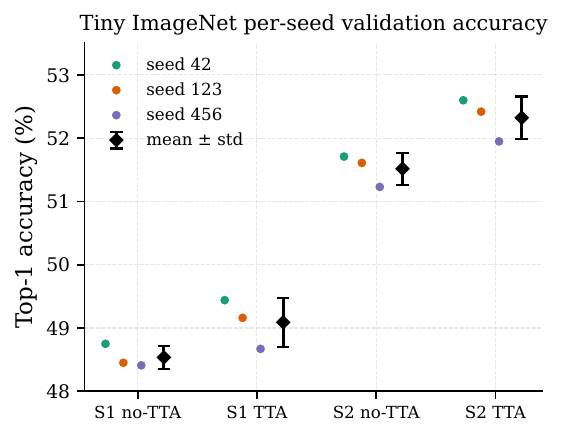}
  \caption{Tiny ImageNet per-seed accuracy on the labeled validation
           split for seeds 42, 123 and 456, S1/S2 $\times$ no-TTA/TTA
           (TTA and S1 no-TTA values in
           Table~\ref{tab:tiny_imagenet_multiseed}; S2 no-TTA
           $51.71$/$51.61$/$51.23\%$). Black diamonds with error bars
           show the mean $\pm$ sample std across seeds; every std is
           below $0.4$\,pp.}
  \label{fig:tinyimagenet_perseed}
  \end{minipage}\hfill
  \begin{minipage}[t]{0.49\textwidth}
  \centering
  \includegraphics[width=\linewidth]{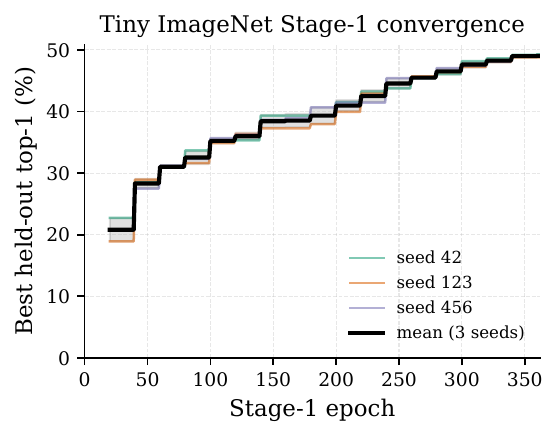}
  \caption{Tiny ImageNet Stage-1 convergence: best-so-far Stage-1
           flip-TTA top-1 on the trainer's 10k held-out subset of the
           training images, used for checkpoint selection (not the labeled
           validation split of Fig.~\ref{fig:tinyimagenet_perseed}),
           evaluated every 20 epochs from epoch~20, for the three seeds
           (thin colored lines), each joined across its resumed training
           segments, and their mean $\pm$ sample std band (black). The
           three trajectories nearly coincide.}
  \label{fig:tinyimagenet_dynamics}
  \end{minipage}
\end{figure}

\renewcommand{\subsection}{\origsubsection}
\section{Per-Class Accuracy}
\label{app:per_class}

\begin{table}[htbp]
  \caption{Per-class test accuracy (S1, TTA) across seven completed variants.
           \textbf{Cat} is universally the hardest class (74.5--80.2\%), while
           ship, truck, and automobile consistently exceed 94\%.
           Bold = best per class.}
  \label{tab:per_class}
  \centering
  \small
  \begin{tabular}{lccccccc}
    \toprule
    Class & baseline & nomem-l2 & fixed-nl & pfx-frz & cp-fair & cp-lc & dp-fair \\
    \midrule
    airplane    & \textbf{94.1} & 93.1 & 93.8 & 93.0 & 93.4 & 93.4 & 92.7 \\
    automobile  & 95.6 & 95.6 & \textbf{96.5} & 95.3 & 95.7 & 96.0 & 96.0 \\
    bird        & 88.2 & \textbf{89.8} & 89.5 & 88.3 & 89.6 & 88.8 & 89.4 \\
    \textbf{cat}& 78.7 & 77.5 & 74.5 & \textbf{80.2} & 79.9 & 77.9 & 79.6 \\
    deer        & 90.5 & 90.6 & \textbf{92.7} & 91.1 & 90.7 & 91.2 & 91.3 \\
    dog         & 84.3 & 84.9 & 84.4 & 84.0 & 85.3 & \textbf{85.4} & 82.9 \\
    frog        & 95.0 & \textbf{96.0} & \textbf{96.0} & 95.0 & 94.8 & 94.3 & 94.2 \\
    horse       & 94.2 & \textbf{94.6} & 93.8 & 92.9 & 94.5 & 95.1 & 94.4 \\
    ship        & \textbf{96.2} & 96.0 & 95.9 & 94.5 & 94.8 & 96.0 & 95.8 \\
    truck       & \textbf{95.2} & 95.0 & 94.4 & \textbf{95.2} & 94.9 & 95.1 & 94.9 \\
    \midrule
    \textbf{Overall} & 91.20 & 91.31 & 91.15 & 90.95 & 91.36 & 91.32 & 91.12 \\
    \bottomrule
  \end{tabular}
\end{table}

Cat is universally the hardest class at 74.5--80.2\% across all variants,
12--17 percentage points below the overall average. Cat$\leftrightarrow$dog
confusion dominates every variant's error list. Ship, truck, and automobile
are the easiest classes ($>94\%$ in all variants). No single variant dominates
all classes: \textsc{prefix-freeze} wins cat (80.2\%), \textsc{fixed-nl} wins
automobile, deer, and frog, while \textsc{cp-fair-lc} wins dog (85.4\%).
This suggests that architectural and loss differences create complementary
per-class strengths.

\section{Deep-Dive Ablations of Load-Bearing Components}
\label{sec:component_deep_dives}

The component ablation (Appendix~\ref{app:components},
Table~\ref{tab:component_ablation}) identified hard negative mining
($-0.45$\,pp) and multi-aspect goodness ($-0.41$\,pp) as the only two
load-bearing components. The following subsections decompose each in
detail, and additionally isolate the EMA teacher contribution.

\subsection{\texorpdfstring{Hard-Negative Mining: $k$-Sweep and NL/NI Decomposition}{Hard-Negative Mining: k-Sweep and NL/NI Decomposition}}
\label{sec:hnm_deep_dive}

The component ablation (Table~\ref{tab:component_ablation}) replaced the
adaptive $k{=}8{\to}16$ hard negative mining schedule with random negatives
($k{=}1$), yielding a $-0.45$\,pp drop. Here we sweep the hardness
parameter $k$ at four fixed levels and separately decompose the two
negative types: wrong-label (NL) and wrong-image (NI). All runs use the
full CP-FAIR configuration (L4, D256, 362 epochs, seed~42).

\begin{table}[htbp]
  \caption{Hard-negative mining decomposition. The baseline uses an
           adaptive $k{=}8{\to}16$ schedule with 50/50 NL/NI blend.
           Each row modifies one aspect of negative sampling; $\Delta$ is
           the change in S2 TTA relative to the full model ($91.24\%$).}
  \label{tab:hnm_ksweep}
  \centering
  \small
  \begin{tabular}{lcccc}
    \toprule
    Configuration & S1 TTA & S2 (no-TTA) & S2 (TTA) & $\Delta$ (pp) \\
    \midrule
    Full CP-FAIR baseline ($k{=}8{\to}16$) & 91.31\% & 90.42\% & 91.24\% & --- \\
    \midrule
    $k{=}1$ (random negatives)   & 90.76\% & 89.93\% & 90.79\% & $-0.45$ \\
    $k{=}4$ (mild hardness)      & 91.27\% & 90.32\% & 91.43\% & $+0.19$ \\
    $k{=}8$ (moderate, flat)     & 90.93\% & 90.49\% & 90.99\% & $-0.25$ \\
    $k{=}16$ (hard, flat)        & 91.44\% & 90.61\% & 91.33\% & $+0.09$ \\
    NL only (wrong-label neg.)   & 91.33\% & 90.24\% & 91.24\% & $\phantom{+}0.00$ \\
    NI only (wrong-image neg.)   & 90.31\% & 89.50\% & 90.37\% & $-0.87$ \\
    \bottomrule
  \end{tabular}
\end{table}

Table~\ref{tab:hnm_ksweep} reveals three
findings. First, the accuracy response to hardness $k$ is
\emph{non-monotonic}: $k{=}4$ (mild hardness) is the best fixed
setting ($+0.19$\,pp), while $k{=}8$ is worse ($-0.25$\,pp) and
$k{=}16$ partially recovers ($+0.09$\,pp). The current adaptive
$8{\to}16$ schedule should therefore be read as a robustness default,
not as evidence that the schedule traverses the empirically optimal
$k$ range; on this sweep, fixed $k{=}4$ outperforms it slightly.

Second, the NL/NI decomposition is striking: \textbf{wrong-label negatives
carry 100\% of the discrimination signal}. NL-only exactly matches the full
baseline ($\Delta{=}0.00$\,pp), while NI-only produces the largest single
drop in either ablation set ($-0.87$\,pp). Wrong-image negatives, which
pair a correct label with a mismatched image, provide no useful separation
signal in isolation.

Third, the baseline's 50/50 NL/NI blend is well-tuned: NL provides all
discrimination, and NI adds diversity without harming performance ($\Delta
{=}0.00$\,pp when NI is removed). This suggests that NI negatives serve as
regularization rather than a learning signal.

\subsection{Multi-Aspect Goodness: Aspect Isolation and EMA Teacher}
\label{sec:aspect_ema_deep_dive}

The component ablation showed that replacing the 4-slot (3 active)
goodness head with a single slot (prototype alignment) costs
$-0.41$\,pp. Here we isolate each of the four slots individually
(keeping one aspect-loss weight at its original $\lambda$ while zeroing
the other three) and separately test the EMA teacher's contribution. All
runs use the full CP-FAIR configuration (L4, D256, 362 epochs, seed~42,
memory disabled). Only three slots are active: with memory disabled the
attention-sharpness slot is identically zero
(App.~\ref{app:loss_formulas}), so its rank loss is the constant
$\log 2$ and has zero gradient. The ``attn-sharp slot only'' row
therefore trains with \emph{no} aspect-loss gradient and serves as a
no-aspect-loss control. In every row all four slots still enter the
block goodness through the learned mixing weights; only the per-slot
aspect losses are switched off.

\begin{table}[htbp]
  \caption{Aspect isolation and EMA ablation. Each ``X-only'' row keeps
           one aspect-loss weight at its original $\lambda$ and zeros the
           other three. $\Delta$ is relative to the full model ($91.24\%$).
           The energy aspect alone surpasses the full 4-slot (3 active)
           model; removing EMA is the most damaging single change
           ($-0.34$\,pp). The attn-sharp slot is zero with memory
           disabled, so that row has no aspect-loss gradient at all.}
  \label{tab:aspect_ema_ablation}
  \centering
  \small
  \begin{tabular}{lcccc}
    \toprule
    Configuration & S1 TTA & S2 (no-TTA) & S2 (TTA) & $\Delta$ (pp) \\
    \midrule
    Full CP-FAIR (4 slots, 3 active; EMA) & 91.31\% & 90.42\% & 91.24\% & --- \\
    \midrule
    Proto alignment only ($\lambda_\text{align}{=}1.5$) & 91.16\% & 90.17\% & 90.97\% & $-0.27$ \\
    Energy (rank) only ($\lambda_\text{energy}{=}0.2$) & 91.13\% & 90.14\% & 91.51\% & $+0.27$ \\
    No aspect loss (attn-sharp only, $\equiv 0$) & 91.09\% & 90.33\% & 91.26\% & $+0.02$ \\
    Learned head only ($\lambda_\text{learned}{=}0.3$) & 91.34\% & 90.51\% & 91.33\% & $+0.09$ \\
    \midrule
    No EMA teacher (eval-EMA also disabled)  & 90.91\% & 90.19\% & 90.90\% & $-0.34$ \\
    EMA eval retained; HNM uses online model & 91.54\% & 90.96\% & 91.48\% & $+0.24$ \\
    \bottomrule
  \end{tabular}
\end{table}

Table~\ref{tab:aspect_ema_ablation} reveals several findings about the multi-aspect goodness design.

\paragraph{Aspect redundancy.} No single aspect is catastrophic to lose:
all four slot-only conditions maintain $>90.9\%$ S2 TTA (within $0.3$\,pp
of the full model). One of the four conditions is the attn-sharp
row, which has no aspect-loss gradient at all, and it matches the full
model ($+0.02$\,pp). The per-slot aspect losses are therefore largely
redundant with the block-level losses in this configuration; the
ablation does not show that the three active aspects complement each
other.

\paragraph{Energy dominates.} The energy (rank) aspect alone achieves
$91.51\%$, \emph{surpassing} the full 4-slot model by $+0.27$\,pp. This
is the highest single result in either ablation set. In contrast, the proto
alignment aspect---despite carrying the highest $\lambda$ ($1.5$ vs.\ $0.2$
for energy)---is the weakest individual aspect ($-0.27$\,pp). This
$\lambda$--performance inversion suggests that the current weighting
over-emphasizes proto alignment at the expense of the more effective energy
signal, and that tuning the aspect ratios could yield further gains.

\paragraph{EMA is load-bearing but not for HNM.} Removing EMA entirely
costs $-0.34$\,pp, confirming it as a genuinely load-bearing component
(comparable to multi-aspect goodness at $-0.41$\,pp). However, removing
the hard-negative mining role from EMA while retaining evaluation smoothing
\emph{improves} accuracy by $+0.24$\,pp. This indicates that EMA's value
lies in providing a stable evaluation target for the goodness heads,
not in its role as a negative-mining teacher.

\section{Text-Domain Experiments: Full Protocol}
\label{app:textdomain}

This appendix expands Section~\ref{sec:textdomain}. The text-domain
experiments serve as a modality-transfer sanity check only; we rely on
the CIFAR-10 / CIFAR-100 / Tiny~ImageNet experiments for the main
mechanistic claims (see the caveat at the end of this section).

\paragraph{Architecture.} We reuse the FF hybrid block (L4, D256, 8
heads, 16 experts, $\gamma{=}0$) without modification. Frozen
\texttt{bert-base-uncased}~\citep{devlin2019bert} sentence embeddings
(768-d pooled) are linearly projected and reshaped into a sequence
of 8 virtual tokens, which the FF blocks consume exactly as image
patches in the vision experiments. Datasets:
IMDb~\citep{maas2011imdb}, 20~Newsgroups~\citep{lang1995newsweeder},
and AG~News~\citep{zhang2015charcnn}.

\paragraph{Training.} Stage-1 trains the backbone with local FF (AdamW,
no SAM) + cosine warmup for 25--40 epochs; Stage-2 trains an attentive head on
frozen features for 8--10 epochs. Remaining hyperparameters (learning
rate $10^{-3}$, weight decay $0.05$, gradient clip $1.0$) match the
CIFAR-10 recipe. Three seeds (42, 43, 44) per dataset.

\paragraph{Per-seed results.} All 9 runs (3 datasets $\times$ 3 seeds)
exceed their published FF targets on Stage-2 mean accuracy with
across-seed std ${\leq}0.3$\,pp. The Stage-1 (backbone-only) numbers
already clear every published target on IMDb and AG~News before the
attentive head is fit; Stage-2 contributes only $0.2$--$1.0$\,pp on
those datasets.

\begin{table}[htbp]
  \caption{Text-domain, full per-seed means. Stage-1 / Stage-2 are mean
           $\pm$ std across three seeds (42, 43, 44). \emph{Best} is the
           single strongest Stage-2 run. $\Delta$ is Stage-2 minus
           Target, computed from unrounded values.}
  \label{tab:text_domain_full}
  \centering
  \small
  \begin{tabular}{lcccccc}
    \toprule
    Dataset & $K$ & Stage-1 & Stage-2 & Best & Target & $\Delta$ \\
    \midrule
    IMDb          & 2  & $86.97 \pm 0.12$ & $\mathbf{87.25 \pm 0.15}$ & $87.41$ & $84.86$~\citep{extendff2023}     & $+2.39$ \\
    20 Newsgroups & 20 & $61.84 \pm 0.28$ & $\mathbf{62.66 \pm 0.24}$ & $62.92$ & $61.63$~\citep{ffcomplete2024}   & $+1.02$ \\
    AG News       & 4  & $93.38 \pm 0.13$ & $\mathbf{93.36 \pm 0.07}$ & $93.45$ & $91.29$~\citep{forwardlayer2024} & $+2.07$ \\
    \bottomrule
  \end{tabular}
\end{table}

\paragraph{Caveat.} The text setup relies on a frozen BERT encoder
(trained with BP), so the FF blocks operate on top of pre-learned
representations. It therefore tests whether the same hybrid block
architecture generalizes to another modality, not whether FF alone can
learn text classifiers from scratch. We make no claim stronger than
that in the main text. The AG~News target is taken from
\citet{forwardlayer2024} (layer-wise local with triplet loss; not
strict FF in the sense of Hinton's goodness function); IMDb and
20~Newsgroups targets are strict-FF baselines.

\section{Training-Cost and Locality Trade-offs}
\label{app:efficiency}

Beyond accuracy, the FF-vs-BP comparison has an efficiency axis.
Table~\ref{tab:efficiency} summarizes the trade-offs we observe on the
same co-designed backbone. FF pays a ${\sim}2.2{\times}$ wall-clock
cost per epoch on our hardware---driven by sequential per-block
forward/backward passes with SAM---but buys two concrete locality
benefits: activation storage during backward is ${\sim}L{\times}$
smaller (one block's activations at a time instead of all $L$), and
each backward updates only ${\sim}1/L$ of the parameters. These are
the classical FF motivations and they remain intact under
$\gamma{=}0$. The wall-clock and memory numbers should be read as
hardware-dependent (Google~Colab A100-SXM4-40GB); relative ordering is
reproducible.

\begin{table}[htbp]
  \caption{FF-vs-BP training-cost and locality trade-offs on the same
           FF-stripped backbone (CIFAR-10, $L{=}4$, $D{=}256$, 362
           epochs, seed 42). Wall-clock is the median per-epoch time
           from \texttt{history.jsonl}; params-per-backward is the
           number of parameters receiving non-zero gradient in a single
           \texttt{.backward()} (FF: one block at a time, BP: the whole
           model); activation memory is the theoretical peak scaling
           factor ($L$ blocks for BP vs.\ 1 for FF). Augmentation
           protocol is matched across FF and BP(strong).}
  \label{tab:efficiency}
  \centering
  \small
  \setlength{\tabcolsep}{4pt}
  \begin{tabular}{llcccc}
    \toprule
    Backbone & Train & Params & t/epoch & Params/backward & Act.\ memory \\
    \midrule
    FF-stripped & BP (weak aug)    & 5.54\,M &  26.9\,s & 5.54\,M (all) & $L{\cdot}A$ \\
    FF-stripped & BP (strong aug)  & 5.54\,M &  53.0\,s & 5.54\,M (all) & $L{\cdot}A$ \\
    FF-stripped & FF ($\gamma{=}0$) & 6.74\,M & 117.3\,s & ${\sim}1.68$\,M (1/$L$) & $1{\cdot}A$ \\
    \bottomrule
  \end{tabular}

  \vspace{1ex}
  \begin{flushleft}\footnotesize
  FF uses the same 4-block backbone plus per-block aggregation heads, so
  total parameter count is marginally larger than the BP variant
  (${\sim}{+}1.2$\,M). Params/backward for FF is total / $L$ (uniform
  block sizes); the \texttt{verify\_locality.py} harness
  (App.~\ref{app:locality}) confirms
  that in the patched implementation the cumulative-goodness path
  puts no gradient on blocks $0{\ldots}\ell{-}1$ during block-$\ell$'s
  backward (the depth-order exception of App.~\ref{app:locality} adds a discarded gradient
  on block $\ell{-}1$'s four mixing logits).
  \end{flushleft}
\end{table}

\section{Detailed Architecture and Training Diagrams}
\label{sec:arch_diagrams}

This appendix collects four diagrams: the architecture
(Fig.~\ref{fig:architecture}), the per-block loss decomposition
(Fig.~\ref{fig:loss_decomposition}), the multi-aspect goodness head and
its use in the training and prediction scores
(Fig.~\ref{fig:goodness_aspects}), and the Stage-1 training loop
(Fig.~\ref{fig:stage1_training}).

\begin{figure*}[t]
\centering
\scalebox{0.55}{%
\begin{tikzpicture}[
    >=latex,
    node distance=0.4cm and 0.5cm,
    myshadow/.style={blur shadow={shadow blur steps=5, shadow xshift=0.7pt,
                     shadow yshift=-0.7pt, shadow opacity=25}},
    block/.style={draw, rounded corners=3pt, minimum height=0.6cm, minimum width=1.4cm,
                  font=\scriptsize, align=center, thick, myshadow},
    blocksmall/.style={draw, rounded corners=3pt, minimum height=0.45cm, minimum width=1.1cm,
                       font=\scriptsize, align=center, myshadow},
    blocktiny/.style={draw, rounded corners=2pt, minimum height=0.28cm, minimum width=0.8cm,
                      font=\tiny, align=center, myshadow},
    input/.style={block, fill=blue!15, draw=blue!55},
    stem/.style={blocksmall, fill=blue!15, draw=blue!55},
    ffblock/.style={block, fill=blue!15, draw=blue!60, minimum width=2.4cm, minimum height=0.7cm},
    labelnode/.style={blocksmall, fill=yellow!18, draw=yellow!65!black},
    head/.style={block, fill=purple!18, draw=purple!65},
    headsmall/.style={blocksmall, fill=purple!15, draw=purple!65},
    lossnode/.style={blocksmall, fill=red!12, draw=red!55, font=\scriptsize},
    procnode/.style={blocksmall, fill=gray!10, draw=gray!50},
    attnnode/.style={blocksmall, fill=blue!15, draw=blue!55},
    moenode/.style={blocksmall, fill=orange!20, draw=orange!55},
    goodnessnode/.style={blocksmall, fill=green!20, draw=green!65!black},
    arrow/.style={->, thick, draw=gray!60, >=stealth},
    posarrow/.style={->, thick, draw=green!60!black, >=stealth},
    negarrow/.style={->, thick, draw=red!50, >=stealth},
    dashedarrow/.style={->, dashed, draw=gray!50, >=stealth},
    emaarrow/.style={->, dashed, thick, draw=teal!60, >=stealth},
    redarrow/.style={->, thick, draw=red!50, >=stealth},
]

\begin{scope}[shift={(0,0)}]

\node[font=\small\bfseries, text=blue!60!black] at (3.5, 4.8) {(a) Stage 1: Forward-Forward Backbone};

\node[input, minimum width=1.5cm] (img) at (3.5, 4.0) {Image\\{\tiny $[B,3,32,32]$}};

\node[stem, minimum width=2.8cm] (stempatch) at (3.5, 3.0)
    {ConvStem + PatchEmbed};
\node[font=\tiny, text=gray!55, anchor=west] at (5.0, 3.0) {$\to$ 64 tokens, $D{=}256$};

\draw[arrow] (img) -- (stempatch);

\node[labelnode, minimum width=1.6cm] (ylabel) at (2.3, 1.9)
    {$y$ (correct)};
\node[blocksmall, fill=red!10, draw=red!50, minimum width=1.6cm] (yhat) at (4.7, 1.9)
    {$y^{-}$ (wrong)};

\draw[posarrow] (stempatch.south) -- ++(0, -0.3) -| (ylabel.north);
\draw[negarrow] (stempatch.south) -- ++(0, -0.3) -| (yhat.north);

\node[labelnode, minimum width=2.8cm, minimum height=0.55cm] (labelcond) at (3.5, 1.05)
    {Label Conditioning\\{\tiny concat emb $\to$ proj to $D$}};

\draw[posarrow] (ylabel.south) -- ++(0, -0.15) -| ([xshift=-0.4cm]labelcond.north);
\draw[negarrow] (yhat.south) -- ++(0, -0.15) -| ([xshift=0.4cm]labelcond.north);

\node[font=\tiny, text=teal!70, align=center] (ema) at (7.0, 1.9)
    {EMA teacher};
\draw[emaarrow] (ema.west) -- (yhat.east);
\node[below=0.01cm of ema, font=\tiny, text=teal!50] {hard neg.\ mining};

\node[font=\tiny\bfseries, text=gray!60] at (7.0, 4.5) {Training streams:};
\fill[green!50] (7.0, 4.2) rectangle ++(0.22, 0.1);
\node[font=\tiny, text=gray!55, anchor=west] at (7.3, 4.25) {Pos: $x_{1,2}$ + $y$};
\fill[red!40] (7.0, 3.95) rectangle ++(0.22, 0.1);
\node[font=\tiny, text=gray!55, anchor=west] at (7.3, 4.0) {Neg: NL / NI};

\def\blockspacing{1.5}

\node[ffblock] (b0) at (3.5, 0) {FFHybridBlock 0};
\node[blocktiny, fill=gray!12, draw=gray!40] (l2n0) at (3.5, -0.65) {$L_2$ norm + detach};

\node[ffblock] (b1) at (3.5, {-1*\blockspacing}) {FFHybridBlock 1};
\node[blocktiny, fill=gray!12, draw=gray!40] (l2n1) at (3.5, {-1*\blockspacing - 0.65}) {$L_2$ norm + detach};

\node[ffblock] (b2) at (3.5, {-2*\blockspacing}) {FFHybridBlock 2};
\node[blocktiny, fill=gray!12, draw=gray!40] (l2n2) at (3.5, {-2*\blockspacing - 0.65}) {$L_2$ norm + detach};

\node[ffblock] (b3) at (3.5, {-3*\blockspacing}) {FFHybridBlock 3\\{\tiny $[B,64,256]$}};

\draw[arrow] (b3.south) -- ++(0, -0.8)
    node[right, font=\tiny, text=gray!55] {$[B,64,256]$};

\draw[arrow] (labelcond.south) -- (b0.north);
\draw[arrow] (b0) -- (l2n0);
\draw[arrow] (l2n0) -- (b1);
\draw[arrow] (b1) -- (l2n1);
\draw[arrow] (l2n1) -- (b2);
\draw[arrow] (b2) -- (l2n2);
\draw[arrow] (l2n2) -- (b3);

\coordinate (lcbus_top) at (6.2, 0.5);
\coordinate (lcbus_bot) at (6.2, {-3*\blockspacing});
\draw[->, thick, draw=yellow!60!black] (labelcond.east) -| (lcbus_top);
\draw[thick, draw=yellow!60!black] (lcbus_top) -- (lcbus_bot);
\draw[->, thick, draw=yellow!60!black] (6.2, 0) -- (b0.east);
\draw[->, thick, draw=yellow!60!black] (6.2, {-1*\blockspacing}) -- (b1.east);
\draw[->, thick, draw=yellow!60!black] (6.2, {-2*\blockspacing}) -- (b2.east);
\draw[->, thick, draw=yellow!60!black] (6.2, {-3*\blockspacing}) -- (b3.east);
\node[font=\tiny, text=yellow!60!black, rotate=-90, anchor=south] at (6.5, {-1.5*\blockspacing})
    {label emb};

\draw[->, thick, draw=green!60!black] (b0.west) -- +(-1.0, 0)
    node[left, font=\tiny, text=green!50!black] {$g^{(0)}$};
\draw[->, thick, draw=green!60!black] (b1.west) -- +(-1.0, 0)
    node[left, font=\tiny, text=green!50!black] {$g^{(1)}$};
\draw[->, thick, draw=green!60!black] (b2.west) -- +(-1.0, 0)
    node[left, font=\tiny, text=green!50!black] {$g^{(2)}$};
\draw[->, thick, draw=green!60!black] (b3.west) -- +(-1.0, 0)
    node[left, font=\tiny, text=green!50!black] {$g^{(3)}$};

\draw[dashedarrow, green!50!black]
    ([xshift=-0.12cm]b0.south west) to[bend right=18]
    node[left, font=\tiny, text=green!50!black] {$g_\text{prev}$}
    ([xshift=-0.12cm]b1.north west);
\draw[dashedarrow, green!50!black]
    ([xshift=-0.12cm]b1.south west) to[bend right=18]
    ([xshift=-0.12cm]b2.north west);
\draw[dashedarrow, green!50!black]
    ([xshift=-0.12cm]b2.south west) to[bend right=18]
    ([xshift=-0.12cm]b3.north west);

\begin{scope}[on background layer]
    \draw[dashed, blue!30, rounded corners=5pt, thick]
        ([xshift=-1.8cm, yshift=0.4cm]b0.north west) rectangle
        ([xshift=3.5cm, yshift=-1.0cm]b3.south east);
    \node[font=\tiny\itshape, text=blue!40, anchor=west] at (0.7, {-3*\blockspacing - 1.1})
        {local FF learning};
\end{scope}

\end{scope}

\begin{scope}[shift={(12.0, 3.0)}]

\node[font=\small\bfseries, text=blue!60!black] at (2.2, 1.8) {(b) FFHybridBlock};

\node[font=\scriptsize, text=gray!60] (tokin) at (2.2, 1.0) {tokens $[B,N,D]$};

\node[labelnode, minimum width=2.4cm] (lbl) at (2.2, 0.25) {+ Label emb $y$\;{\tiny $[64]$}};
\node[font=\tiny, text=gray!55, anchor=west] at (2.3, -0.12) {broadcast + concat + proj $\to D$};
\draw[arrow] (tokin) -- (lbl);

\node[attnnode, minimum width=2.6cm] (sa) at (2.2, -0.75)
    {Self-Attention\\{\tiny RoPE, 8 heads}};
\draw[arrow] (lbl) -- (sa);

\node[moenode, minimum width=2.6cm] (moe) at (2.2, -1.7)
    {Top-$k$ MoE FFN\\{\tiny 32 experts, $k{=}4$, GEGLU}};
\draw[arrow] (sa) -- (moe);

\coordinate (branchpt) at (2.2, -2.35);
\draw[arrow] (moe) -- (branchpt);

\node[goodnessnode, fill=green!18, draw=green!60!black, minimum width=3.2cm, minimum height=0.7cm]
    (goodness) at (2.2, -3.2)
    {aspects $[g_a]$ $\to$ $\displaystyle\sum_a \pi_a g_a$\\{\tiny $=\; g^{(d)}$ (scalar goodness)}};
\draw[arrow, green!50!black] (branchpt) -- (goodness);

\node[font=\tiny, text=gray!45, anchor=east, align=right] at (0.5, -3.2)
    {detail:\\Fig.~\ref{fig:goodness_aspects}};

\node[lossnode, minimum width=1.6cm] (discloss) at (2.2, -4.3)
    {$\to$ discrimination\\loss};
\draw[redarrow] (goodness) -- (discloss);

\node[procnode, minimum width=1.0cm] (relu) at (4.6, -3.4) {ReLU};
\node[procnode, minimum width=2.0cm] (l2out) at (2.2, -5.35)
    {$L_2$ norm + detach};
\node[font=\scriptsize, text=gray!50] at (2.2, -5.85) {$\to$ next block};

\draw[arrow] (branchpt) -- (4.6, -2.35) -- (relu.north);
\node[font=\tiny, text=gray!50, anchor=west] at (4.65, -2.8) {tokens $h$};
\draw[arrow] (relu.south) |- (l2out.east);

\node[font=\tiny, text=gray!40, anchor=east] at (0.3, -1.2) {+ residual};
\node[font=\tiny, text=gray!40, anchor=east] at (0.3, -2.1) {+ residual};

\end{scope}

\begin{scope}[shift={(18.5, 3.0)}]

\node[font=\small\bfseries, text=purple!70] at (2.5, 1.8) {(c) Stage 2: AttentiveHybridHead};

\node[block, fill=blue!6, draw=blue!40, minimum width=3.0cm] (frozen) at (2.5, 1.0)
    {Frozen Stage-1\\Backbone};
\node[font=\tiny, text=blue!50, anchor=west] at (4.1, 1.15) {frozen};

\node[labelnode, minimum width=3.2cm] (hypo) at (2.5, 0.0)
    {Run all $k = 0,\ldots,9$\\{\tiny 10 label hypotheses}};
\draw[arrow] (frozen) -- (hypo);

\node[blocksmall, fill=blue!8, draw=blue!50, minimum width=3.2cm] (feats) at (2.5, -1.0)
    {Extract features\\{\tiny $[B, 10, 1040]$}};
\draw[arrow] (hypo) -- (feats);
\node[font=\tiny, text=gray!50, anchor=west] at (4.2, -1.0) {goodness + pooled};

\node[headsmall, minimum width=3.2cm] (proj2) at (2.5, -1.9)
    {RMSNorm $\to$ Linear $\to$ GELU\\{\tiny $[B, 10, 1536]$}};
\draw[arrow] (feats) -- (proj2);

\node[block, fill=purple!22, draw=purple!70, thick, minimum width=3.6cm, minimum height=0.8cm,
      font=\scriptsize\bfseries] (classattn) at (2.5, -3.1)
    {Self-Attention across\\10 class hypotheses};
\draw[arrow] (proj2) -- (classattn);

\draw[purple!55, thick, rounded corners=3pt, dashed]
    ([xshift=-0.15cm, yshift=0.12cm]classattn.north west) rectangle
    ([xshift=0.15cm, yshift=-0.12cm]classattn.south east);

\node[font=\tiny, text=purple!60, align=left, text width=2.2cm, anchor=west] at (4.5, -3.1)
    {models inter-class\\competition};

\node[headsmall, minimum width=3.2cm] (ffn2) at (2.5, -4.1)
    {GEGLU FFN + residual};
\draw[arrow] (classattn) -- (ffn2);

\node[block, fill=purple!28, draw=purple!75, minimum width=2.6cm, font=\scriptsize\bfseries]
    (logits) at (2.5, -4.95)
    {Linear $\to$ 10 logits};
\draw[arrow] (ffn2) -- (logits);

\node[lossnode] (celoss) at (2.5, -5.65) {Cross-Entropy};
\draw[redarrow] (logits) -- (celoss);

\node[font=\tiny, text=gray!50, align=center] at (2.5, -6.15)
    {only head params trained};

\end{scope}

\draw[->, thick, dashed, draw=purple!50]
    (3.5, -5.3) -- (3.5, -6.2) -- (18.8, -6.2)
    node[midway, below, font=\tiny, text=purple!50] {freeze backbone}
    -- (18.8, 4.0) -- (frozen.west);

\end{tikzpicture}%
}
\caption{%
\textbf{(a)}~Stage-1 backbone: an image is tokenized and processed
with a hypothesized label through four locally trained FF Hybrid
Blocks. Label conditioning is applied at every block (yellow bus
line). Each block produces a scalar goodness $g^{(d)}$, trained high
for the correct label and low for incorrect ones. Gradient
detachment between blocks enforces local learning; an EMA teacher
provides hard negative mining.
\textbf{(b)}~FF Hybrid Block internals: tokens concatenated with a
label embedding pass through self-attention (RoPE) and a sparse MoE
FFN. Multi-aspect goodness is computed from the post-MoE tokens
(attention-pooled for the prototype-alignment and learned-head
aspects) and
aggregated into a scalar $g^{(d)}$ (see
Fig.~\ref{fig:goodness_aspects}); ReLU, $L_2$ normalization and
detachment apply only to the tokens passed to the next block.
Auxiliary contrastive and reconstruction losses are omitted for
clarity.
\textbf{(c)}~Stage~2: the frozen backbone extracts features for all
10 class hypotheses; the Attentive Hybrid Head applies self-attention
\emph{across class hypotheses} (inter-class competition), followed by
a GEGLU block and per-class logit prediction; only the head is trained
(AdamW, cosine schedule) on per-hypothesis goodness features ($L{\times}A{=}4{\times}4$)
and pooled activations ($L{\times}D{=}4{\times}256$). Both stages support optional
horizontal-flip TTA (scores or logits averaged over the flipped image).%
}
\label{fig:architecture}
\end{figure*}

\begin{figure}[t]
\centering
\resizebox{\columnwidth}{!}{%
\begin{tikzpicture}[
    >=latex,
    myshadow/.style={blur shadow={shadow blur steps=5, shadow xshift=0.7pt,
                     shadow yshift=-0.7pt, shadow opacity=25}},
    lossnode/.style={draw, rounded corners=3pt, minimum height=0.45cm, minimum width=1.4cm,
                     font=\scriptsize, align=center, thick, fill=red!15, draw=red!50, myshadow},
    subloss/.style={draw, rounded corners=3pt, minimum height=0.38cm, minimum width=1.1cm,
                    font=\tiny, align=center, fill=red!10, draw=red!30, myshadow},
    highlight/.style={draw, rounded corners=3pt, minimum height=0.45cm, minimum width=1.4cm,
                      font=\scriptsize, align=center, thick, fill=green!20, draw=green!60!black,
                      line width=1.2pt, myshadow},
    rootnode/.style={draw, rounded corners=3pt, minimum height=0.55cm, minimum width=2.0cm,
                     font=\small\bfseries, align=center, thick, fill=red!18, draw=red!60, myshadow},
    depthbar/.style={fill=green!50!black, draw=none},
    arrow/.style={->, thick, draw=gray!60, >=stealth},
    label/.style={font=\tiny, text=gray!60},
]

\node[rootnode] (root) at (0, 0) {$\mathcal{L}^{(d)}_\text{total}$};

\node[lossnode] (aspects) at (-5.5, -1.5) {$\mathcal{L}_\text{aspects}$};
\node[lossnode] (block) at (-3.2, -1.5) {$\mathcal{L}_\text{block}$};
\node[highlight] (curr) at (-0.8, -1.5) {$\lambda_\text{curr}(d)\,\mathcal{L}_\text{curr}$};
\node[lossnode] (depth) at (1.6, -1.5) {$\mathcal{L}_\text{depth}$};
\node[lossnode] (supcon) at (3.5, -1.5) {$\mathcal{L}_\text{SupCon}$};
\node[lossnode] (moeloss) at (5.4, -1.5) {$\mathcal{L}_\text{MoE}$};
\node[lossnode] (reconloss) at (7.2, -1.5) {$\mathcal{L}_\text{recon}$};

\draw[arrow] (root) -- (aspects);
\draw[arrow] (root) -- (block);
\draw[arrow, green!60!black, line width=1.2pt] (root) -- (curr);
\draw[arrow] (root) -- (depth);
\draw[arrow] (root) -- (supcon);
\draw[arrow] (root) -- (moeloss);
\draw[arrow] (root) -- (reconloss);

\node[subloss] (align) at (-7.0, -3.0) {$g_\text{align}$\\$\lambda{=}1.5$};
\node[subloss] (energy) at (-5.8, -3.0) {$g_\text{energy}$\\$\lambda{=}0.2$};
\node[subloss, fill=gray!10, draw=gray!40, text=gray!50] (attns) at (-4.6, -3.0) {$g_\text{attn}$\\$\lambda{=}0.3$};
\node[subloss] (learned) at (-3.4, -3.0) {$g_\text{learned}$\\$\lambda{=}0.3$};

\draw[arrow, gray!40] (aspects) -- (align);
\draw[arrow, gray!40] (aspects) -- (energy);
\draw[arrow, dashed, gray!25] (aspects) -- (attns);
\draw[arrow, gray!40] (aspects) -- (learned);

\node[below=0.05cm of align, font=\tiny, text=gray!50] {margin};
\node[below=0.05cm of energy, font=\tiny, text=gray!50] {rank};
\node[below=0.05cm of attns, font=\tiny, text=gray!40] {inactive};
\node[below=0.05cm of learned, font=\tiny, text=gray!50] {margin};

\node[below=0.15cm of block, font=\tiny, text=gray!60, align=center] {cumulative block\\discrimination};

\node[below=0.2cm of curr, font=\tiny, text=green!50!black, align=center] (dslabel)
    {depth-scaled};

\begin{scope}[shift={(-0.8, -3.35)}]%
    \fill[depthbar, opacity=0.4] (-0.6, 0) rectangle (-0.35, 0.25);
    \fill[depthbar, opacity=0.6] (-0.25, 0) rectangle (0.0, 0.5);
    \fill[depthbar, opacity=0.8] (0.1, 0) rectangle (0.35, 0.75);
    \fill[depthbar, opacity=1.0] (0.45, 0) rectangle (0.7, 1.0);
    \node[font=\tiny, text=green!40!black] at (-0.475, -0.15) {.25};
    \node[font=\tiny, text=green!40!black] at (-0.125, -0.15) {.50};
    \node[font=\tiny, text=green!40!black] at (0.225, -0.15) {.75};
    \node[font=\tiny, text=green!40!black] at (0.575, -0.15) {1.0};
    \node[font=\tiny, text=gray!50] at (0.05, -0.35) {blocks $0 \to 3$};
\end{scope}

\node[below=0.15cm of depth, font=\tiny, text=gray!60, align=center] {monotonic\\ordering};

\node[below=0.15cm of supcon, font=\tiny, text=gray!60, align=center] {per-block\\contrastive};

\node[subloss] (bal) at (4.8, -3.0) {$\mathcal{L}_\text{bal}$\\$\lambda{=}0.01$};
\node[subloss] (zloss) at (6.0, -3.0) {$\mathcal{L}_z$\\$\lambda{=}0.001$};
\draw[arrow, gray!40] (moeloss) -- (bal);
\draw[arrow, gray!40] (moeloss) -- (zloss);

\node[below=0.05cm of bal, font=\tiny, text=gray!50] {balance};
\node[below=0.05cm of zloss, font=\tiny, text=gray!50] {router};

\node[below=0.15cm of reconloss, font=\tiny, text=gray!60, align=center] {MSE\\$\lambda{=}0.05$};

\draw[green!60!black, thick, rounded corners=3pt, dashed]
    ([xshift=-0.3cm, yshift=0.15cm]curr.north west) rectangle
    ([xshift=0.3cm, yshift=-2.2cm]curr.south east);

\end{tikzpicture}%
}
\caption{Per-block loss decomposition at depth $d$. The depth-scaled current-block term $\mathcal{L}_\text{curr}$ (green, dashed border) is the repair studied in \S\ref{sec:repairs} (the depth-scaled auxiliary): its weight rises linearly from $0.25$ to $1.00$ across blocks $0$--$3$; it only partially restores deeper-block separation (App.~Tab.~\ref{tab:freeriding}). The aspect loss has four slots, of which three are active: the attention-sharpness slot $g_\text{attn}$ (grayed) is identically zero with memory disabled, our default, so its term carries no gradient (App.~\ref{app:loss_formulas}).}
\label{fig:loss_decomposition}
\end{figure}

\begin{figure}[t]
\centering
\definecolor{gaTokFill}{HTML}{D9D9FF}\definecolor{gaTokLine}{HTML}{7373FF}%
\definecolor{gaGoodFill}{HTML}{CCFFCC}\definecolor{gaGoodLine}{HTML}{00A600}%
\definecolor{gaLblFill}{HTML}{FFFFD1}\definecolor{gaLblLine}{HTML}{A6A600}%
\definecolor{gaLossFill}{HTML}{FFE0E0}\definecolor{gaLossLine}{HTML}{E06666}%
\definecolor{gaLossText}{HTML}{990000}%
\definecolor{gaProcFill}{HTML}{F2F2F2}\definecolor{gaProcLine}{HTML}{A6A6A6}%
\definecolor{gaOffText}{HTML}{7A7A7A}\definecolor{gaArrow}{HTML}{666666}%
\definecolor{gaOwn}{HTML}{009E73}\definecolor{gaHist}{HTML}{D55E00}\definecolor{gaSum}{HTML}{0072B2}%
\begin{tikzpicture}[
    >={Stealth[length=5pt,width=4pt]},
    font=\footnotesize,
    box/.style={draw, rounded corners=2pt, align=center, line width=0.6pt,
                inner xsep=3pt, inner ysep=2.5pt},
    tok/.style={box, fill=gaTokFill, draw=gaTokLine},
    proc/.style={box, fill=gaProcFill, draw=gaProcLine},
    lbl/.style={box, fill=gaLblFill, draw=gaLblLine},
    asp/.style={box, fill=gaGoodFill, draw=gaGoodLine, minimum width=3.2cm, minimum height=1.3cm},
    off/.style={fill=gaProcFill, draw=gaProcLine, dashed, text=gaOffText},
    loss/.style={draw=gaLossLine, fill=gaLossFill, rounded corners=1.5pt, text=gaLossText,
                 inner xsep=2pt, inner ysep=1.5pt, line width=0.5pt},
    flow/.style={->, line width=0.7pt, draw=gaArrow},
    offflow/.style={->, line width=0.7pt, draw=gaProcLine, dashed},
    blk/.style={box, fill=gaTokFill, draw=gaTokLine, minimum width=1.5cm, minimum height=0.7cm},
    plus/.style={circle, draw=#1, line width=0.7pt, fill=white, inner sep=0pt, minimum size=0.34cm},
    ptitle/.style={font=\small\bfseries, anchor=west, inner sep=0pt},
    hist/.style={line width=0.8pt, draw=gaHist, dash pattern=on 2.5pt off 1.5pt},
]
\node[ptitle] at (0cm, 0.6cm) {(a) Inside block $d$: four goodness aspects $\to$ one scalar $g^{(d)}$};

\node[box, fill=gaProcFill, draw=gaProcLine] (pool) at (5.9cm, -0.35cm)
  {attention pooling (one learned query)\\$\bar h\in\mathbb{R}^{B\times D}$};
\node[tok] (h) at (5.15cm, -1.6cm) {post-MoE tokens\\$h\in\mathbb{R}^{B\times N\times D}$};
\node[lbl] (proto) at (2.85cm, -1.6cm) {prototype $p_y$\\of label $y$};
\node[box, off] (mem) at (8.6cm, -1.6cm) {memory cross-attention\\(off by default)};

\node[asp] (a0) at (1.7cm, -3.1cm)
  {{\bfseries prototype alignment}\\$g_{\text{align}}=\cos(\bar h,p_y)/\tau_p$\\$\tau_p=0.1$};
\node[asp] (a1) at (5.15cm, -3.1cm)
  {{\bfseries activation energy}\\$g_{\text{energy}}=\operatorname{mean}_{n,i}\,h_{n,i}^{2}$\\(over all $N{\times}D$ entries)};
\node[asp, off] (a2) at (8.6cm, -3.1cm)
  {{\bfseries attention sharpness}\\$g_{\text{attn}}=0$\\inactive by default};
\node[asp] (a3) at (12.05cm, -3.1cm)
  {{\bfseries learned head}\\$g_{\text{learned}}=\operatorname{softplus}(f(\bar h))$\\$f$: RMSNorm, GELU MLP};

\draw[flow] (h.north) -- (h.north |- pool.south);
\draw[flow] (pool.west) -| (a0.north);
\draw[flow] (pool.east) -| (a3.north);
\draw[flow] (h.south) -- (a1.north);
\draw[flow, draw=gaLblLine] (proto.south) -- (proto.south |- a0.north);
\draw[offflow] (mem.south) -- (a2.north);

\node[box, fill=gaGoodFill, draw=gaGoodLine, line width=0.9pt, minimum width=13.3cm, minimum height=0.62cm]
  (agg) at (6.875cm, -5.05cm)
  {{\bfseries weighted sum:}\ \ $g^{(d)}=\sum_{a}\pi_a\,g_a$,\quad $\pi=\operatorname{softmax}(\alpha_{\mathrm{raw}})$,\ \ $\alpha_{\mathrm{raw}}\in\mathbb{R}^{4}$ learned per block (initialized to $0$: uniform)};

\foreach \n/\x/\w/\lossname in {a0/1.7cm/0/{aspect loss $\ell_{\text{margin}}$}, a1/5.15cm/1/{$\ell_{\text{rank}}$},
                                   a3/12.05cm/3/{$\ell_{\text{margin}}$}} {
  \draw[flow, draw=gaGoodLine] (\n.south) -- (\x, 0cm |- agg.north)
     node[at end, anchor=south east, inner sep=1.5pt, text=black] {$\pi_{\w}$};
  \fill[gaLossLine] (\x, -4.1cm) circle (1.3pt);
  \draw[gaLossLine, line width=0.5pt] (\x, -4.1cm) -- ++(0.3cm, 0cm)
     node[loss, anchor=west] {\lossname};
}
\draw[offflow] (a2.south) -- (8.6cm, 0cm |- agg.north)
   node[at end, anchor=south east, inner sep=1.5pt, text=gaOffText] {$\pi_{2}$};
\fill[gaProcLine] (8.6cm, -4.1cm) circle (1.3pt);
\draw[gaProcLine, line width=0.5pt, dashed] (8.6cm, -4.1cm) -- ++(0.3cm, 0cm)
   node[loss, anchor=west, draw=gaProcLine, fill=gaProcFill, text=gaOffText, dashed] {$\ell_{\text{rank}}$, no gradient};

\begin{scope}[yshift=-6.45cm]
\node[ptitle] at (0cm, 0cm) {(b) Across blocks $d=0,\dots,3$: training loss score vs.\ prediction score};

\node[anchor=west, inner sep=0pt] (lab) at (0cm, -0.55cm) {label $y$};
\draw[line width=0.7pt, draw=gaLblLine] (lab.east) -- (10.1cm, -0.55cm);
\node[anchor=west, align=center, inner sep=0pt] (in) at (0cm, -1.25cm) {image\\$x$};
\foreach \d/\x in {0/2.0cm, 1/4.7cm, 2/7.4cm, 3/10.1cm} {
  \node[blk] (b\d) at (\x, -1.25cm) {block $\d$};
  \draw[->, line width=0.7pt, draw=gaLblLine] (\x, -0.55cm) -- (b\d.north);
}
\draw[flow] (in.east) -- (b0.west);
\foreach \d/\e in {0/1, 1/2, 2/3} {
  \draw[flow] (b\d.east) -- (b\e.west) node[midway, below, inner sep=1.5pt] {sg};
}

\foreach \d/\x/\bx in {0/2.0cm/2.7cm, 1/4.7cm/5.4cm, 2/7.4cm/8.1cm, 3/10.1cm/10.8cm} {
  \node[plus=gaOwn] (G\d) at (\x, -2.75cm) {$+$};
  \draw[line width=0.7pt, draw=gaOwn] (b\d.south) -- (\x, -1.92cm)
     node[midway, left, inner sep=1.5pt, text=black] {$g^{(\d)}$};
  \fill[gaOwn] (\x, -1.92cm) circle (1.3pt);
  \draw[->, line width=0.7pt, draw=gaOwn] (\x, -1.92cm) -- (G\d.north);
  \draw[line width=0.7pt, draw=gaSum] (\x, -1.92cm) -- (\bx, -1.92cm);
  \node[loss, anchor=east] (L\d) at ([xshift=-0.62cm]G\d.west) {$\mathcal{L}^{(\d)}_{\text{block}}$};
  \draw[->, line width=0.7pt, draw=gaOwn] (G\d.west) -- (L\d.east)
     node[midway, above, inner sep=1.5pt, text=black] {$G^{(\d)}_{\gamma}$};
}

\foreach \d/\bx in {1/5.4cm, 2/8.1cm, 3/10.8cm} {
  \node[plus=gaSum] (S\d) at (\bx, -3.75cm) {$+$};
  \draw[->, line width=0.7pt, draw=gaSum] (\bx, -1.92cm) -- (S\d.north);
}
\draw[->, line width=0.7pt, draw=gaSum] (2.7cm, -1.92cm) |- (S1.west);
\draw[->, line width=0.7pt, draw=gaSum] (S1.east) -- (S2.west);
\draw[->, line width=0.7pt, draw=gaSum] (S2.east) -- (S3.west);
\node[box, fill=white, draw=gaSum, anchor=west] (pred) at ([xshift=0.3cm]S3.east)
  {prediction\\$S_y=\sum_{d} g^{(d)}$\\$\hat y=\arg\max_y S_y$};
\draw[->, line width=0.7pt, draw=gaSum] (S3.east) -- (pred.west);
\foreach \d/\x in {0/3.7cm, 1/6.4cm, 2/9.1cm} {
  \node[below, inner sep=2pt, text=gaSum] at (\x, -3.75cm) {$S^{(\d)}$};
}

\foreach \d/\x in {1/4.7cm, 2/7.4cm, 3/10.1cm} {
  \fill[gaHist] (\x, -3.75cm) circle (1.3pt);
  \draw[->, hist] (\x, -3.75cm) -- (G\d.south)
     node[midway, left, inner sep=1.5pt, text=gaHist] {$\times\gamma$, sg};
}

\node[anchor=west, inner sep=0pt] (lg1) at (0.55cm, -4.7cm)
  {training: $G^{(d)}_\gamma=g^{(d)}+\gamma\,\mathrm{sg}\big(S^{(d-1)}\big)$};
\draw[line width=0.7pt, draw=gaOwn] (0cm, -4.63cm) -- (0.45cm, -4.63cm);
\draw[hist] (0cm, -4.77cm) -- (0.45cm, -4.77cm);
\draw[line width=0.7pt, draw=gaSum] ([xshift=0.5cm]lg1.east) -- ([xshift=0.95cm]lg1.east);
\node[anchor=west, inner sep=0pt] (lg2) at ([xshift=1.05cm]lg1.east)
  {running sum $S^{(d)}=\sum_{j\le d} g^{(j)}$; prediction uses $S_y=S^{(3)}$ (no $\gamma$)};
\end{scope}
\end{tikzpicture}
\caption{Multi-aspect goodness and how it is used. \textbf{(a)}~Inside
block~$d$, four goodness slots are computed from the post-MoE tokens $h$ and
mixed by the learned weights $\pi$ of App.~\ref{app:loss_formulas} into the
scalar $g^{(d)}$. The
attention-sharpness slot (gray, dashed) is inactive by default: memory
cross-attention is off, so the slot is identically $0$ but still takes weight
$\pi_2$. Each slot also has its own aspect loss (red), part of block~$d$'s loss.
\textbf{(b)}~Blocks pass on the detached (sg), $L_2$-normalized
$\mathrm{ReLU}(h)$. In training, $\mathcal{L}^{(d)}_\text{block}$ scores each
hypothesis by $G^{(d)}_\gamma$, whose positive-minus-negative margin is
Eq.~\ref{eq:cum_margin}; the history $S^{(d-1)}$ is detached, so gradient
reaches block~$d$ only through $g^{(d)}$ ($\gamma{=}0.7$ in our cumulative
runs, $\gamma{=}1$ for LCFF, $\gamma{=}0$ history-free). Prediction, for every
variant, uses $S_y$ with no $\gamma$. When the history already separates a
positive from its negatives, $\mathcal{L}^{(d)}_\text{block}$ is nearly flat
in $g^{(d)}$: the free-riding of Theorem~\ref{thm:attenuation}. The single
label line simplifies the negatives: a wrong-label negative's label is
re-mined at every block, and its history is that negative stream's own
(the score $G^{(d)}_{\gamma,nl}$ of App.~\ref{app:loss_formulas}).}
\label{fig:goodness_aspects}
\end{figure}

\begin{figure*}[t]
\centering
\resizebox{\textwidth}{!}{%
\begin{tikzpicture}[
    >=latex,
    node distance=0.5cm and 0.6cm,
    myshadow/.style={blur shadow={shadow blur steps=5, shadow xshift=0.7pt,
                     shadow yshift=-0.7pt, shadow opacity=25}},
    proc/.style={draw, rounded corners=3pt, minimum height=0.55cm, minimum width=2.0cm,
                 font=\scriptsize, align=center, thick, myshadow},
    data/.style={proc, fill=blue!15, draw=blue!50},
    compute/.style={proc, fill=green!15, draw=green!50!black},
    optim/.style={proc, fill=orange!18, draw=orange!60},
    lossbox/.style={proc, fill=red!15, draw=red!50},
    loopbox/.style={draw, rounded corners=4pt, thick, draw=green!60!black, fill=green!3, myshadow},
    decision/.style={draw, diamond, aspect=2.2, minimum height=0.5cm, font=\scriptsize,
                     align=center, fill=yellow!15, draw=yellow!60!black, thick, myshadow},
    arrow/.style={->, thick, draw=gray!60, >=stealth},
    dashedarrow/.style={->, dashed, thick, draw=gray!50, >=stealth},
    label/.style={font=\tiny, text=gray!60},
    scissors/.style={font=\scriptsize\bfseries, text=orange!70},
]

\node[font=\small\bfseries, text=blue!60] at (0, 0.8) {Data Preparation};

\node[data] (batch) at (0, 0) {Minibatch\\$(x, y)$, $B{=}512$};
\node[data, below=0.5cm of batch] (aug) {3 augmented views\\$x_1, x_2, x_3$};
\node[data, below=0.5cm of aug] (ni) {Wrong-image neg.\\$x_\text{ni} \leftarrow \text{shuffle}(x)$};
\node[data, below=0.5cm of ni] (patemb) {PatchEmbed\\$x \to [B, 64, 256]$};
\node[data, below=0.5cm of patemb] (initg) {Init $g_\text{hist} = []$\\Init token cache};

\draw[arrow] (batch) -- (aug);
\draw[arrow] (aug) -- (ni);
\draw[arrow] (ni) -- (patemb);
\draw[arrow] (patemb) -- (initg);

\node[font=\small\bfseries, text=green!60!black] at (5.9, 2.0) {Per-Block Training Loop};

\node[loopbox, minimum width=7.6cm, minimum height=12.5cm] (loopframe) at (5.9, -4.7) {};
\node[font=\scriptsize\bfseries, text=green!60!black] at (5.9, 1.3) {for $d = 0, 1, 2, 3$:};

\node[compute] (detach) at (5.5, 0) {Detach tokens\\+ $g_\text{prev}$ from cache};

\node[compute, fill=yellow!10, draw=yellow!60!black] (hnm) at (5.5, -1.2)
    {Hard negative mining (EMA teacher)\\$k$ draws ($8{\to}16$), keep hardest};

\node[compute] (fwd) at (5.5, -2.6)
    {Forward block $d$\\on 4 streams};

\node[right=0.2cm of fwd, font=\tiny, text=gray!60, align=left]
    {pos$_1$: $(x_1, y)$\\pos$_2$: $(x_2, y)$\\NL: $(x_3, y^{-})$\\NI: $(x_\text{ni}, y)$};

\node[lossbox, minimum width=3.0cm] (loss) at (5.5, -4.2)
    {Compute $\mathcal{L}^{(d)}_\text{total}$};

\node[right=0.2cm of loss, font=\tiny, text=red!60, align=left]
    {$\mathcal{L}_\text{aspects}$ + $\lambda_\text{block}\mathcal{L}_\text{block}$\\
     + $\lambda_\text{curr}(d)\mathcal{L}_\text{curr}$\\
     + $\mathcal{L}_\text{depth}$ + $\mathcal{L}_\text{SupCon}$\\
     + $\mathcal{L}_\text{MoE}$ + $\mathcal{L}_\text{recon}$};

\node[optim] (sam1) at (5.5, -5.6) {SAM step 1:\\$\nabla \to$ perturb weights};

\node[compute] (refwd) at (5.5, -6.8) {Re-forward block $d$\\(at perturbed point)};

\node[optim] (sam2) at (5.5, -8.0) {SAM step 2:\\optimizer update};

\node[optim, fill=teal!10, draw=teal!60] (ema) at (5.5, -9.1) {EMA update\\block $d$ weights};

\node[compute] (norm) at (5.5, -10.2) {ReLU $\to$ $L_2$ norm\\output tokens};

\draw[arrow] (detach) -- (hnm);
\draw[arrow] (hnm) -- (fwd);
\draw[arrow] (fwd) -- (loss);
\draw[arrow] (loss) -- (sam1);
\draw[arrow] (sam1) -- (refwd);
\draw[arrow] (refwd) -- (sam2);
\draw[arrow] (sam2) -- (ema);
\draw[arrow] (ema) -- (norm);

\draw[arrow] (initg.east) -| (2.75, 0) -- node[scissors, above, pos=0.33] {detach} (detach.west);

\node[font=\small\bfseries, text=purple!60] at (12.3, 0.8) {Post-Block Updates};

\node[compute, fill=green!8, draw=green!50!black] (storeg) at (12.3, -1.0)
    {Append $g_\text{vec}^{(d)}$\\to $g_\text{hist}$};

\node[compute, fill=blue!8, draw=blue!50] (enqueue) at (12.3, -2.5)
    {Enqueue contrastive\\embeddings $z$ to queue};

\node[compute] (cache) at (12.3, -4.0)
    {Cache normalized tokens\\for block $d{+}1$};

\node[decision] (lastblock) at (12.3, -5.5) {$d = 3$?};

\node[compute, fill=gray!10, draw=gray!50] (nextblock) at (14.8, -5.5) {$d \leftarrow d{+}1$};

\node[optim, fill=purple!10, draw=purple!60] (evalchk) at (12.3, -7.5)
    {Evaluate (periodic)\\Save checkpoint};

\draw[arrow] (norm.east) -| (10.3, -1.0) -- (storeg.west);
\draw[arrow] (storeg) -- (enqueue);
\draw[arrow] (enqueue) -- (cache);
\draw[arrow] (cache) -- (lastblock);
\draw[arrow] (lastblock.east) -- node[above, font=\tiny] {no} (nextblock);
\draw[arrow] (nextblock.north) |- (detach.east);
\draw[arrow] (lastblock.south) -- node[right, font=\tiny] {yes} (evalchk);

\node[font=\tiny, text=blue!60, align=center] at (3.25, -5.75)
    {Block 0 optimizer\\also trains\\PatchEmbed};

\draw[thick, orange!60, dashed, rounded corners=2pt]
    ([xshift=-0.5cm, yshift=0.25cm]detach.north west) rectangle
    ([xshift=0.2cm, yshift=-0.15cm]detach.south east);
\node[font=\tiny, text=orange!60, anchor=south west] at ([xshift=-0.5cm, yshift=0.25cm]detach.north west)
    {gradient isolation};

\end{tikzpicture}%
}
\caption{Stage~1 training procedure. \textbf{Left:} Data preparation produces three augmented views and wrong-image negatives per minibatch. \textbf{Center:} Each block $d$ is trained locally with its own SAM optimizer: hard negative labels are mined via the EMA teacher, four streams are forwarded, and the composite loss $\mathcal{L}^{(d)}_\text{total}$ drives a two-step SAM update. Gradient detachment between blocks (orange dashed box) enforces local learning. \textbf{Right:} After each block, aspect histories and contrastive embeddings are cached, and normalized tokens are passed to the next block.}
\label{fig:stage1_training}
\end{figure*}

\newpage

\section{Secondary Analysis}
\label{sec:secondary_analysis}

\begin{figure}[htbp]
  \centering
  \includegraphics[width=0.94\textwidth]{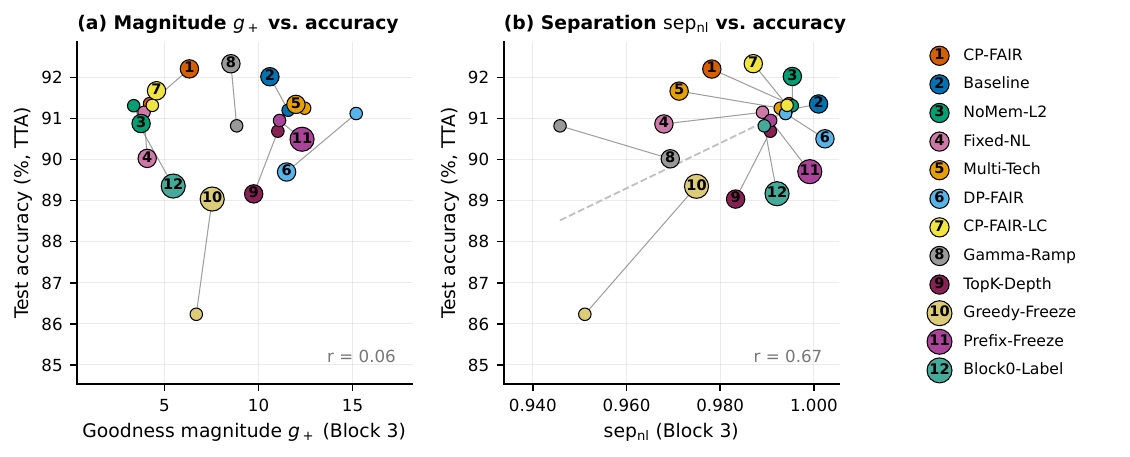}
  \caption{Across 12 seed-42 variants of Table~\ref{tab:full_main_results} (all but $\gamma{=}0$; flip-TTA test accuracy, except single-crop \textsc{greedy-freeze} and Stage-2 \textsc{multi-tech}): \textbf{(a)}~goodness magnitude $g_+$ at Block~3 shows no correlation with test accuracy ($r{=}0.06$); \textbf{(b)}~wrong-label separation $\text{sep}_\text{nl}$ at Block~3 is a strong predictor ($r{=}0.67$). Separation quality, not magnitude, is associated with accuracy (a correlation that includes the degenerate \textsc{greedy-freeze}). Block-3 values: last Stage-1 epoch of each plotted log (for \textsc{baseline}, 286, and \textsc{multi-tech}, 330, an earlier segment of a run that reached 287 and 362).}
  \label{fig:magnitude_vs_sep}
\end{figure}

\paragraph{Goodness magnitude vs.\ separation quality.}
A surprising finding is that large goodness magnitudes do not improve
accuracy (Fig.~\ref{fig:magnitude_vs_sep}). \textsc{nomem-l2} (no memory, $L_2$-normalized tokens)
achieves $g_+ \approx 3.4$ at Block~3, versus $g_+ \approx 11.5$ for
\textsc{baseline} (with memory cross-attention; both at their
plotted epoch)---a ${\sim}3.4{\times}$ difference---yet
matches its test accuracy ($91.31\%$ vs.\ $91.20\%$ S1~TTA, one seed).
Separation quality ($\text{sep}_\text{nl}$) is the predictive metric,
not magnitude. Across the $\gamma$-ablations, $\text{sep}_\text{nl}$ at
Block~3 converges to $\sim$0.995 for both $\gamma{=}0$ and $\gamma{=}0.7$,
consistent with their similar final accuracies. However,
$\text{sep}^\text{cur}_\text{nl}$ (per-block independence) clearly differentiates
variants: $\gamma{=}0$ achieves $3.28$ at Block~3 vs.\ $0.61$ for
\textsc{cp-fair} (seed~42; the $\gamma{=}0$ seeds 123/456 were trained
with different auxiliary-loss weights, Table~\ref{tab:gamma0_multiseed}). Within the $\gamma$-sweep, higher
$\text{sep}^\text{cur}_\text{nl}$ correlates with higher accuracy and zero
free-riding. However, the hardness-gated experiments
(Section~\ref{sec:adaptive_gamma}) show that artificially boosting
per-block independence via adaptive gating does not improve accuracy,
suggesting that $\text{sep}^\text{cur}_\text{nl}$ is a reliable \emph{diagnostic}
of block health but not a direct causal driver of test performance.

\paragraph{The complexity trap.}
\textsc{multi-tech} synthesizes 10 additional techniques (triangle activation,
multi-negative mining, diversity loss, per-block probes, gamma ramp,
warm-start, dual-phase contrastive, temperature annealing, reduced MoE,
wider HNM) on top of the same depth-scaled loss as \textsc{cp-fair}, yet
achieves $91.22\%$ vs.\ $91.36\%$ S1~TTA ($91.25\%$ vs.\ $91.27\%$ S2~TTA; seed~42). Caution is required in interpreting this
comparison directly (the two variants differ in architecture: 24 vs.\ 32
experts, presence of triangle activation), but it is consistent with the
general finding that identifying and targeting the root failure mode
(free-riding) is more productive than layering techniques.

\paragraph{Stage~2 is expendable.}
Across the 12 seed-42 variants of Table~\ref{tab:full_main_results}
that were run through Stage~2 (all but \textsc{gamma-ramp}; the
\textsc{dp-fair} Stage~2 ran as a separate continuation from its saved
Stage-1 checkpoint), Stage~2 (frozen backbone + attentive head, 20
epochs; 300 for three variants) changes TTA test accuracy by
$+0.03$\,pp on average (median $+0.01$\,pp; range $-0.12$ to
$+0.37$\,pp). The largest gain belongs to the degenerate
\textsc{greedy-freeze} variant; without it the mean over the other 11
is $-0.002$\,pp. Six of the 12 variants lose accuracy in Stage~2. We
recommend omitting Stage~2 for practitioners prioritizing simplicity.

\paragraph{Batch size affects wrong-image separation.}
$\text{sep}_\text{ni}$ (wrong-image separation) at Block~3 drops from
$>99.6\%$ (batch size 512) to ${\sim}94.7\%$ (batch size 128). This is
because hard negative mining selects incorrect labels from within the batch;
larger batches provide harder, more diverse negatives. We recommend $\geq$512.

\appendixpart{Appendix Part III: Changes Since Submission}{Part~III
lists the corrections made since the submitted version and the review
discussion (App.~\ref{app:changes}).}

\section{Changes Since Submission and the Review Discussion}
\label{app:changes}

{\footnotesize
This version corrects statements in the author response during the
review discussion (public on OpenReview) and errors in the submitted
version; each item points to the corrected text. Rounding-,
SD-convention- and wording-only changes are not listed.

\paragraph{Statements in the author response.}
\begin{enumerate}[label=A\arabic*.,leftmargin=*,nosep]
\item The MGC comparison, called matched and read as passing the
  falsification test, also removes block~0's auxiliary term; a control
  at L4/D128 finds no compensation gain masked by this change ($n{=}3$), and
  criterion~(b), read literally, is met at Stage~1 in the harmful
  direction (App.~\ref{app:attenuation_compensated},
  \ref{app:f1_control}).
\item The response credited the compensator with restoring deep-block
  separation, but at L4/D128 dropping MGC's replaced auxiliary alone
  gives $92\%$ of the gain ($4.88$ vs.\ $5.25$; cumulative $0.75$),
  failing a registered prediction (App.~\ref{app:f1_control}).
\item As adopted in the discussion, MGC's effect is a consistent
  Stage-1 cost that Stage~2 recovers; its CIFAR-100 S1 cost
  ($-1.05$\,pp, 95\% CI $[-1.46, -0.62]$) exceeds the response's
  sub-1\,pp bound for training-rule interventions
  (App.~\ref{app:attenuation_compensated}).
\item MGC's deepest-block separation is highest only within matched
  configurations (L8/D128 $\kappa{=}0$: $6.65 \pm 0.33$,
  Tab.~\ref{tab:dissociation_full}).
\item L4/D128 MGC runs (SAM on) were compared with seed-42 anchors (SAM
  off) and a SAM-mixed $\kappa{=}0$ mean; a SAM-on cumulative arm is now
  the baseline (Tab.~\ref{tab:f1_control}).
\item The corrected L4/D128 $\kappa{=}0$ mean used a population SD; the
  sample SD gives ${87.15 \pm 0.17\%}$ (Tab.~\ref{tab:gated_k0_multiseed}).
\item The CIFAR-10 $\gamma{=}0$ per-seed example ($3.28$/$7.85$/$7.97$),
  presented as seed variation, compared two configurations at the last
  epoch; at the selected checkpoints it is $3.22$/$5.65$/$6.41$
  (Tabs.~\ref{tab:gamma0_multiseed}, \ref{tab:e1_aligned}); seed~42
  rerun in the shared configuration gives $7.72$ ($3.69$ at its
  selected epoch). The
  per-seed test, on the misaligned pairing the response flagged,
  reported S2~TTA only ($\rho{=}0.28$, $p{=}0.20$; without
  TTA $\rho{=}{+}0.53$, $p{=}0.007$); aligned, no protocol reaches
  $p{<}0.05$ (S2~TTA $\rho{=}{+}0.08$, $p{=}0.73$; configurations
  split, $\rho{=}{-}0.03$, $p{=}0.91$) (Fig.~\ref{fig:accuracy_vs_ratio}).
\item The $r{=}{-}0.49$ (Fig.~\ref{fig:accuracy_vs_ratio}) reflects the
  4- vs.\ 8-block contrast, not the degenerate variants (within-depth
  $r{=}{+}0.10$, ${-}0.10$).
\item The early-exit case for the repairs used test-selected points, as
  stated. On validation (rule T),
  truncation saves ${48.1 \pm 29.5\%}$ of FLOPs on CIFAR-10 (not
  $72\%$) and ${63.5 \pm 13.1\%}$ on CIFAR-100 (6/9 runs at or above
  full depth); T$'$ picks depth~2 in 9/9 CIFAR-100 runs; dynamic exit
  saves ${73.4 \pm 1.6\%}$. Cumulative runs also exit early (as
  the response noted), so saving is not repair-specific; repaired runs'
  larger T$'$ gains reflect lower full-depth accuracy
  (App.~\ref{app:early_exit}).
\item Readouts mixed seeds: the MGC advantage was a seed-42 range next
  to the three-seed MGC mean, and seed-42 readout blocks were compared
  with three-seed head means. Matched on seed~42, the MGC block is
  $+1.31$/$+1.47$/$+2.21$\,pp above cumulative/$\kappa{=}0$/$\gamma{=}0$
  and the readout blocks are $5.3$--$6.4$\,pp below the head
  (App.~\ref{app:readouts}).
\item The $k$-NN trio span ($1.47$\,pp) used block-softmax goodness
  (unstated); raw goodness gives $8.58$\,pp
  (App.~\ref{app:probes}).
\item The gradient decomposition's linearity error exceeds the
  response's $10^{-4}$ bound at block~0 (below $3{\times}10^{-4}$)
  (Tab.~\ref{tab:gradshare}).
\end{enumerate}

\paragraph{Errors in the submitted version.}
\begin{enumerate}[label=B\arabic*.,leftmargin=*,nosep]
\item Locality: in the trainers of the reported runs, with
  $\lambda_\text{depth}{>}0$ the depth-order term wrote a gradient to
  block $d{-}1$'s mixing logits, so the block-$d$ graph was not
  block-$d$-only and the harness did not exercise it; the gradient was
  never applied, so updates are unchanged, and the released trainers
  compute the term without gradient tracking, which the harness now
  checks (App.~\ref{app:locality}).
\item Method text: $\text{sep}_\text{nl}$ is a fraction in $[0,1]$, not
  an expected margin, and $\text{sep}^\text{cur}_\text{nl}$ uses the hard-mined
  label (Eq.~\ref{eq:diagnostics}); in Fig.~\ref{fig:grad_attenuation}
  only the $\gamma{=}0.7$ run has nonzero $\lambda_\text{curr}$,
  $\lambda_\text{depth}$; $\mathcal{L}_\text{depth}$ and
  $\mathcal{L}_\text{recon}$ counted their weights twice and
  $\mathcal{L}_z$ squared the token mean rather than each token's
  term, so none matched the code (App.~\ref{app:loss_formulas}), and
  Prop.~\ref{prop:depthorder} covers only an idealized form; the
  text-domain Stage~1 used no SAM (App.~\ref{app:textdomain}).
\item The 32-expert CIFAR-10 model has $107.2$\,M parameters
  ($18.9$\,M active), not ${\sim}6$--$8$\,M (Sec.~\ref{sec:experiments},
  Tab.~\ref{tab:related_work}).
\item CIFAR-10 $\gamma{=}0$: seed~42's Stage~2, said to be unrecorded,
  is in its run log (now $n{=}3$), and seed~42 used different loss
  weights, not the shared ones stated (matched rerun: $n{=}3$ S1~TTA
  $91.30 \pm 0.17\%$; Tab.~\ref{tab:gamma0_multiseed}).
\item LCFF deepest-block separation is $0.70$, not $0.64$ (from an
  aborted run) (Sec.~\ref{sec:experiments}, Tab.~\ref{tab:freeriding}).
\item Three L4/D128 seed-42 S2~TTA entries ($\kappa{=}0$, constant
  $\gamma$, LCFF) were Stage-1 values
  (Tabs.~\ref{tab:gated_k0_multiseed}, \ref{tab:dissociation_full}).
\item The $87.40\%$ early-exit figure ($\tau{=}3,\kappa{=}0$, block~0)
  was a validation flip-TTA accuracy (App.~\ref{app:early_exit},
  Fig.~\ref{fig:depth_truncation}).
\item L8/D128 $\kappa{=}0$: the submitted mean averaged seed~42's
  validation accuracy ($86.88\%$) with two test values; the three-seed
  test mean (seed~42 evaluated post hoc) is $87.12 \pm 0.35\%$
  (Tab.~\ref{tab:gated_k0_multiseed}).
\item $L{=}8$: the truncation values had no raw source; the logged
  curves are flat (${\le}0.26$\,pp per run), not rising 1--2\,pp per
  block, and per-block goodness is strictly monotonic only in the two
  cumulative $\kappa{=}0$ runs (Tabs.~\ref{tab:depth_truncation_L8},
  \ref{tab:L8_gpos}).
\item The train/validation split is seeded by each run's seed, not
  fixed at seed~42 (Sec.~\ref{sec:experiments},
  App.~\ref{sec:full_config}).
\item Compute: ${\sim}820$ A100-hours and ${\sim}1{,}910$ GB10-hours on
  DGX Spark machines, not ${\sim}289$ A100-hours
  (\texttt{metric\_summaries/camera\_ready/compute\_budget.json} in the code repository);
  a Tiny ImageNet seed took ${\sim}29$ A100-hours, not the ${\sim}8$\,h of its
  last resumed segment ($8.0$ and $8.1$\,h for seeds 42 and 123;
  App.~\ref{sec:tiny_imagenet_metrics}).
\item Sec.~\ref{sec:method} said $\gamma{=}0$ matches or beats every
  other CIFAR-10 variant. It has the highest validation accuracy, but on
  test it is within 1\,pp of every non-degenerate variant, and some
  score higher (\textsc{baseline} $90.69\%$ vs.\ $90.52\%$ single crop,
  seed~42; Sec.~\ref{sec:c10_accuracy}, Tab.~\ref{tab:full_main_results}).
\item Depth saturation was stated for all variants of
  Tab.~\ref{tab:full_main_results}, of which the paper shows per-depth
  accuracy for two; it is now stated only for the three runs of
  Fig.~\ref{fig:depth_gap}, as a two- vs.\ four-block accuracy gap, and
  kept as a signature that is consistent with, but not specific to,
  free-riding (Sec.~\ref{sec:intro}).
\item The flat-separation signature cited loss-collapse (LC) metrics
  that no figure or table reports; it now cites
  Fig.~\ref{fig:freeriding} and Tab.~\ref{tab:freeriding}
  (Sec.~\ref{sec:intro}).
\item Convergence: LCFF-prefix was said to reach $99\%$ of its final
  validation accuracy earlier than the other runs; all three reach
  $95\%$ at epoch~80 and $99\%$ at epoch~140
  (Fig.~\ref{fig:convergence_speed}).
\item Tiny ImageNet: the convergence figure, captioned as three seeds,
  showed only each run's last resumed segment (seed~456 from epoch~209;
  seeds 42 and 123 and the mean only after epoch~330) and is measured on
  the held-out part of the training images, not the validation split; the
  per-block figure's Block-0 points did not match the run logs; both are
  regenerated from the logs (Figs.~\ref{fig:tinyimagenet_crossdata},
  \ref{fig:tinyimagenet_dynamics}), and no stated result changes. Its
  configuration paragraph gave CIFAR-10's patch size as~$4$, not~$2$
  (App.~\ref{app:gated_full}).
\item Of the nine FF methods in the submitted Tab.~\ref{tab:related_work}, six report CIFAR-100, not three:
  Distance-Forward and Trifecta, said not to, report $59.01\%$ and
  $35.79\%$, and HFF reports $54.34\%$ (Tab.~\ref{tab:cifar100_comparison}).
\item Captions described what the plots do not show: a baseline collapse
  at Block~3 (Fig.~\ref{fig:freeriding}), monotonic growth of every
  $\gamma{=}0$ block (Fig.~\ref{fig:freeriding_onset}), an earlier
  \textsc{fixed-nl} plateau (Fig.~\ref{fig:training_dynamics}),
  constant-$\gamma$ panels and
  monotonic growth throughout training (Figs.~\ref{fig:goodness_decomposition_L8},
  \ref{fig:training_dynamics_L8}), per-block loss and adaptive-gating
  stability (Fig.~\ref{fig:training_stability}), and all 13 variants
  with a causal reading (Fig.~\ref{fig:magnitude_vs_sep}); they now
  describe the plots. App.~\ref{sec:secondary_analysis} compared
  \textsc{multi-tech}'s Stage-2 with \textsc{cp-fair}'s Stage-1
  accuracy; both protocols are now paired.
\item Prior-work entries: CFFM's CIFAR-10 and CIFAR-100 top-5 accuracies
  are $85.37\%$ and $84.29\%$, not $85.40\%$ and $84.30\%$; ASGE's
  $90.62\%$/$65.42\%$ are its VGG11 results, not VGG8; HFF's CNN has 3
  layers, not 14 (Tabs.~\ref{tab:related_work},
  \ref{tab:related_work_3ds} and~\ref{tab:cifar100_comparison}).
\item Prior-work attributions (same tables): Hinton's FF reaches
  $59\%$ on CIFAR-10, not ${\sim}70\%$ (Sec.~\ref{sec:intro} gave
  $59$--$75\%$, now $54$--$59\%$), with locally connected nets, not an MLP/CNN of
  ${\sim}0.3$\,M parameters, and it does not inject labels at every
  block (Sec.~\ref{sec:method}); SCFF has no direct feedback and is
  strict FF; Layer Collaboration's $48.4\%$ is its own five-run result,
  not an HFF reproduction on another split; Trifecta's $83.51\%$ is
  its own three-run result, not a DeeperForward reproduction; ASGE uses a local
  cross-entropy, not binary goodness (Sec.~\ref{sec:related}), and
  ASGE and CFFM read out with trained linear classifiers, not FF
  goodness (for ASGE also Sec.~\ref{sec:c10_accuracy});
  DeeperForward's $86.22\%$ used no augmentation, its $88.72\%$ comes
  from a separate 1000-epoch run whose enhanced augmentation has no
  crops, and, by our count from its official code, its ResNets have
  ${\sim}31$ and ${\sim}275$\,M parameters, not ${\sim}5$--$10$ and ${\sim}15$--$45$\,M; DSN and
  PEPITA do not use local backprop (Tab.~\ref{tab:related_work}
  caption); SFF's block-local backprop does not make it non-strict,
  and SFF and CwC are now in Tab.~\ref{tab:related_work};
  augmentation (SCFF, CFFM, Distance-Forward, HFF, ASGE) and seed
  counts (six rows) are corrected, and the VGG19 parameter count given
  for Karimi et al.\ is removed.
\end{enumerate}
\par}


\begin{thebibliography}{99}

\bibitem[Aghagolzadeh and Ezoji(2025)]{cffm2025}
Aghagolzadeh, H. and Ezoji, M. (2025).
\newblock Contrastive Forward-Forward: A Training Algorithm of Vision
Transformer.
\newblock \emph{Neural Networks}, 192:107867. arXiv:2502.00571.

\bibitem[Chen et~al.(2025)Chen, Liu, Laydevant, and Grollier]{scff2025}
Chen, X., Liu, D., Laydevant, J., and Grollier, J. (2025).
\newblock Self-Contrastive Forward-Forward algorithm.
\newblock \emph{Nature Communications}, 16:5978.

\bibitem[Coates et~al.(2011)Coates, Ng, and Lee]{coates2011stl}
Coates, A., Ng, A., and Lee, H. (2011).
\newblock An Analysis of Single-Layer Networks in Unsupervised Feature
Learning.
\newblock In \emph{Proceedings of the 14th International Conference on
Artificial Intelligence and Statistics (AISTATS)}.

\bibitem[Cubuk et~al.(2020)Cubuk, Zoph, Shlens, and Le]{cubuk2020randaugment}
Cubuk, E.~D., Zoph, B., Shlens, J., and Le, Q.~V. (2020).
\newblock RandAugment: Practical automated data augmentation with a reduced
search space.
\newblock In \emph{Conference on Computer Vision and Pattern Recognition
Workshops (CVPRW)}.

\bibitem[Dellaferrera and Kreiman(2022)]{dellaferrera2022pepita}
Dellaferrera, G. and Kreiman, G. (2022).
\newblock Error-driven Input Modulation: Solving the Credit Assignment Problem
without a Backward Pass.
\newblock In \emph{Proceedings of the 39th International Conference on Machine
Learning (ICML)}.

\bibitem[Devlin et~al.(2019)Devlin, Chang, Lee, and Toutanova]{devlin2019bert}
Devlin, J., Chang, M.-W., Lee, K., and Toutanova, K. (2019).
\newblock BERT: Pre-training of Deep Bidirectional Transformers for Language
Understanding.
\newblock In \emph{Proceedings of the 2019 Conference of the North American
Chapter of the Association for Computational Linguistics (NAACL-HLT)}.

\bibitem[Dooms et~al.(2024)Dooms, Tsang, and Oramas]{trifecta2024}
Dooms, T., Tsang, I.~J., and Oramas, J. (2024).
\newblock The Trifecta: Three simple techniques for training deeper
Forward-Forward networks.
\newblock \emph{Transactions on Machine Learning Research (TMLR)}.
arXiv:2311.18130.

\bibitem[Fedus et~al.(2022)Fedus, Zoph, and Shazeer]{fedus2022switch}
Fedus, W., Zoph, B., and Shazeer, N. (2022).
\newblock Switch Transformers: Scaling to Trillion Parameter Models with
Simple and Efficient Sparsity.
\newblock \emph{Journal of Machine Learning Research}, 23(120):1--39.

\bibitem[Foret et~al.(2021)Foret, Kleiner, Mobahi, and Neyshabur]{foret2021sam}
Foret, P., Kleiner, A., Mobahi, H., and Neyshabur, B. (2021).
\newblock Sharpness-Aware Minimization for Efficiently Improving Generalization.
\newblock In \emph{International Conference on Learning Representations (ICLR)}.

\bibitem[Frankle and Carbin(2019)]{frankle2019lottery}
Frankle, J. and Carbin, M. (2019).
\newblock The Lottery Ticket Hypothesis: Finding Sparse, Trainable Neural
Networks.
\newblock In \emph{International Conference on Learning Representations (ICLR)}.

\bibitem[Gandhi et~al.(2023)Gandhi, Gala, Kornberg, and Sridhar]{extendff2023}
Gandhi, S., Gala, R., Kornberg, J., and Sridhar, A. (2023).
\newblock Extending the Forward Forward Algorithm.
\newblock \emph{arXiv preprint arXiv:2307.04205}.

\bibitem[Gong et~al.(2026)Gong, Staszewski, and Xu]{asge2025}
Gong, Q., Staszewski, R.~B., and Xu, K. (2026).
\newblock Adaptive Spatial Goodness Encoding: Scaling the Forward-Forward
Algorithm for Convolutional Neural Networks.
\newblock In \emph{IEEE International Conference on Acoustics, Speech and
Signal Processing (ICASSP)}. Extended version: arXiv:2509.12394.

\bibitem[Hinton(2022)]{hinton2022forward}
Hinton, G. (2022).
\newblock The Forward-Forward Algorithm: Some Preliminary Investigations.
\newblock \emph{arXiv preprint arXiv:2212.13345}.

\bibitem[Karimi et~al.(2024)Karimi, Kalhor, and Sadeghi Tabrizi]{forwardlayer2024}
Karimi, A., Kalhor, A., and Sadeghi Tabrizi, M. (2024).
\newblock Forward layer-wise learning of convolutional neural networks through
separation index maximizing.
\newblock \emph{Scientific Reports}, 14:8576.

\bibitem[Khosla et~al.(2020)Khosla, Teterwak, Wang, Sarna, Tian, Isola, Maschinot, Liu, and Krishnan]{khosla2020supcon}
Khosla, P., Teterwak, P., Wang, C., Sarna, A., Tian, Y., Isola, P.,
Maschinot, A., Liu, C., and Krishnan, D. (2020).
\newblock Supervised Contrastive Learning.
\newblock In \emph{Advances in Neural Information Processing Systems (NeurIPS)}.

\bibitem[Krizhevsky(2009)]{krizhevsky2009cifar}
Krizhevsky, A. (2009).
\newblock Learning Multiple Layers of Features from Tiny Images.
\newblock Technical Report, University of Toronto.

\bibitem[Krutsylo(2025)]{sff2025}
Krutsylo, A. (2025).
\newblock Scalable Forward-Forward Algorithm.
\newblock \emph{arXiv preprint arXiv:2501.03176}.

\bibitem[Lang(1995)]{lang1995newsweeder}
Lang, K. (1995).
\newblock NewsWeeder: Learning to Filter Netnews.
\newblock In \emph{Proceedings of the 12th International Conference on Machine
Learning (ICML)}.

\bibitem[Le and Yang(2015)]{le2015tinyimagenet}
Le, Y. and Yang, X. (2015).
\newblock Tiny ImageNet Visual Recognition Challenge.
\newblock Technical Report CS 231N, Stanford University.

\bibitem[Lee et~al.(2015a)Lee, Zhang, Fischer, and Bengio]{lee2015difference}
Lee, D.-H., Zhang, S., Fischer, A., and Bengio, Y. (2015a).
\newblock Difference Target Propagation.
\newblock In \emph{Joint European Conference on Machine Learning and Knowledge
Discovery in Databases (ECML PKDD)}.

\bibitem[Lee et~al.(2015b)Lee, Xie, Gallagher, Zhang, and Tu]{lee2015dsn}
Lee, C.-Y., Xie, S., Gallagher, P., Zhang, Z., and Tu, Z. (2015b).
\newblock Deeply-Supervised Nets.
\newblock In \emph{Proceedings of the 18th International Conference on
Artificial Intelligence and Statistics (AISTATS)}.

\bibitem[Lorberbom et~al.(2024)Lorberbom, Gat, Adi, Schwing, and Hazan]{layercollab2023}
Lorberbom, G., Gat, I., Adi, Y., Schwing, A., and Hazan, T. (2024).
\newblock Layer Collaboration in the Forward-Forward Algorithm.
\newblock In \emph{Proceedings of the AAAI Conference on Artificial
Intelligence (AAAI)}, volume~38. arXiv:2305.12393.

\bibitem[Loshchilov and Hutter(2019)]{loshchilov2019adamw}
Loshchilov, I. and Hutter, F. (2019).
\newblock Decoupled Weight Decay Regularization.
\newblock In \emph{International Conference on Learning Representations (ICLR)}.

\bibitem[Maas et~al.(2011)Maas, Daly, Pham, Huang, Ng, and Potts]{maas2011imdb}
Maas, A.~L., Daly, R.~E., Pham, P.~T., Huang, D., Ng, A.~Y., and Potts, C.
(2011).
\newblock Learning Word Vectors for Sentiment Analysis.
\newblock In \emph{Proceedings of the 49th Annual Meeting of the Association
for Computational Linguistics: Human Language Technologies (ACL-HLT)}.

\bibitem[Papachristodoulou et~al.(2024)Papachristodoulou, Kyrkou, Timotheou, and Theocharides]{cwc2024}
Papachristodoulou, A., Kyrkou, C., Timotheou, S., and Theocharides, T.
(2024).
\newblock Convolutional Channel-wise Competitive Learning for the
Forward-Forward Algorithm.
\newblock In \emph{Proceedings of the AAAI Conference on Artificial
Intelligence (AAAI)}, volume~38. arXiv:2312.12668.

\bibitem[Raghu et~al.(2021)Raghu, Unterthiner, Kornblith, Zhang, and Dosovitskiy]{raghu2021vision}
Raghu, M., Unterthiner, T., Kornblith, S., Zhang, C., and Dosovitskiy, A. (2021).
\newblock Do Vision Transformers See Like Convolutional Neural Networks?
\newblock In \emph{Advances in Neural Information Processing Systems (NeurIPS)}.

\bibitem[Rao and Ballard(1999)]{rao1999predictive}
Rao, R.~P.~N. and Ballard, D.~H. (1999).
\newblock Predictive coding in the visual cortex: A functional interpretation
of some extra-classical receptive-field effects.
\newblock \emph{Nature Neuroscience}, 2(1):79--87.

\bibitem[Russakovsky et~al.(2015)Russakovsky, Deng, Su, Krause, Satheesh, Ma, Huang, Karpathy, Khosla, Bernstein, Berg, and Fei-Fei]{russakovsky2015imagenet}
Russakovsky, O., Deng, J., Su, H., Krause, J., Satheesh, S., Ma, S., Huang,
Z., Karpathy, A., Khosla, A., Bernstein, M., Berg, A.~C., and Fei-Fei, L.
(2015).
\newblock ImageNet Large Scale Visual Recognition Challenge.
\newblock \emph{International Journal of Computer Vision}, 115(3):211--252.

\bibitem[Sarode et~al.(2026)Sarode, Moser, Folz, Raue, Nauen, Frolov, and Dengel]{hff2026}
Sarode, S., Moser, B., Folz, J., Raue, F., Nauen, T., Frolov, S., and Dengel,
A. (2026).
\newblock Hyperspherical Forward-Forward with Prototypical
Representations.
\newblock \emph{arXiv preprint arXiv:2605.00082}.

\bibitem[Shazeer(2020)]{shazeer2020glu}
Shazeer, N. (2020).
\newblock GLU Variants Improve Transformer.
\newblock \emph{arXiv preprint arXiv:2002.05202}.

\bibitem[Shazeer et~al.(2017)Shazeer, Mirhoseini, Maziarz, Davis, Le, Hinton, and Dean]{shazeer2017outrageously}
Shazeer, N., Mirhoseini, A., Maziarz, K., Davis, A., Le, Q., Hinton, G., and
Dean, J. (2017).
\newblock Outrageously Large Neural Networks: The Sparsely-Gated
Mixture-of-Experts Layer.
\newblock In \emph{International Conference on Learning Representations (ICLR)}.

\bibitem[Su et~al.(2024)Su, Ahmed, Lu, Pan, Bo, and Liu]{su2024roformer}
Su, J., Ahmed, M., Lu, Y., Pan, S., Bo, W., and Liu, Y. (2024).
\newblock RoFormer: Enhanced Transformer with Rotary Position Embedding.
\newblock \emph{Neurocomputing}, 568:127063.

\bibitem[Sun et~al.(2025)Sun, Zhang, He, Wen, Shen, and Xie]{deeperforward2025}
Sun, L., Zhang, Y., He, W., Wen, J., Shen, L., and Xie, W. (2025).
\newblock DeeperForward: Enhanced Forward-Forward Training for Deeper
and Better Performance.
\newblock In \emph{The Thirteenth International Conference on Learning
Representations (ICLR)}.

\bibitem[Wu et~al.(2024)Wu, Xu, Wu, Deng, Xu, Wen, and Li]{distanceforward2024}
Wu, Y., Xu, S., Wu, J., Deng, L., Xu, M., Wen, Q., and Li, G. (2024).
\newblock Distance-Forward Learning: Enhancing the Forward-Forward
Algorithm Towards High-Performance On-Chip Learning.
\newblock \emph{arXiv preprint arXiv:2408.14925}.

\bibitem[Yang et~al.(2024)Yang, Zhang, Song, Zhang, Zhang, Ma, and Yu]{ffcomplete2024}
Yang, L., Zhang, H., Song, Z., Zhang, J., Zhang, W., Ma, J., and Yu, P.~S.
(2024).
\newblock Cyclic Neural Network.
\newblock \emph{arXiv preprint arXiv:2402.03332}.

\bibitem[Yun et~al.(2019)Yun, Han, Oh, Chun, Choe, and Yoo]{yun2019cutmix}
Yun, S., Han, D., Oh, S.~J., Chun, S., Choe, J., and Yoo, Y. (2019).
\newblock CutMix: Regularization Strategy to Train Strong Classifiers
with Localizable Features.
\newblock In \emph{Proceedings of the IEEE/CVF International Conference
on Computer Vision (ICCV)}.

\bibitem[Zhang et~al.(2017)Zhang, Bengio, Hardt, Recht, and Vinyals]{zhang2017understanding}
Zhang, C., Bengio, S., Hardt, M., Recht, B., and Vinyals, O. (2017).
\newblock Understanding Deep Learning Requires Rethinking Generalization.
\newblock In \emph{International Conference on Learning Representations (ICLR)}.

\bibitem[Zhang et~al.(2018)Zhang, Cisse, Dauphin, and Lopez-Paz]{zhang2018mixup}
Zhang, H., Cisse, M., Dauphin, Y.~N., and Lopez-Paz, D. (2018).
\newblock mixup: Beyond Empirical Risk Minimization.
\newblock In \emph{International Conference on Learning Representations (ICLR)}.

\bibitem[Zhang et~al.(2015)Zhang, Zhao, and LeCun]{zhang2015charcnn}
Zhang, X., Zhao, J., and LeCun, Y. (2015).
\newblock Character-level Convolutional Networks for Text Classification.
\newblock In \emph{Advances in Neural Information Processing Systems (NeurIPS)}.

\bibitem[Zoph et~al.(2022)Zoph, Bello, Kumar, Du, Huang, Dean, Shazeer, and Fedus]{zoph2022stmoe}
Zoph, B., Bello, I., Kumar, S., Du, N., Huang, Y., Dean, J., Shazeer, N., and
Fedus, W. (2022).
\newblock ST-MoE: Designing Stable and Transferable Sparse Expert Models.
\newblock \emph{arXiv preprint arXiv:2202.08906}.

\end{thebibliography}
\end{document}